\documentclass[journal,twocolumn,10pt]{IEEEtran}
\usepackage{amsmath,epsfig}

\usepackage{diagbox}

\usepackage{floatflt} 
\usepackage{paralist} 
\usepackage{algorithm} 
\usepackage{graphicx}
\usepackage{amsthm}
\usepackage{amssymb}
\usepackage{listings}    
\usepackage{xcolor}
\usepackage{hyperref}
\usepackage{indentfirst}
\usepackage{booktabs}
\usepackage{subfigure}
\usepackage{caption}
\usepackage{calligra}
\usepackage{bm}
\usepackage{fancyhdr}
\usepackage{float}
\usepackage{pifont}
\usepackage{colortbl} 
\usepackage{amsmath}
\usepackage{MnSymbol}%
\usepackage{amsbsy}

\def\c{\mathbf{c}}
\def\x{\mathbf{x}}

\def\q{\mathbf{q}}
\def\s{\mathbf{s}}

\def\z{\mathbf{z}}
\def\S{\mathcal{S} }

\def\V{\mathcal{V}}

\def\one{{\bf 1}}

\newcommand{\mypar}[1]{{\bf #1.}}
\theoremstyle{definition}
\newtheorem{defn}{Definition}

\newtheorem{myThm}{Theorem}
\newtheorem{myCorollary}{Corollary}

\DeclareMathOperator{\TV}{TV}

\DeclareMathOperator{\Adj}{A}

\DeclareMathOperator{\D}{D}

\DeclareMathOperator{\Mm}{M}

\DeclareMathOperator{\Vm}{V}
\DeclareMathOperator{\X}{X}

\DeclareMathOperator{\Z}{Z}

\newcommand{\R}{\ensuremath{\mathbb{R}}}

\DeclareMathOperator{\Id}{I}

\title{Deep Unsupervised Learning of 3D Point Clouds via  Graph Topology Inference and Filtering}

\author{Siheng Chen, Chaojing Duan, Yaoqing Yang, Duanshun Li, Chen Feng, Dong Tian}

\begin{document}

\maketitle

\begin{abstract}
We propose a deep autoencoder with graph topology inference and filtering to achieve compact representations of unorganized 3D point clouds in an unsupervised manner. Many previous works discretize 3D points to voxels and then use lattice-based methods to process and learn 3D spatial information; however, this leads to inevitable discretization errors. In this work, we try to handle raw 3D points without such compromise. The proposed networks follow the autoencoder framework with a focus on designing the decoder. The encoder of the proposed networks adopts similar architectures as in PointNet, which is a well-acknowledged method for supervised learning of 3D point clouds. The decoder of the proposed networks involves three novel modules: the folding module, the graph-topology-inference module, and the graph-filtering module. The folding module folds a canonical 2D lattice to the underlying surface of a 3D point cloud, achieving coarse reconstruction; the graph-topology-inference module learns a graph topology to represent pairwise relationships between 3D points, pushing the latent code to preserve both coordinates and pairwise relationships of points in 3D point clouds; and the graph-filtering module couples the above two modules, refining the coarse reconstruction through a learnt graph topology to obtain the final reconstruction. The proposed decoder leverages a learnable graph topology to push the codeword to preserve representative features and further improve the unsupervised-learning performance. We further provide theoretical analyses of the proposed architecture. We provide an upper bound for the reconstruction loss and further show the superiority of graph smoothness over spatial smoothness as a prior to model 3D point clouds. In the experiments, we validate the proposed networks in three tasks, including 3D point cloud reconstruction, visualization, and transfer classification. The experimental results show that (1) the proposed networks outperform the state-of-the-art methods in various tasks, including reconstruction and transfer classification; (2) a graph topology can be inferred as auxiliary information without specific supervision on graph topology inference; (3) graph filtering refines the reconstruction, leading to better performances; and (4) designing a powerful decoder could improve the unsupervised-learning performance, just like a powerful encoder.
\end{abstract}


\begin{table}[htb!]
  \begin{center}
    \begin{tabular}{c |  c  c }
      \hline
       & 
       2D lattice & 
       Airplane
       \\
       &
      \includegraphics[trim={2cm 4cm 1cm 1cm}, clip=true , scale=0.02] {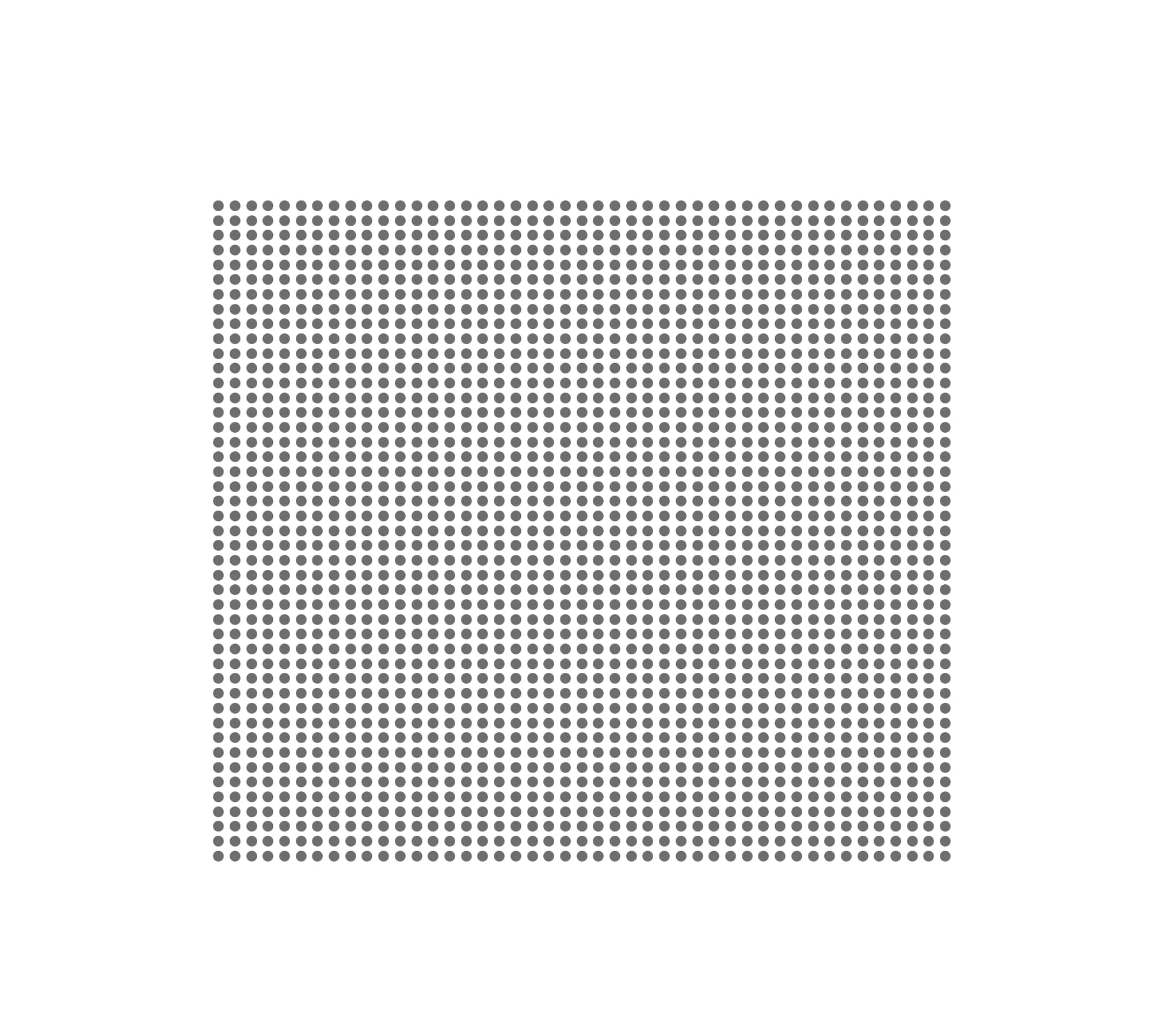} &
      \includegraphics[trim={4cm 4cm 2cm 4cm}, clip=true , scale=0.045] {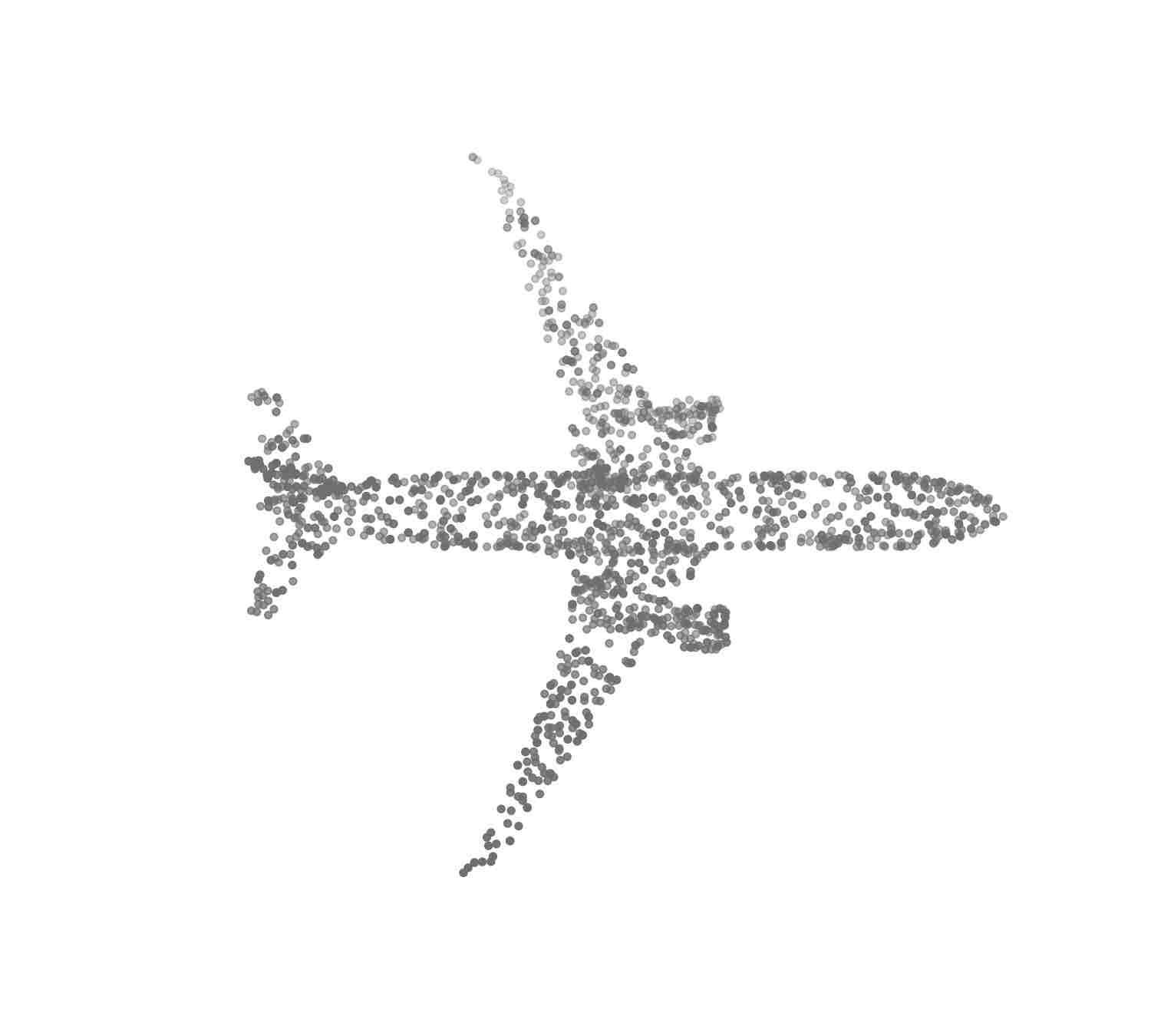}
      \\
      \hline
      Epoch & 
       Coarse reconstruction & 
       Refined reconstruction
       \\
           & 
       \emph{before graph filtering} & 
          \emph{after graph filtering}
       \\
      0 &
      \includegraphics[trim={6cm 6cm 4cm 6cm}, clip=true , scale=0.06] {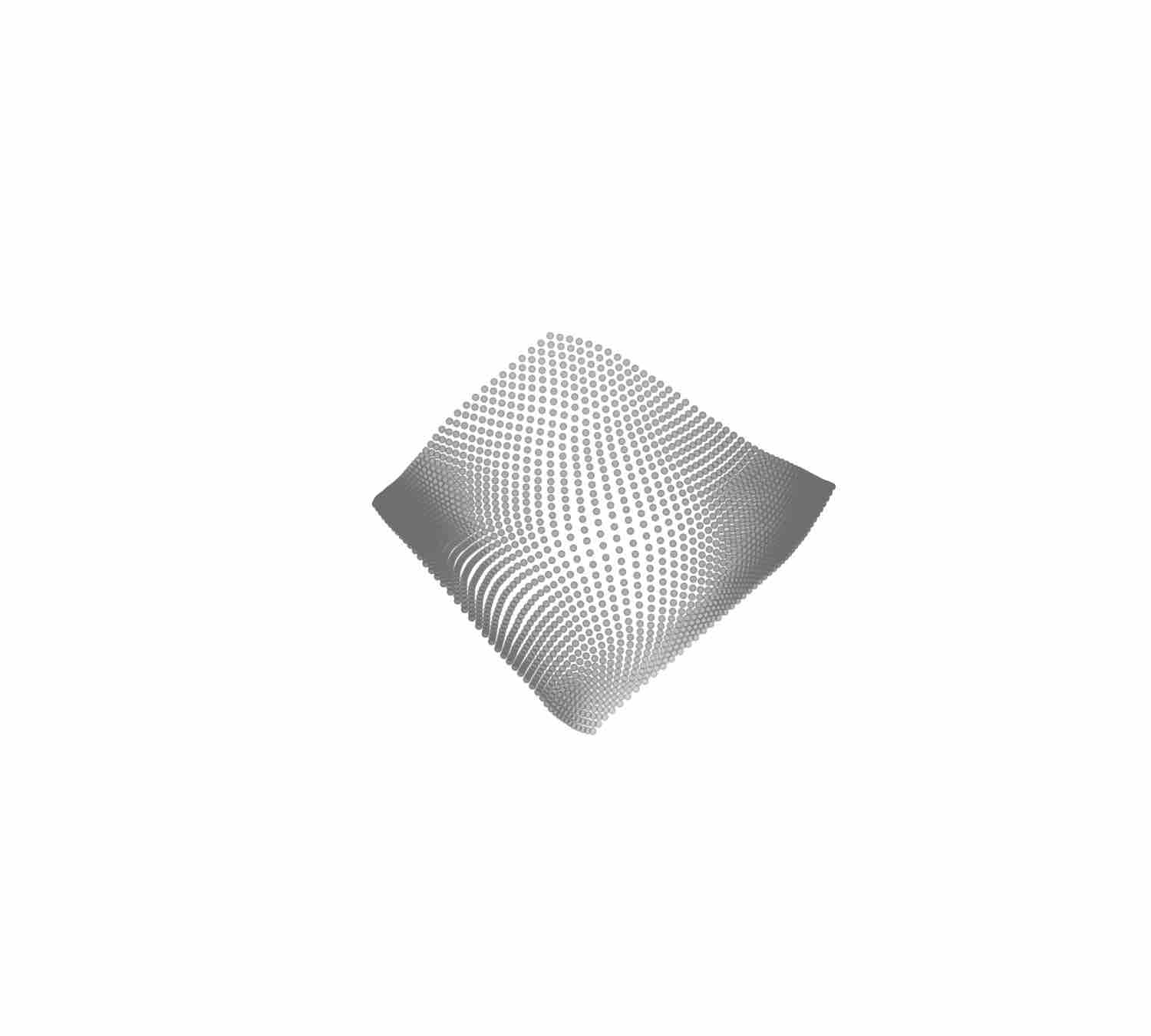} & 
      \includegraphics[trim={5cm 4cm 3cm 4cm}, clip=true , scale=0.06] 
      {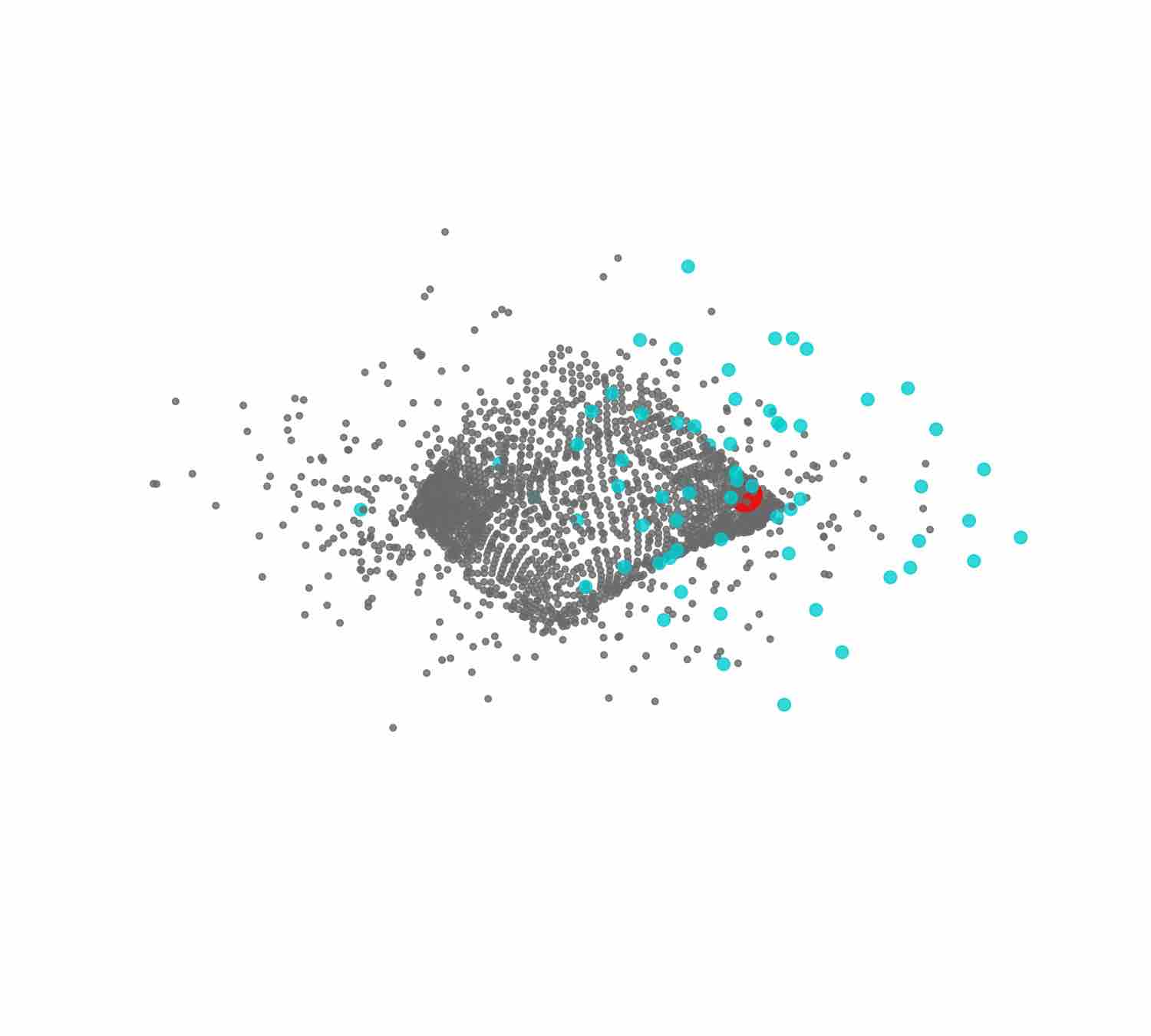}\\

      10 &
      \includegraphics[trim={6cm 6cm 4cm 6cm}, clip=true , scale=0.06] {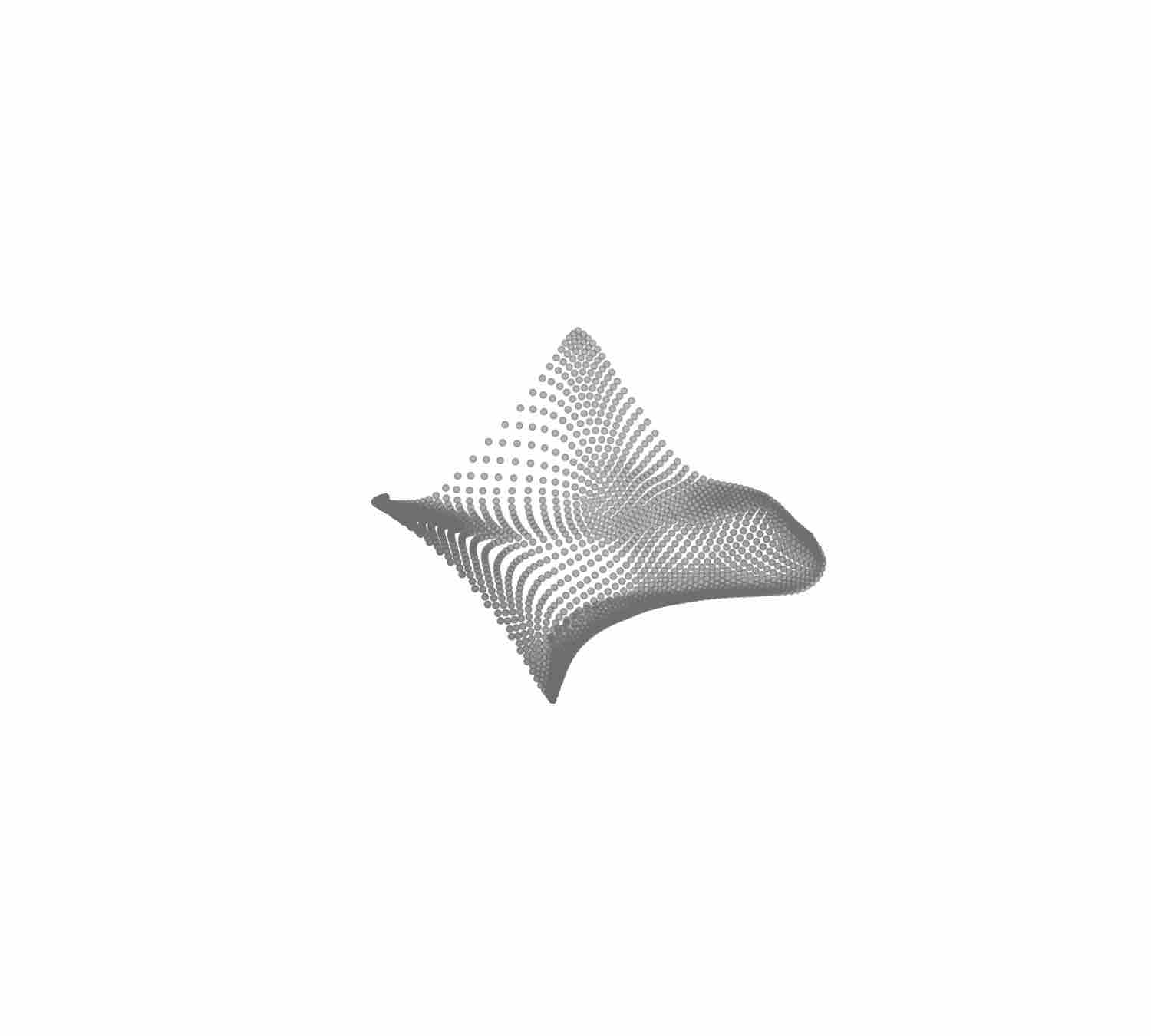} & 
      \includegraphics[trim={5cm 4cm 3cm 4cm}, clip=true , scale=0.06] 
      {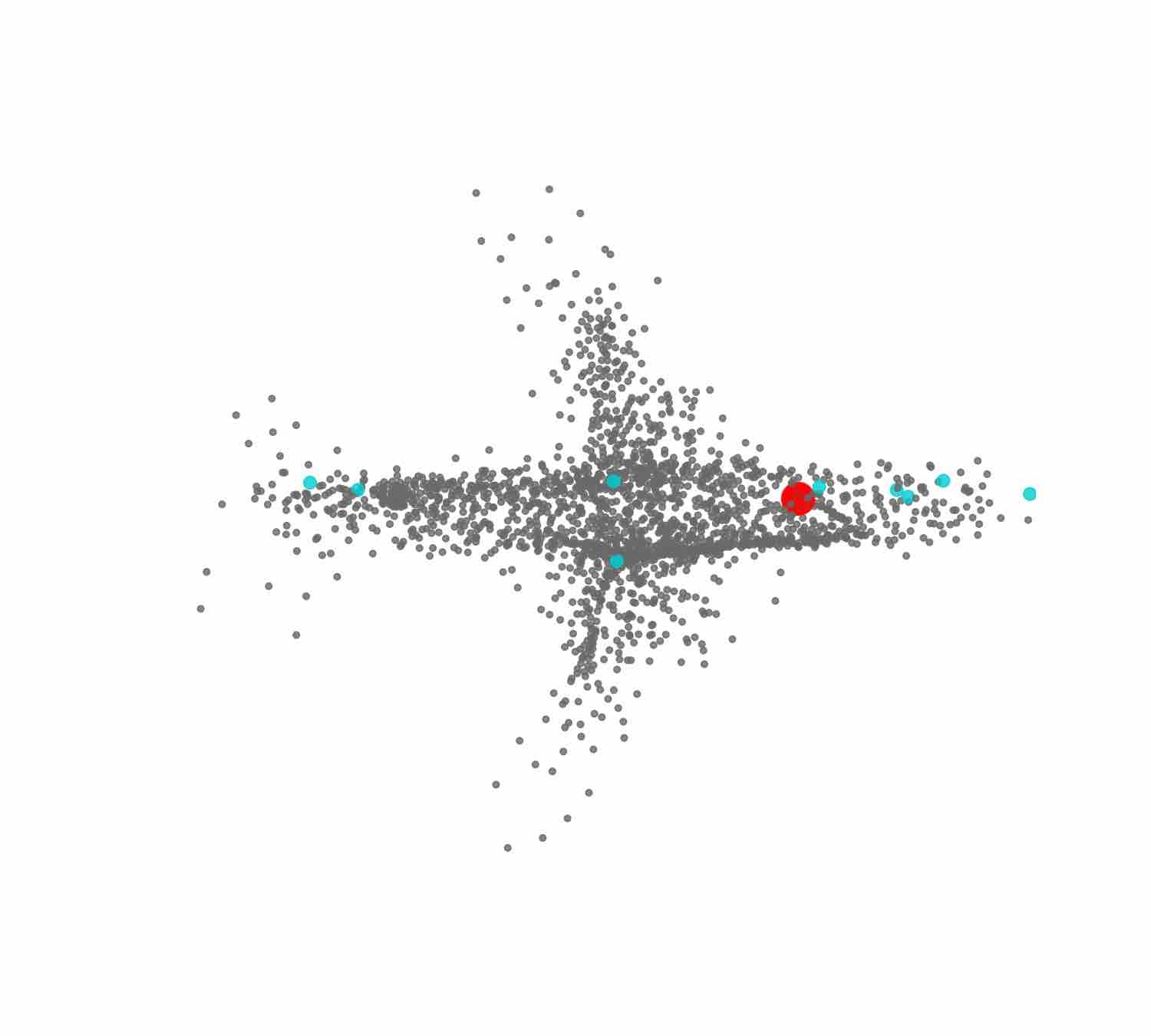}\\
            
      100 &
      \includegraphics[trim={6cm 6cm 4cm 6cm}, clip=true , scale=0.06] {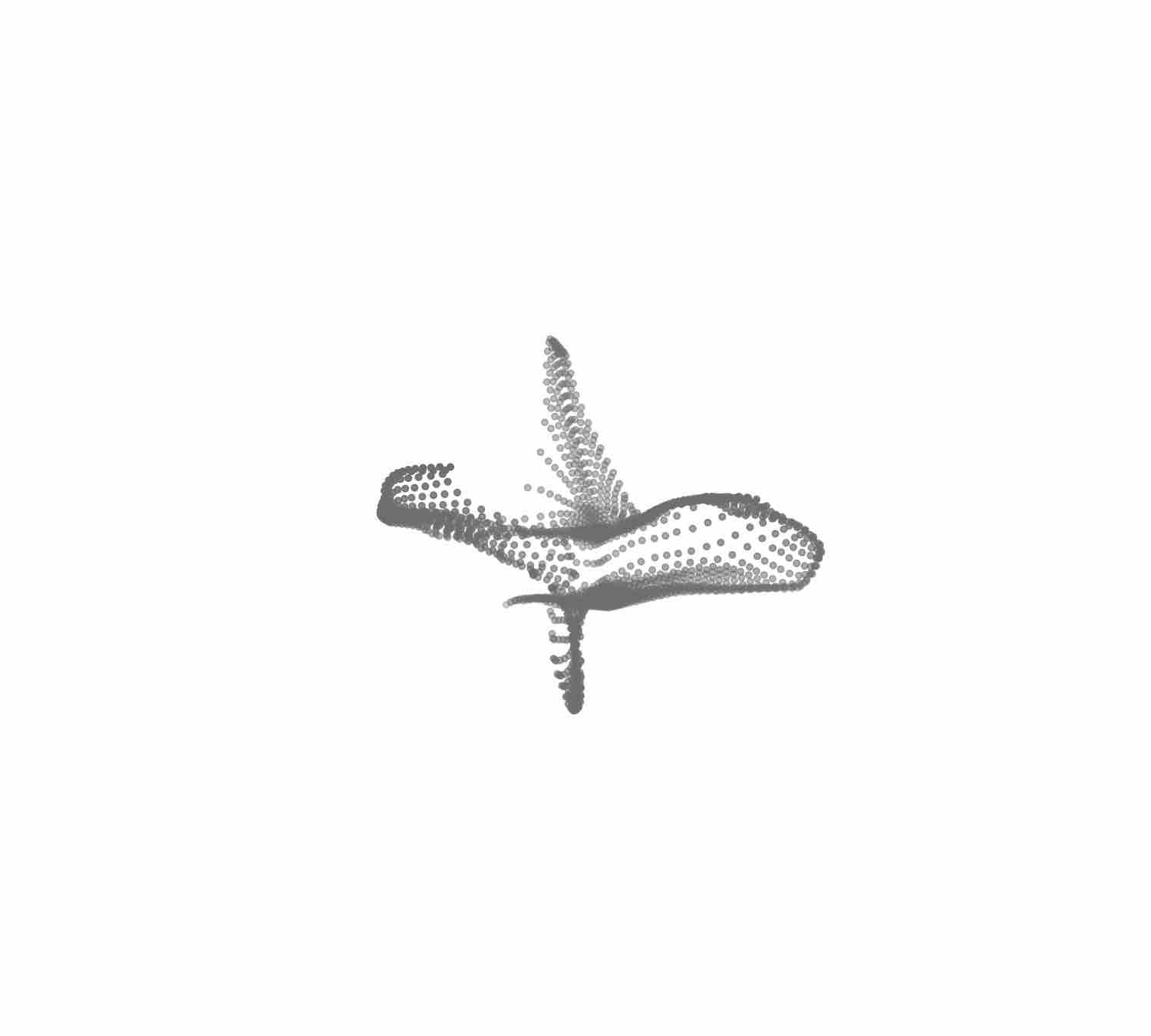} & 
      \includegraphics[trim={5cm 4cm 3cm 4cm}, clip=true , scale=0.06] 
      {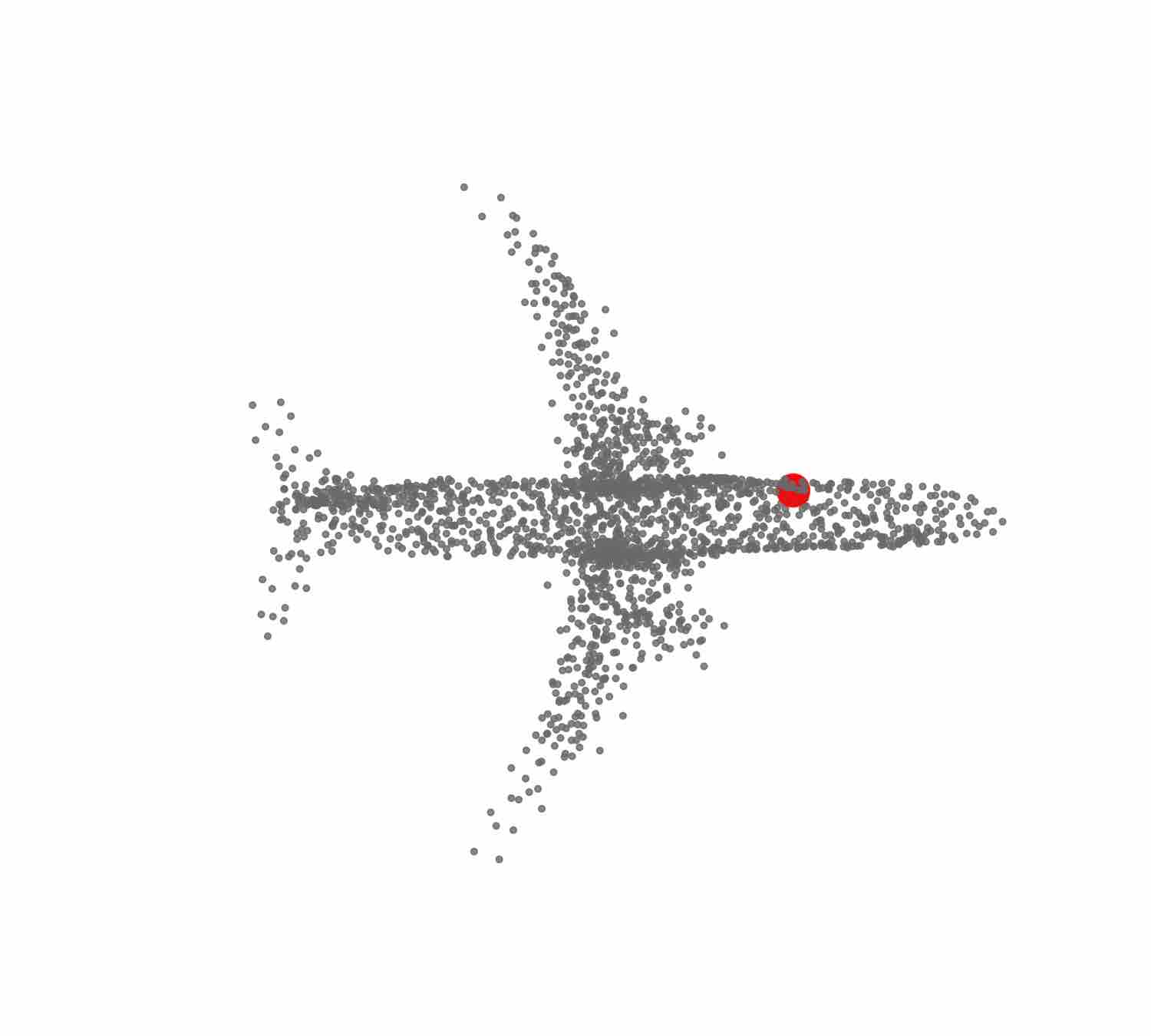}\\
            
      300 &
      \includegraphics[trim={6cm 6cm 4cm 6cm}, clip=true , scale=0.06] {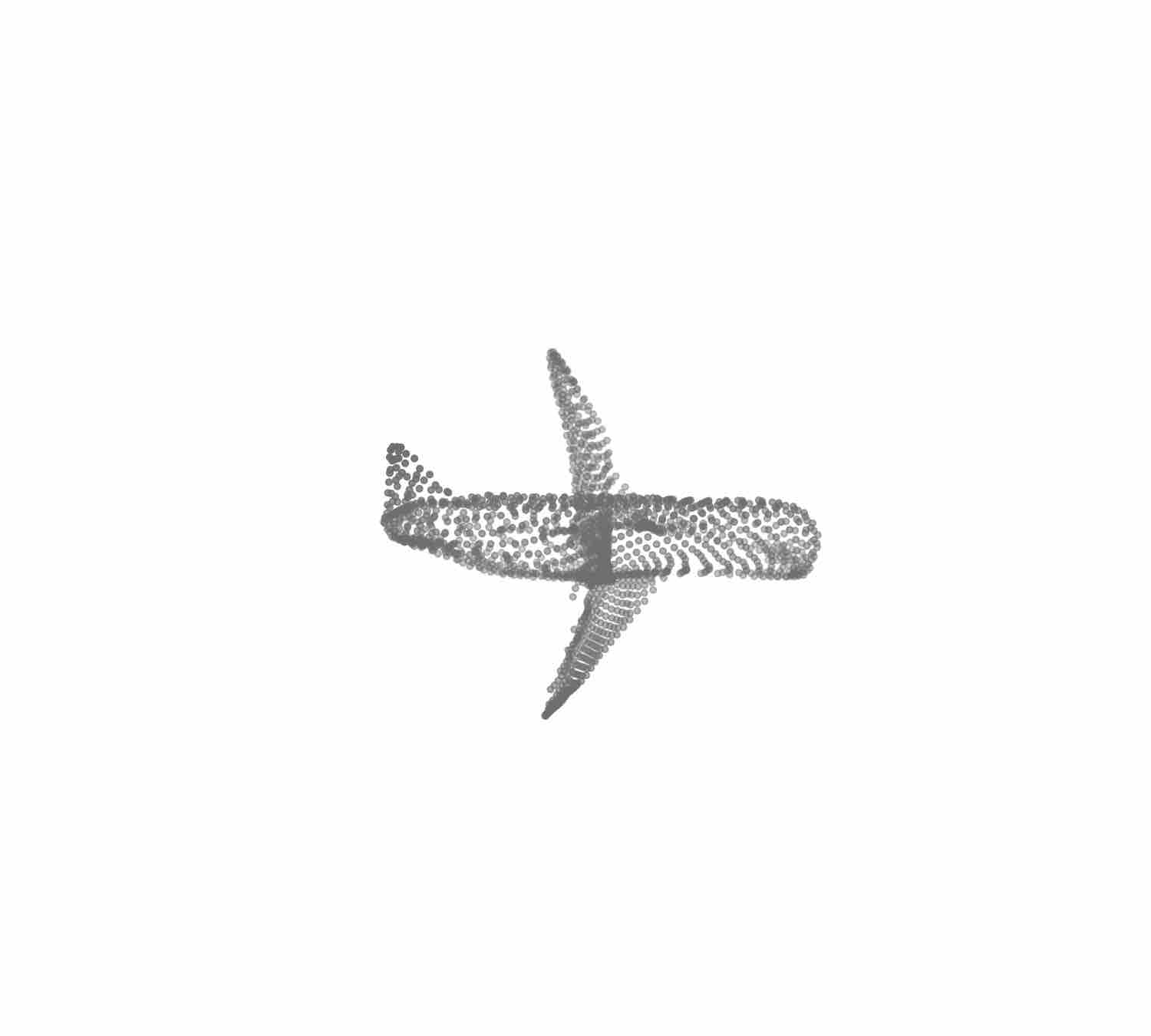} & 
      \includegraphics[trim={5cm 4cm 3cm 4cm}, clip=true , scale=0.06] 
      {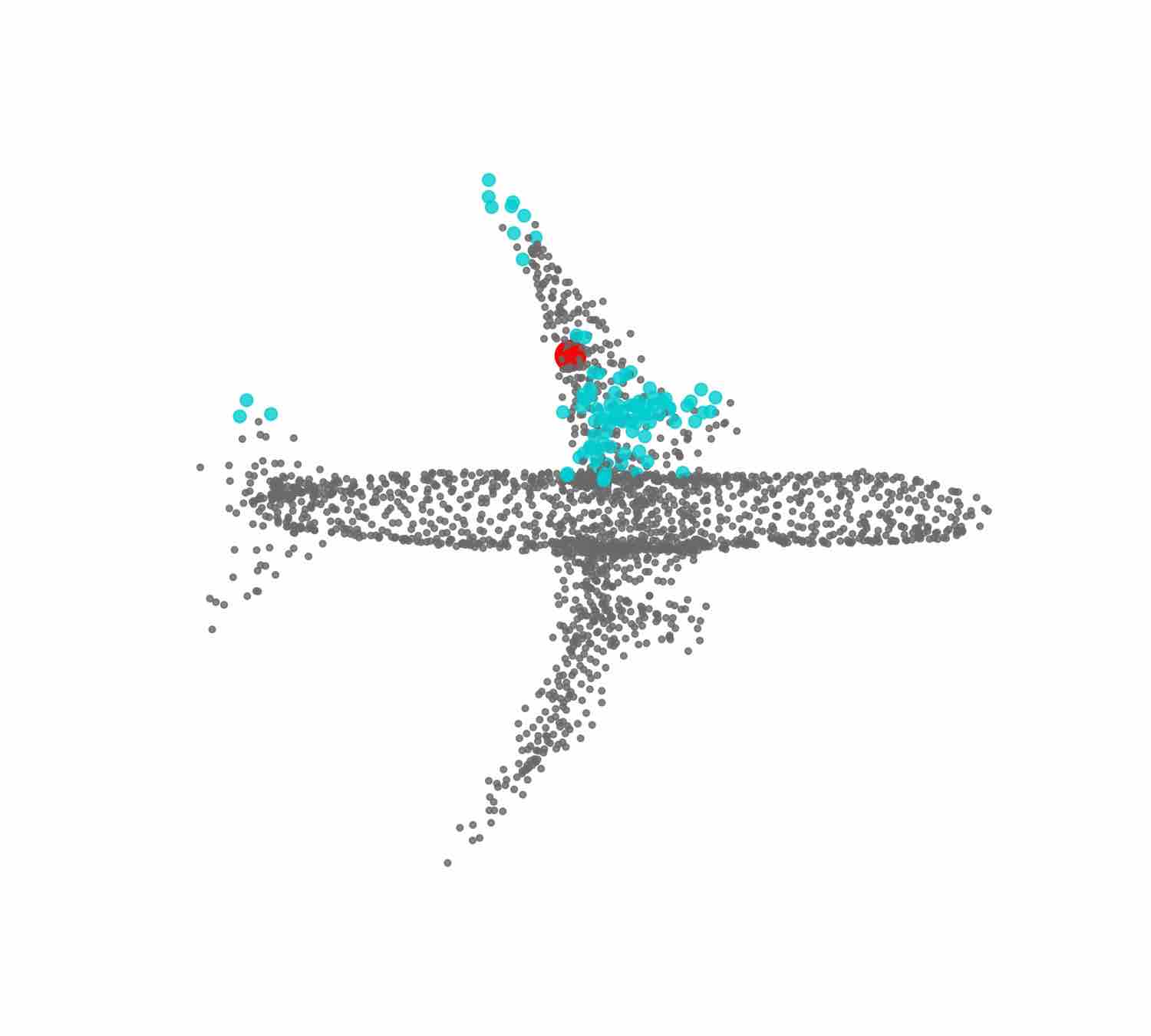}\\
      \hline
      
    \end{tabular}
  \end{center}
  \caption{\label{tab:crown_jewel}\textbf{Graph filtering improves the reconstruction of a 3D point cloud.}  A graph filter learnt through the graph-topology-inference module is used to refine the coarse reconstructions produced by the folding module (left column) and obtain the final reconstructions (right column). We select one 3D point (in red) and trace the evolvement of its neighbors in the learnt graph topology (in cyan) during the training process. Graphs guide the networks to preserves details.}
\end{table}

\section{Introduction}
\label{sec:intro}
3D point clouds are discrete representations of continuous surfaces in the 3D space, which have be widely used in autonomous driving, industrial robotics, augmented reality and many others~\cite{PCL}. Based on the storage order and spatial connectivity among 3D points, we distinguish between two types of point clouds: \emph{organized point clouds}, such as those collected by camera-like 3D sensors or 3D laser scanners and arranged on a lattice~\cite{3dmultiview}, and \emph{unorganized point clouds}, such as those that, due to their complex structure, are scanned from multiple viewpoints and are subsequently merged leading to the loss of ordering of indices~\cite{3dsensor}. Organized point clouds are easier to process as the underlying lattice produce natural spatial connectivity and reflect the sensing order. For generality, we consider unorganized point clouds in this paper. Different from 1D speech data or 2D images, which are associated with regular lattices~\cite{duanMWP}, unorganized 3D point clouds are usually sparsely and irregularly scattered in the 3D space;  this makes traditional latticed-based algorithms difficult to handle 3D point clouds. To solve this issue, many previous works discretize 3D point clouds by transforming them to either 3D voxels or multi-view images, causing volume redundancies and  quantization artifacts. As a pioneering work, PointNet is a  deep-neural-network-based supervised method that uses pointwise multi-layer perceptron followed by maximum pooling to guarantee permutation invariance. It achieves successes on a series of supervised-learning tasks, such as recognition, segmentation, and semantic scene segmentation of 3D point clouds~\cite{2016pointnet}. After that, similar techniques are also applied to many other tasks, such as 3D point cloud detection~\cite{vox3}, classification~\cite{recognition}, and upsampling~\cite{YuLFCH:18}.

In this work, we consider unsupervised learning of 3D point clouds; that is, learning compact representations of 3D point clouds without any labeling information. In this way, representative features are automatically extracted  from 3D point clouds and can be further applied to arbitrary subsequent tasks as auxiliary or prior information. Some works have been proposed recently to pursue this goal~\cite{3dgan,AchlioptasDMG:17}. They adopt the encoder-decoder framework.~\cite{3dgan} discretizes 3D points to 3D voxels and uses 3D convolutions to design both encoder and decoder; however, it leads to unavoidable discretization errors and 3D convolutions are expensive.~\cite{AchlioptasDMG:17} directly handles 3D points; it uses PointNet as the encoder and fully-connected layers as the decoder. This approach is effective; however, it does not explore geometric structures of 3D point clouds and requires an unnecessarily huge number of training parameters.

The proposed networks directly handle 3D points and explore geometric structures of 3D point clouds by graph structures. The proposed networks follow the classical autoencoder framework with a focus on designing a powerful decoder. The encoder adopts PointNet and the decoder consists of three novel modules: the folding module, the graph-topology-inference module, and the graph-filtering module. The folding module maps each node in a 2D lattice to a point in the 3D space based on the latent code generated by the encoder, achieving coarse reconstructions. Intuitively, this process \textit{folds} the underlying 2D flat of a 2D lattice to the underlying 3D surface of a 3D point cloud. The graph-topology-inference module learns a graph topology to explicitly capture the relationships between 3D points. Intuitively, the learnt graph topology is able to deform a 3D point cloud by \textit{cutting} or \textit{gluing} local shapes. Finally, the graph-filtering module couples the above two modules. It designs graph filters based on a learnable graph topology and refines the coarse reconstruction to obtain the final reconstruction.
Intuitively,  the graph filtering module guides the networks to fold, cut, and glue a 2D flat to form a refined and complex 3D surface; see Table~\ref{tab:crown_jewel}. The overall decoder leverages a learnable graph topology to push the codeword to preserve representative features and further improve the unsupervised-learning performance.
We further provide theoretical analyses of the proposed architecture. We provide an upper bound on the reconstruction loss, which is proportional to the cube root of the code length. With a certain smoothness assumption, we are able to show that the filtering process lowers the upper bound by smoothing the coarse reconstruction. We also provide a theoretical framework to show that graph smoothness is a better prior than  spatial smoothness and  graph filtering refines 3D point clouds by promoting graph smoothness. This reflects the needs of an appropriate graph topology and the subsequent graph filtering.

Experimentally, we validate the effectiveness of the proposed networks in three tasks, including reconstruction, visualization, and transfer classification. To specifically show the fine-grained performance, we manually label the 3D shapes in the dataset of ModelNet40 and provide subcategory labels. For example, airplanes are categorized into $9$ subcategories according to the shape of wings. The experimental results show that (1) the proposed networks outperform state-of-the-art methods in quantitative tasks; (2) graph filtering guides the networks to refine details and outperforms its competitors in fine-grained classification; and (3) graph topology inference can be achieved without direct supervision in an end-to-end architecture.

The main contributions of this paper are:
\begin{itemize}
\item We propose a novel deep autoencoder with a focus on designing a graph-based decoder. The overall decoder leverages a learnable graph topology to push the codeword to preserve representative features;

\item We provide theoretical analyses for proposed networks and show the effectiveness of leveraging a graph structure;

\item We propose a fine-grained 3D shape dataset with subcategory labels based on ModelNet40, which could be used for fine-grained recognition, visualization, and clustering; and

\item We validate the proposed networks in 3D point cloud reconstruction, visualization, and transfer classification. Both qualitative and quantitative results show that  the proposed networks outperform state-of-the-art methods.
\end{itemize}

\section{Related Works}
In this section, we overview the related works from three aspects: unsupervised learning, graph signal processing and geometric deep neural networks.
\vspace{-3mm}
\subsection{Unsupervised Learning}
Compared to supervised learning where training data is associated with ground-truth labels, unsupervised learning does not have any label and uses self-organization to model raw data. Some common unsupervised-learning methods include $k$-means clustering, Gaussian mixture models principal component analysis~\cite{Bishop:06}, matrix factorization~\cite{BerryBLPP:07}, autoencoders, and generative adversarial networks~\cite{Goodfellow-et-al-2016}.  Recently, a series of unsupervised-learning models are proposed to learn from 3D point clouds~\cite{sph,lfd,tlnetwork,vconv, FanSG:17}. For example,  3D GAN converts 3D points to 3D voxels~\cite{3dgan}, which introduces a lot of empty voxels and loses precision; LatentGAN handles 3D point clouds directly~\cite{AchlioptasDMG:17}; however, the decoder uses fully-connected layers, which does not explore specific geometric structures of 3D point clouds and requires a huge number of training parameters; and VIP-GAN uses  recurrent-neural-network-based architecture to solve multiple view inter-prediction tasks for each shape~\cite{innergan}; 
~\cite{abs-1901-05103} learns a continuous signed distance function representation of a class of shapes that enables high quality shape representation, interpolation and completion from partial and noisy 3D input data. AtlasNet~\cite{AtlasNet} models a 3D shape as a collection of parametric surface elements, which is similar to the folding module in the proposed decoder; 3DCapsNet~\cite{3dcapsule} adopts the dynamic routing scheme and the peculiar 2D latent space.  In this work, we use deep autoencoder to directly handle unorganized 3D points and propose graph-based operations to explore geometric structures of 3D point clouds.

\vspace{-3mm}
\subsection{Graph Signal Processing}
Graph signal processing is a theoretical framework for the analysis of high-dimensional data with irregular structures~\cite{ShumanNFOV:13,SandryhailaM:14,OrtegaFKMV:18}. This framework extends classical discrete signal processing to signals with an underlying irregular structure. The framework models underlying structure by a graph and signals by graph signals, generalizing concepts and tools from classical discrete signal processing to the graph domain. Some techniques involve representations for graph signals~\cite{ShumanFV:16, ChenSK:18}, sampling for graph signals~\cite{ChenVSK:15, AnisGO:15, ChenVSK:16}, recovery for graph signals~\cite{NarangGO:13, ChenSMK:14}, denoising~\cite{NarangO:12, ChenSMK:14a}, graph-based filter banks~\cite{NarangO:12,TremblayB:16}, graph-based transforms~\cite{NarangSO:10, ShumanFV:16}, graph topology inference~\cite{DongTRF:18}, and graph neural networks~\cite{GamaMLR:19, NiuCGTSK:18}. To process 3D point clouds,~\cite{ChenTFVK:18} uses graph filters and graph-based resampling strategies to select most informative 3D nodes;~\cite{ZhangR:18} uses graph convolutional neural networks to classify 3D point clouds. In this work, we use graph signal processing techniques to achieve filter design in the proposed graph-filtering module.

\vspace{-3mm}
\subsection{Geometric Deep Neural Networks}
People have used various ways to represent 3D data in deep learning, including voxels~\cite{maturana2015voxnet}, multi-view~\cite{su2015multi}, meshes~\cite{graph_data1}, and point clouds~\cite{2016pointnet}. As a common raw data format in autonomous driving, robotics and augmented reality, 3D point clouds are of particular interests because of their flexibility and expressivity~\cite{splatnet,graph_data1,graph_pc4,vox3,wang2018deep,tatarchenko2018tangent,xu2018spidercnn,li2018so}. To process and learn from 3D point clouds, many previous methods convert irregular 3D points to regular data structures by either voxelization or projection to 2D images, such that they can take the advantage of convolutional neural networks (CNN)~\cite{vox1, choy20163d, deep_image1, wang2018adaptive}; however, they have to  trade-off between resolution and memory. To handle raw point clouds directly, PointNet~\cite{2016pointnet} uses point-wise multilayer perceptron (MLP) and max-pooling to ensure the permutation invariance. A branch of 3D deep learning methods follows PointNet~\cite{2016pointnet} as their base networks, with applications in classification, segmentation, and up-sampling~\cite{YangFST:18, AtlasNet, sgpn, huang2018recurrent, YuLFCH:18, shen2018}. Other than that,~\cite{graph_data1,wang2018deep} define an MLP-based continuous graph convolution for processing unorganized point clouds with CNN-like networks;~\cite{tatarchenko2018tangent} defines a continuous tangent convolution for point clouds;~\cite{xu2018spidercnn} uses a simple polynomial convolution weight function instead of MLP;~\cite{vox3} uses PointNet before voxelization to combine both PointNet and CNN;~\cite{li2018so} uses a self-organizing map to learn ordered information for MLP; and~\cite{splatnet} designs a bilateral convolution layer that projects features onto a regularly partitioned space before convolution. Most of these previous works focus on supervised learning. In this work, we use deep learning techniques to achieve unsupervised learning of raw 3D points.

\begin{figure*}[htb]
\centerline{\includegraphics[width=6.4in , height = 2.7in, keepaspectratio]{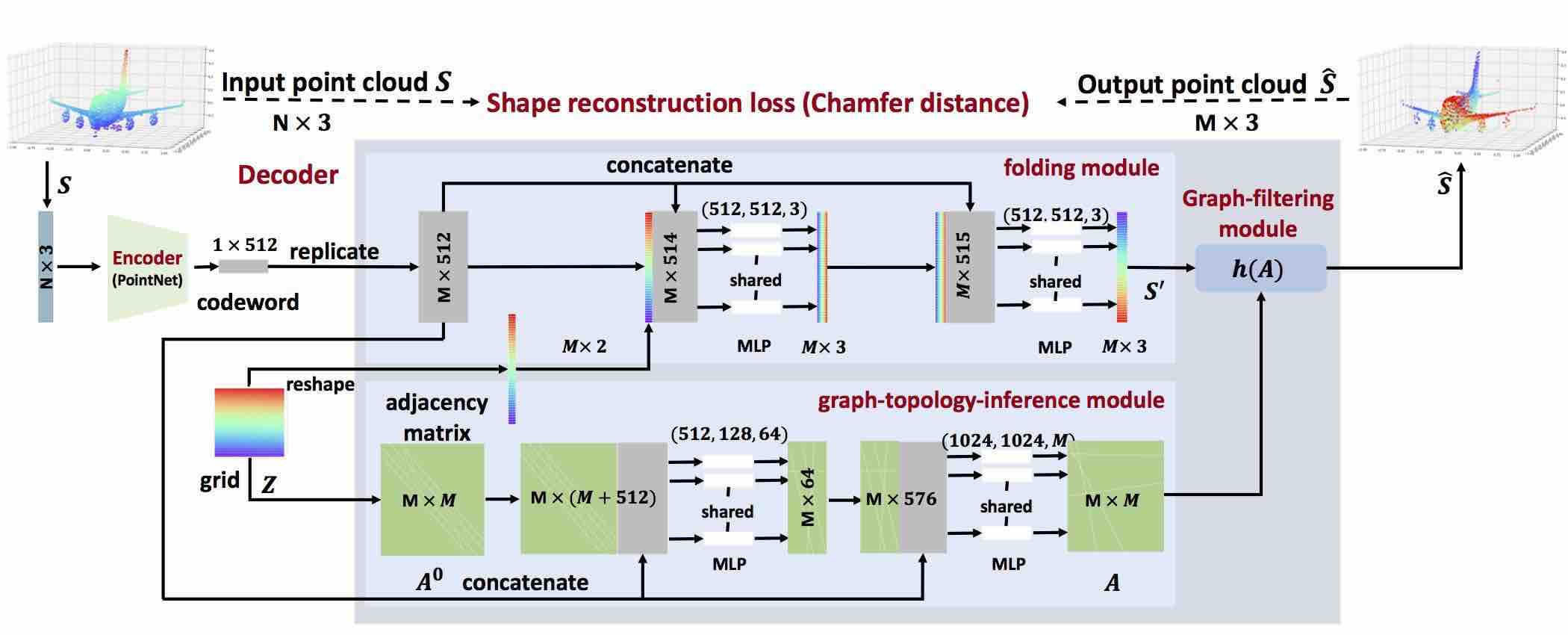}}
\caption{\label{fig:architecture} {\bf Proposed networks follow the encoder-decoder frameworks.}  The decoder includes three novel modules:  the folding module, the graph-topology-inference module, and the graph-filtering module.}
\end{figure*}

\section{Problem Formulation}
We now cover the background material necessary for the
rest of the paper. We start by introducing the properties of 3D point clouds. We then formulate the task of unsupervised learning of 3D point clouds.

\subsection{3D Point Clouds}
\label{sec:3DPoints}
Let $\S = \{  \x_i \in  \R^{K}~|~i= 1, \cdots, N\}$ be a 3D point cloud with ($K-3$) attributes. The corresponding matrix representation is
\begin{eqnarray*}
  \label{eq:PC}
  \X \ = \   \begin{bmatrix}
    \s_1 & \s_2 & \ldots  &  \s_K
  \end{bmatrix} \ = \ \begin{bmatrix} \x_1 & \x_2 & \ldots &
    \x_N
  \end{bmatrix}^T \ \in \ \R^{N \times K},
\end{eqnarray*}
where $\s_i \in \R^N$ denotes the $i$th attribute and $\x_j \in \R^K$
denotes the $j$th point; depending on the sensing device, attributes can be 3D coordinates, RGB colors, intensity, surface normal vector, and many others. Here we mainly consider 3D coordinates; thus from now on, $K = 3$. 

3D point clouds have their own specific properties:
\begin{itemize}
    \item Permutation invariant. 3D point clouds are collections of 3D points represented by the corresponding 3D coordinates. The order of points stored in collections can be changed by any permutation operation and the points are still in the collections with the same coordinates;
 
   \item Transformation equivalence. The 3D coordinates of points can be translated, rotated or reflected along any directions without changing the surfaces they represent, indicating those transformations change the 3D coordinates, but do not change the intrinsic topology or relationships between points of a point cloud;

    \item Piecewise-smoothness. Since 3D point clouds are usually sampled from the surfaces of objects, they are intrinsically 2D surfaces folded in the 3D space. Most parts of a surface are smooth and can be locally approximated by 2D tangent planes, indicating neighboring points share similar geometric structures. Some parts of a surface have significant curvatures and are non-smooth. Overall, the underlying surface of a 3D point cloud is mostly piecewise-smooth in the 3D spatial domain.
\end{itemize}

\subsection{Unsupervised learning of 3D Point Clouds}
The goal is to model the distribution of 3D point clouds by self-organization. Specifically, we aim to use a deep autoencoder to explore compact representations of 3D point clouds that preserve the ability to reconstruct the original point clouds. The proposed autoencoder is based on a deep neural-network framework. It learns to compress a 3D point cloud to a low-dimensional code, and then decompress the code back to a reconstruction that closely matches the original 3D point cloud. The compress module is called an encoder and the decompressed module is called a decoder.

\subsubsection{Encoder}
The functionality of an encoder $\Psi(\cdot)$ is to produce a low-dimensional code to represent the original point cloud; that is,
\begin{equation}
\label{eq:encoder}
\mathbf{c}  \ = \  \Psi  \left( \S \right) \in \mathbb{R}^{C},
\end{equation}
where $C \ll 3N$, reflecting that $\mathbf{c} $ is a compact representation of the original point cloud.

\subsubsection{Decoder}
The functionality of a decoder $\Phi(\cdot)$ is to reconstruct a 3D point cloud that closely matches the original one; that is, the reconstructed  3D point cloud is
\begin{equation}
\label{eq:decoder}
\widehat{\S}  \ = \   \Phi  \left( \mathbf{c}  \right) \in \mathbb{R}^{M \times 3}.
\end{equation}
Note that we do not restrict the number of points in the reconstructed 3D point cloud $M$ to be the same with  the number of points in the original 3D point cloud $N$.

We use deep neural networks to design both the encoder and the decoder. 

\subsubsection{Loss function}
To push the reconstructed 3D point cloud to match the original  3D point cloud, we aim to minimize their difference. Given a set of $n$ 3D point clouds and a fixed code length $C$, the overall optimization problem is 
\begin{eqnarray}
\label{eq:loss}
 \min_{ \Psi(\cdot),  \Phi(\cdot)}  && \sum_{i \in 1}^n d(\S_i, \widehat{\S}_i)  
\\ \nonumber
 {\rm subject~to~}  &&  \mathbf{c}_i  \ = \  \Psi \left( \S_i \right) \in \R^{C},
\\ \nonumber
&& \widehat{\S}_i  \ = \  \Phi \left( \mathbf{c}_i  \right),
\end{eqnarray}
where $\S_i$ is the $i$th 3D point cloud in the dataset and $d(\cdot,  \cdot)$ is the distance metric that measures the difference between two point clouds. Here we consider the augmented Chamfer distance,
\begin{eqnarray}
  \label{eq:aug_chamfer}
  d(\S, \widehat{\S}) & = &
   \max  \bigg\{  \frac{1}{N} \sum_{\x \in \S}  \min_{ \widehat{\x} \in  \widehat{S}}  \left\|  \x - \widehat{\x} \right\|_2, 
   \\ \nonumber
   && ~~~~~~~   \frac{1}{M} \sum_{\widehat{\x} \in  \widehat{S}}  \min_{\x \in \S}  \left\|   \widehat{\x} - \x \right\|_2  \bigg\},
\end{eqnarray}
where the first term $\min_{j = 1, 2, \cdots, M }  \left\|  \x_i - \widehat{\x}_j \right\|_2$ measures the $\ell_2$ distance between  each 3D point 
in the original point cloud and its correspondence in the
reconstructed point cloud; the second term $ \min_{j = 1, 2, \cdots, N }  \left\|   \widehat{\x}_i  - \x_j \right\|_2$ measures the $\ell_2$ distance between  each 3D point  in the reconstructed point cloud and its correspondence in the original point cloud. The maximum operation outside the bracket
enforces the distance from the original point cloud to the reconstructed point cloud  and the distance vice
versa be small simultaneously. The augmented Chamfer distance is essentially the Hausdorff distance between two 3D point clouds. 
Compared to the original Chamfer distance, it is more robust to prevent some ill cases. For example, the original Chamfer distance is less informative when one 3D point cloud have only a few 3D points. This augmented Chamfer distance enforces the underlying manifold of the reconstruction to stay close to that of the original point cloud.  Since we use the minimum and average operations to remove the influence from the number of points,  the reconstructed 3D point cloud does not necessarily have the same number of points as in the original 3D point cloud. 

We use stochastic gradient descent to solve~\eqref{eq:loss}. Since we train the entire networks end-to-end, the performance of unsupervised learning depends on both the encoder and the decoder:  the encoder extracts sufficient information such that the decoder is able to reconstruct; on the other hand, a decoder uses specific structures to push the encoder to extract specific information. Since we can reconstruct the original 3D point cloud,  the code $\mathbf{c}$ preserves key features that describe 3D shapes of the original 3D point cloud. The code thus can be used in classification, matching and other related tasks. In this paper, we mainly consider the design of a decoder.

\section{Network Architecture}
\label{sec:networks}
In this section, we introduce the proposed networks; see an overview in Figure~\ref{fig:architecture}. The proposed networks follow the encoder-decoder framework. The encoder follows PointNet~\cite{2016pointnet}. The decoder includes three novel modules:  the folding module, the graph-topology-inference module, and the graph-filtering module.

\subsection{Encoder}
Here we mainly adopt the architecture of PointNet~\cite{2016pointnet}. The encoder mainly includes a cascade of pointwise multi-layer perceptrons (MLP).  For example, the first MLP maps a point from 3D space to a high-dimensional feature space. Since  all 3D points share the same weights in the convolution, similar points will map to similar positions in the feature space. We next use the max-pooling to remove the point dimension, preserving global features. We finally use MLPs to map  global features to codes. Mathematically, the encoder~\eqref{eq:encoder} is implemented as 
\begin{subequations}
\label{eq:encoder_impl}
\begin{eqnarray}
\label{eq:encoder_1}
\mathbf{c}_i  & = & {\rm MLP}^{(L_1)} \left( \x_i  \right) \in \mathbb{R}^{C}, ~~~{\rm for}~i = 1, \cdots, N, 
\\
\label{eq:encoder_2}
\mathbf{c}' & = & {\rm max pool}   \left( \{ \mathbf{c}_i  \}_{i=1}^N \right) \in \mathbb{R}^{C},
\\
\label{eq:encoder_3}
\mathbf{c}  & = & {\rm MLP}^{(L_2)} \left(\mathbf{c}'  \right) \in \mathbb{R}^{C},
\end{eqnarray}
\end{subequations}
where $\x_i \in \R^3$ is the $i$th 3D point, ${\rm MLP}^{(\ell)} \left( \cdot \right)$ denotes $\ell$ layers of MLPs and $\mathbf{c}_i $ is the feature representation of the $i$th 3D point. Step~\eqref{eq:encoder_1} uses a cascade of MLPs to extract point-wise features $\mathbf{c}_i$; Step~\eqref{eq:encoder_2} uses max-pooling to aggregate point-wise features and obtain global features $\mathbf{c}'$; and Step~\eqref{eq:encoder_3} uses a cascade of MLPs to obtain the final code $\mathbf{c}$.

\begin{table}[htb!]
  \begin{center}
    \begin{tabular}{ c  c  c }
      \hline   
      Input & Folding module only & Overall \\
      \hline 
      $\quad$&
      $\quad$&
      $\quad$\\
      & \includegraphics[trim={10cm 10cm 10cm 10cm}, clip=true , scale=0.09] {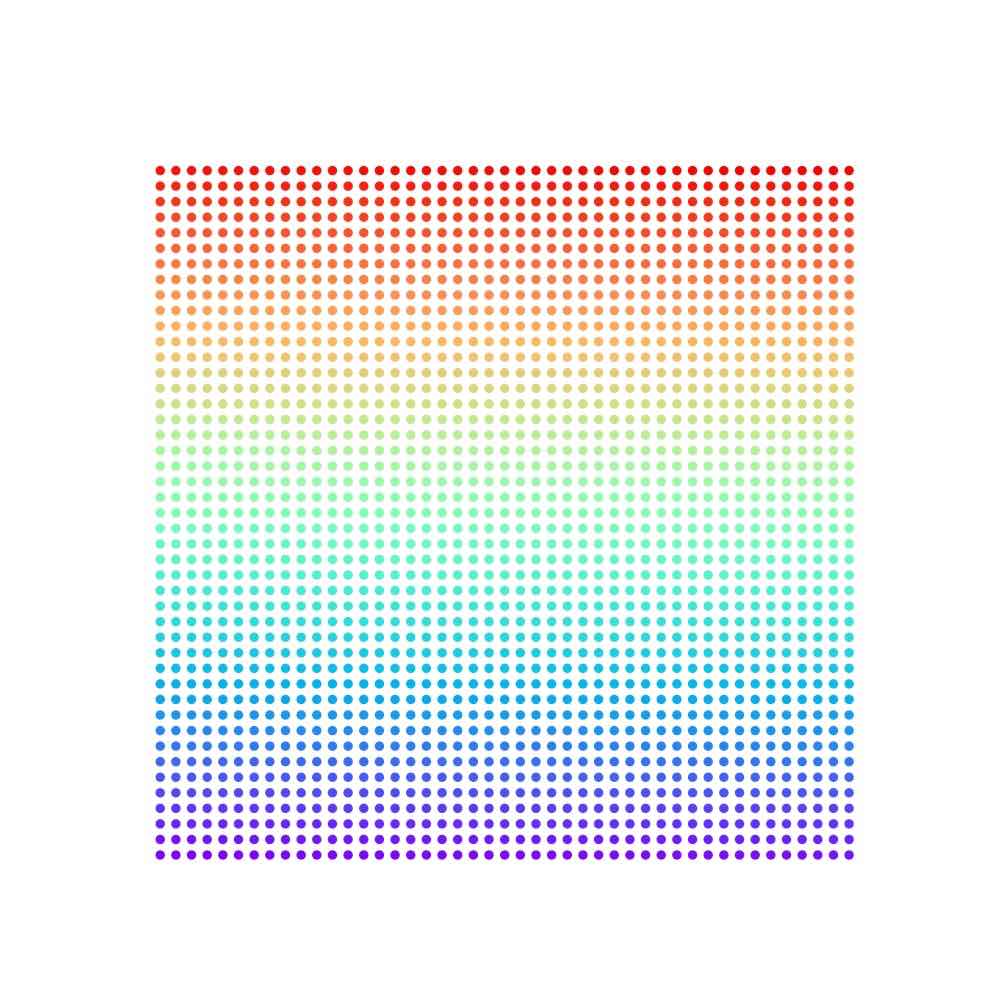} & \includegraphics[trim={10cm 10cm 10cm 10cm}, clip=true , scale=0.09] {figures/torus/grid.jpg}\\
      \includegraphics[trim={10cm 10cm 10cm 10cm}, clip=true , scale=0.045] {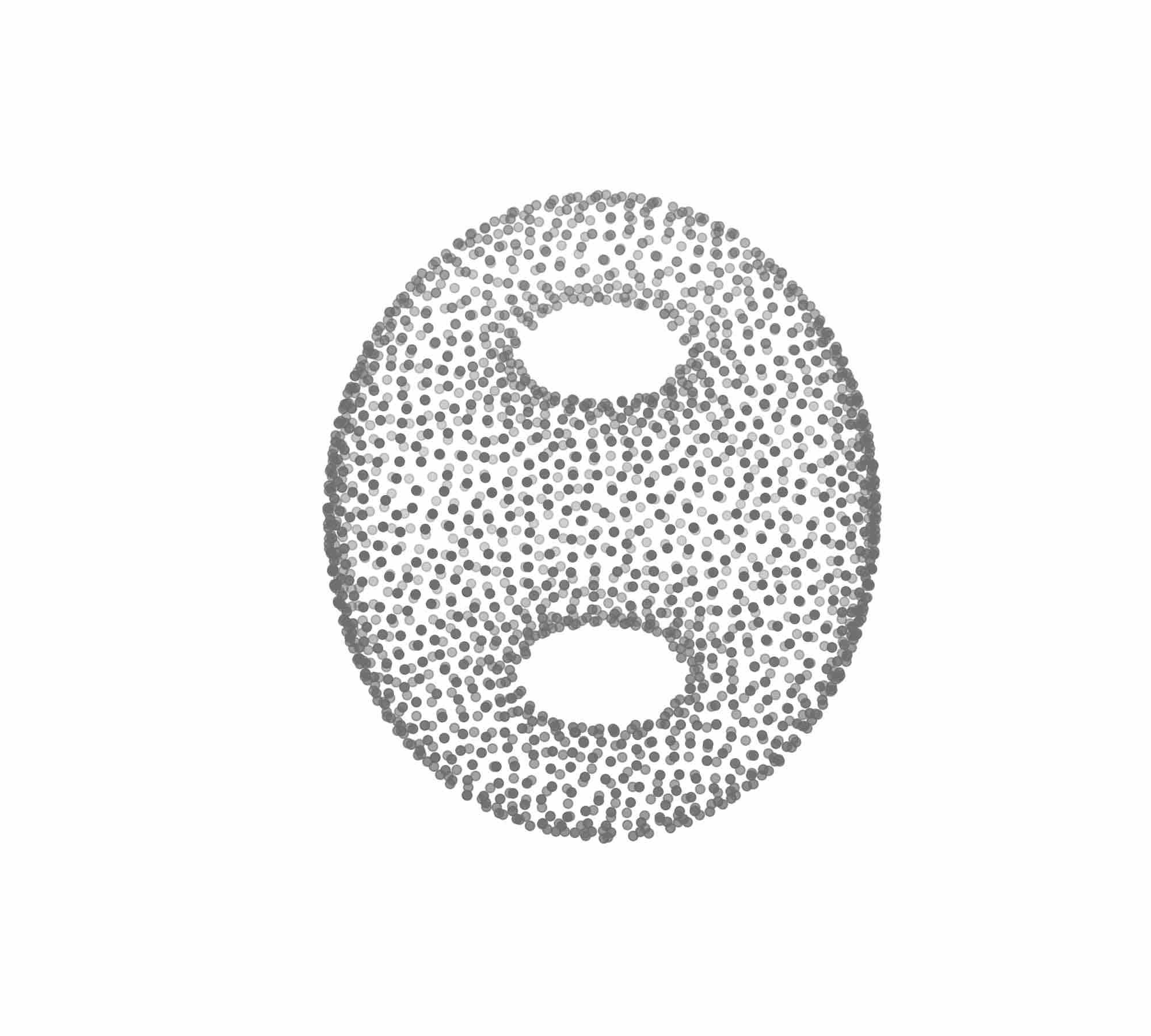} & \includegraphics[trim={10cm 10cm 10cm 10cm}, clip=true , scale=0.045] {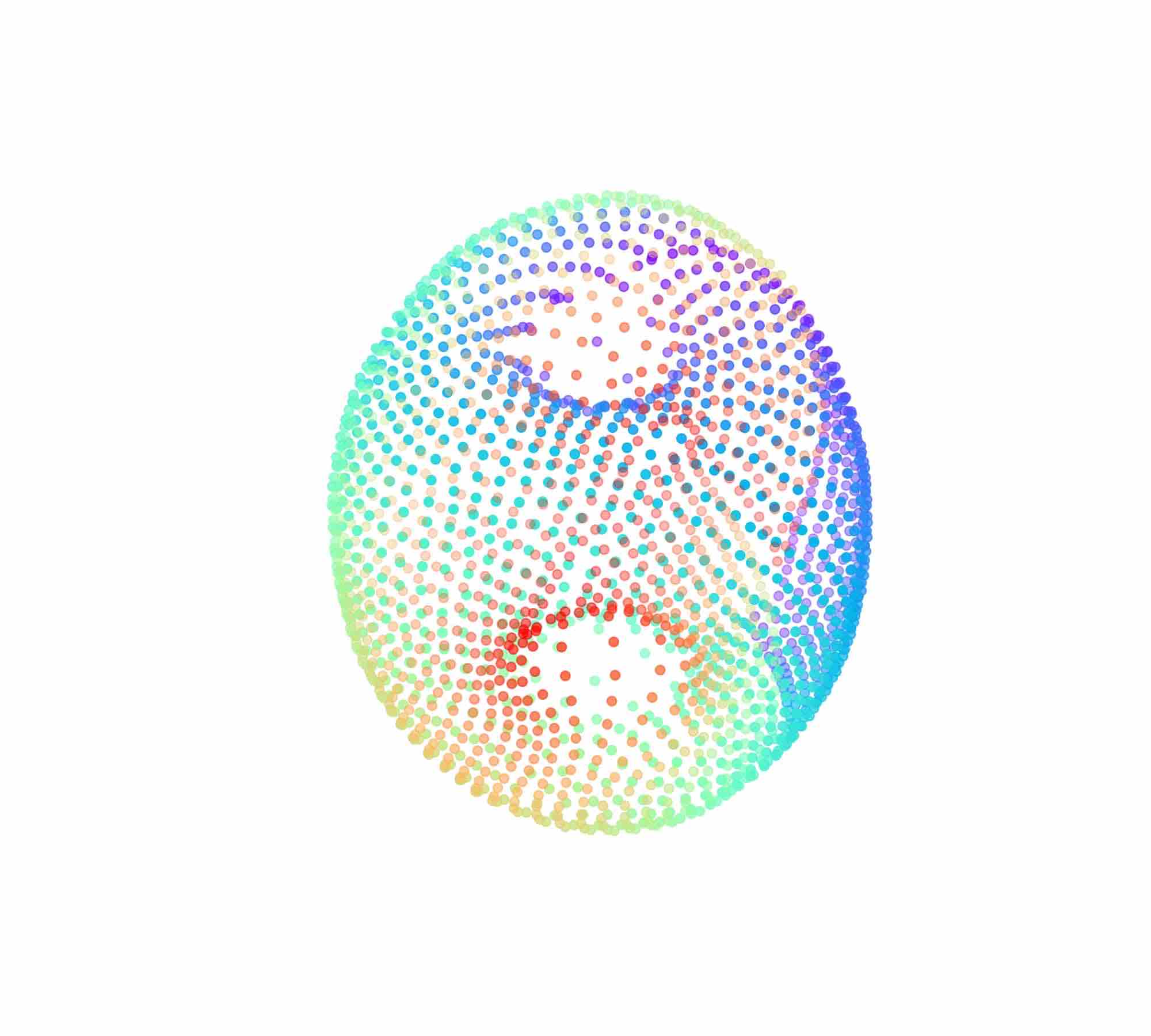} & \includegraphics[trim={10cm 10cm 10cm 10cm}, clip=true , scale=0.045] {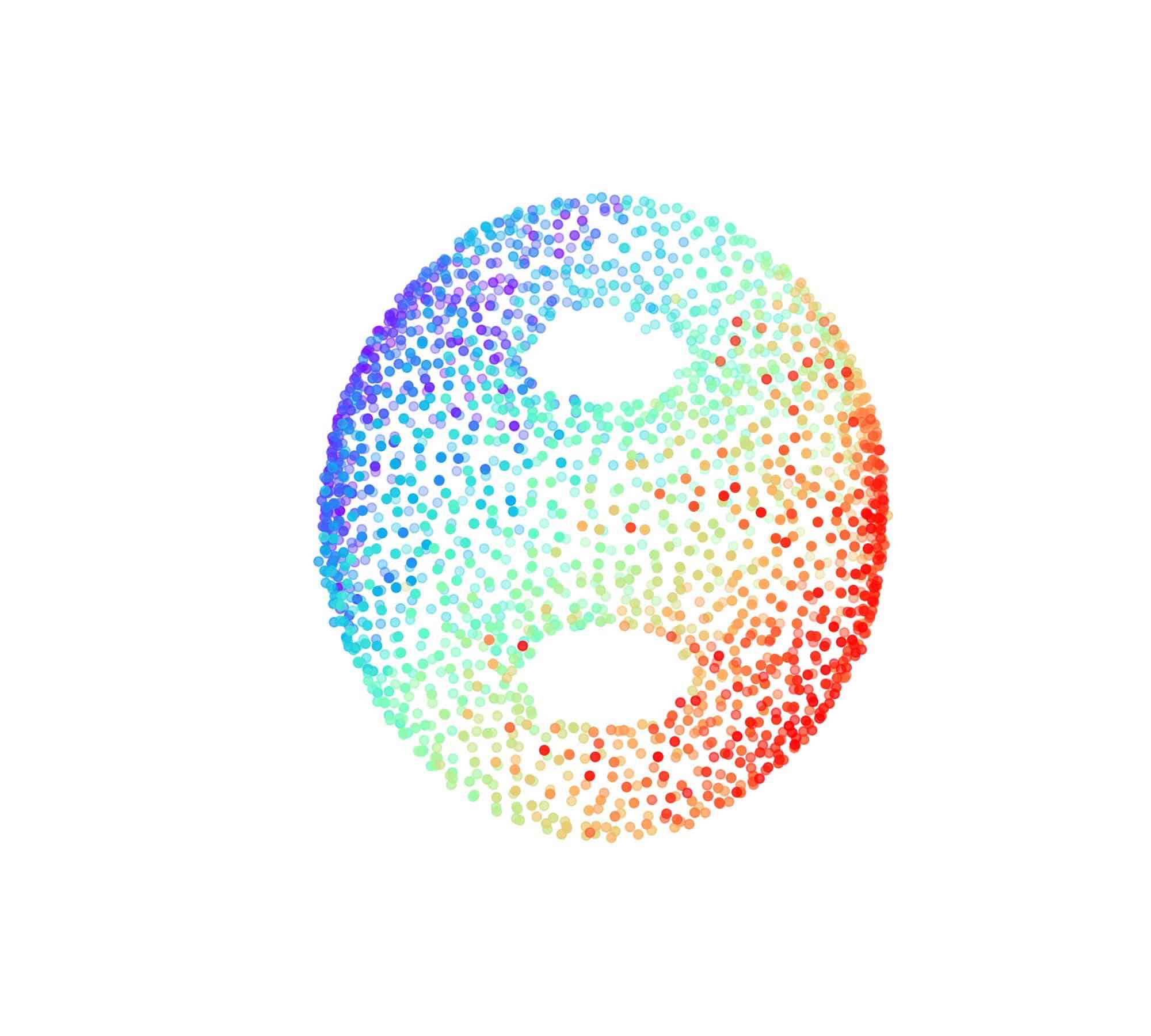}\\
      \includegraphics[trim={10cm 10cm 10cm 10cm}, clip=true , scale=0.045] {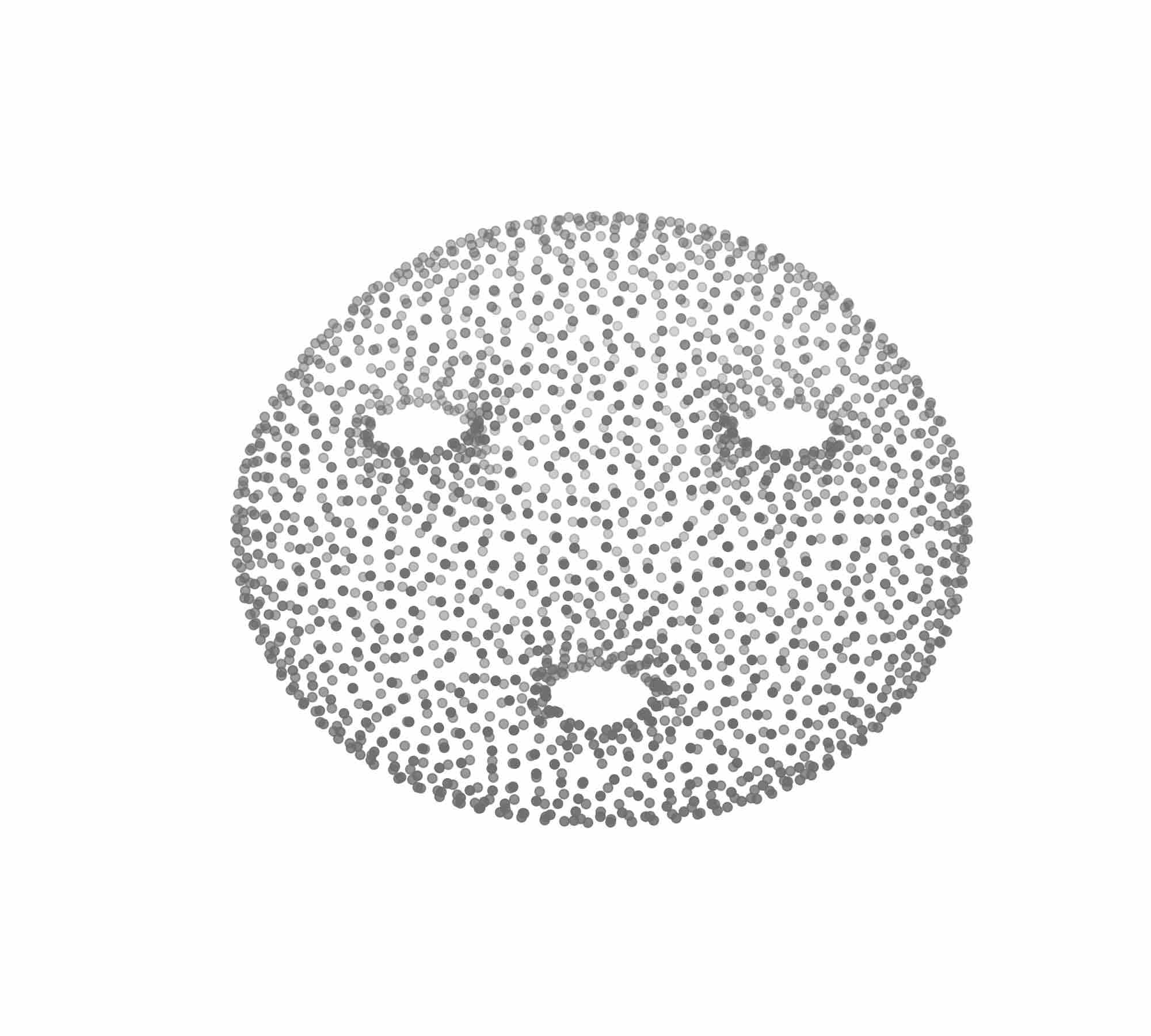} & \includegraphics[trim={10cm 10cm 10cm 10cm}, clip=true , scale=0.045] {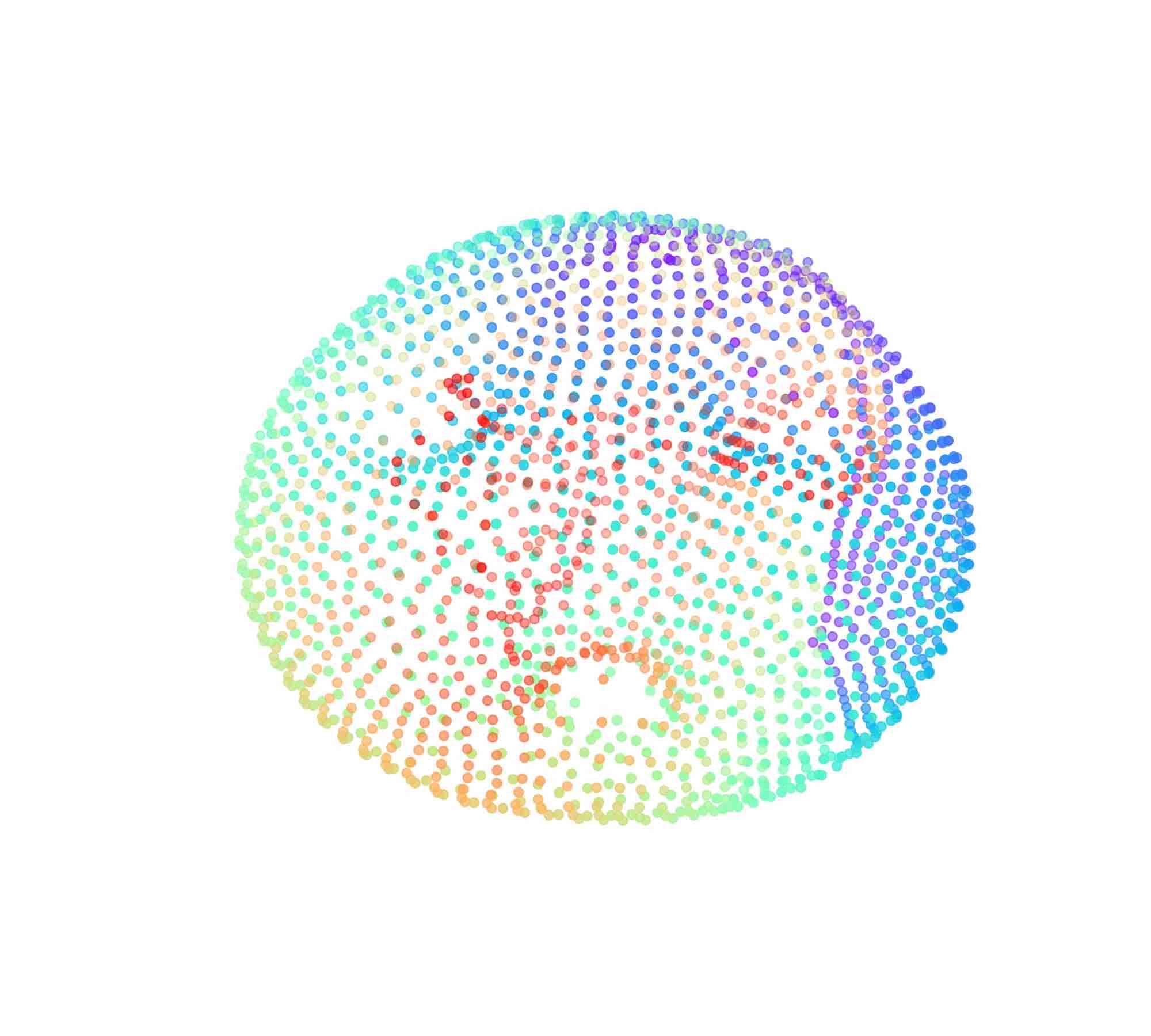} & \includegraphics[trim={10cm 10cm 10cm 10cm}, clip=true , scale=0.045] {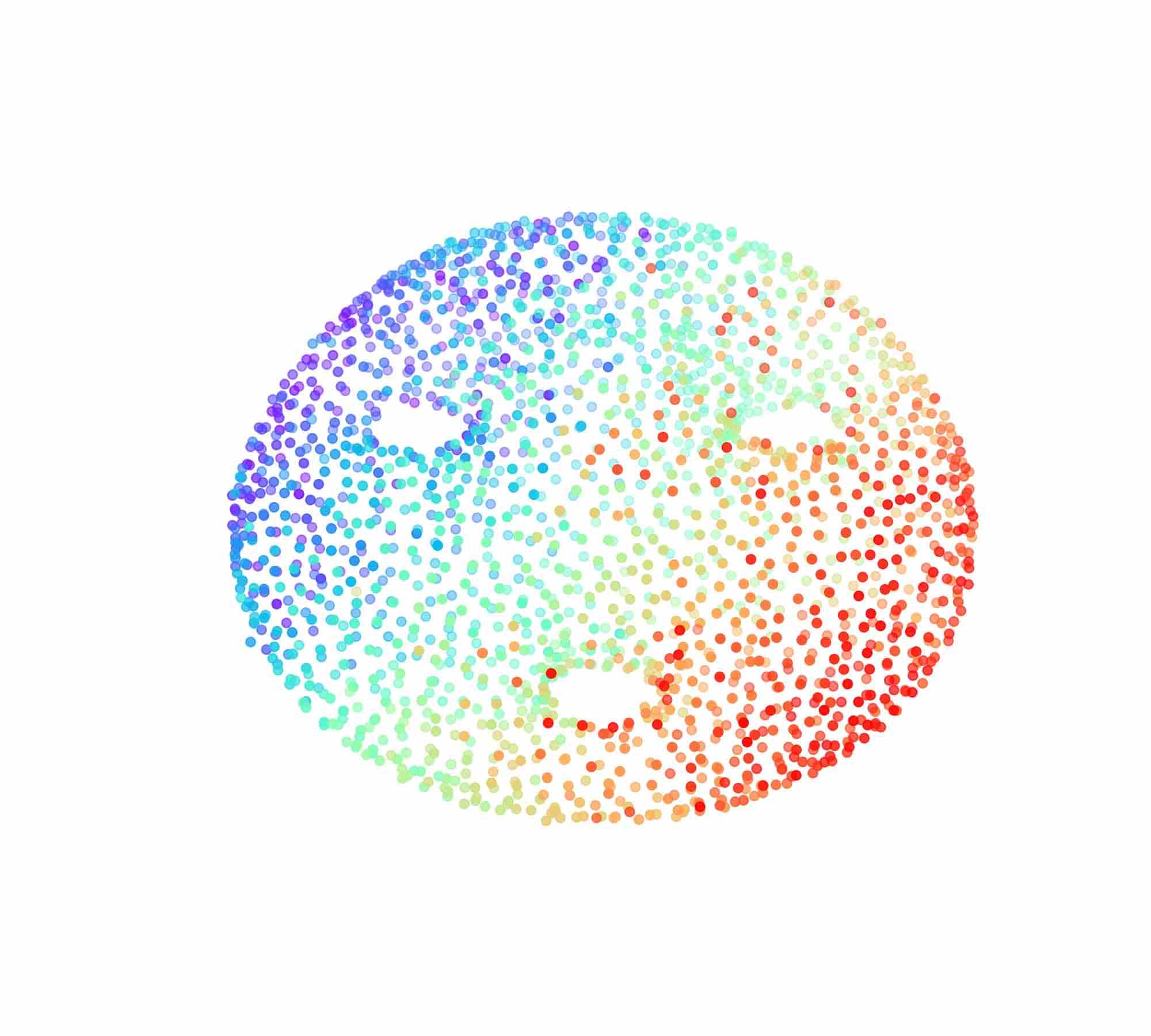}\\
      \hline
    \end{tabular}
  \end{center}
  \caption{\label{fig:torus_toy_example} \textbf{Only training the folding module cannot reconstruct  torus with high-order genus.} The first column shows the original point clouds sampled from tori generated in MeshLab \cite{meshlab}; the second column shows the reconstructions when we only train the folding module; the third column shows the reconstructions when we train all three modules together. The color associated with each point indicates the correspondence between a node in the 2D lattice and a 3D point. The smoothness of the color transition reflects the difficulty of the folding process. The reconstruction based on the folding mechanism only cannot capture the holes and the folding process is difficult to glue surfaces smoothly.}
\end{table}

\subsection{Decoder: Folding module}
\label{sec:folding_module}
The folding module decodes the code to form a coarse reconstructed 3D point cloud.  A straightforward way to  decode is to use fully-connect layers that directly map a code to a 3D point cloud~\cite{AchlioptasDMG:17}; however, it does not explore any geometric property of 3D point clouds and requires a huge number of training parameters\footnote{The decoding is to construct a mapping from the code space to the space of $3N$-dimensional point sets. The fully-connected decoding uses fully-connected layers to learn such a mapping, while our idea is to parameterize such a mapping by introducing a canonical 2D lattice.}. As mentioned in Section~\ref{sec:3DPoints}, 3D point clouds are intrinsically sampled from 2D surfaces lying in the 3D space. We then consider the reconstruction as folding from a 2D surface to a 3D surface and the folding mechanism is determined by the code produced by the encoder.

Let $\Z \in \mathbb{Z}^{M \times 2}$ be a matrix representation of  nodes sampled uniformly from a fixed regular 2D lattice and the $i$th row vector $\z_i \in \R^2$ be the 2D coordinate of the $i$th node in the 2D lattice. Note that $\Z$ is fixed and given for any 3D point cloud. It is used as a canonical base for the reconstruction and does not depend on the original point cloud.  The functionality of the folding module is to fold a 2D lattice to a surface in the 3D space. Since the code is trained in a data-driven manner, it preserves the folding mechanism. We thus can concatenate the code with each 2D coordinate and then uses MLPs to implement the folding process. Mathematically, the $i$th point after folding is
\begin{equation}
\label{eq:folding}
 \x'_i  \ = \  f_{\mathbf{c} }( \z_i )  =  {\rm MLP} \left( \left[   {\rm MLP} \left( \left[ \z_i,  \mathbf{c}  \right]  \right),  \mathbf{c}  \right]  \right) \in \mathbb{R}^{3},
\end{equation}
where the code $\mathbf{c}$ is the output of the encoder and $\left[  \cdot, \cdot \right] $ denotes the concatenation of two vectors. 
The folding function $f_{\mathbf{c} }( \cdot )$ consists of two-layer MLPs and the code is introduced in each layer to guide the folding process. We collect all the 3D points $ \x'_i $ to form the reconstruction $\S'  = \{ \x'_i \in \R^3~|~i= 1, \cdots, M \}$ with the corresponding matrix representation $\X' \in \R^{M \times 3}$. 

Intuitively, $f_{\mathbf{c} }( \cdot )$ is supposed to be a smooth function; that is, when two nodes are close in the 2D lattice, their correspondence after folding are also close in the 3D space. The smoothness makes the networks easy to train; however, when the 3D surfaces have a lot of curvatures and complex shapes, the smoothness of the 2D to 3D mapping limits the representation power. Table~\ref{fig:torus_toy_example} shows that only training the folding module cannot reconstruct tori with high-order genus. The color associated with each point indicates the correspondence between a node in the 2D lattice and a 3D point. The smoothness of the color transition reflects the difficulty of the folding process. We see that only training the folding module cannot capture the holes in tori and the networks cannot find an appropriate  way to fold. The reason behind is that the folding process implies the spatial smoothness, which can hardly construct an arbitrarily complex shape. Therefore, we consider $\X' \in \R^{M \times 3}$ a coarse reconstruction and we next propose two other modules to refine the reconstruction.

\subsection{Decoder: Graph-topology-inference module}
\label{sec:topology_learning_module}
The graph-topology-inference module decodes the code to reflect the pairwise relationships between 3D points.  As the initialization of 3D points used in the folding module, the 2D lattice sets a default and uniform connection pattern to each pair of 3D points; however, it cannot capture irregular connection patterns, especially a topology that is genus-wise different from a 2D plane.  To solve this, we learn a graph to capture irregular pairwise relationships, empowering the networks to refine a coarse reconstruction.

The learnt graph is initialized by the same 2D lattice used  in the folding module. The nodes are fixed and the edges are updated during the training process. The initial graph adjacent matrix $\Adj^{0}  \in \mathbb{R}^{M \times M}$ is 
\begin{equation*}
\Adj^{0}_{ij} \ = \ 
\begin{cases}
\frac{1}{Z_i} \exp(-\frac{\lVert \z_i - \z_j \rVert_2^2}{2\sigma^2}) & \text{if } \z_j \in \mathcal{N}_i, \\
0& \text{otherwise},
\end{cases}
\end{equation*}
where  $\z_i$ is the $i$th node in the canonical 2D lattice, the hyper-parameters $\sigma$ reflects the decay rate, $\mathcal{N}_i$ represents the $k$-nearest neighboring nodes of node $\z_i$ and a normalization term $Z_i = \sum_j \exp(-\lVert \z_i - \z_j \rVert_2^2/(2\sigma^2))$ ensures that $\sum_j \Adj^{0}_{ij} = 1$. Note that we use the same node set in the 2D lattice, but introduce more connections by considering $k$-nearest neighbors for each node, instead of restricting to 4. This can increase the reception field and make the model easier to train.

Since the code produced by the encoder preserves information of the original point cloud, we concatenate the code with each row of $\Adj^{0}$, and then uses MLPs to implement the graph topology inference. Mathematically, the $i$th row of the learnt graph adjacency matrix is obtained as
\begin{eqnarray}
\label{eq:adj}
 \Adj_i  & = &  g_{\mathbf{c} }( \z_i ) 
\\ \nonumber 
  & = & {\rm softmax} \left(  {\rm MLP} \left( \left[   {\rm MLP} \left( \left[  \Adj^{0}_i ,  \mathbf{c}  \right]  \right),  \mathbf{c}  \right]  \right) \right) \in \mathbb{R}^{M},
\end{eqnarray}
where $\Adj_i \in  \mathbb{R}^{M}$ is  the $i$th row of $\Adj$ and $\mathbf{c} $ is the code produced by the encoder. The softmax operation promotes sparse connections, which can help reduce overfitting. Note that (i) the last layer of MLP uses ReLU as the nonlinear activation function, ensuring that all the edge weights are nonnegative; (ii) the softmax ensures that the sum of each row in $\Adj$ is one; that is, $\Adj \one = \one \in \R^{M}$. We also recognize this as a random-walk matrix~\cite{Newman:10}, whose element reflects the transition probability jumping from one node to another. We then introduce one more step to make the graph adjacency matrix symmetric; that is,
$\Adj = (\Adj' + \Adj'^T)/2$.
This step ensures the pairwise relationships between 3D points are undirected. During the experiments, we see that this step can stabilize the training process.

The intuition behind this is that the initial graph adjacent matrix provides initial pairwise relationships in the 2D lattice and the latent code provides the mechanism to adjust the edge weights in the graph adjacent matrix and properly reflect the pairwise relationships. At the same time, during training, the graph-topology-inference module pushes the latent code to preserve the spatial relationships in the original 3D point cloud, guiding the evolvement of the graph adjacency matrix.

Note that the memory cost of learning a graph topology is $O(M^2)$. One potential solution of dealing with a large-scale 3D point cloud is to use the divide-and-conquer strategy~\cite{ChenNLL:19}; however, the further improvement is subject to the future work.

\subsection{Decoder: Graph-filtering module}
\label{sec:graph_filtering_module}
The graph-filtering module couples the previous two modules, refining a coarse reconstruction through a learnable graph topology. In Section~\ref{sec:3DPoints}, we mention that the underlying surface of 3D points is piecewise-smooth in the 3D spatial domain. The curvatures along the surface cause discontinuities; however, 3D coordinates are always smooth along the surface. In other words, when a graph perfectly reflects the underlying surface,  3D coordinates should be smooth in the graph domain. We thus design low-pass graph filters based on the learnt graph adjacency matrix to filter the coarse reconstruction and obtain the refined reconstruction. A graph filter allows each point to aggregate information from its neighbors to refine its 3D position. At the same time, the graph filtering pushes the networks to learn a graph topology that preserves smoothness on the graph for 3D points. Note that previous works design graph filters to process 3D point clouds, where both the graph topologies and the input 3D point clouds are fixed and known~\cite{ThanouCF:16,ZengCNPY:18,LozesEL:15,ChenTFVK:18}, while the proposed system uses a fixed graph filter to couple the graph topology inference and the folding process in an end-to-end learning system, where both the graph topologies and the input 3D point clouds are trainable.

Here we consider two types of graph filters.

\mypar{Graph-adjacency-matrix based filter design} Similarly to~\cite{ChenSMK:14a}, we can design smooth graph filters based on the learnt graph adjacency matrix~\eqref{eq:adj}; that is, 
\begin{equation}
\label{eq:adj_filter}
h(\Adj)     \ = \  \sum_{\ell=0}^{L-1} h_{\ell} \Adj^\ell \in \R^{M \times M},
\end{equation}
where $h_{\ell}$'s are filter coefficients and $L$ is the order of the graph filter. A larger $L$ indicates a bigger receptive field; however, a large $L$ slows down the training process. For simplicity, we thus set $L=1$ and $h_{0} = h_{1} = 0.5$. The final reconstruction is then
\begin{equation}
\label{eq:adj_filtering}
\widehat{\X}  \ = \  \frac{1}{2}(\Id + \Adj) \X' \in \R^{M \times 3},
\end{equation}
where $\X'$ is the coarse reconstruction obtained by the folding module. Correspondingly, the $i$th row of $\widehat{\X} $ is  $\widehat{\x}_i = (\x'_i + \sum_{j \in \mathcal{N}_i} \Adj_{ij} \x'_j ) / 2$, where $\x'_i$ is obtained in~\eqref{eq:folding}. The corresponding set representation is $\widehat{\S} = \{  \widehat{\x}_i \in  \R^{3}~|~i= 1, \cdots, M\}$. We recognize~\eqref{eq:adj_filtering} as the low-pass graph Haar filter~\cite{VetterliKG:12}. In traditional signal processing, the Haar filter is used to locally smooth time-series or images and remove noises. Here we use the graph Haar filter to locally smooth the coarse 3D point cloud to obtain a refined one. Theorem~\ref{thm:adj_smooth} shows the smoothness effect of the proposed low-pass graph Haar filter.

\mypar{Graph-Laplacian-matrix based filter design}
We can also convert the learnt graph adjacency matrix to 
the graph Laplacian matrix and then design smooth graph filters.  Let  $\mathcal{L}  = \widetilde{\D}  -  \widetilde{\Adj}$ be the graph Laplacian matrix, $\widetilde{\Adj} = \left( \Adj + \Adj^T \right)/2$ is  the symmetric graph adjacency matrix and $\widetilde{\D} = {\rm diag}  \left( \widetilde{\Adj}  \one \right)$ is the degree matrix.  The eigendecomposition of $\mathcal{L} $ is
\begin{eqnarray}
\label{eq:graph_lap}
  \mathcal{L} 
  \ = \    \Vm  \Sigma \Vm^T =  \Vm 
  \begin{bmatrix}
    \lambda_1 &  0 & \cdots & 0  \\
    0  &  \lambda_2  & \cdots & 0 \\
    \vdots &  \vdots & \ddots & \vdots  \\
    0 & 0 & \cdots &  \lambda_M
  \end{bmatrix}
  \Vm^{T},
\end{eqnarray}
where the eigenvector matrix $\Vm \in \R^{M \times M}$ is the graph Fourier basis and the eigenvalues ($\lambda_1 \leq \lambda_2 \leq \cdots \leq \lambda_M$) capture the graph frequencies. Since $\mathcal{L}  \one = 0 \cdot \one$, the first eigenvalue is zero ($\lambda_1 = 0$) and the corresponding eigenvector is all ones, representing the smoothest graph signal. A larger eigenvalue is associated with a less smooth graph signal. The graph Laplacian $\mathcal{L}$ captures the second-order differences and its inverse promotes the global smoothness~\cite{Chung:96}. 

We thus consider the graph filters as follows
\begin{equation}
\label{eq:lap_filter}
h( \mathcal{L} )    \ = \   \left( \mu \Id +   \mathcal{L}  \right)^{-1}   \in \R^{M \times M},
\end{equation}
where $\mu > 0$ is a hyperparameter to avoid computational issues. The graph-spectral representations of the graph filters are
\begin{eqnarray*}
  h( \mathcal{L} ) 
  \ = \  \Vm 
  \begin{bmatrix}
    \frac{1}{\mu} &  0 & \cdots & 0  \\
    0  &   \frac{1}{\mu + \lambda_2}  & \cdots & 0 \\
    \vdots &  \vdots & \ddots & \vdots  \\
    0 & 0 & \cdots &  \frac{1}{\mu + \lambda_M} 
  \end{bmatrix}
  \Vm^{T}.
\end{eqnarray*}
The final reconstruction is then $\widehat{\X}  \ = \    h( \mathcal{L} )  \X' \in \R^{M \times 3}.$

Note that since~\eqref{eq:lap_filter} involves a matrix inversion, the computational cost of the graph-Laplacian-matrix based filter is much more expensive than the graph-adjacency-matrix based filter~\eqref{eq:adj_filtering}.

\section{Theoretical Analysis}
In this section, we show the theoretical analyses about the proposed method.

\subsection{Reconstruction error}
The aim here is to show when the proposed encoder and decoder can lead to bounded reconstruction errors. Without loss of generality, we assume that all the 3D points lie in a unit cubic space.  The following theorem shows the existence of a pair of encoder and decoder to  reconstruct an arbitrary 3D point cloud.
\begin{myThm}
\label{thm:upper_bound}
Let $\mathcal{X}_n = \{  \S \subset [0, 1]^{3}, |\S| = n \}$ be the space of 3D point clouds with $n$ points. Then,  there exits a pair of encoder  $\Psi: \mathcal{X}_N \rightarrow \R^C$  and decoder  $\Phi: \R^C \rightarrow \mathcal{X}_C$ that satisfies
\begin{equation*}
 d \left(  \S,  \Phi \left( \Psi(\S) \right) \right)  \ \leq \  \frac{\sqrt{3}}{2 \sqrt[3]{C}}  \approx  \frac{0.866}{\sqrt[3]{C}}  ,  {\rm~for~all~} \S \in \mathcal{X}_N,
\end{equation*}    
where $d(\cdot, \cdot)$ is the augmented Chamfer distance~\eqref{eq:aug_chamfer}.
\end{myThm}
\begin{proof}
We  proof by construction. We partition the 3D space into equally-spaced nonoverlapping voxels along each of three dimensions. We evenly split the unit 3D space $[0, 1]^{3}$ into $K^3$ 3D voxels. The $(i, j, k)$th voxel represents a 3D space, 
\begin{eqnarray*}
\V_{i, j, k}  \ = \ \{ (x, y, z)  |   (i-1) / K   \leq & x & < i / K, 
\\ \nonumber
 (j-1) / K  \leq & y  & <  (j-1) / K ,
\\ \nonumber
 (k-1) / K  \leq &  z  & <  (k-1) / K \}.
\end{eqnarray*}

Let $g_{i,j,k} (\x) = \max_{\z \in \V_{i, j, k} } \exp\left( - \left\| \x - \z \right\|_2^2 ) \right) \in \R$ be  a soft indicator function, reflecting the presence of the point $\x$ in the $(i, j, k)$th voxel. Let $g(\x) = [g_{1,1,1}(\x), \cdots, g_{K,K,K}(\x) ] \in \R^{K^3}$, reflecting the influence of the point $\x$ to the unit space. Let the encoder work as
\begin{equation}
\label{eq:encoder_construction}
\c  \ = \ \Psi(\S) = [\Psi_{1,1,1} (\S), \cdots, \Psi_{K,K,K} (\S)] \in \R^{C},
\end{equation}
whose element $\Psi_{i,j,k} (\S) =  \max_{\x \in \S} g_{i,j,k} (\x) \in \R$ indicates the occupancy of the $(i,j,k)$th voxel by 3D points in $S$ with $C = K^3$. By construction, the encoder discretizes a 3D point cloud to voxels and the code reflects the voxel occupancy. 

Let the decoder work as
\begin{equation}
\label{eq:decoder_construction}
\widehat{\S}  \ = \ \Phi(\c) =  \{ \q_{i,j,k} \in [0,1]^3~|~\c_\ell = \Psi_{i,j,k} (\S) = 1 \},
\end{equation}
where $\q_{i,j,k}  =  [ i/K, j/K, k/K ] \in [0,1]^3$ is the 3D point, which acts as the  proxy of the $(i,j,k)$ voxel. In~\eqref{eq:decoder_construction}, we use two ways to index the code, indicating the one-to-one correspondence between $\ell$ and a triple of $i,j,k$. The code works as a lookup table:
When an element in the code is one, the corresponding voxel is activated. In other words, the index reflects the corresponding voxel and the value is a mask to select the activated voxel.

Now, we bound $d(\S, \widehat{\S})$. Because of the discretization, for any point $\x$ in $\S$, we can find a triple of $i,j,k$ to satisfy $\left\| \x - \q_{i,j,k} \right\|_2 \leq \sqrt{3}/(2K)$; at the same time, for any point $\q_{i,j,k} \in  \widehat{\S}$, we can find at least one point $\x \in \S$ to satisfy $\left\| \q_{i,j,k} - \x \right\|_2 \leq \sqrt{3}/(2K)$; see Figure~\ref{fig:proof_max_error} (a). Then, we have
\begin{eqnarray*}
  d(\S, \widehat{\S}) & = &
   \max  \bigg\{  \frac{1}{N} \sum_{\x \in \S}  \min_{ \widehat{\x} \in  \widehat{S}}  \left\|  \x - \widehat{\x} \right\|_2,   \frac{1}{M} \sum_{\widehat{\x} \in  \widehat{S}}  \min_{\x \in \S}  \left\|   \widehat{\x} - \x \right\|_2  \bigg\}
   \\
   & = & 
      \max  \bigg\{  \max_{\x \in \S} \min_{ i,j,k }  \left\|  \x - \q_{i,j,k} \right\|_2, 
    \\
    && ~~~~~~~  \max_{i,j,k} \min_{\x \in \S}  \left\|  \q_{i,j,k} - \x \right\|_2  \bigg\}
      \\
   & = & 
      \max  \bigg\{ \frac{\sqrt{3}}{2K},  \frac{\sqrt{3}}{2K} \bigg\}  =  \frac{\sqrt{3}}{2K} = \frac{\sqrt{3}}{2 \sqrt[3]{C}}. 
\end{eqnarray*}
\end{proof}

Theorem~\ref{thm:upper_bound} constructs a naive pair of encoder and decoder to show the upper bound of the reconstruction error. To achieve this upper bound, the encoder only needs to discretize the 3D space, such that each point can be mapped to the corresponding voxel;  and the decoder only needs to learn the index-to-voxel correspondence, such that each element in the code can represent the corresponding voxel. Note that the proposed structures of the encoder~\eqref{eq:encoder_impl} and the decoder~\eqref{eq:folding} are consistent with the hypothetical constructions, which is shown in details in the following corollary.

\begin{myCorollary}
The proposed encoder~\eqref{eq:encoder_impl} implements the constructed encoder~\eqref{eq:encoder_construction}. The proposed folding module~\eqref{eq:folding} implements the decoder~\eqref{eq:decoder_construction}. The proposed encoder followed by the proposed decoder leads to the same upper bound as shown in Theorem~\ref{thm:upper_bound}.
\end{myCorollary}

\begin{proof}
Here we mainly prove the decoder part. Specifically, we aim to show the following: Let $M$ be the number of output points and $\mathbf{c}_0  = [\mathbf{c},\mathbf{x}_0]^\top$ be the code, where the code $\mathbf{c}$ follows~\eqref{eq:encoder_construction} and $\mathbf{x}_0 \in \R^3$ is an arbitrary 3D point in the original point cloud $\S$. Then, we use the folding module~\eqref{eq:folding} to decode $\mathbf{c}_0$ and  reconstruct a 3D point cloud $\widehat{\S}$ to satisfy that $d(\S,\widehat{\S}) \leq \sqrt{3}/{(2K)}$.

Now we explicitly provide the coefficients of each layer in the MLP. We denote by $\bm{\theta}_l$ the input and output to each layer. Note that by Definition 1 of \cite{YangFST:18}, the code is concatenated with $M$ pairs of two dimensional grid coordinates $\z_m = (x_m,y_m)$, $m=1,2,\ldots,M$, and all of the $M$ concatenated vectors are fed into the same MLP in parallel. In our proof, we assume that the number of output points $M$ is the same as the number of 3D grid points $K^3$, i.e., $M=K^3$. The input to the first layer of the MLP is a vector of length $K^3+5$, i.e., $\bm{\theta}_1 = [\mathbf{c},\mathbf{x}_0,\z_m]^\top$, where recall that $\mathbf{x}_0$ is a point in the original point cloud $S$. Note that the index $m$ represents the $m$th 3D voxel, and that also corresponds to the ($i$,$j$,$k$)-th grid point in Theorem~\ref{thm:upper_bound}. In the proof below, we will sometimes refer to the index $m$ and the ($i$,$j$,$k$)-th grid point exchangeably.

The first layer is a simple nonlinear layer that only operates on the first $K^3$ elements on the input $\bm{\theta}_1$ to change the value to a binary 1 or 0, i.e., $f_1(x) = x$ if $x\geq 1$. Therefore, after the first layer, the output $\bm{\theta}_2$ represents the occupancy of the 3D voxel grid of size $K^3$, concatenated by one point $\mathbf{x}_0$ in the original point cloud $S$ and the 2D-grid point $\z_m$.

The second layer is a linear layer that maps the first $K^3$ elements in the input $\bm{\theta}_2$ to $4K^3$ elements, and does not change the concatenated values of the last 5 entries. The feature map is the $4K^3$-by-$K^3$ matrix
\begin{equation}
 \mathbf{F}_2 = \mathbf{I}_{K^3}\otimes [1,0,0,0]^\top,  
\end{equation}
where $\mathbf{I}$ is the identity matrix, and $\otimes$ represents the Kronecker product. The bias vector is 
\begin{equation}
    \mathbf{b}_2 = [\mathbf{b}_{2,1},\mathbf{b}_{2,2},\ldots,\mathbf{b}_{2,K^3}]^\top,
\end{equation}
where each $\mathbf{b}_{2,m}$ is a length-4 vector $[0,\mathbf{q}_m]$ and the length-3 vector $\mathbf{q}_m = \mathbf{q}_{i,j,k}$, i.e., the ($i$,$j$,$k$)-th 3D grid point. We apply this feature map and the bias vector to the first $K^3$ entries in the input $\bm{\theta}_2$ (the total length of which is $K^3+5$), and the output is a vector of length $4K^3+5$. The output vector can be written as
\begin{equation}
    \bm{\theta}_3 = [\bm{\theta}_{3,1},\bm{\theta}_{3,2},\bm{\theta}_{3,3},\ldots,\bm{\theta}_{3,K^3},\mathbf{x}_0,\mathbf{u}_m],
\end{equation}
where $\bm{\theta}_{3,m} = [f_1(\Psi_{i,j,k}(\S)),\mathbf{q}_{i,j,k}]$, where recall that $f_1(\Psi_{i,j,k}(\S))$ means the binary value that indicates whether the ($i$,$j$,$k$)-th 3D voxel has been occupied or not. 

The third and the fourth layers are exactly the same as Theorem 3.1 in~\cite{YangFST:18}. The readers are referred to~\cite{YangFST:18} and the proof therein for details. The Theorem 3.1 in~\cite{YangFST:18} shows that one can use the grid point $\mathbf{u}_m$ to select the $m$th segment in a vector of length $3M$, and get a segment of size 3. Therefore, we can use exactly the same construction to get the $m$th segment of the first $4K^3$ elements in $\bm{\theta}_3$. Therefore, after the third and the fourth layer, the input to the fifth layer can be written as
\begin{equation}
    \bm{\theta}_5 = [f_1(\Psi_{i,j,k}(\S)),\mathbf{q}_{i,j,k},\mathbf{x}_0],
\end{equation}
where $f_1(\Psi_{i,j,k}(\S))$ indicates whether the ($i$,$j$,$k$)-th 3D voxel has been occupied or not, $\mathbf{q}_{i,j,k}$ means the position of the ($i$,$j$,$k$)-th 3D voxel, and $\mathbf{x}_0$ means one point in the original point cloud.

The fifth-layer can choose $\mathbf{q}_{i,j,k}$ or $\mathbf{x}_0$ based on whether $f_1(\Psi_{i,j,k}(\S))=1$ or not. One may wonder how we can use an MLP to make this selection. One way is to simply let the output be $\mathbf{y}_m = f_1(\Psi_{i,j,k}(\S))\cdot\mathbf{q}_{i,j,k}+(1-f_1(\Psi_{i,j,k}(\S)))\cdot\mathbf{x}_0$. However, this is not an MLP because it contains multiplication. In fact, Theorem 3.1 in~\cite{YangFST:18} again provides the construction: one can use an MLP to select one input in the vector to be the output. Thus, the second way is to construct a selecter-MLP as in Theorem 3.1 in~\cite{YangFST:18} to get the output.

The output of the MLP is thus either the grid point $\mathbf{q}_{i,j,k}$, or the point $\mathbf{x}_0$ that is already in the original point cloud. Therefore, the final output is the same as in Theorem~\ref{thm:upper_bound} plus one extra point $\mathbf{x}_0$ that is already in the original point cloud, which does not change the Chamfer distance. This concludes the proof.
\end{proof}
The intuition behind the proof is that the folding module ~\ref{sec:folding_module} can be represented as
\begin{equation*}
\Phi_{\Z}(\c) =  \{ f(\z_{\ell}) \in [0,1]^3~|~\c_\ell = 1, \z_{\ell} \in \Z, |\c| = |\Z| = C \},
\end{equation*}
where $\z_\ell  \in \mathbb{Z}^2$ is a node sampled from a 2D lattice $\Z$ and $f: \mathbb{Z}^2 \rightarrow [0,1]^3$.  It uses $\z_{\ell}$ to initialize $\q_{i,j,k}$in~\eqref{eq:decoder_construction} and only needs to learn the mapping function $f(\cdot)$ that lifts 2D points to 3D points. While LatentGAN~\cite{AchlioptasDMG:17} uses fully connected layers as the decoder, which has to learn $\q_{i,j,k}$ from scratch.

\begin{figure}[thb]
  \begin{center}
    \begin{tabular}{cc}
    \includegraphics[width=0.42\columnwidth]{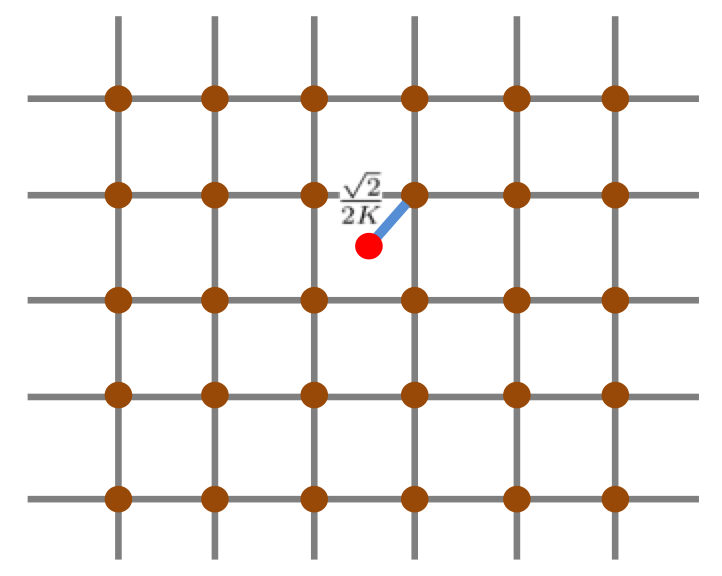}     & \includegraphics[width=0.42\columnwidth]{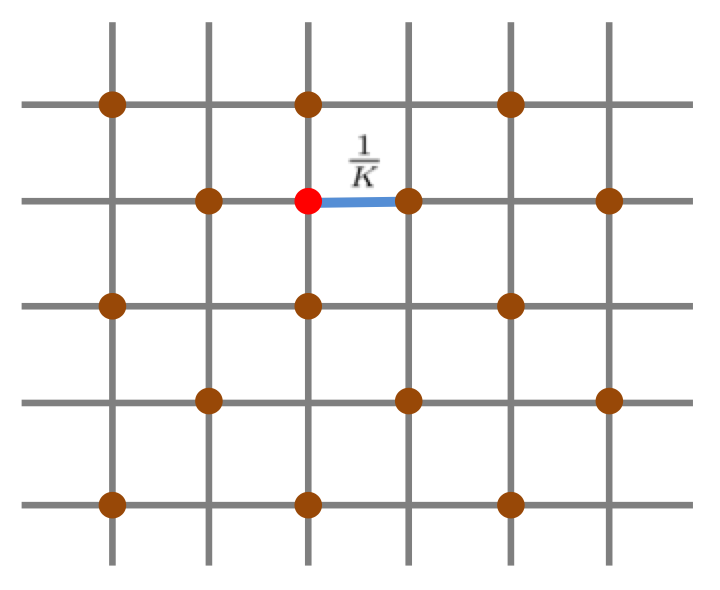}   
    \\
    {\small (a)  Error in Theorem~\ref{thm:upper_bound}.} &  {\small (b) Error in Theorem~\ref{thm:upper_bound_smooth}.}
  \end{tabular}
\end{center}
\caption{\label{fig:proof_max_error} 2D Illustration of Theorems~\ref{thm:upper_bound} and~\ref{thm:upper_bound_smooth}. As shown in both theorems, we discretize the space. For simplicity, we show 2D space, instead of 3D.  The code in Theorem~\ref{thm:upper_bound} preserves the occupancy of all voxels. Plot (a) shows the maximum error is $\sqrt{2}/(2K)$. The code in Theorem~\ref{thm:upper_bound_smooth} only preserves the occupancy of every other voxel. Plot (b) shows the maximum error is $1/K$. }
\end{figure}

The following theorem further shows that under the smoothness assumption, a filtering step refines the coarse 3D point clouds and achieves the better reconstructions.
\begin{myThm}
\label{thm:upper_bound_smooth}
Let $\mathcal{X}_n = \{  \S \subset [0, 1]^{3}, |\S| = n \}$ be the space of 3D point clouds with $n$ points. Let $\Psi: \mathcal{X}_N \rightarrow \R^C$  be the encoder defined in~\eqref{eq:encoder_construction}. Let $\widetilde{\mathcal{X}}_n \subset \mathcal{X}_n$ satisfy that for all $\S \in  \widetilde{\mathcal{X}}_n$,
\begin{eqnarray}
\label{eq:smooth}
&& \min_{\delta_i=\{0,1\}, \delta_j=\{0,1\}, \delta_k=\{0,1\}}  \Psi_{i+\delta_i,j+\delta_j,k+\delta_k}(\S)  \leq 
\\ \nonumber
 && \Psi_{i,j,k}(\S)  \leq  \max_{\delta_i=\{0,1\}, \delta_j=\{0,1\}, \delta_k=\{0,1\}}  \Psi_{i+\delta_i,j+\delta_j,k+\delta_k}(\S).
\end{eqnarray}
Then, there exits  a coarse decoder $\Gamma: \R^C \rightarrow \mathcal{X}_C$ and a filter $h: \mathcal{X}_C \rightarrow \widetilde{\mathcal{X}}_C$ that satisfies
\begin{equation*}
 d \left(  \S,  h \left( \Gamma \left( \Psi(\S) \right) \right) \right) \  \leq \  \frac{1}{\sqrt[3]{2C}}  \approx  \frac{0.7937}{\sqrt[3]{C}} ,  {\rm~for~all~} \S \in \widetilde{\mathcal{X}}_N,
\end{equation*}    
where $d(\cdot, \cdot)$ is the augmented Chamfer distance~\eqref{eq:aug_chamfer}.
\end{myThm}
\begin{proof}
To take the advantage of the smoothness condition~\eqref{eq:smooth}, instead of encoding each voxel, we can partition the space into a finer resolution and encode every other voxel. The decoder $\Gamma(\cdot)$ can then recover every other voxel. After that, $h(\cdot)$ is introduced to interpolate via neighboring voxels; that is, we add $\q_{i,j,k}$ to the final reconstruction $\widehat{\S}$ when $\q_{i+\delta_i,j+\delta_j,k+\delta_k} \in \widehat{\S}$ with $\delta_i=\{0,1\}, \delta_j=\{0,1\}, \delta_k=\{0,1\}$.  The maximum error we could make is no more than $1/K$; see Figure~\ref{fig:proof_max_error} (b). Since we use the half of the length of code, the code length is $C = K^3/2$, the error is ${1}/{\sqrt[3]{2C}}$.
\end{proof}
The smoothness condition~\eqref{eq:smooth} only requires the voxel is similar to its neighbor voxels, which is soft  and easy to satisfy. This result suggests that exploring smoothness can reduce the reconstruction error. The proposed graph-filtering module in Section~\ref{sec:graph_filtering_module} uses a graph filter to promote smoothness and potentially leads to a smaller reconstruction error. Theorem~\ref{thm:upper_bound_smooth}  assumes smoothness in the 3D spatial domain; however, 3D points may not be smooth in the 3D spatial domain. We next show that graphs and graph smoothness can be used to capture the distribution of 3D points.

\subsection{Graph smoothness}
The aim here is to show graph smoothness is a key prior to model 3D point clouds. When the graph-topology-inference module in Section~\ref{sec:topology_learning_module} can infer appropriate graphs, we can model 3D point clouds as smooth graph signals, instead of nonsmooth spatial signals. Since the graph-filtering module in Section~\ref{sec:graph_filtering_module} promotes graph smoothness, it is expected to improve the reconstruction quality of point clouds.

For simplicity, instead of considering points in the 3D space, we consider nodes in a 2D lattice, which is a simplified and discretized version of points in the 3D space. We use a binary value to indicate if each node in the 2D lattice is occupied. Those binary values on the 2D lattice form a~\emph{2D lattice signal}. Since a 3D point cloud intrinsically represents a 2D surface, 2D lattice signal intrinsically represents a 1D curve.
\begin{defn}
Let $\X \in \{0, 1\}^{N \times N}$ be a 2D lattice signal, where $\X_{i,j}$ denotes the signal value supported on the $(i,j)$th node. The directional variation of the spatial signal $\X$ at $(i,j)$th node is 
\begin{eqnarray*}
{\rm DTV}_{i,j} \left( \X \right) & = & \X_{i,j}  \bigg( \one_{\sum_{\Delta_i , \Delta_j \in \{-1, 1\} }  \X_{i + \Delta_i,j+\Delta_j}  = 0 } 
\\
&& + \sum_{\Delta_i , \Delta_j \in \{-1, 1\} } | \X_{i + \Delta_i,j + \Delta_j} -  \X_{i - \Delta_i,j - \Delta_j}|  \bigg).
\end{eqnarray*}
The directional total variation of the spatial signal $\X$ is 
\begin{equation*}
{\rm DTV} \left( \X \right) \ = \ \sum_{i,j=1}^N {\rm DTV}_{i,j} \left( \X \right).
\end{equation*}
\end{defn}
The first term in the parenthesis captures the isolated nodes and the second term in the parenthesis captures the local discontinuity. The directional total variation measures the curvatures represented by a 2D spatial signal. Here we cannot use the standard total variation~\cite{RudinOF:92} because we consider 1D curves in 2D lattice. The standard total variation
is defined to measure the smoothness level of natural images; however, here we consider 1D curves in the 2D lattice, which can be considered as an extremely sparse image with signals only on the curve. In this case, the standard total variation can hardly measure the complexity of the curvature.

A low directional total variation indicates that the underlying 1D curve is smooth in the 2D space. For example, Figure~\ref{fig:DTV} shows that red dots form a 1D curve in the 2D lattice and the directional total variation can properly measure the curvatures.

\begin{figure}[thb]
  \begin{center}
    \begin{tabular}{cc}
 \includegraphics[width=0.29\columnwidth]{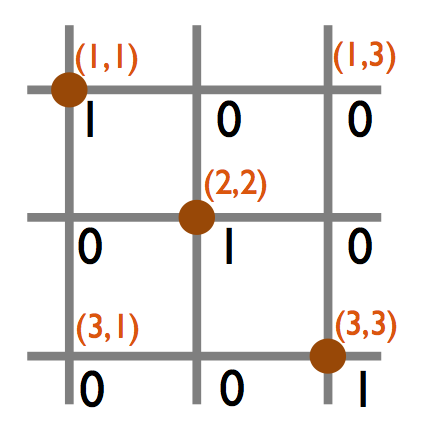}   & \includegraphics[width=0.29\columnwidth]{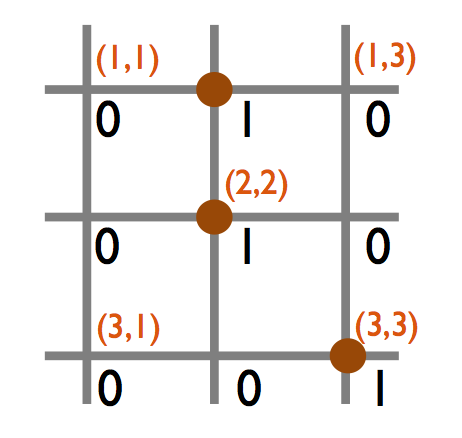}
    \\
    {\small (a) ${\rm DTV}_{2,2} \left( \X \right) = 0$.} &  {\small (b) ${\rm DTV}_{2,2} \left( \X \right) = 4$.}
  \end{tabular}
\end{center}
\caption{\label{fig:DTV} Directional total variation measures the local continuity of a 1D curve embedded in the 2D space. }
\end{figure}

\begin{defn}
Let $\x \in \R^{n}$ be a graph signal associated with a graph adjacency matrix $\Adj \in  \R^{n \times n}$, where $\x_{i} \in \R$ denotes the signal value on the $i$th node. 
The graph total variation of the graph signal $\x$ is 
\begin{equation*}
{\rm TV}_{\Adj} \left( \x \right) =   \left\| \x -  \frac{1}{|\lambda_{\rm max}|} \Adj \x \right\|_2^2 = \sum_i \left\| \x_i -  \frac{1}{|\lambda_{\rm max}|} \sum_{k} \Adj_{ik} \x_k \right\|_2^2.
\end{equation*}
\end{defn}
The graph total variation  measures the graph smoothness. A low graph total variation of a graph signal indicates that the underlying 1D curve is smooth in the graph domain. To represent the same occupancy of nodes in a 2D lattice, we can use either a spatial signal or a pair of graph signals.  A spatial signal uses a binary value to indicate the occupancy of a node; that is, when the value is one, the corresponding node is occupied. A pair of graph signals use the 2D coordinates to record the occupied nodes.  When a spatial signal represents the same information with a pair of graph signals, we consider they are equivalent. 

\begin{defn}
Let $\X \in \{0, 1\}^{N \times N}$ be a 2D lattice signal. Let $\x^{(1)}, \x^{(2)} \in \R^M$ be a pair of graph signals. $\X$ and $\x^{(1)}, \x^{(2)}$ are~\emph{equivalent} when
\begin{itemize}
\item $\X_{ \lceil \x^{(1)}_{\ell} \rceil, \lceil \x^{(2)}_{\ell} \rceil} = 1$, for $\ell=1, \cdots, M$;
\item $\sum_{i,j=1}^N \X_{i,j} = M$.
\end{itemize}
\end{defn}
The first condition indicates that each element in a graph signal reflects the coordinate in the 2D lattice. $\lceil \cdot \rceil$ denotes the ceiling function that rounds a real value to an integer. Note that the underlying graph associated with a graph signal can have arbitrary connections and the optimal one is supposed to provide graph smoothness for the 2D coordinates. For example, the spatial signal 
\begin{equation*}
  \widetilde{\X}
  \ = \ 
  \begin{bmatrix}
    0 &  1 & 1 & 1  \\
    0  & 0 & 1 & 0 \\
    0 &  1 & 0 & 0  \\
    1 & 1 & 1 &  0
  \end{bmatrix} \in \{0, 1\}^{4 \times 4}
\end{equation*}
is equivalent to a pair of graph signals
\begin{equation*}
 \begin{bmatrix}    
    \widetilde{\x}^{(1)} & \widetilde{\x}^{(2)}
  \end{bmatrix}
  \ = \ 
  \begin{bmatrix}    
     1 & 1 & 1 & 2 & 3 & 4 & 4 & 4  \\
     2 & 3 & 4 & 3 & 2 & 1 & 2 & 3
  \end{bmatrix}^T \in \R^{8 \times 2},
\end{equation*}
which can be associated with an arbitrary graph $\widetilde{\Adj}  \in \R^{8 \times 8}$. No matter what graph topology is, both the spatial signal $\widetilde{\X}$ and a pair of graph signals $[\widetilde{\x}^{(1)}, \widetilde{\x}^{(2)}]$ represents the same `Z' shape in the 2D space; see illustration in Figure~\ref{fig:smooth_model}. 

\begin{figure}[thb]
  \begin{center}
    \begin{tabular}{cc}
    \includegraphics[width=0.42\columnwidth]{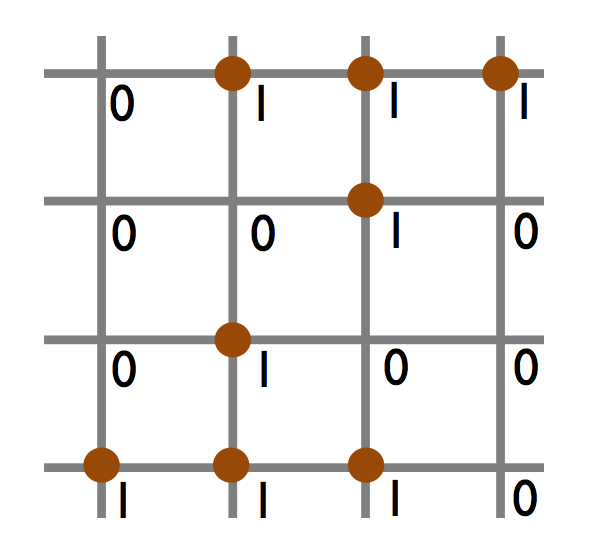}     & \includegraphics[width=0.42\columnwidth]{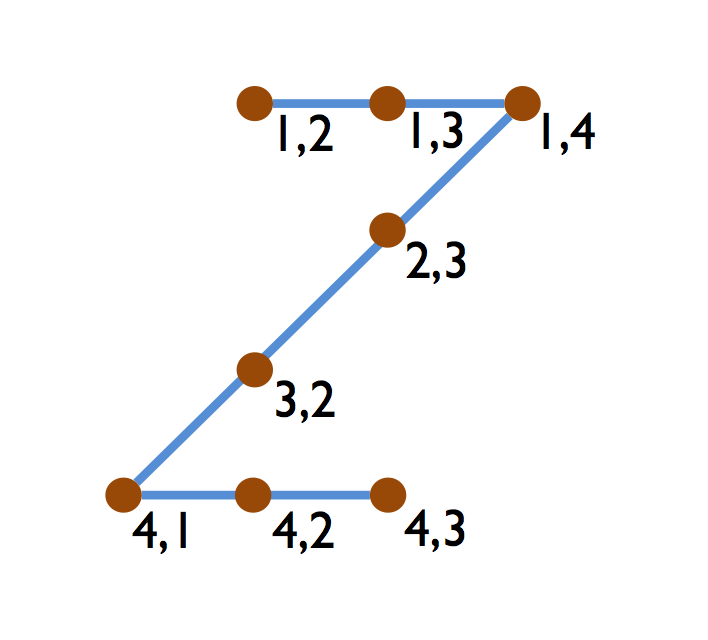}   
    \\
    {\small (a) 2D spatial signal.} &  {\small (b) Graph signals}.
  \end{tabular}
\end{center}
\caption{\label{fig:smooth_model} A nonsmooth 2D spatial signal is equivalent to smooth graph signals, representing a `Z' shape in the 2D space. }
\end{figure}

Among many possible graph topologies,  we can find at least one to promote graph smoothness for $[\widetilde{\x}^{(1)}, \widetilde{\x}^{(2)}]$. Given
\begin{equation*}
  \widetilde{\Adj}
  \ = \ 
  \begin{bmatrix}
    1 &  0 & 0 & 0 & 0 & 0 & 0 & 0 \\
    1/2 &  0 & 1/2 & 0 & 0 & 0 & 0 & 0 \\
    0 &  0 & 1 & 0 & 0 & 0 & 0 & 0  \\
    0 &  0 & 1/2  & 0 & 1/2  & 0 & 0 & 0 \\
    0 &  0 & 0 & 1/2  & 0 & 1/2  & 0 & 0  \\
    0 &  0 & 0 & 0 & 0 & 1 & 0 & 0 \\
    0 &  0 & 0 & 0 & 0 & 1/2 & 0  & 1/2   \\
    0 &  0 & 0 & 0 & 0 & 0 & 0 & 1 \\
  \end{bmatrix} \in \R^{8 \times 8},
\end{equation*}
we can show ${\rm TV}_{\widetilde{\Adj}} (\widetilde{\x}^{(1)}) = {\rm TV}_{\widetilde{\Adj}} (\widetilde{\x}^{(2)}) = 0$. At the same time,  due to the local discontinuity, ${\rm DTV} (\widetilde{\X}) > 0$. 

In the following theorem, we show that there always exists a nontrivial graph to promote graph smoothness for arbitrary graph signals.
\begin{myThm}
\label{thm:graph_smooth}
Let $\x^{(1)}, \x^{(2)} \in \R^M$ ($M>2$) be two vectors. Then, there exists a graph adjacency matrix $\Adj \in \R^{M \times M}$ to satisfy that
\begin{itemize}
\item $\Adj \neq \Id$;
\item ${\rm TV}_{\Adj} (\x^{(1)}) = {\rm TV}_{\Adj} (\x^{(2)}) = 0$.
\end{itemize}
\end{myThm}
\begin{proof}
We want to show the existence of $\Adj$ given two conditions. Based on the second condition, we use the lifting trick to obtain
\begin{eqnarray*}
{\rm vec} \left( (\Id - \Adj) 
\begin{bmatrix} \x^{(1)} & \x^{(2)} \end{bmatrix} \right) 
= \left( \begin{bmatrix} {\x^{(1)}}^T \\  {\x^{(2)}}^T \end{bmatrix} \otimes \Id \right) {\rm vec}(\Id - \Adj) = 0.
\end{eqnarray*} 
Since $\begin{bmatrix} {\x^{(1)}}^T \\  {\x^{(2)}}^T \end{bmatrix} \otimes \Id \in \R^{2M \times M^2}$, rank$\left( \begin{bmatrix} {\x^{(1)}}^T \\  {\x^{(2)}}^T \end{bmatrix} \otimes \Id \right) < M^2$. Thus, there exists  $\Adj \neq \Id$ to ensure $(\Id - \Adj)  \x^{(1)}  =  (\Id - \Adj)  \x^{(2)} = 0$.
\end{proof}
We further show that a nonsmooth spatial can be represented as a smooth graph signal.
\begin{myCorollary}
\label{thm:bridging_smooth}
There exists a class of 2D lattice signals $\X \in \{0, 1\}^{N \times N}$  and corresponding pairs of graph signals $\x^{(1)}, \x^{(2)} \in \R^M$ associated with specifically designed graph adjacency matrices $\Adj \in \R^{M \times M}$
to satisfy that 
\begin{itemize}
\item $\X$ is equivalent to $\x^{(1)}, \x^{(2)}$;
\item $\frac{{\rm DTV} (\X) }{1_N^T \X 1_N} = O(1)$;
\item $\Adj \neq \Id$;
\item ${\rm TV}_{\Adj} (\x^{(1)}) = {\rm TV}_{\Adj} (\x^{(2)}) = 0$.
\end{itemize}
\end{myCorollary}
The first condition indicates the equivalence; the second condition indicates the number of discontinuities increases as the length of the curve increases, reflecting spatial nonsmoothness;  the third condition indicates the graph topology is nontrivial; and the fourth condition indicates the existence of  a graph topology to ensure the graph smoothness. 
\begin{proof}
A class of  periodic `Z' or `L' shapes, which generalize the previous example of $\widetilde{\X}$, is a simple construction proof to this Theorem.  Due to the periodicity, the number of local discontinuities increases as the length of the curve increases. On the other hand, as shown in Theorem~\ref{thm:graph_smooth}, there exists a graph topology to ensure graph total variations are zeros.
\end{proof}

Corollary~\ref{thm:bridging_smooth} shows that with an appropriate graph topology, we can represent a nonsmooth spatial signal as a pair of smooth graph signals. Hypothetically, the graph-topology-inference module is designed to provide such a  graph topology. In practice, many 3D points sampled from complicated shapes and curvatures show nonsmoothness in the 3D space. On the other hand, when an appropriate  graph topology is obtained, graph smoothness is a strong prior to capture the distribution of 3D points.  We thus can use the graph-topology-inference module to learn an appropriate graph and then  promote graph smoothness for 3D points. This also indicates that convolutional neural networks may not be appropriate to learn 3D point clouds~\cite{3dgan}. A 3D convolution is usually used for 3D volume data and does not explore an appropriate manifold prior for 3D point clouds, which are sampled from the surfaces of objects.

We now show graph filtering guides the networks to promote graph smoothness and  refine the positions of 3D points.
\begin{myThm}
\label{thm:adj_smooth}
Let $\Adj \in \R^{M \times M}$ be an adjacency matrix whose eigenvalues are $\lambda_i$, with $\max_i |\lambda_i| = 1.$ Let $h(\Adj)  =  \frac{1}{2} \left( \Id + \Adj \right) \in \R^{M \times M}$ be a graph filter.  For any graph signal $\x \in \R^M$, we have
\begin{equation*}
 \TV_{\Adj} \left( \x  \right)  \geq  \TV_{\Adj} \left( h(\Adj) \x \right).
\end{equation*}
\end{myThm}

\begin{proof}
Let $\Mm = ( \Id - \Adj)^T ( \Id - \Adj ) - h(\Adj) ( \Id - \Adj)^T ( \Id - \Adj ) h(\Adj)  \in \R^{M \times M}$. We want to show
\begin{eqnarray*}
&& \TV_{\Adj} \left( \x  \right) -  \TV_{\Adj} \left( h(\Adj) \x \right)
\\ 
&  = & \x^T \left(  ( \Id - \Adj)^T ( \Id - \Adj ) - h(\Adj) ( \Id - \Adj)^T ( \Id - \Adj ) h(\Adj) \right)  \x
\\
& = &  \x^T \Mm  \x
 \end{eqnarray*}
is nonnegative for any arbitrary $\x$. This is equivalent to show $\Mm$ is positive semidefinite.
We now want to show that all the eigenvalues of $\Mm$ are nonnegative. Since we can show
\begin{eqnarray*}
  && ( \Id - \Adj)^T ( \Id - \Adj )  \\
  & = & \Vm  \begin{bmatrix}
   (1 - \lambda_1)^2 &  0 & \cdots & 0  \\
    0  &  (1-  \lambda_2)^2  & \cdots & 0 \\
    \vdots &  \vdots & \ddots & \vdots  \\
    0 & 0 & \cdots &  (1-  \lambda_N )^2
  \end{bmatrix}
  \Vm^{-1},
 \end{eqnarray*}
and
\begin{eqnarray*}
  && h(\Adj) ( \Id - \Adj)^T ( \Id - \Adj ) h(\Adj)  \\
 & = &
\Vm 
  \begin{bmatrix}
   \frac{(1 - \lambda_1)^2 (1 + \lambda_1)^2}{4}  & \cdots & 0  \\
    \vdots  & \ddots & \vdots  \\
    0  & \cdots &  \frac{(1 - \lambda_N)^2 (1 + \lambda_N)^2}{4}
  \end{bmatrix}
  \Vm^{-1},
\end{eqnarray*}
the eigenvalues of $\Mm$ are $(1 - \lambda_i)^2 -  \frac{(1 - \lambda_i)^2 (1 + \lambda_i)^2}{4}$ for $i = 1, \cdots, N$. Since $-1 \leq \lambda_i \leq 1$,  $$(1 - \lambda_i)^2 -  \frac{(1 - \lambda_i)^2 (1 + \lambda_i)^2}{4} \geq 0.$$ Thus, $\Mm$ is positive semidefinite.
\end{proof}
This shows that the low-pass graph Haar filter promotes graph smoothness. In the context of 3D point clouds, the graph-filtering module refines the reconstructed 3D points by pushing them as smooth graph signals. For graph-Laplacian-based filters~\eqref{eq:lap_filter}, we can use the quadratic term $\x^T \mathcal{L} \x$ to measure the graph variations~\cite{ShumanNFOV:13} and further show that $$\x^T \mathcal{L} \x  \geq (h(\mathcal{L}) \x)^T \mathcal{L} (h(\mathcal{L}) \x),$$
which also pushes the reconstructed 3D points as  smooth graph signals.

\section{Experimental Results}
\label{sec:experiments}
In this section, we conduct experiments to validate the proposed unsupervised model mainly from three aspects: 3D point cloud reconstruction, visualization and transfer classification.

\subsection{Dataset}
We consider four datasets in the experiments, including ShapeNet~\cite{shapenet}, ShapeNetCore~\cite{yi2016scalable}, ModelNet40~\cite{WuSKYZTX:15} and ModelNet10~\cite{WuSKYZTX:15}. ShapeNet contains more than $50000$ unique 3D models from 55 manually verified categories of common objects and ShapeNetCore is a subset of ShapeNet. ModelNet40 (MN40) contains 9843/2468 3D CAD models from 40 categories in train/test datasets; ModelNet10 (MN10) contains 3991/909 3D CAD models from 10 categories in train/test datasets; the datasets are labeled by human workers in Amazon Mechanical Turk and manually selected by researchers.  In our experiments, we  sample $57447$ point clouds from ShapeNet, $12311$ point clouds from ModelNet40, $4900$ point clouds from ModelNet10, and $15011$ point clouds from ShapeNetCore. Each point cloud is composed of $2048$ 3D points. 

\subsection{Experimental Setup}
All the experiments are conducted on GTX TITAN Xp GPU with pytorch 0.4.0. We used ADAM with an initial learning rate of 0.0001. The training batch size is $32$. The features dimensions are shown in Figure~\ref{fig:architecture}. In our experiments, we choose $k = 96$ and $\sigma = 0.08$ in the graph-topology-learning module and $\mu = 0.5$ in the graph-filtering module. The experimental results show the overall performances are not sensitive to those hyperparameters. By default, the number of reconstructed 3D points are set to M = 2025. Note that a squared number is preferred for M because of the structure of the 2D lattice.

\begin{table*}[htb!]
	\begin{center}
		\begin{tabular}{ c|c c c c c c c c c c} 
			\hline
			Category & Airplane & Bag & Cap & Car & Chair & Earphone & Guitar & Knife &\# \textbf{winner}\\
			\hline
			\# shape & 2299 & 68 & 50 & 817 & 3362 & 63 & 709 & 357 & \\
			\hline
			3DCapsNet  (loss: $\times$ 1e-2) & 2.48 & 5.33 & 5.46 & 3.79 & 3.81 & 4.68 & 2.04 & 2.17 &  0 \\
			AtlasNet (loss: $\times$ 1e-2) & 2.39 & 4.33 & 4.17 & 3.88 & 3.58 & 5.13 & 1.96 & 1.75 &  0   \\
			 Folding module only   (loss: $\times$  1e-2) & 2.55 & 4.07 & 4.04 & 3.80 & 3.63 & 4.92 & 1.62 & 1.66 & 0\\
			With graph-adjacency-based filtering~\eqref{eq:adj_filtering}   (loss: $\times$  1e-2) & {\bf 2.34} & {\bf 3.91} & {\bf 3.91} & 3.83 & 3.46 & {\bf 4.28} & {\bf 1.58} & 1.70 & {\bf 5}\\
			With graph-Laplacian-based filtering~\eqref{eq:lap_filter}   (loss: $\times$  1e-2) & 2.38 & 4.06 & 4.53 & {\bf 3.72} & {\bf 3.39} & 4.77 & 1.73 & {\bf 1.49} & 3\\
			\hline
			\hline
			Category & Lamp & Laptop & Motorbike & Mug & Pistol & Rocket & Skateboard & Table & \# \textbf{winner} \\
			\hline
			\# shape & 1404 & 407 & 176 & 168 & 253 & 58 & 137 & 4683 & \\
			\hline
			 3DCapsNet  (loss: $\times$ 1e-2) & 3.51 & 2.96 & 3.91 & 4.86 & 2.67 & 2.21 & 3.09 & 3.44 & 0 \\
			AtlasNet (loss: $\times$ 1e-2) & 3.55 & 2.98 & 4.11 & 5.07 & 2.82 & 2.31 & 2.60 & 3.58 & 0 \\
			Folding module only   (loss: $\times$ 1e-2) & 3.47 & 2.90 & 4.03 & 4.94 & 2.75 & 2.28 & 2.48 & 3.48 & 0 \\
			With graph-adjacency-based filtering~\eqref{eq:adj_filtering}  (loss: $\times$ 1e-2) & 3.45 & 2.94 & 3.88 & {\bf 4.60} & 2.74 & 2.51 & 2.63 & 3.30 & 1 \\
			With graph-Laplacian-based filtering~\eqref{eq:lap_filter}  (loss: $\times$  1e-2) & {\bf 3.43} & {\bf 2.84} & {\bf 3.69} & 5.29 & {\bf 2.64} & {\bf 2.18} & {\bf 2.39} & {\bf 3.05} & {\bf 7}\\
			\hline
		\end{tabular}
	\end{center}
	\caption{\label{tab:reconstruction_loss} \textbf{Graph filtering lowers reconstruction losses.}  The reconstruction loss is measured by the augmented Chamfer distance. The overall networks achieves better reconstruction performance than the folding module only. All three models are trained by the same setting in an end-to-end fashion and use the same number of training parameters on the dataset of ShapeNetCore. }
\end{table*}

\subsection{Reconstruction of 3D Point Clouds}
\label{sec:reconstruction}
As an unsupervised method, we first validate  the reconstruction performance of the proposed networks. Specifically, we compare the reconstructions without graph filtering and with graph filtering. For the reconstructions without graph filtering, we remove the graph-topology-inference module and the graph-filtering module and only train the folding module end-to-end\footnote{In this paper, when we show the performance of the folding module, there are two possible settings. In the first setting, we train three modules together and the supervision is put on the final output of the graph-filtering module. The output of the folding module is the intermediate output from the entire networks; we call this~\emph{before graph filtering}. In the second setting, we only train the folding module and the supervision is directly put on the output of the folding module; we call this~\emph{folding module only}.}. We then use its output $\S'$ in~\eqref{eq:folding} as the final reconstruction; for the reconstructions with graph filtering, we train all three  modules and outputs the refined reconstruction $\widehat{S}$ in~\eqref{eq:adj_filtering}. To make a fair comparison, we adjust the configurations of both models to ensure they have similar numbers of trainable parameters. For the graph-filtering module, we consider both graph-adjacency and graph-Laplacian-matrix based filters.  We train all the networks on ShapeNetCore. Table~\ref{tab:reconstruction_loss} shows the comparisons of the reconstruction losses,  which are measured by the augmented Chamfer distance. We see that the entire networks with graph filtering achieve better reconstruction performance than the folding module only. This indicates that it is beneficial to using a graph topology to refine the reconstruction. We also see that  the graph-Laplacian-based filter~\eqref{eq:lap_filter} achieves slightly better reconstruction performances than the graph-adjacency-based filter~\eqref{eq:adj_filtering}. We also compare the proposed method with AtlasNet~\cite{AtlasNet} and 3DCapsNet~\cite{3dcapsule}. AtlasNet generates a 3D surface by leveraging multiple configurations from the same 2D flat; and 3DCapsNet generates multiple 3D local patches by leveraging distinct 2D lattices with the corresponding latent representations. Here we adopt one single 2D lattice and use a trainable graph topology to increase the expressive power of this 2D lattice. Table~\ref{tab:reconstruction_loss} shows that the proposed method outperforms AtlasNet and 3DCapsNet.

\begin{table*}[htb!]
  \begin{center}
    \begin{tabular}{c | c  c | c  c | c  c }
      \hline   
      \hline
      Epoch & 
      \multicolumn{2}{c|}{Swept wings no engines} & 
      \multicolumn{2}{c|}{Swept wings two engines} & 
      \multicolumn{2}{c}{Straight wings two engines} \\
      \hline
             & 
       \emph{Before graph filtering} & 
       \emph{After graph filtering} & 
       \emph{Before graph filtering} & 
       \emph{After graph filtering} & 
       \emph{Before graph filtering} & 
       \emph{After graph filtering} 
       \\
        \hline
      
      Input &
      \includegraphics[trim={4cm 4cm 2cm 4cm}, clip=true , scale=0.025] {figures/graph_variance/grid.jpg} &
      \includegraphics[trim={4cm 4cm 2cm 4cm}, clip=true , scale=0.04] {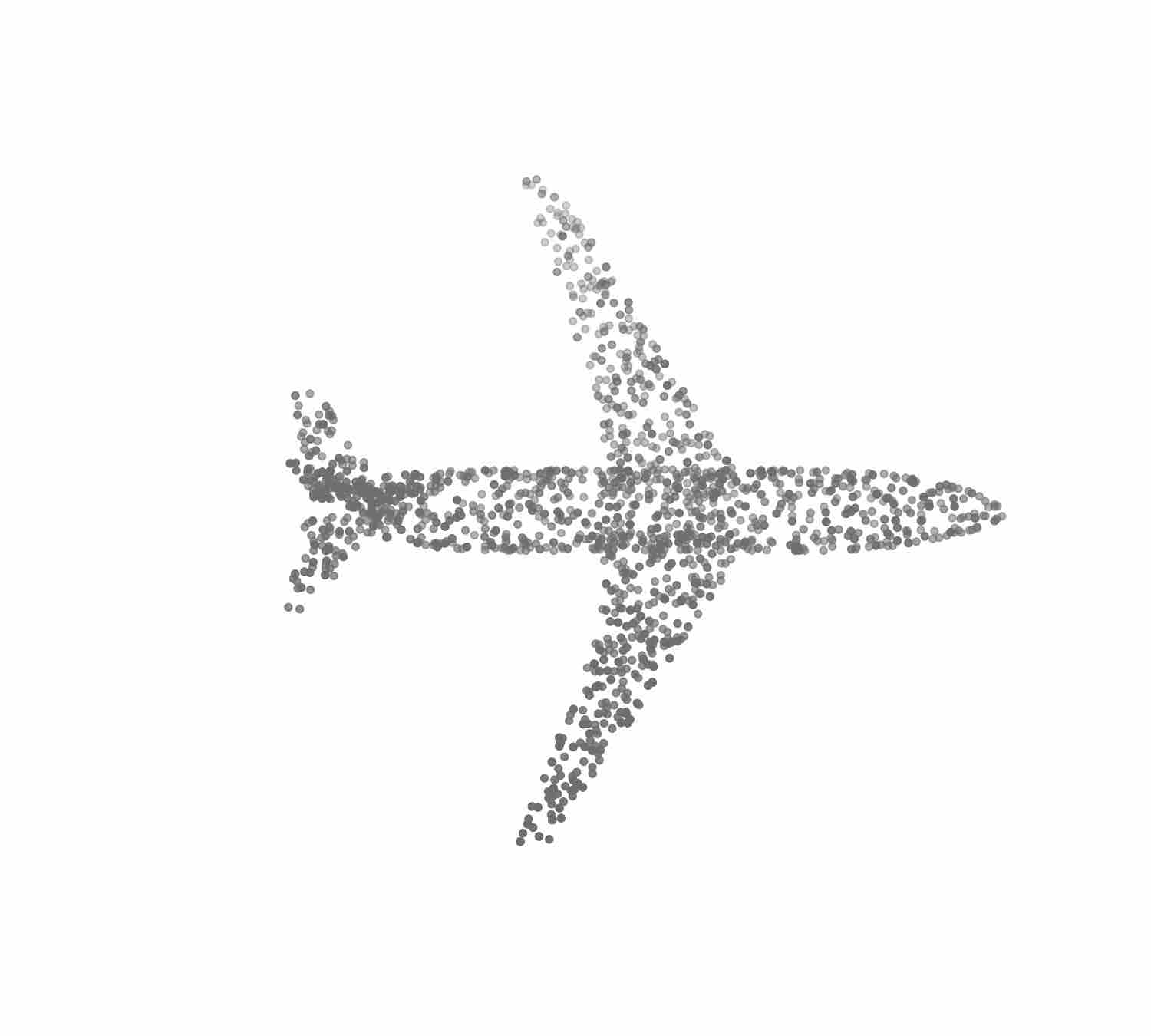} & 
      \includegraphics[trim={4cm 4cm 2cm 4cm}, clip=true , scale=0.025] {figures/graph_variance/grid.jpg} &
      \includegraphics[trim={4cm 4cm 2cm 4cm}, clip=true , scale=0.04]  {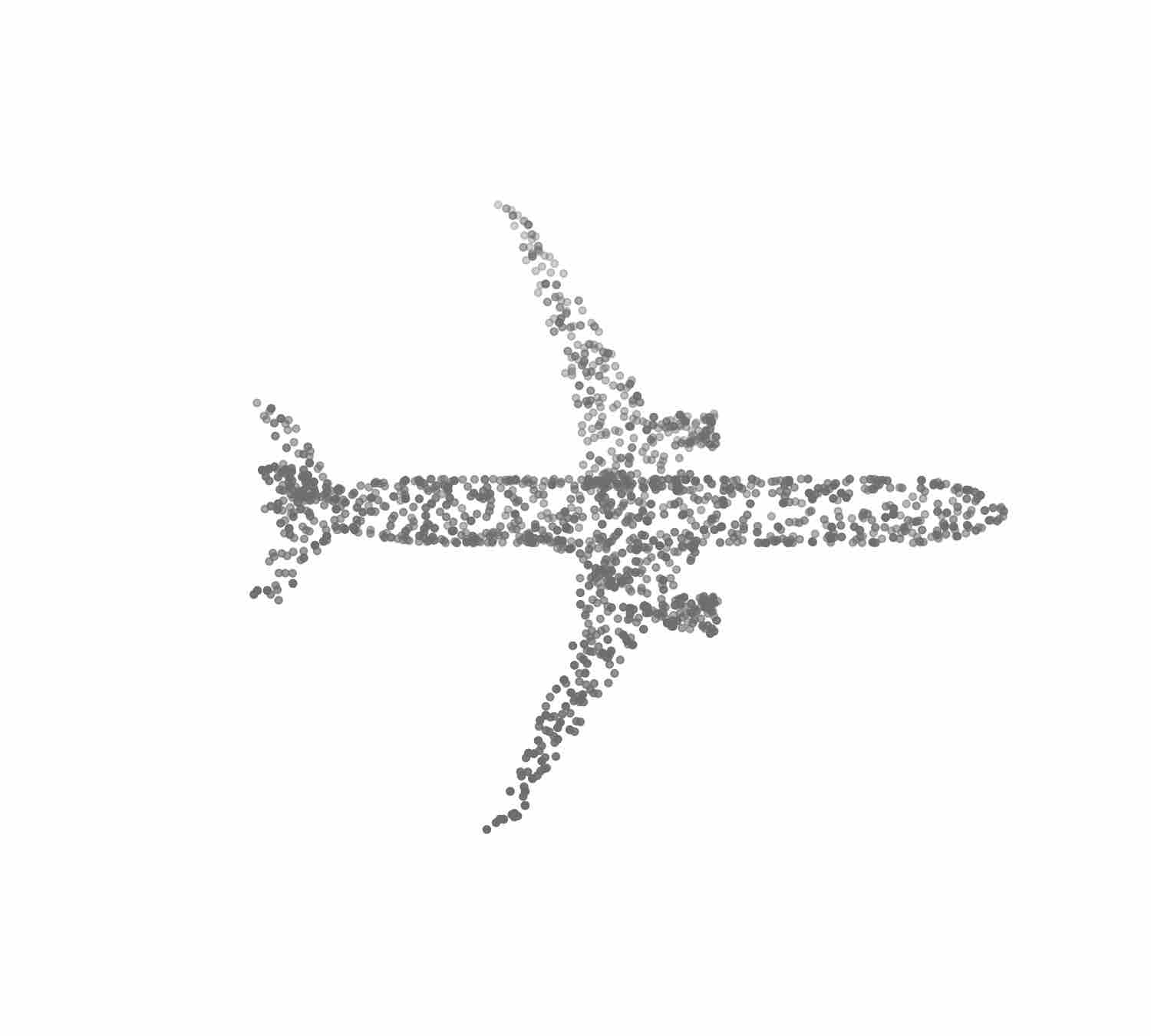} & 
      \includegraphics[trim={4cm 4cm 2cm 4cm}, clip=true , scale=0.025] {figures/graph_variance/grid.jpg} &
      \includegraphics[trim={4cm 4cm 2cm 4cm}, clip=true , scale=0.04] {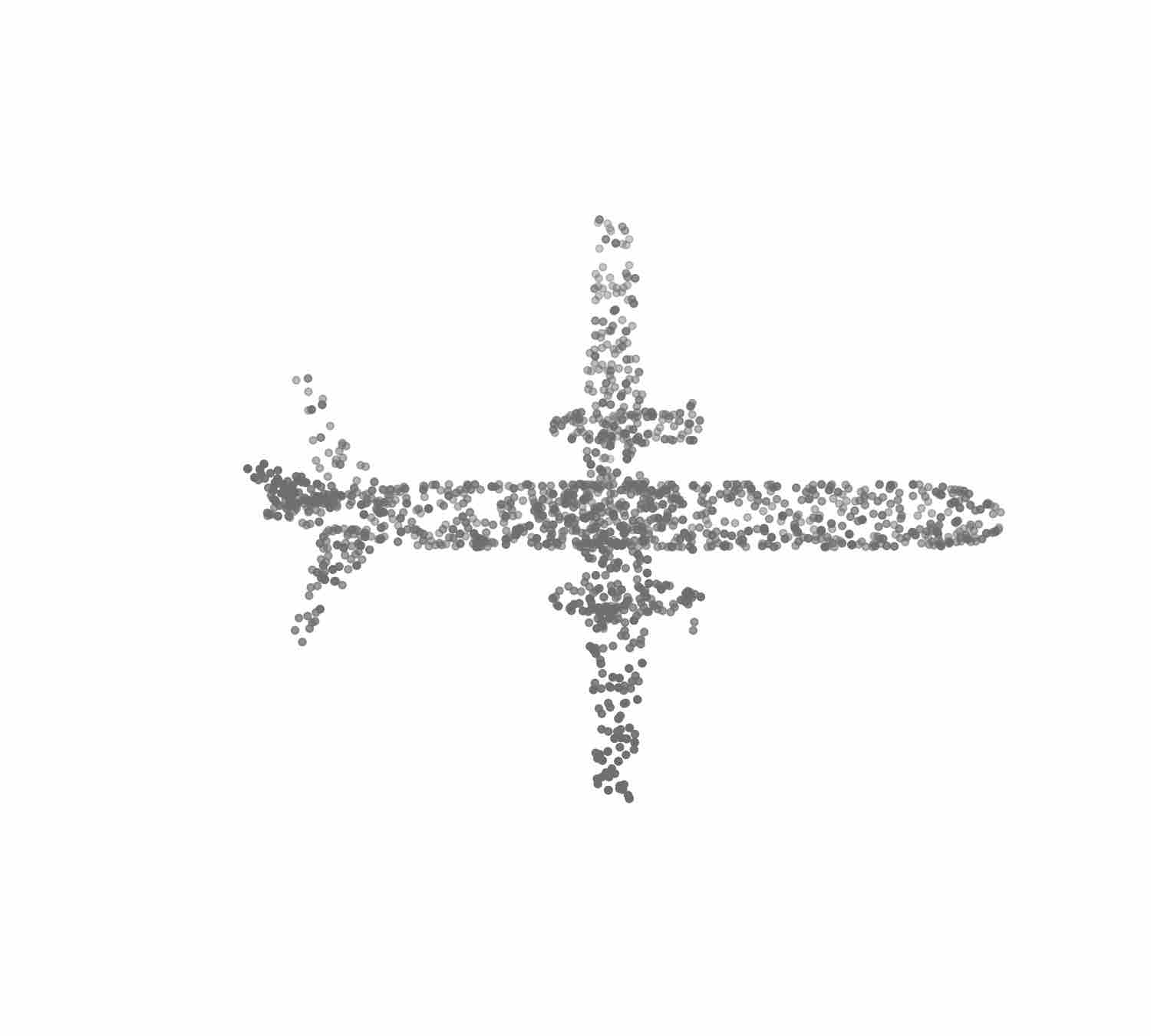}
 \\
      \hline
       &&&&&&\\[-1em] 
      0 &
      \includegraphics[trim={6cm 6cm 4cm 6cm}, clip=true , scale=0.045] {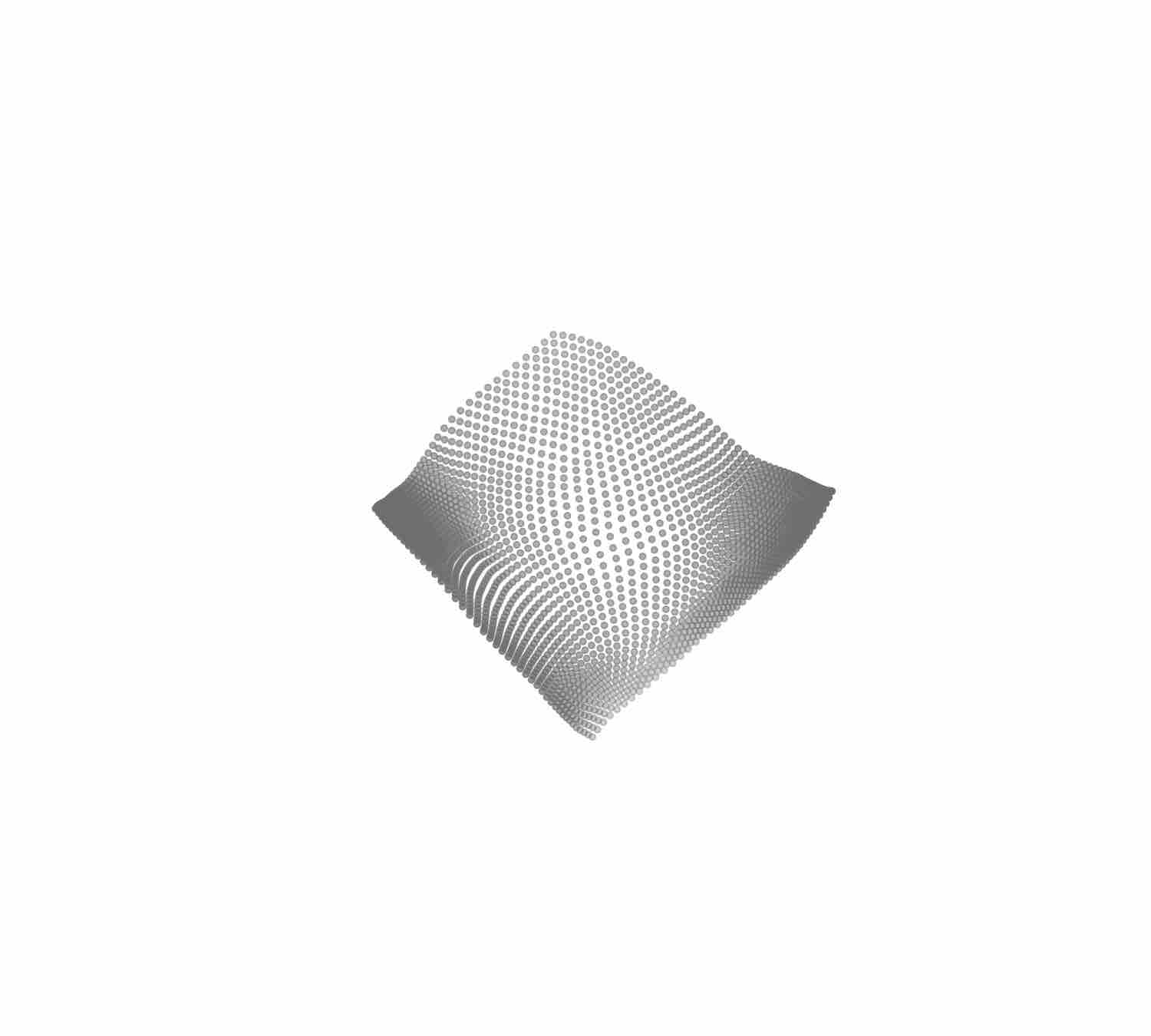} & 
      \includegraphics[trim={5cm 4cm 3cm 4cm}, clip=true , scale=0.045] {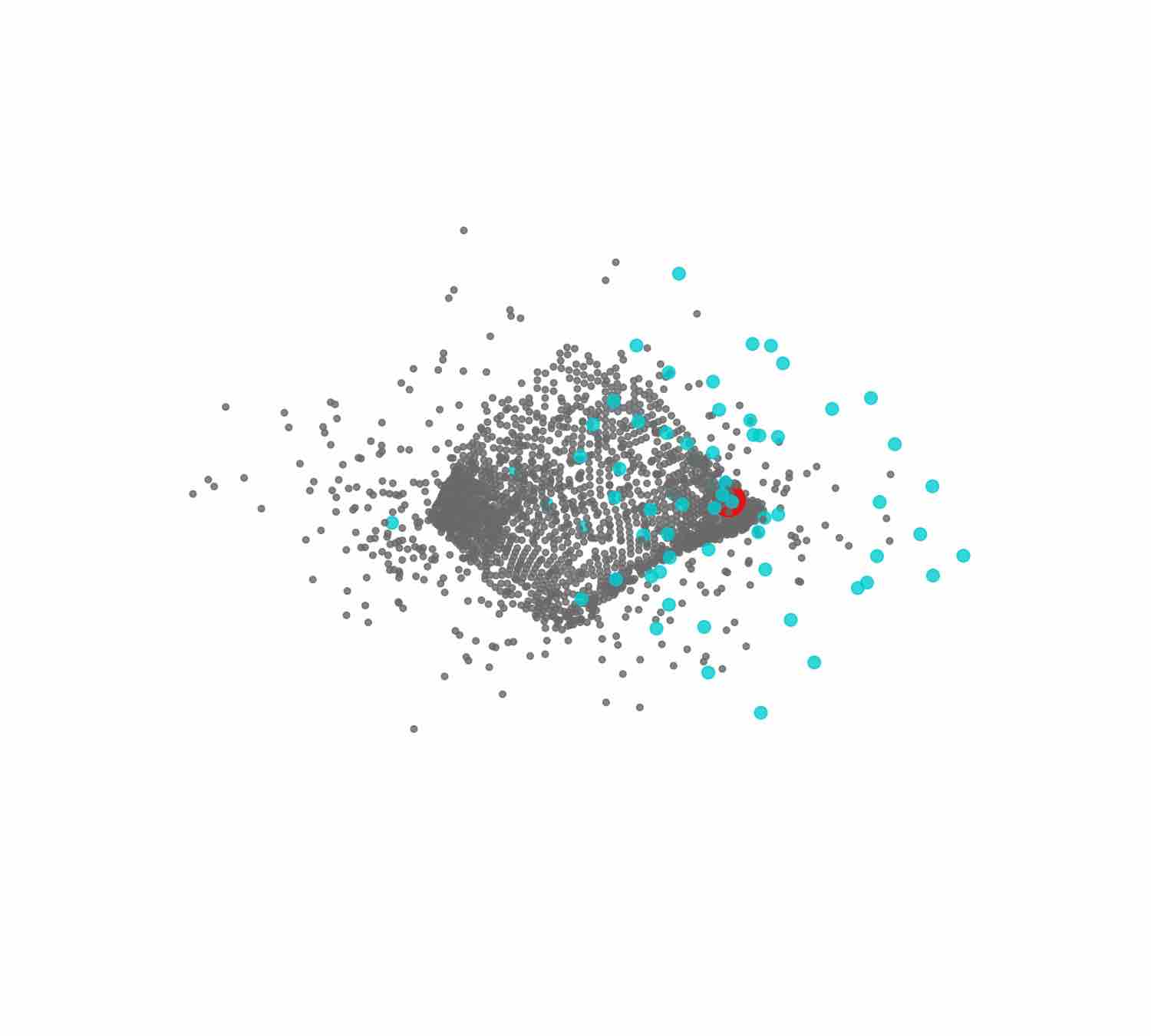} & 
      \includegraphics[trim={6cm 6cm 4cm 6cm}, clip=true , scale=0.045] {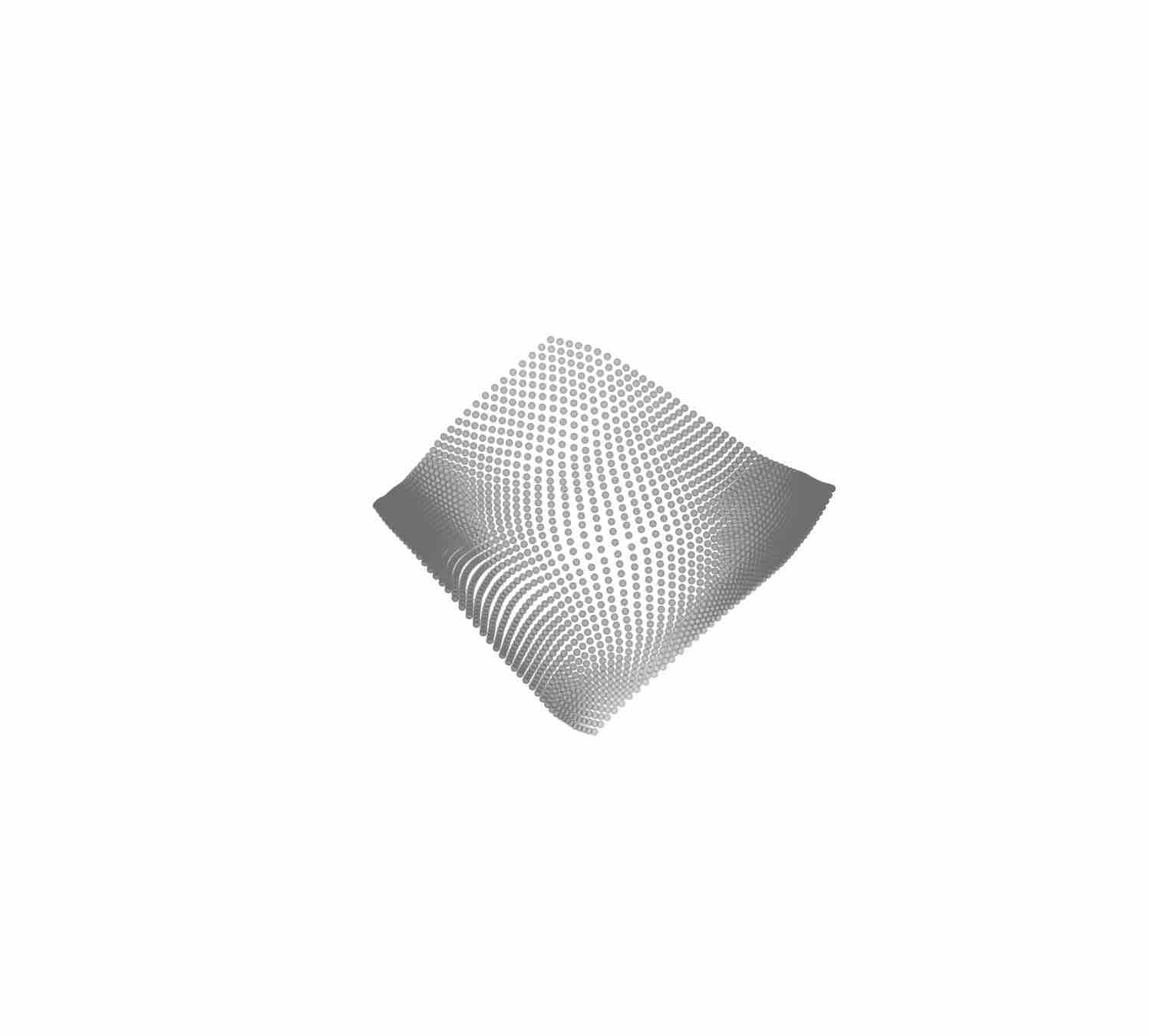} &
      \includegraphics[trim={5cm 4cm 3cm 4cm}, clip=true , scale=0.045] {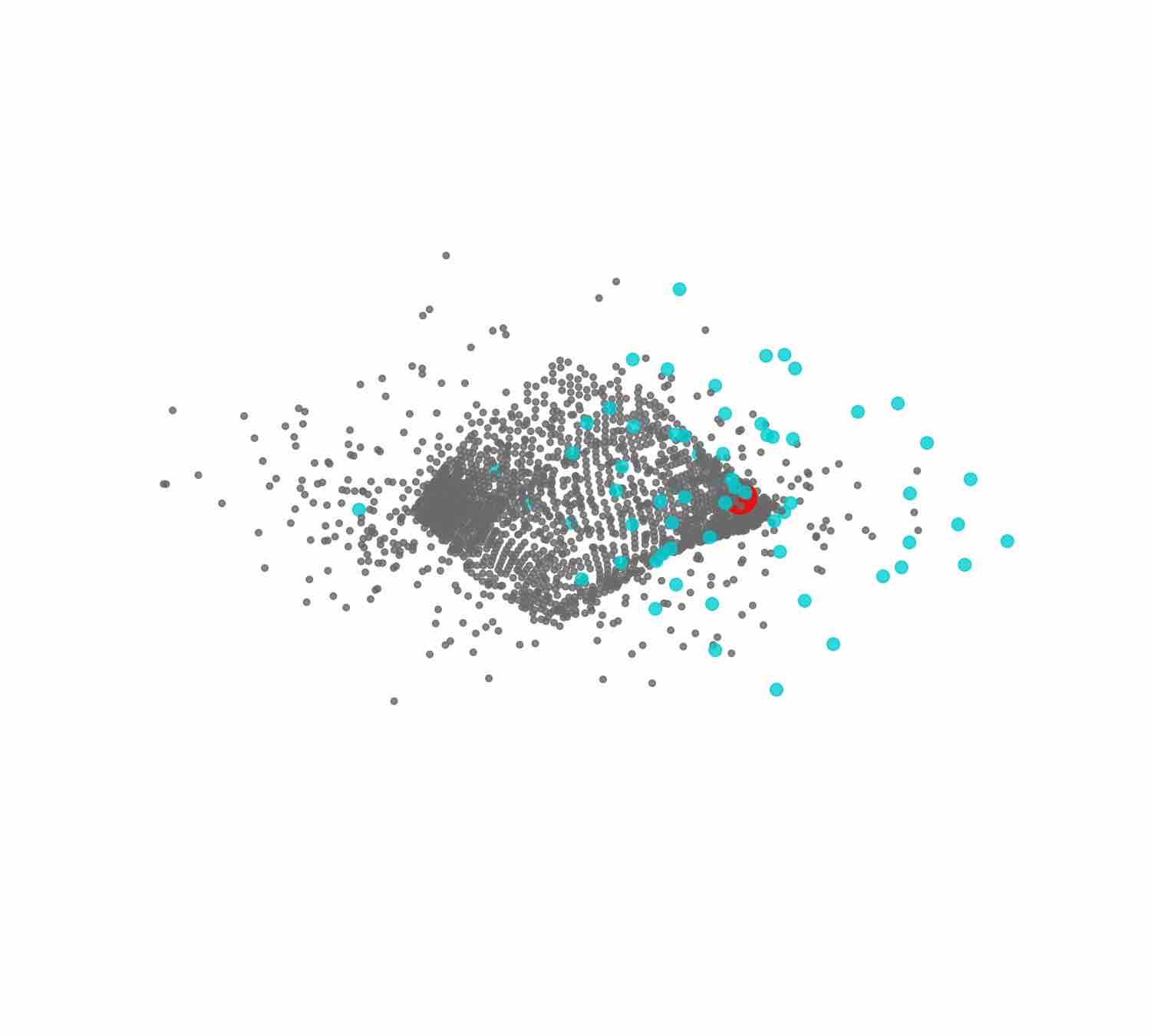} & 
      \includegraphics[trim={6cm 6cm 4cm 6cm}, clip=true , scale=0.045] {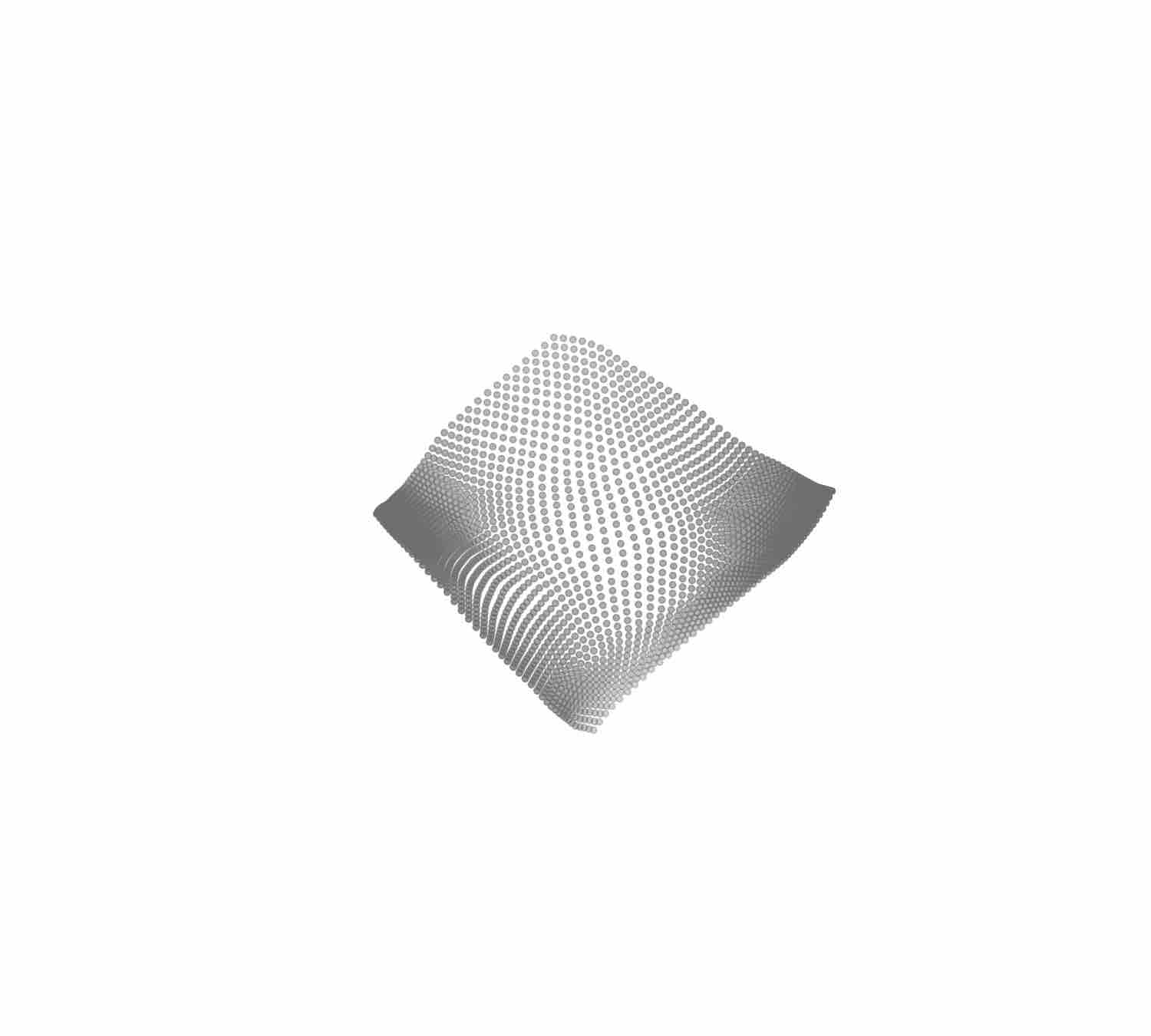} & 
      \includegraphics[trim={5cm 4cm 3cm 4cm}, clip=true , scale=0.045] {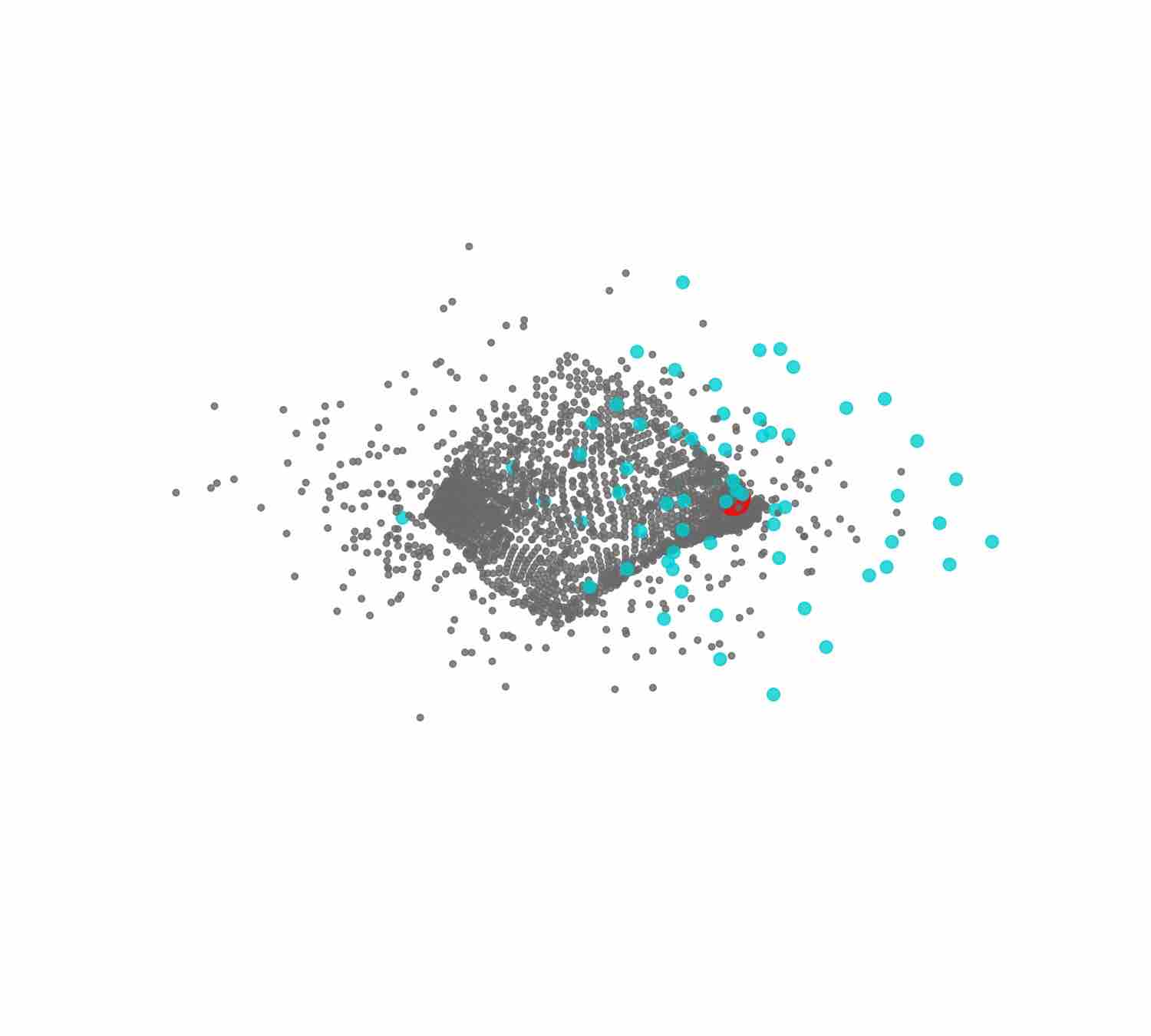}    
      \\      
      \hline
       &&&&&&\\[-1em]
      5 &
      \includegraphics[trim={6cm 6cm 4cm 6cm}, clip=true , scale=0.045] {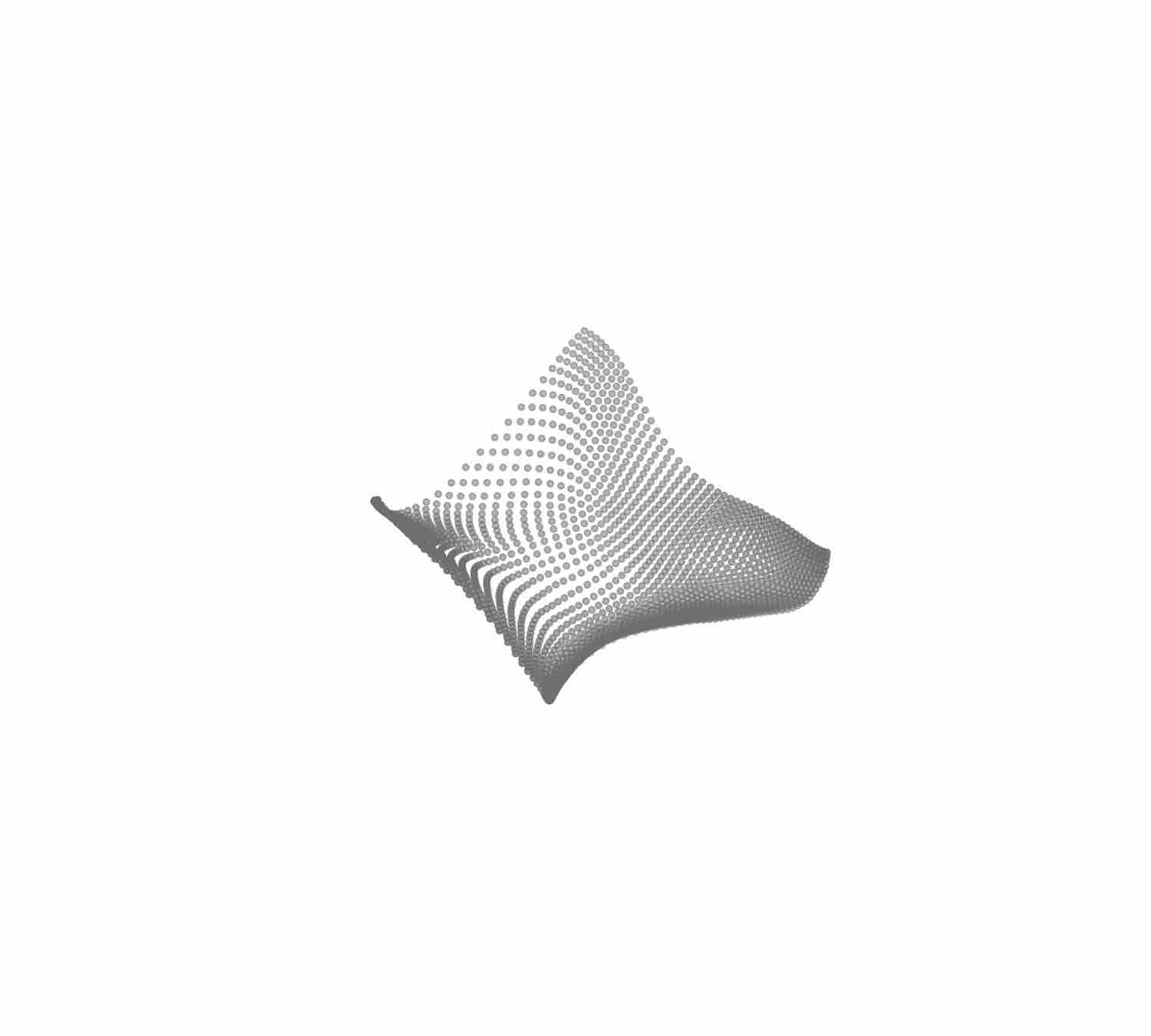} & 
      \includegraphics[trim={5cm 4cm 3cm 4cm}, clip=true , scale=0.045] {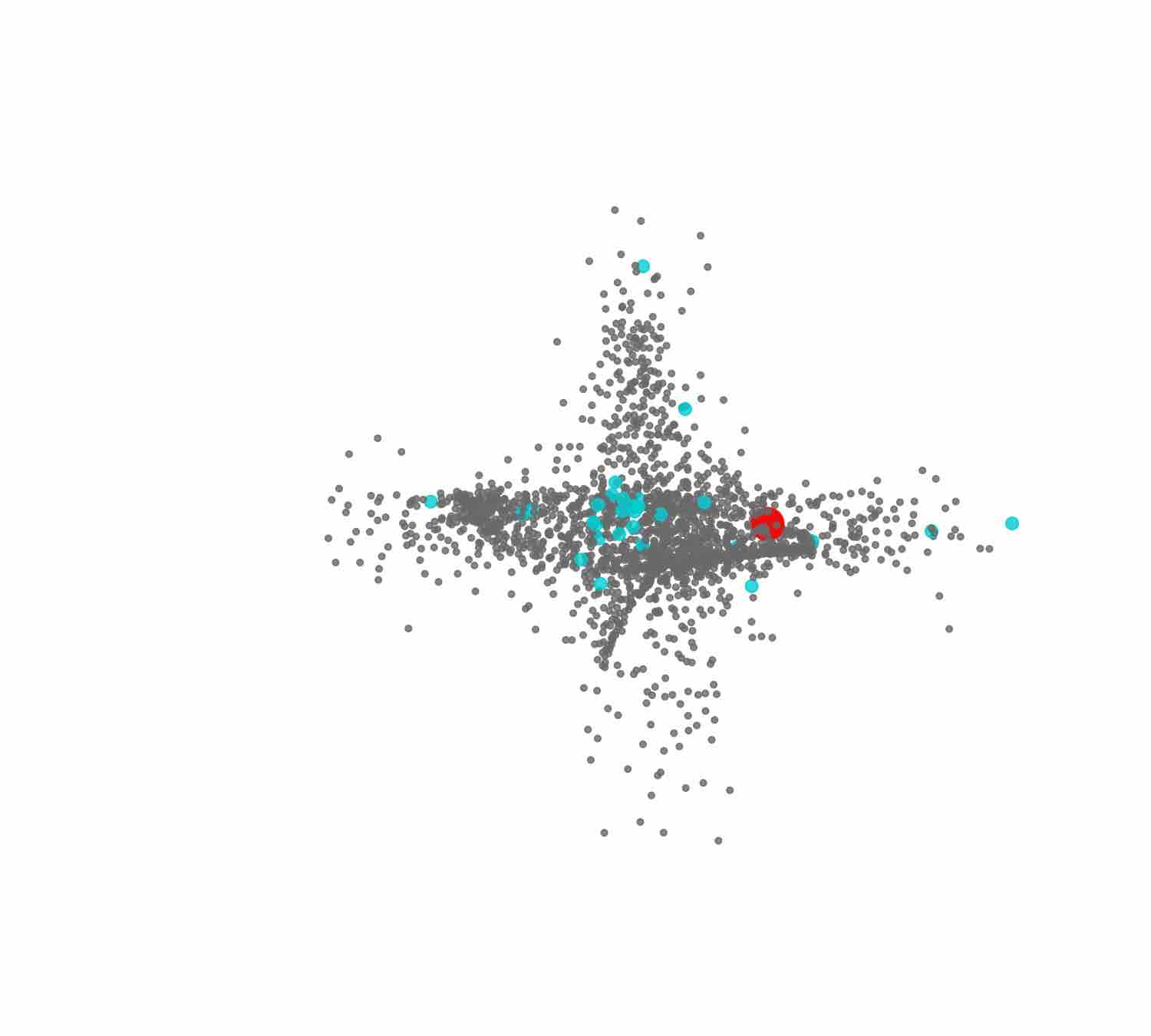} & 
      \includegraphics[trim={6cm 6cm 4cm 6cm}, clip=true , scale=0.06] {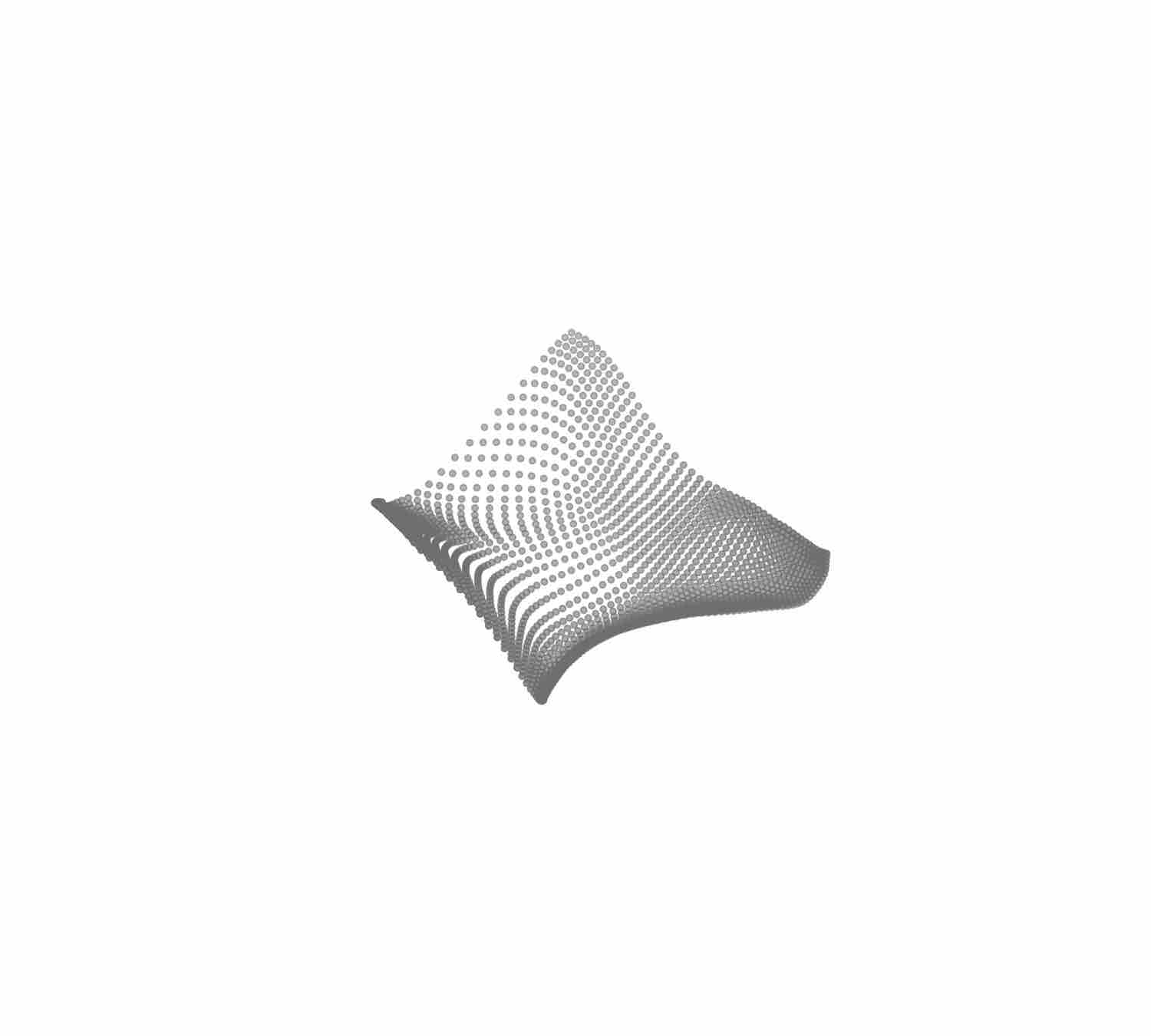} &
      \includegraphics[trim={5cm 4cm 3cm 4cm}, clip=true , scale=0.045] {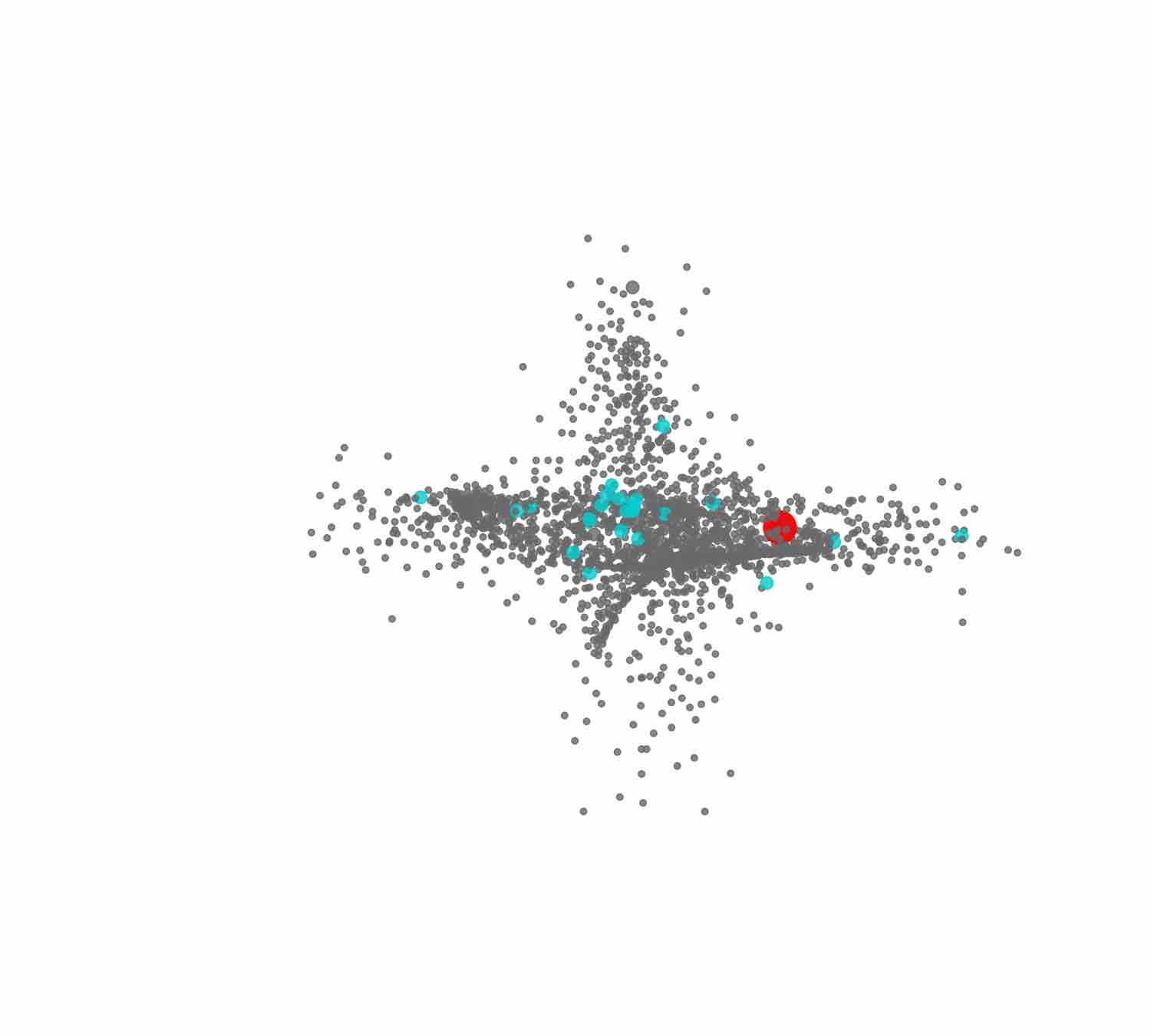} &
      \includegraphics[trim={6cm 6cm 4cm 6cm}, clip=true , scale=0.045] {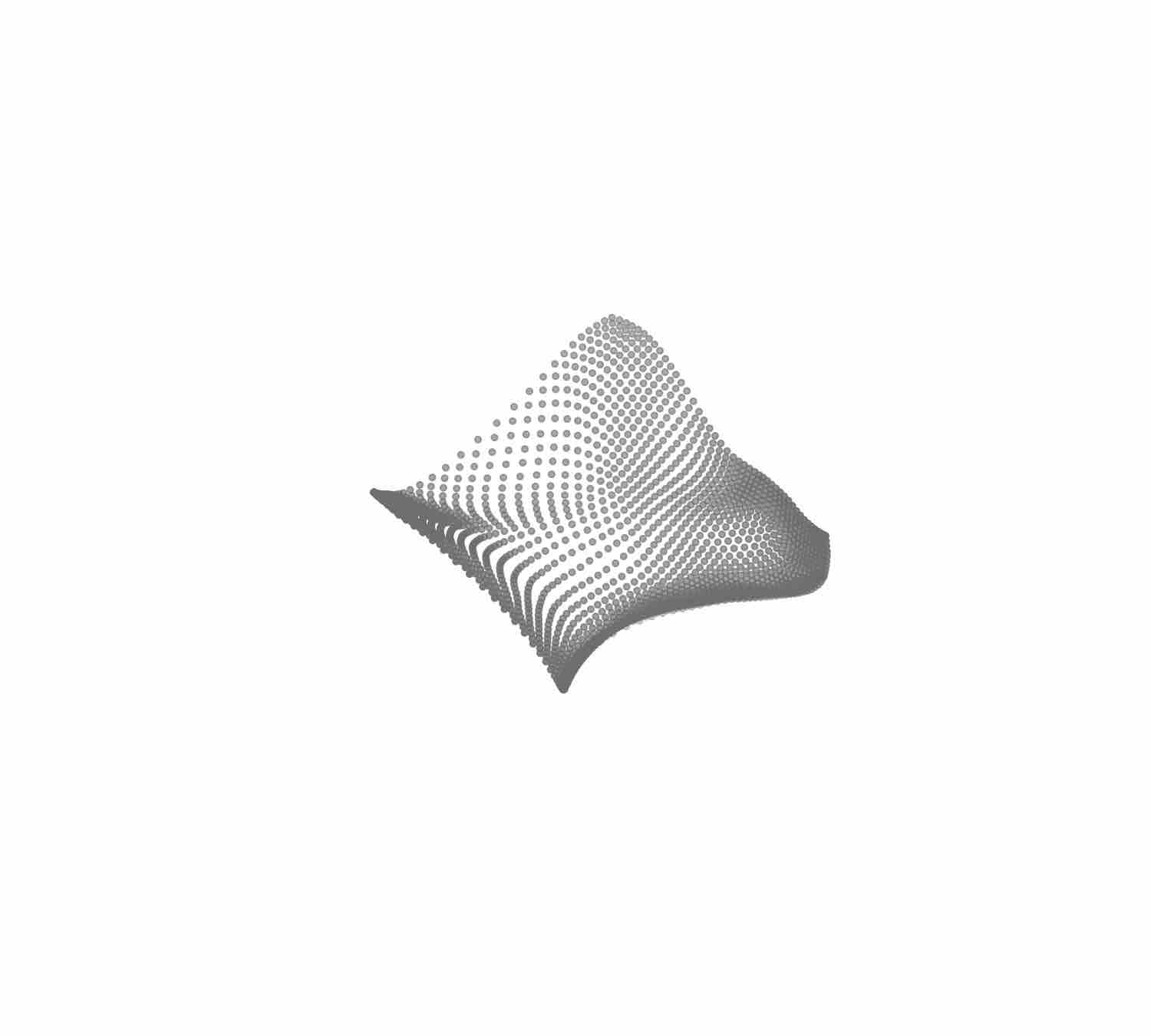} & 
      \includegraphics[trim={5cm 4cm 3cm 4cm}, clip=true , scale=0.045] {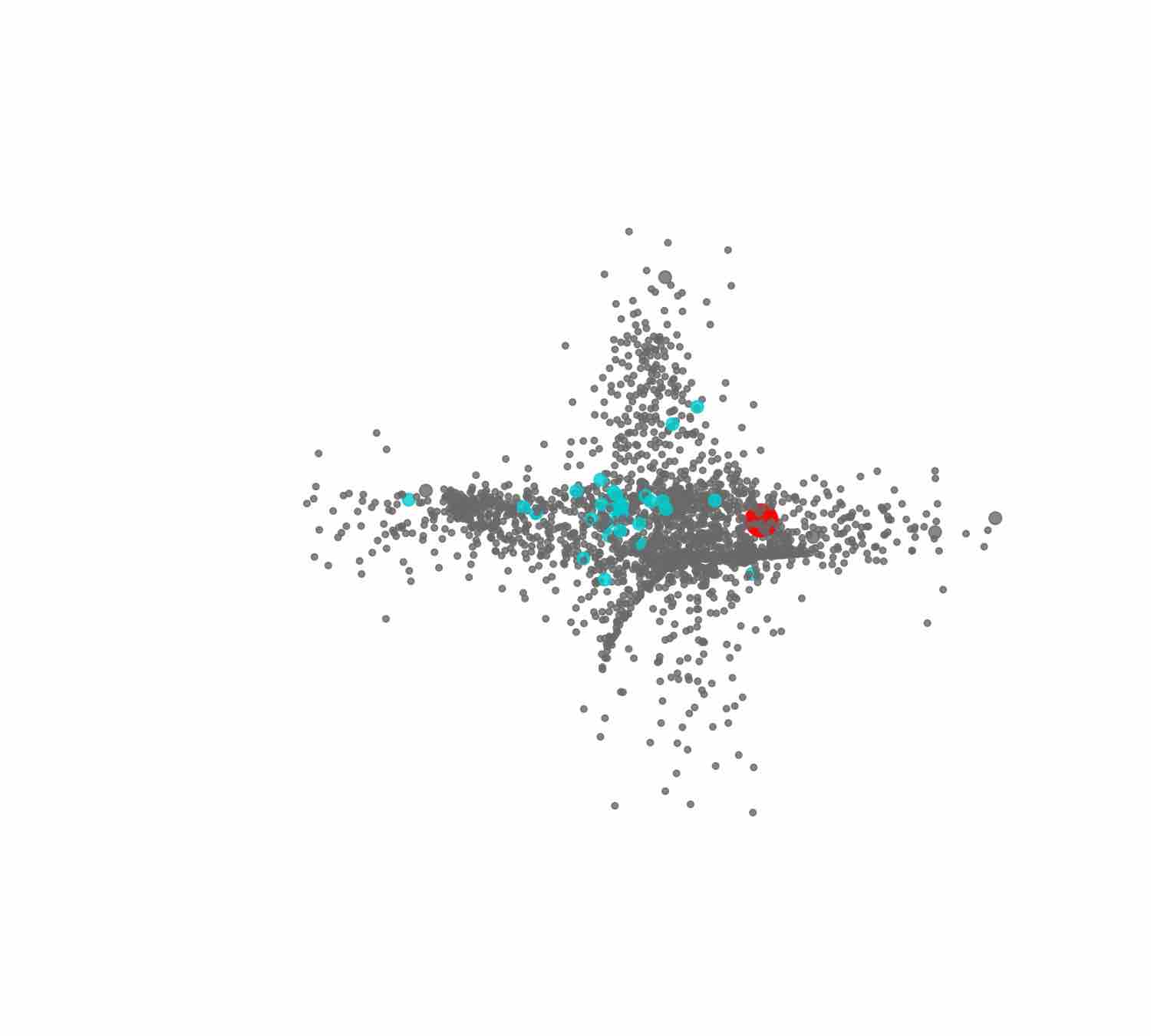}
      \\      
      \hline
       &&&&&&\\[-1em]
            
      300 &
      \includegraphics[trim={6cm 6cm 4cm 6cm}, clip=true , scale=0.06] {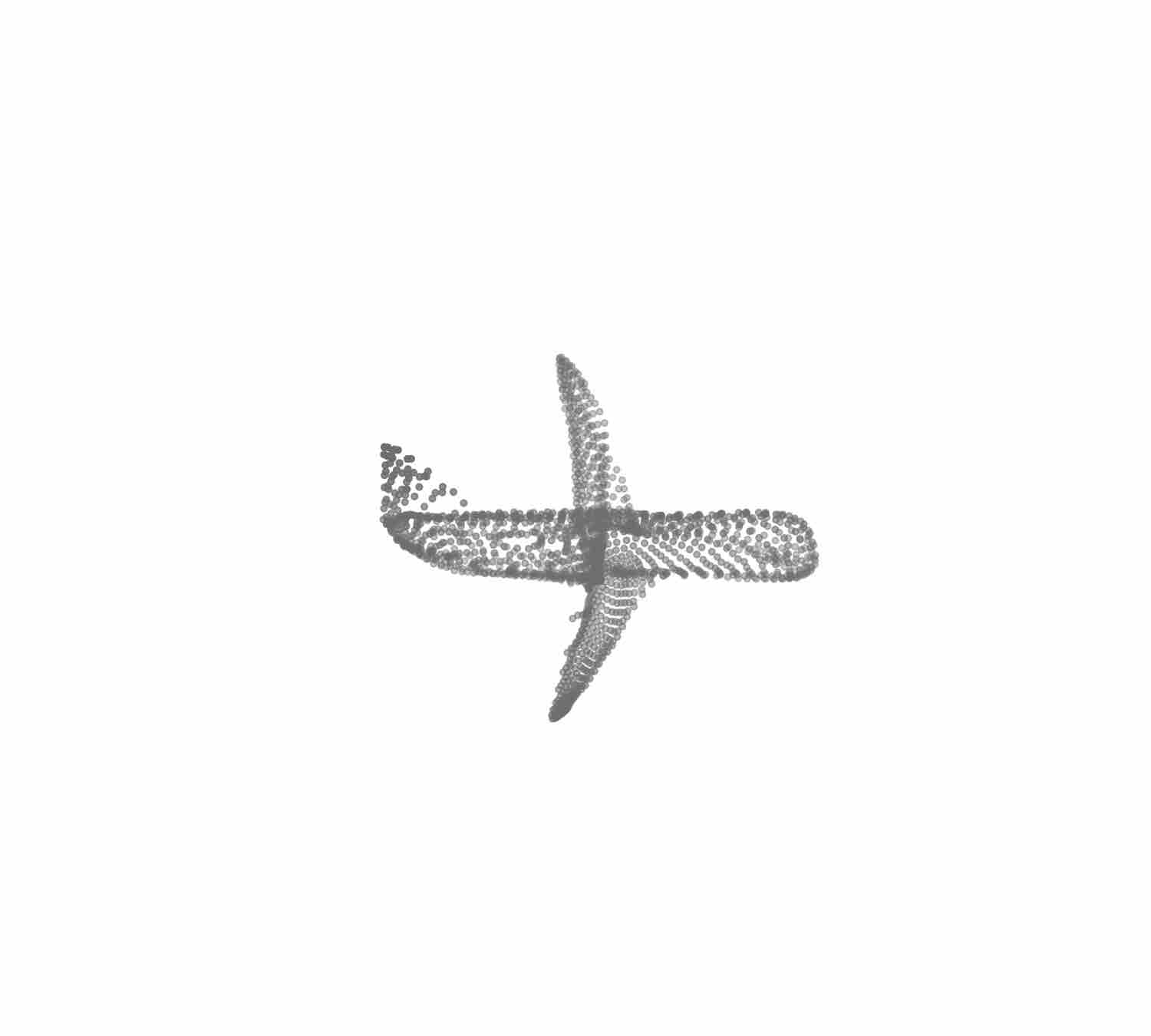} & 
      \includegraphics[trim={5cm 4cm 3cm 4cm}, clip=true , scale=0.045] {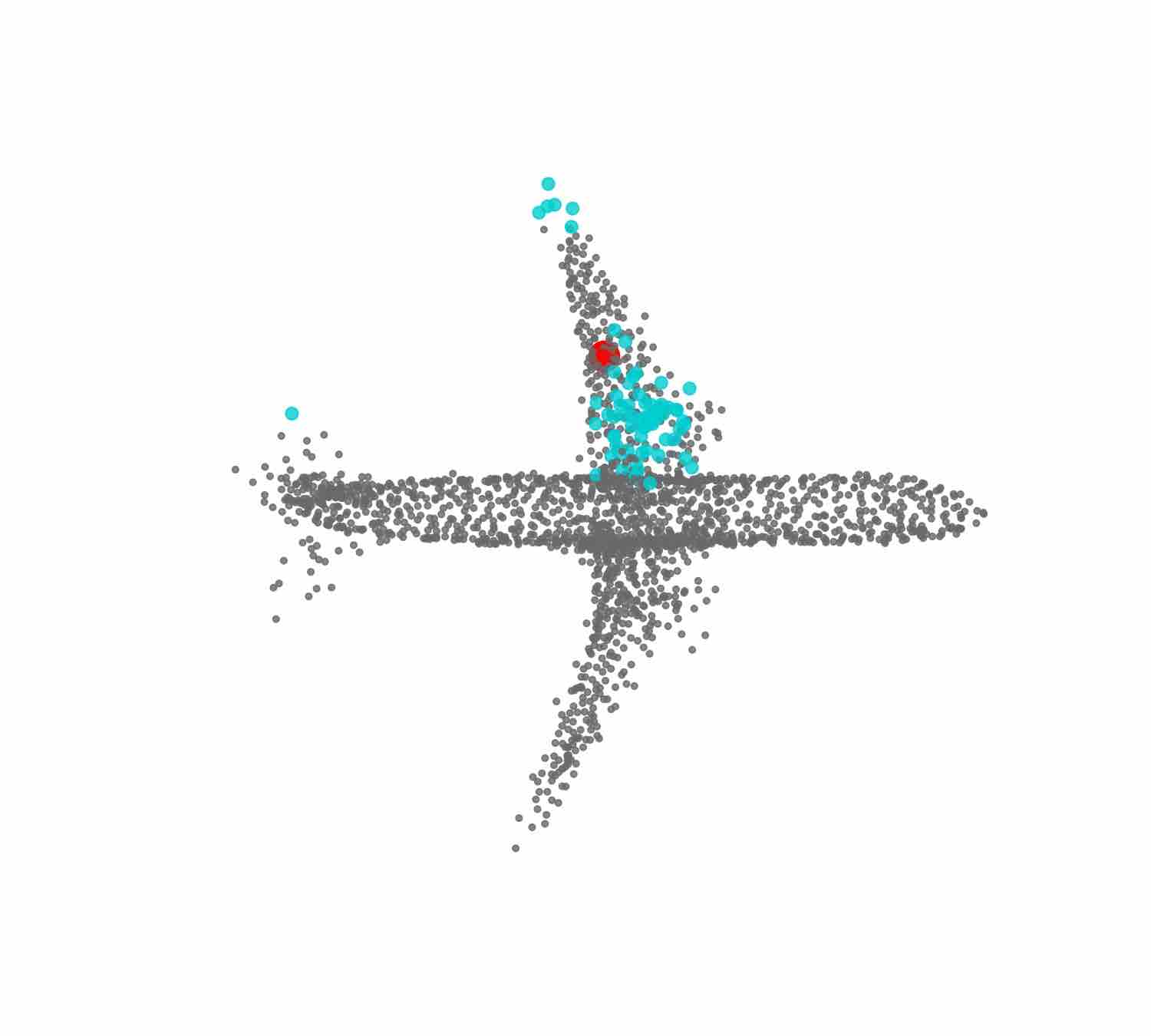} & 
      \includegraphics[trim={6cm 6cm 4cm 6cm}, clip=true , scale=0.06] {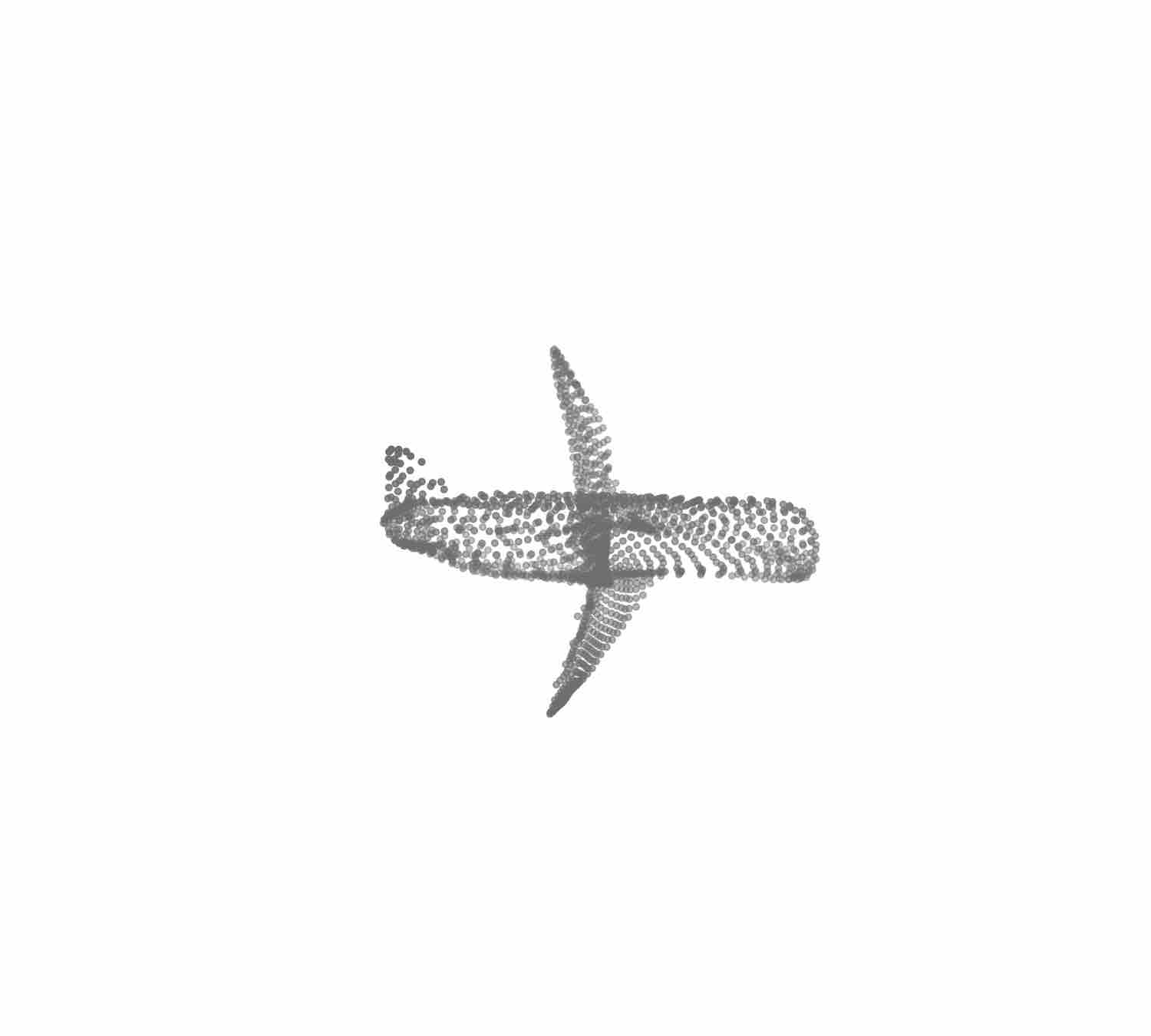} &
      \includegraphics[trim={5cm 4cm 3cm 4cm}, clip=true , scale=0.045] {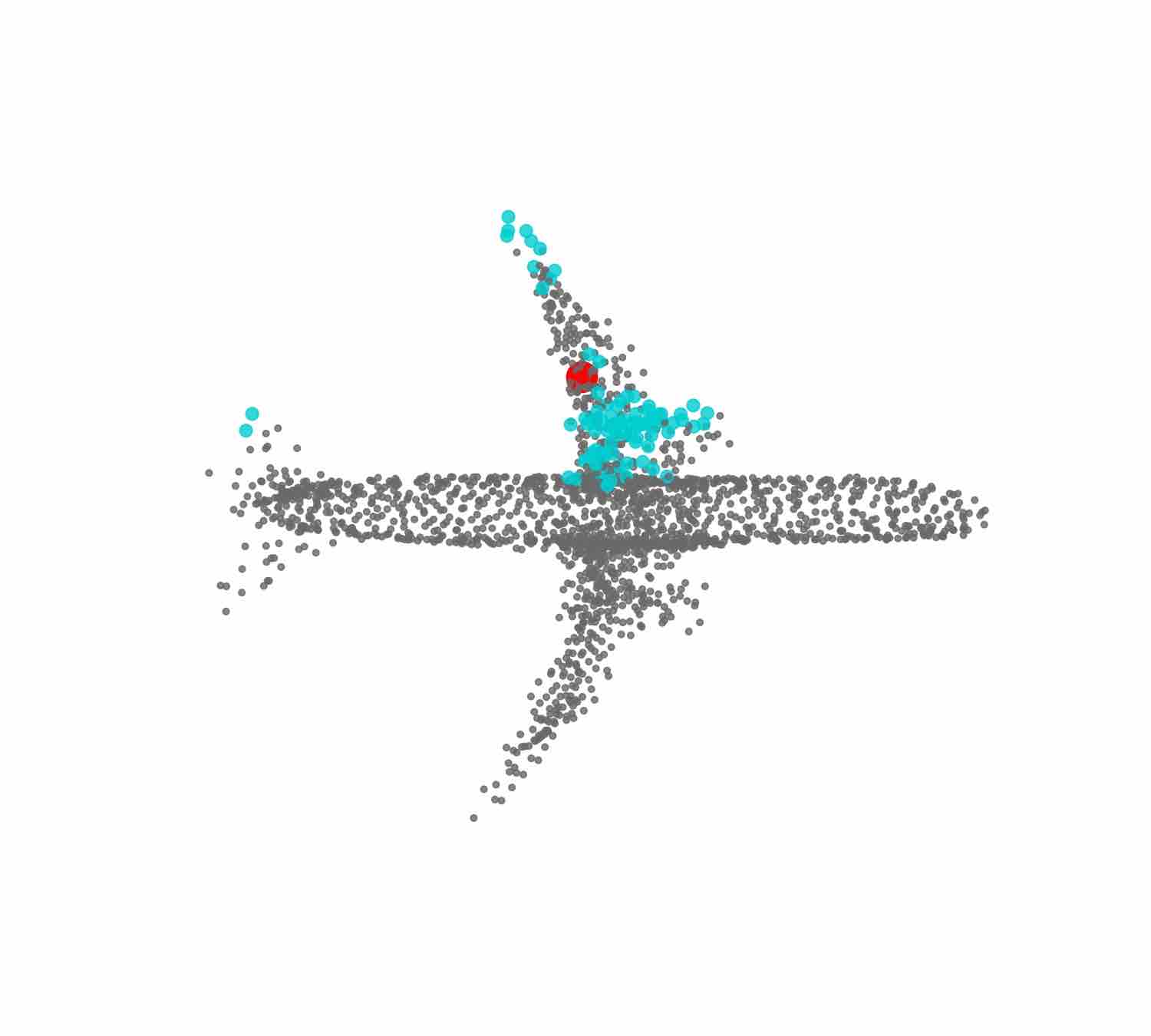} &
      \includegraphics[trim={6cm 6cm 4cm 6cm}, clip=true , scale=0.06] {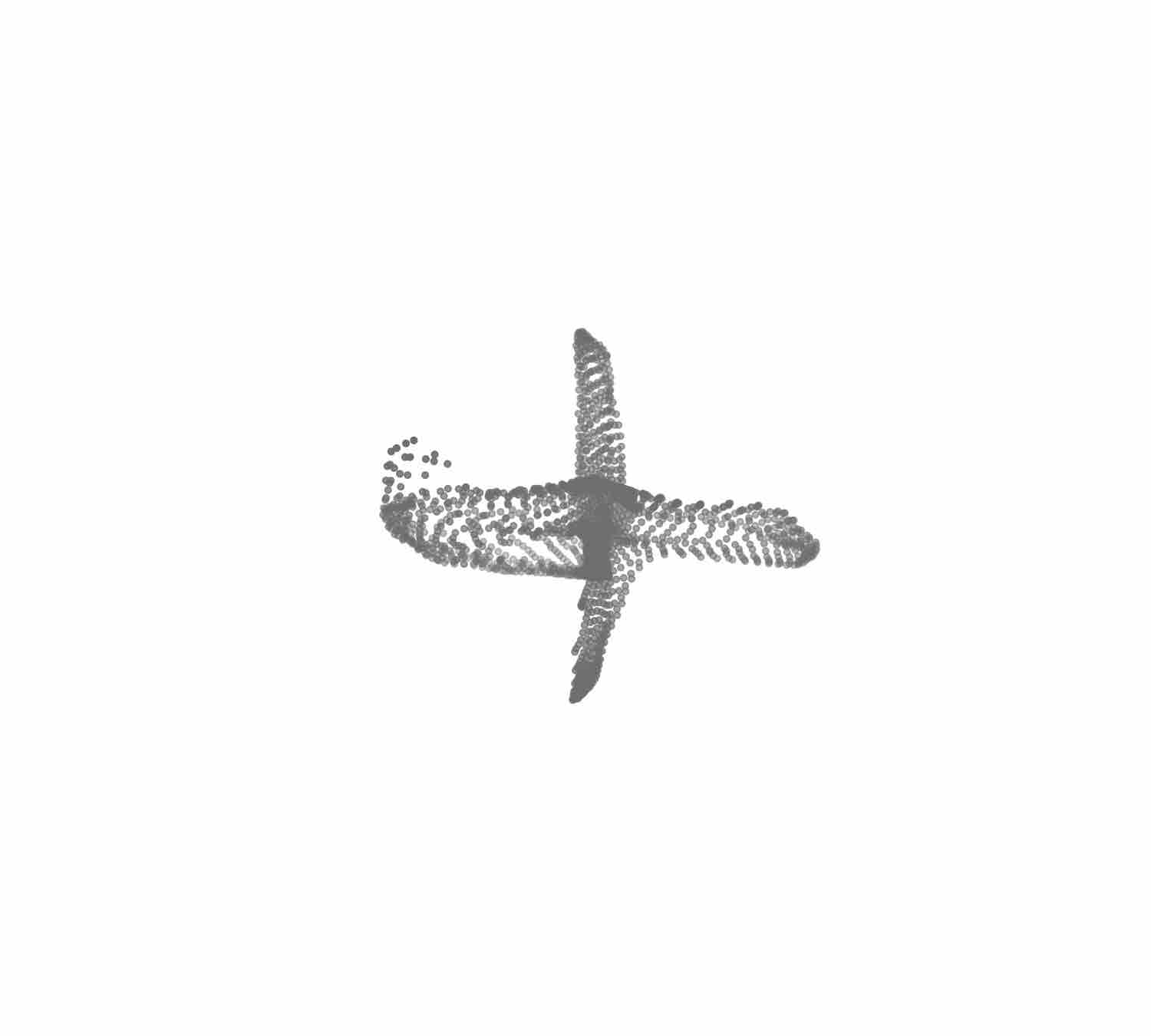} & 
      \includegraphics[trim={5cm 4cm 3cm 4cm}, clip=true , scale=0.045] {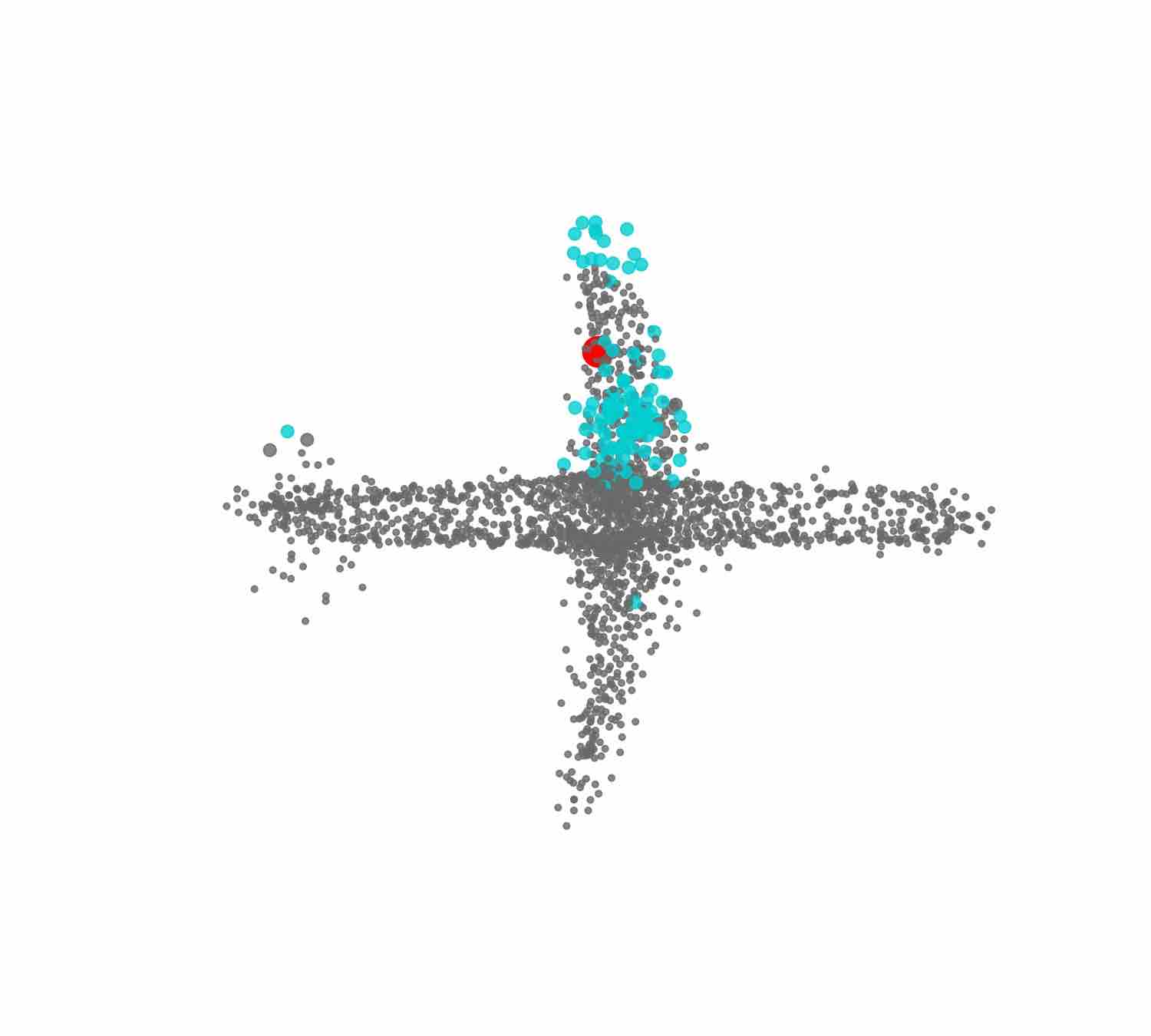} 
      \\      \hline
    \end{tabular}
  \end{center}
  \caption{\label{tab:graph_topology_evolvement_adj} \textbf{Visualizing the evolvement of reconstructions and learnt graph topologies during the training process (based on the graph-adjacency-matrix based filtering~\eqref{eq:adj_filtering}).} In the \textit{Before graph filtering} columns, we show the coarse reconstructions produced by the folding module, which is the intermediate output; in the \textit{After graph filtering} columns, we show the reconstructions after the graph-filtering module, which is the final output. We put supervision only to the final output. We choose one point in the 2D lattice and record its neighboring points in the training process. The chosen point is colored in red and its neighboring points are colored in cyan. The chosen point and its neighboring points evolve to form an engine, which is a fine detail. Comparing to the reconstructed point clouds in the \textit{Before graph filtering} columns, the reconstructed point clouds in the \textit{After graph filtering} columns preserve more fine details, such as engines, tails, nodes, and wing tips; see the improvement on fine-detailed classification in Table~\ref{tab:sub_classification}.}
\end{table*}

\begin{table*}[htb!]
  \begin{center}
    \begin{tabular}{c | c  c | c  c | c  c }
      \hline   
      \hline
      Epoch & 
      \multicolumn{2}{c|}{Swept wings no engines} & 
      \multicolumn{2}{c|}{Swept wings two engines} & 
      \multicolumn{2}{c}{Straight wings two engines} \\
      \hline
             & 
       \emph{Before graph filtering} & 
       \emph{After graph filtering} & 
       \emph{Before graph filtering} & 
       \emph{After graph filtering} & 
       \emph{Before graph filtering} & 
       \emph{After graph filtering} 
       \\
        \hline
      
      Input &
      \includegraphics[trim={4cm 4cm 2cm 4cm}, clip=true , scale=0.025] {figures/graph_variance/grid.jpg} &
      \includegraphics[trim={4cm 4cm 2cm 4cm}, clip=true , scale=0.045] {figures/graph_variance/graph_fig_e0/input_pc_0_e0.jpg} & 
      \includegraphics[trim={4cm 4cm 2cm 4cm}, clip=true , scale=0.025] {figures/graph_variance/grid.jpg} &
      \includegraphics[trim={4cm 4cm 2cm 4cm}, clip=true , scale=0.045]  {figures/graph_variance/graph_fig_e0/input_pc_13_e0.jpg} & 
      \includegraphics[trim={4cm 4cm 2cm 4cm}, clip=true , scale=0.025] {figures/graph_variance/grid.jpg} &
      \includegraphics[trim={4cm 4cm 2cm 4cm}, clip=true , scale=0.045] {figures/graph_variance/graph_fig_e0/input_pc_23_e0.jpg}
 \\
      \hline
       &&&&&&\\[-1em] 
      0 &
      \includegraphics[trim={6cm 6cm 4cm 6cm}, clip=true , scale=0.06] {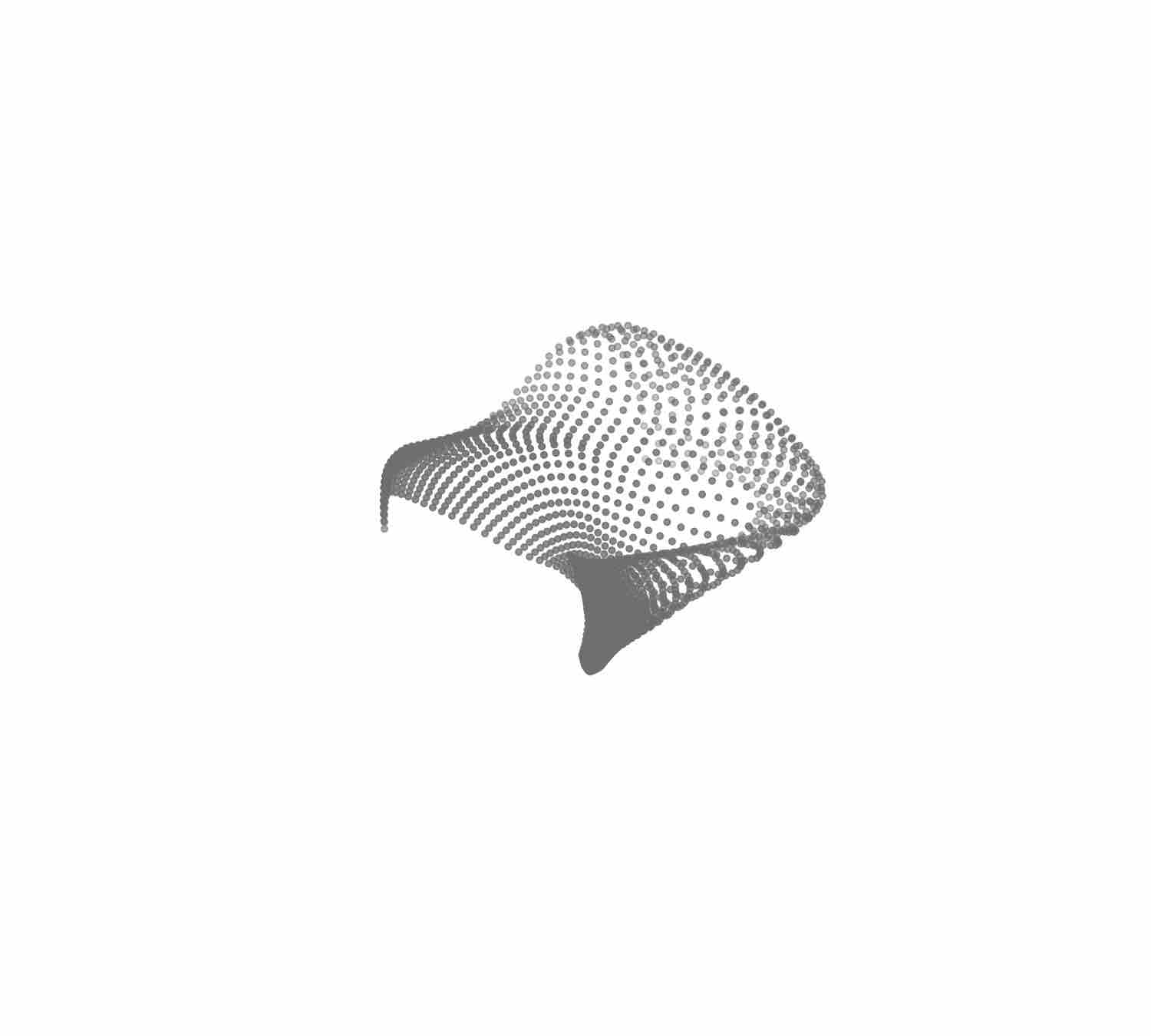} & 
      \includegraphics[trim={5cm 4cm 3cm 4cm}, clip=true , scale=0.045] {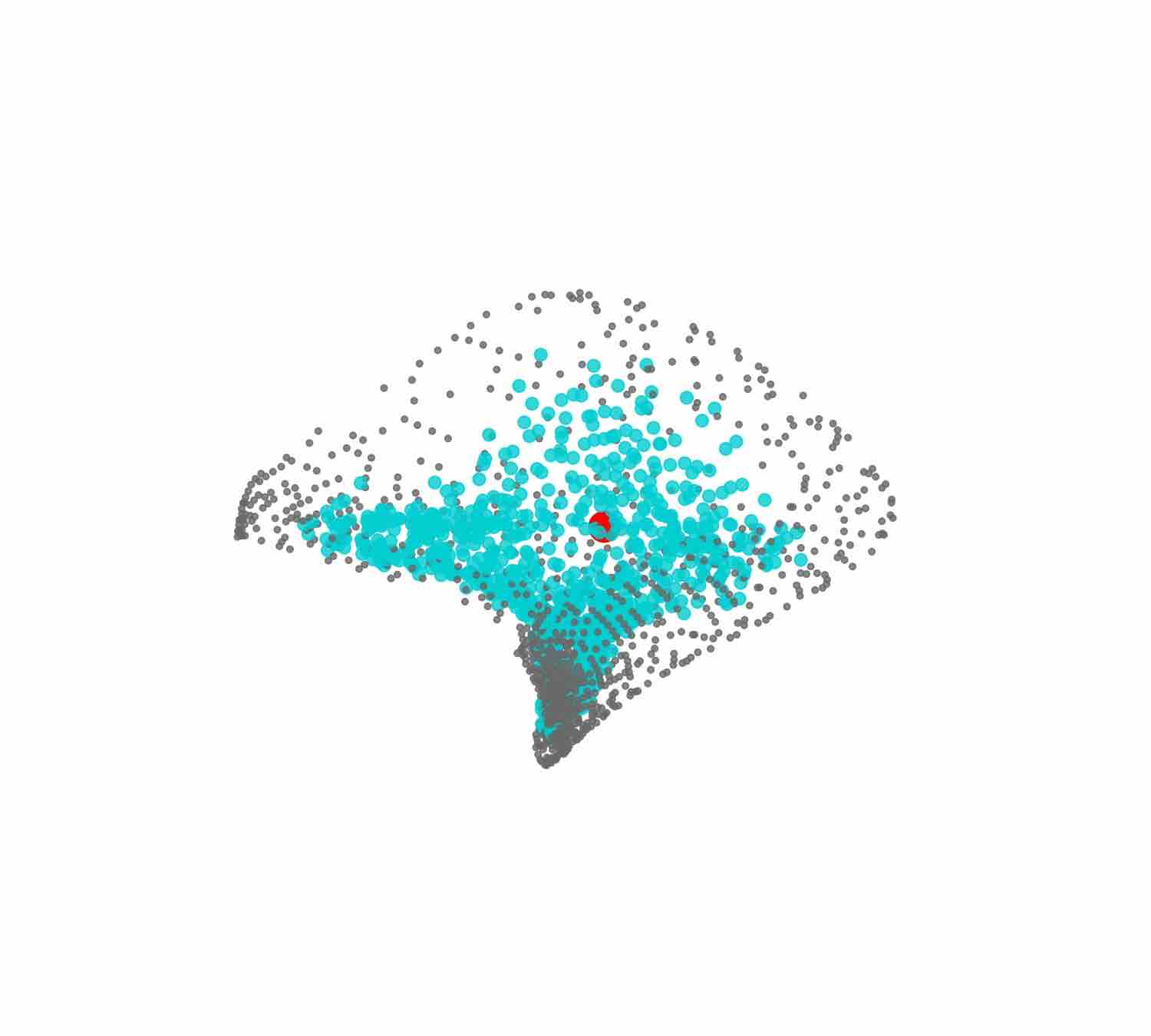} & 
      \includegraphics[trim={6cm 6cm 4cm 6cm}, clip=true , scale=0.06] {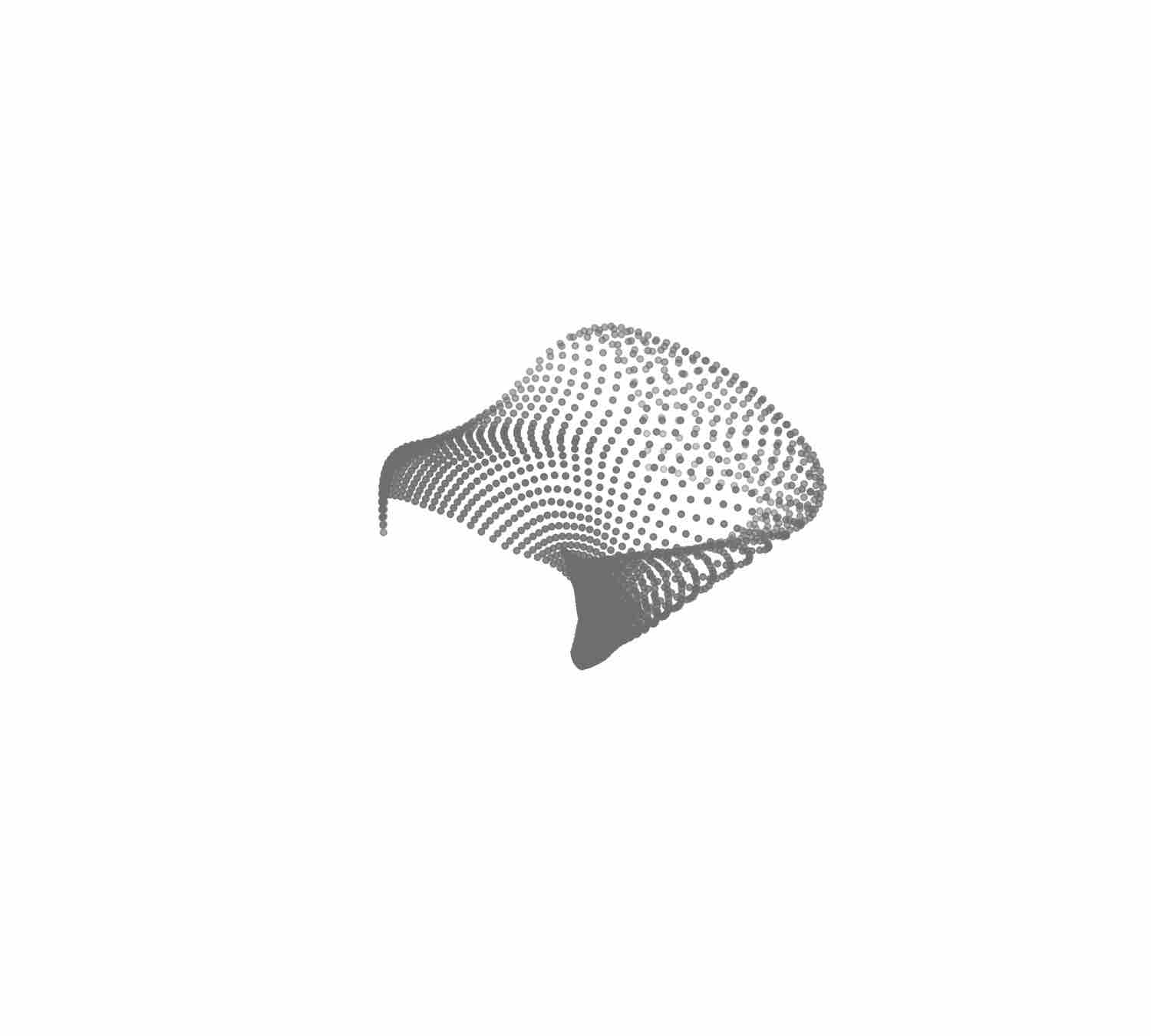} &
      \includegraphics[trim={5cm 4cm 3cm 4cm}, clip=true , scale=0.045] {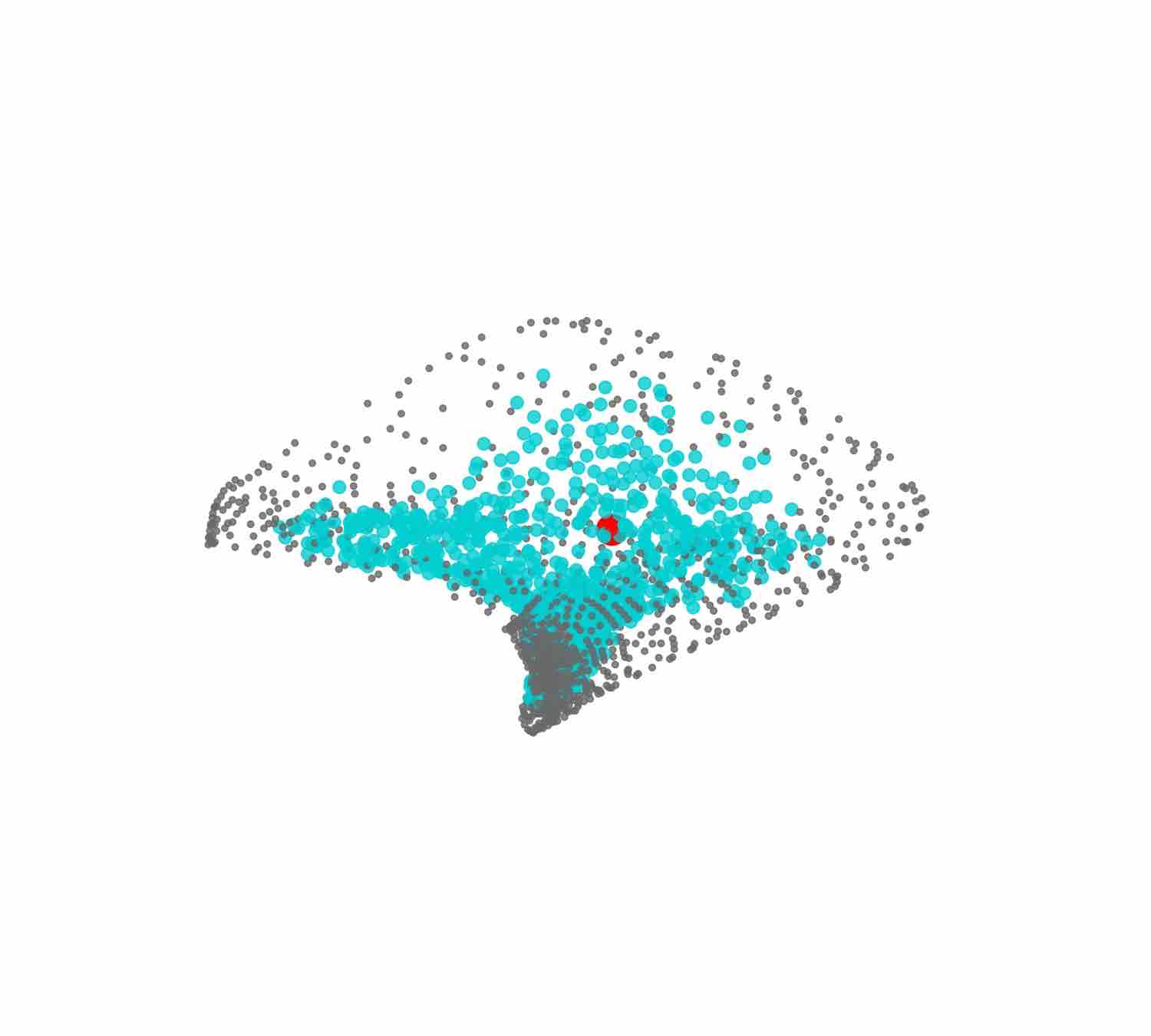} & 
      \includegraphics[trim={6cm 6cm 4cm 6cm}, clip=true , scale=0.06] {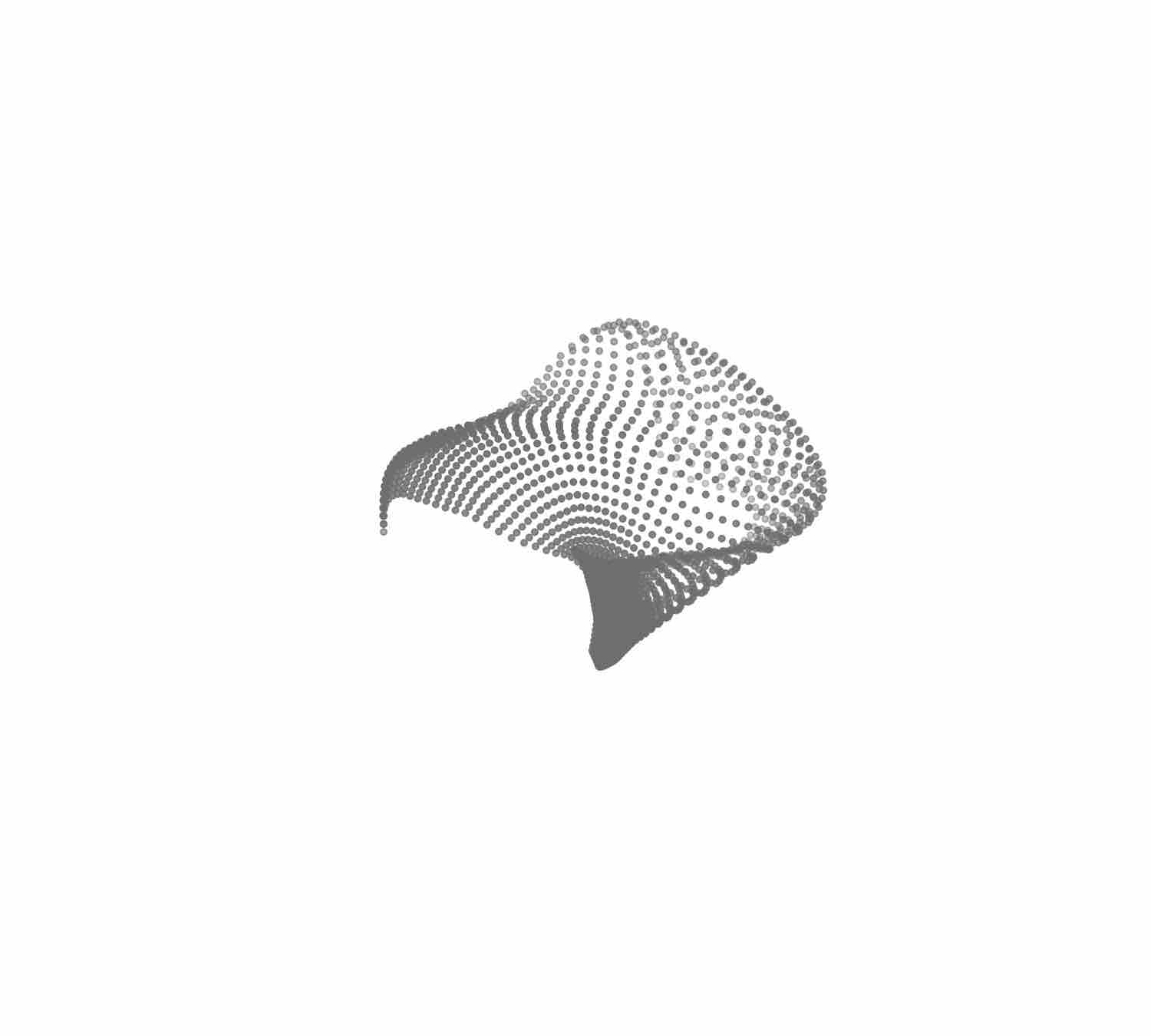} & 
      \includegraphics[trim={5cm 4cm 3cm 4cm}, clip=true , scale=0.045] {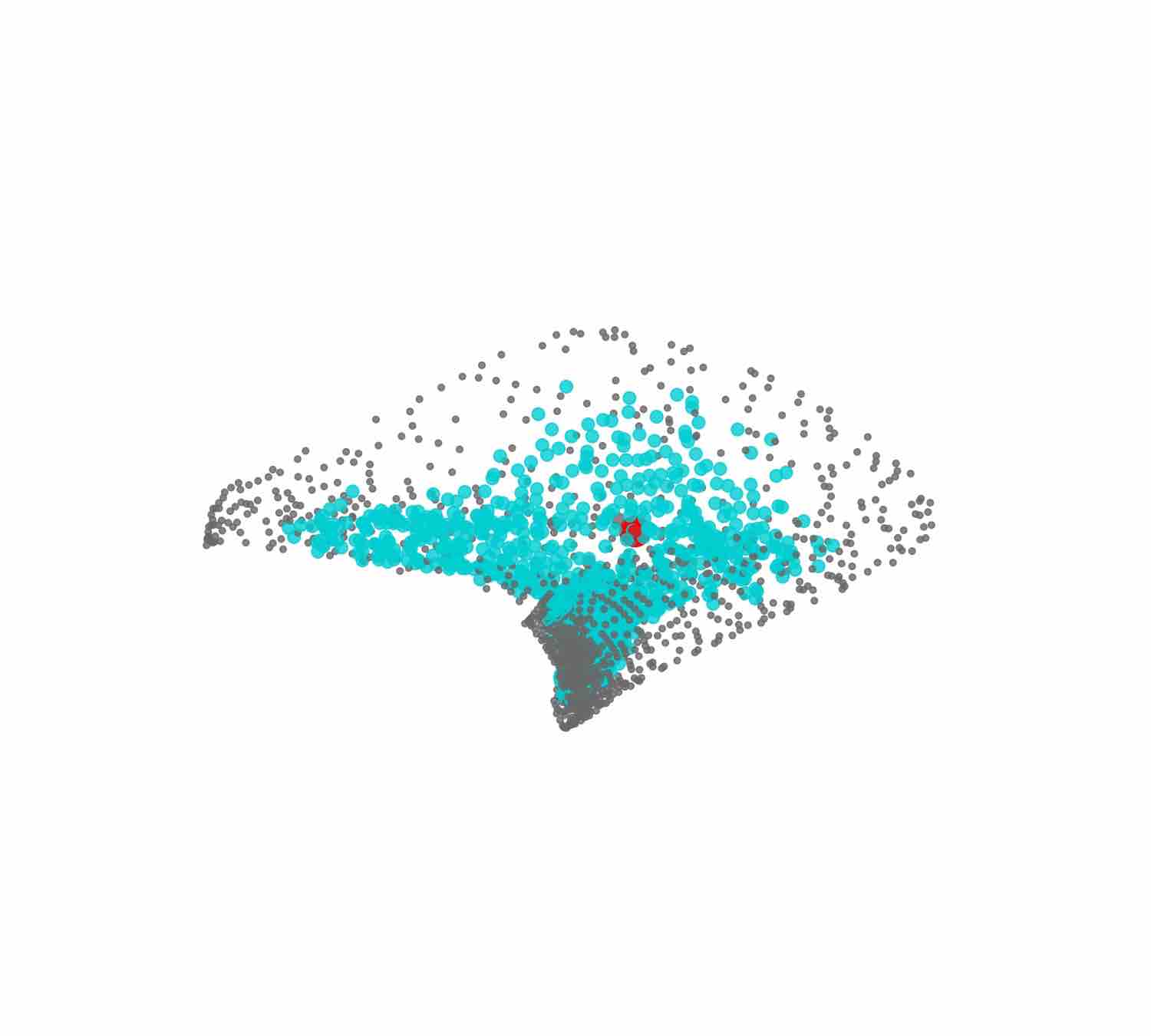}    
      \\      
      \hline
       &&&&&&\\[-1em]
      5 &
      \includegraphics[trim={6cm 6cm 4cm 6cm}, clip=true , scale=0.06] {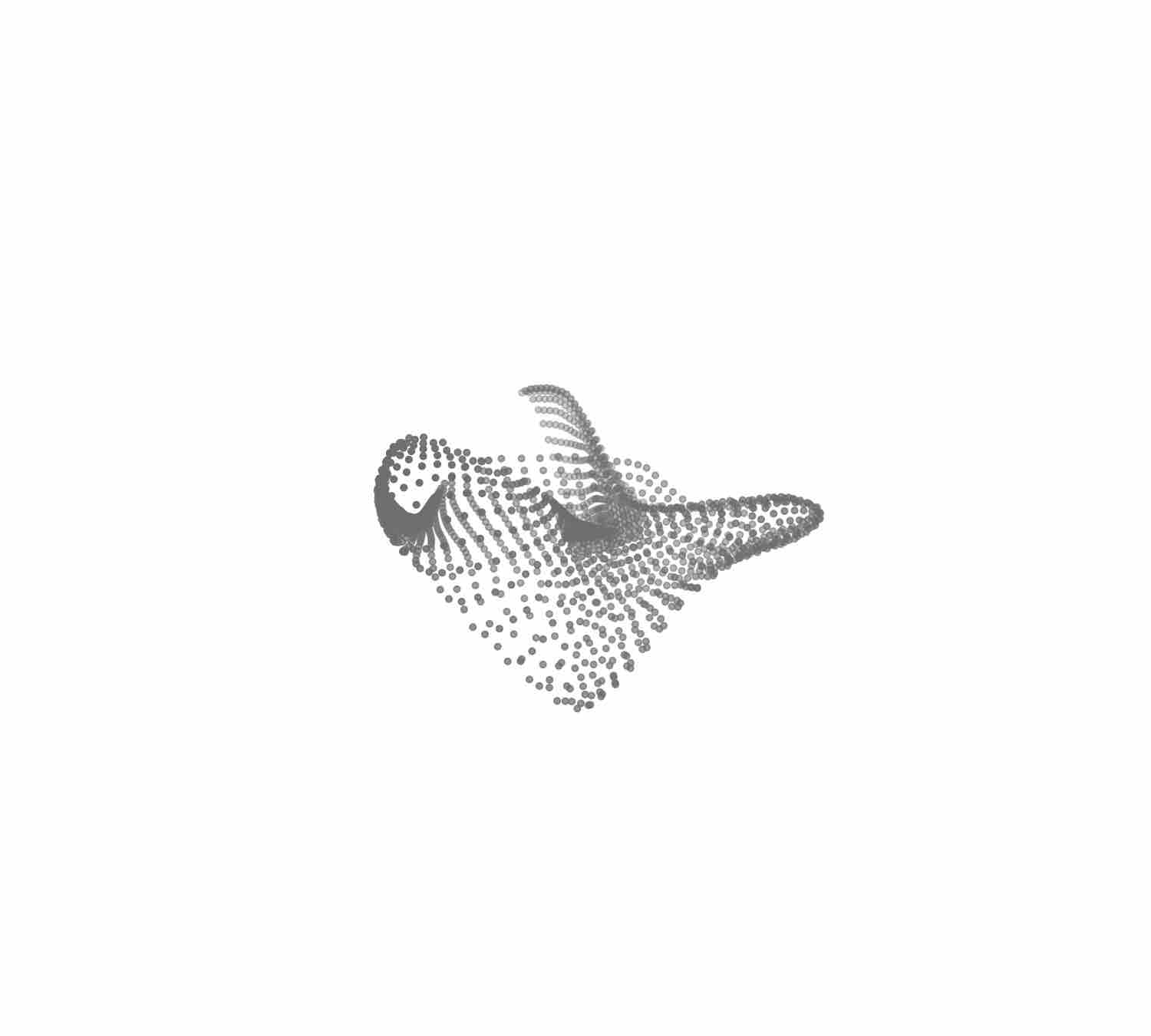} & 
      \includegraphics[trim={5cm 4cm 3cm 4cm}, clip=true , scale=0.045] {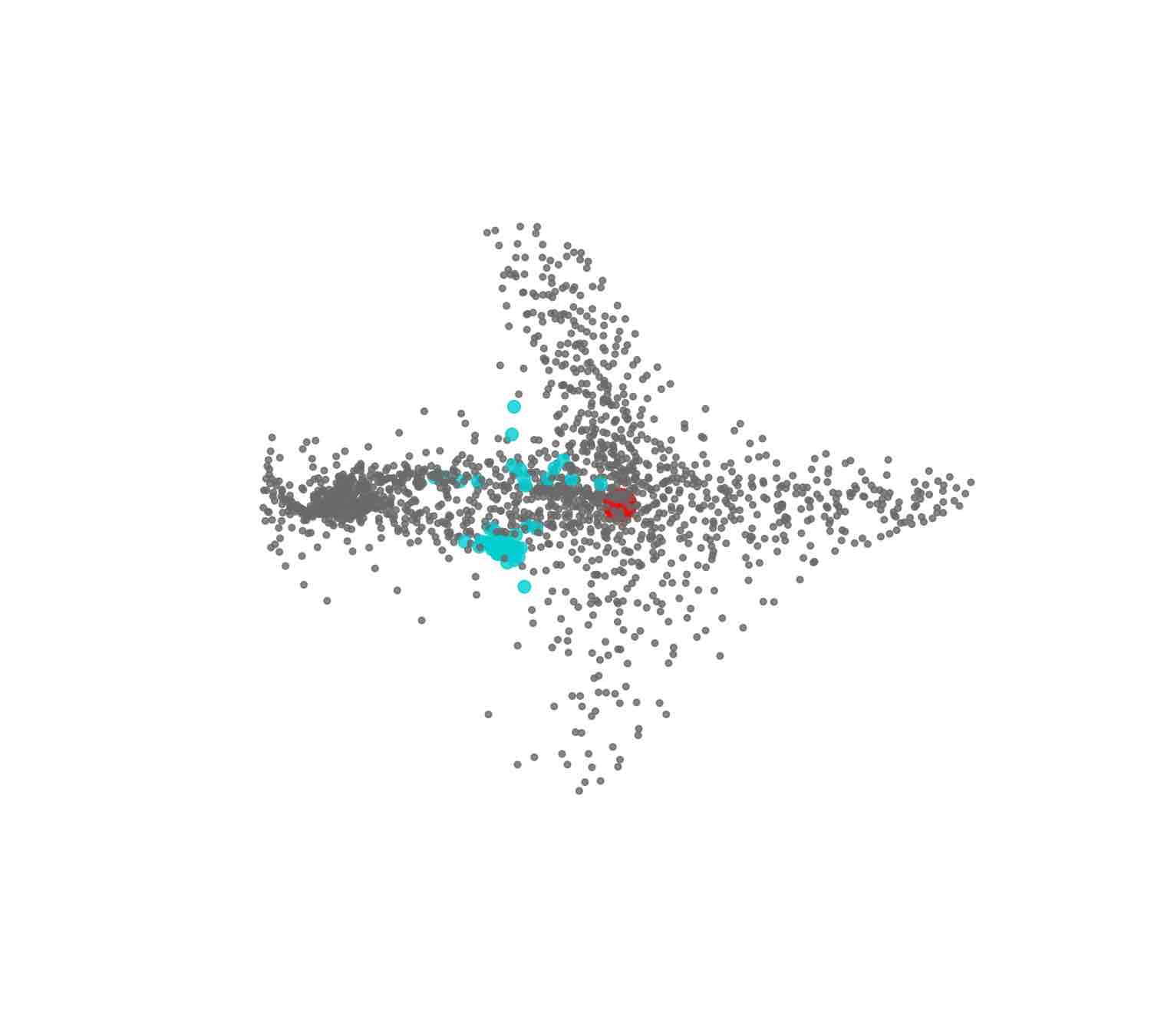} & 
      \includegraphics[trim={6cm 6cm 4cm 6cm}, clip=true , scale=0.06] {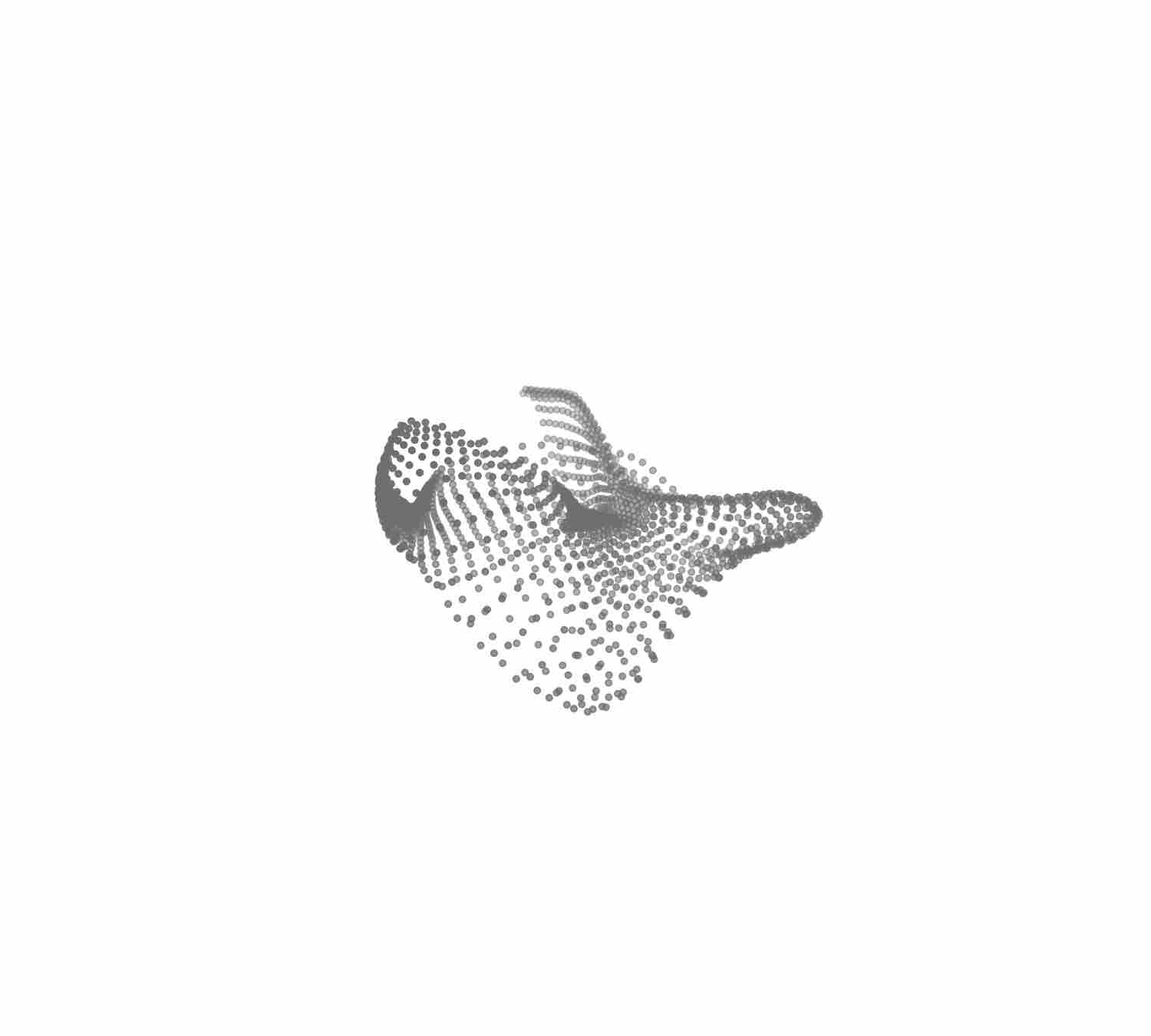} &
      \includegraphics[trim={5cm 4cm 3cm 4cm}, clip=true , scale=0.045] {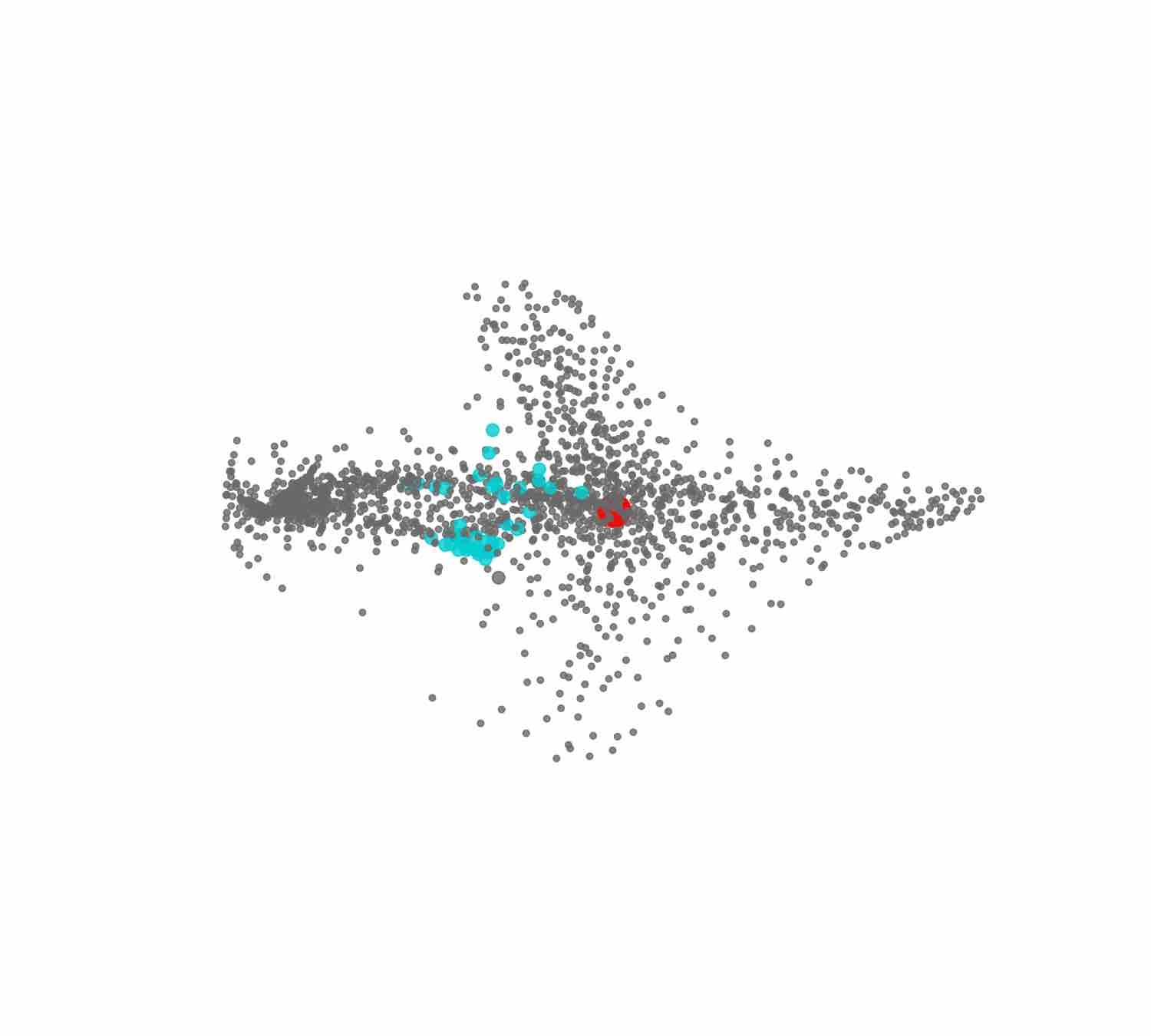} &
      \includegraphics[trim={6cm 6cm 4cm 6cm}, clip=true , scale=0.06] {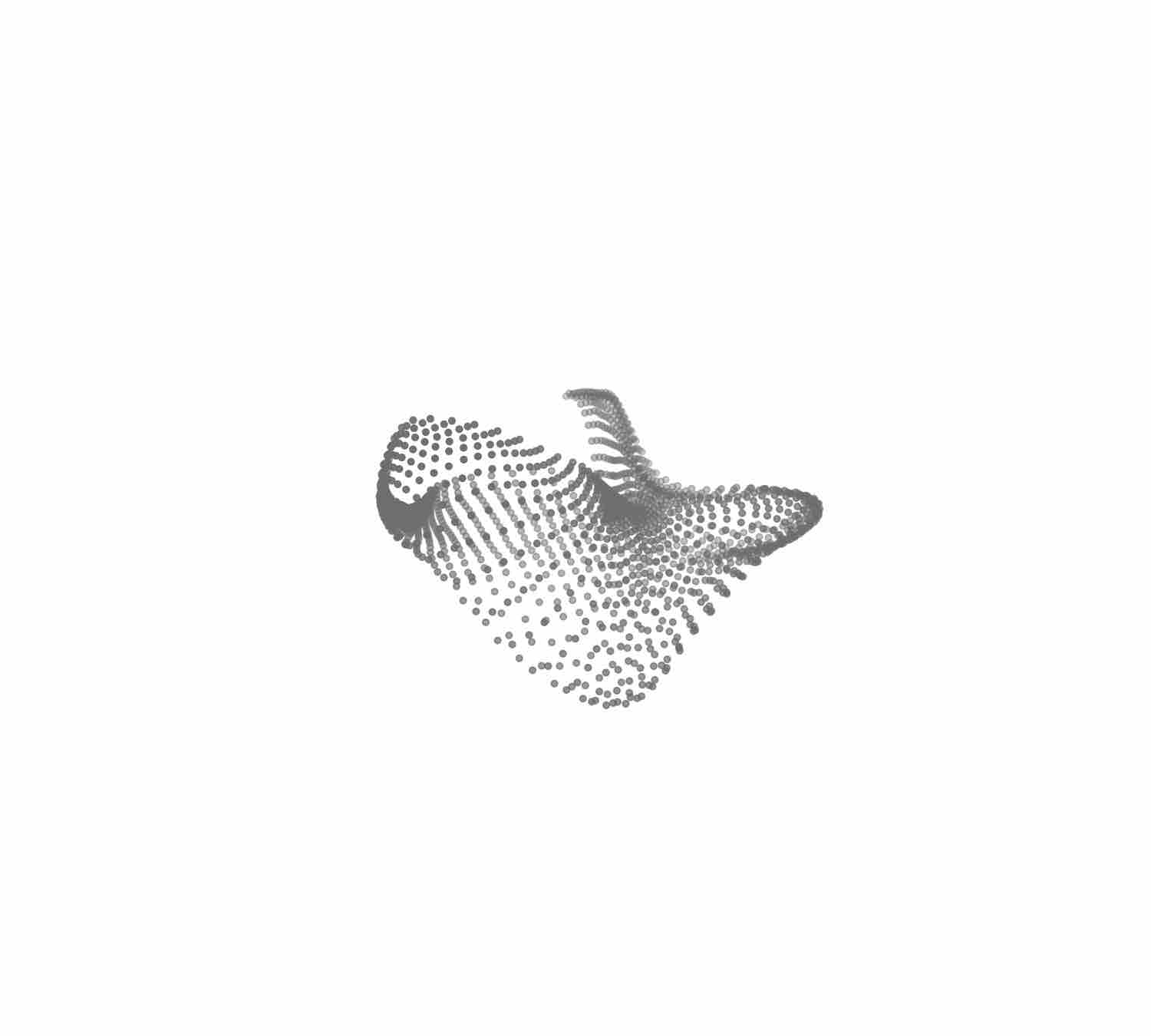} & 
      \includegraphics[trim={5cm 4cm 3cm 4cm}, clip=true , scale=0.045] {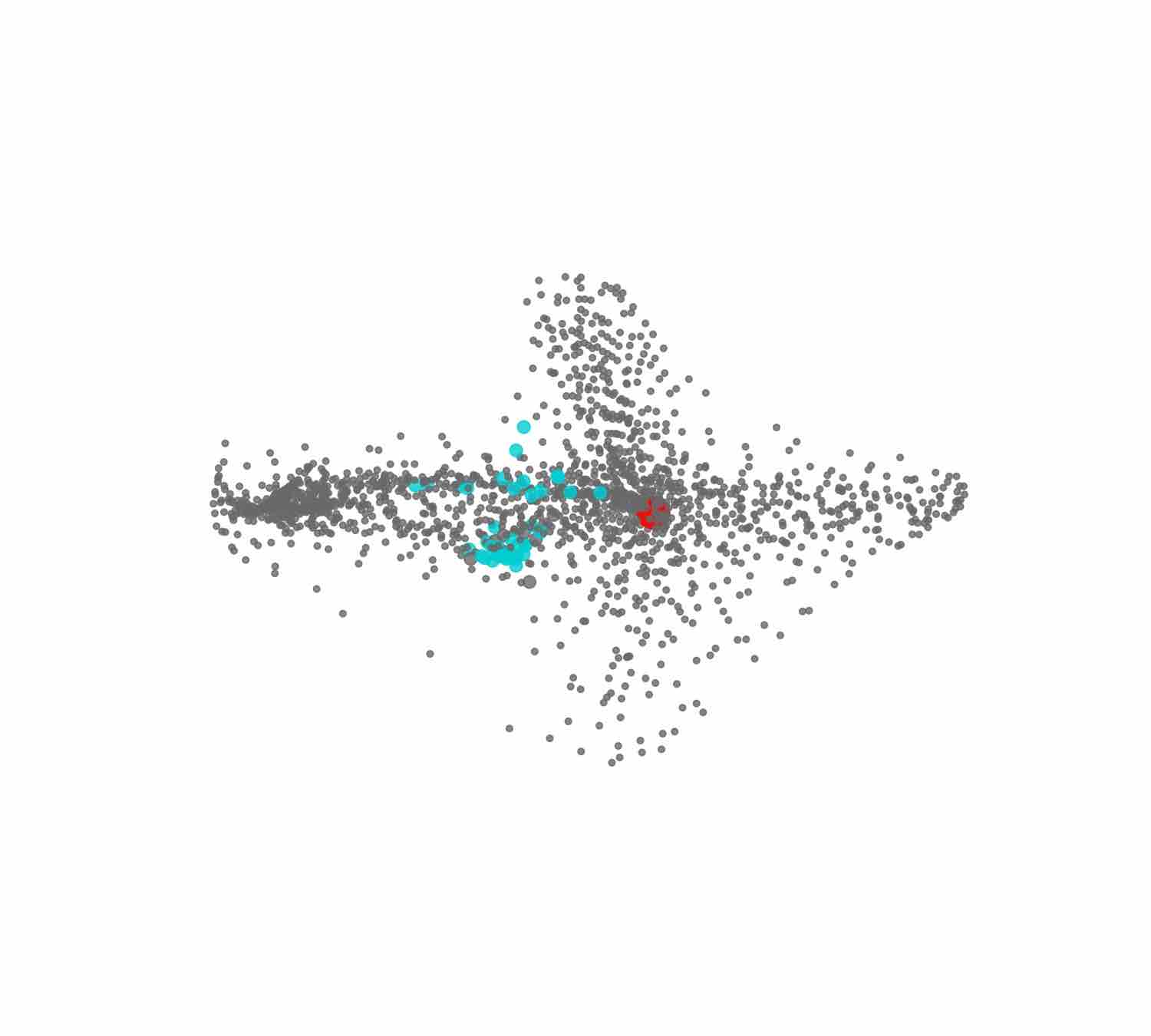}
      \\      
      \hline
       &&&&&&\\[-1em]
            
      300 &
      \includegraphics[trim={6cm 6cm 4cm 6cm}, clip=true , scale=0.06] {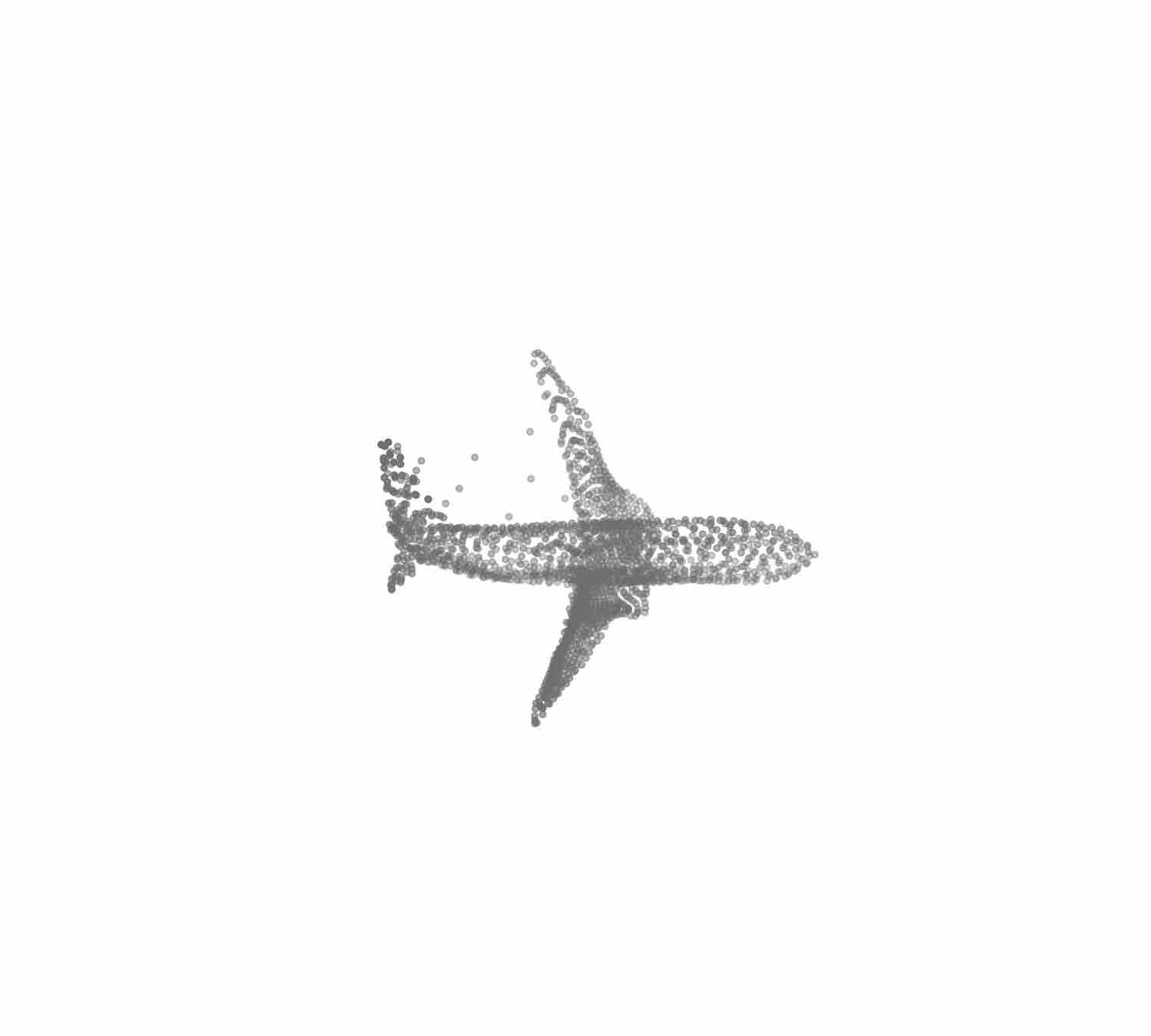} & 
      \includegraphics[trim={5cm 4cm 3cm 4cm}, clip=true , scale=0.045] {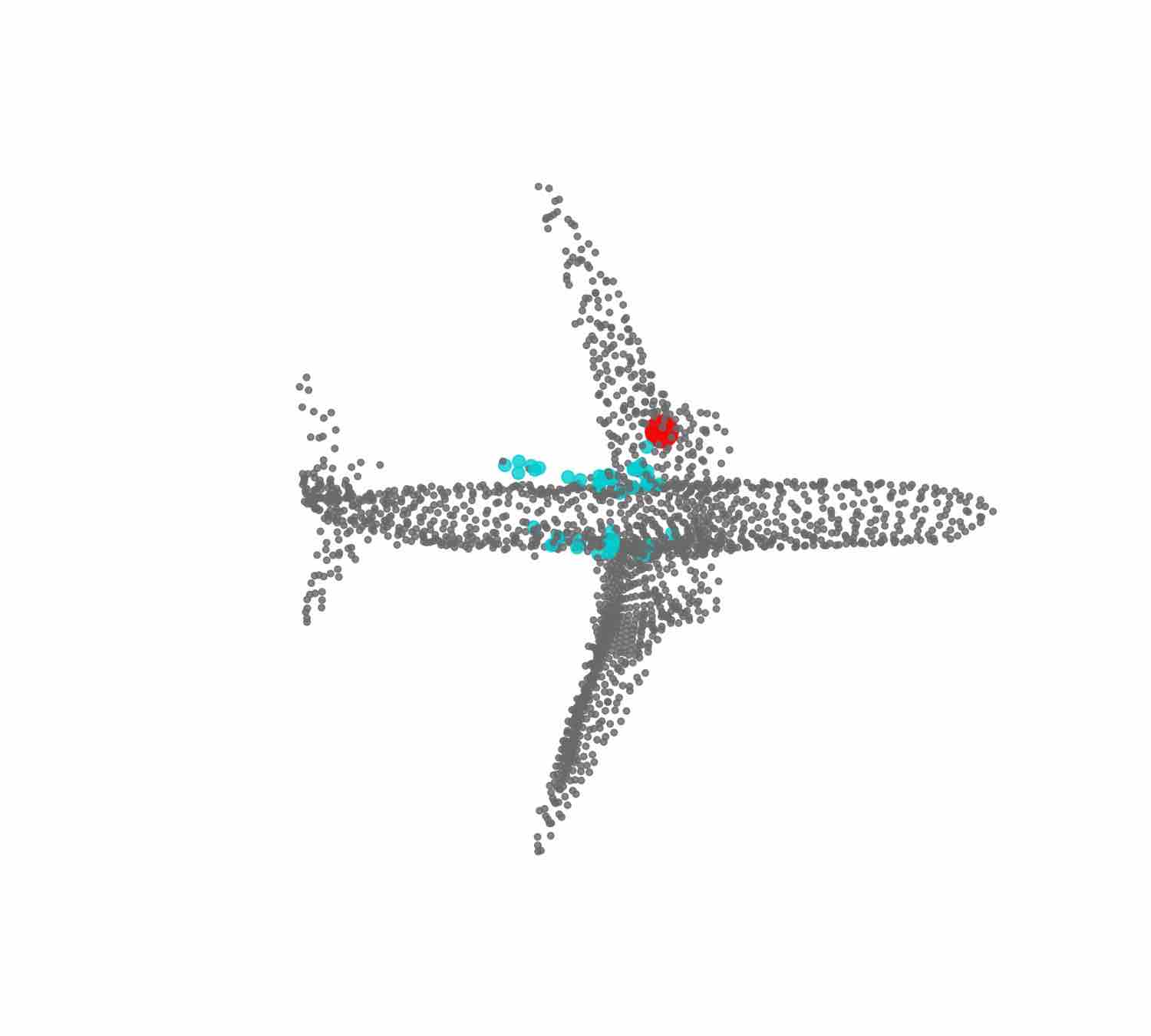} & 
      \includegraphics[trim={6cm 6cm 4cm 6cm}, clip=true , scale=0.06] {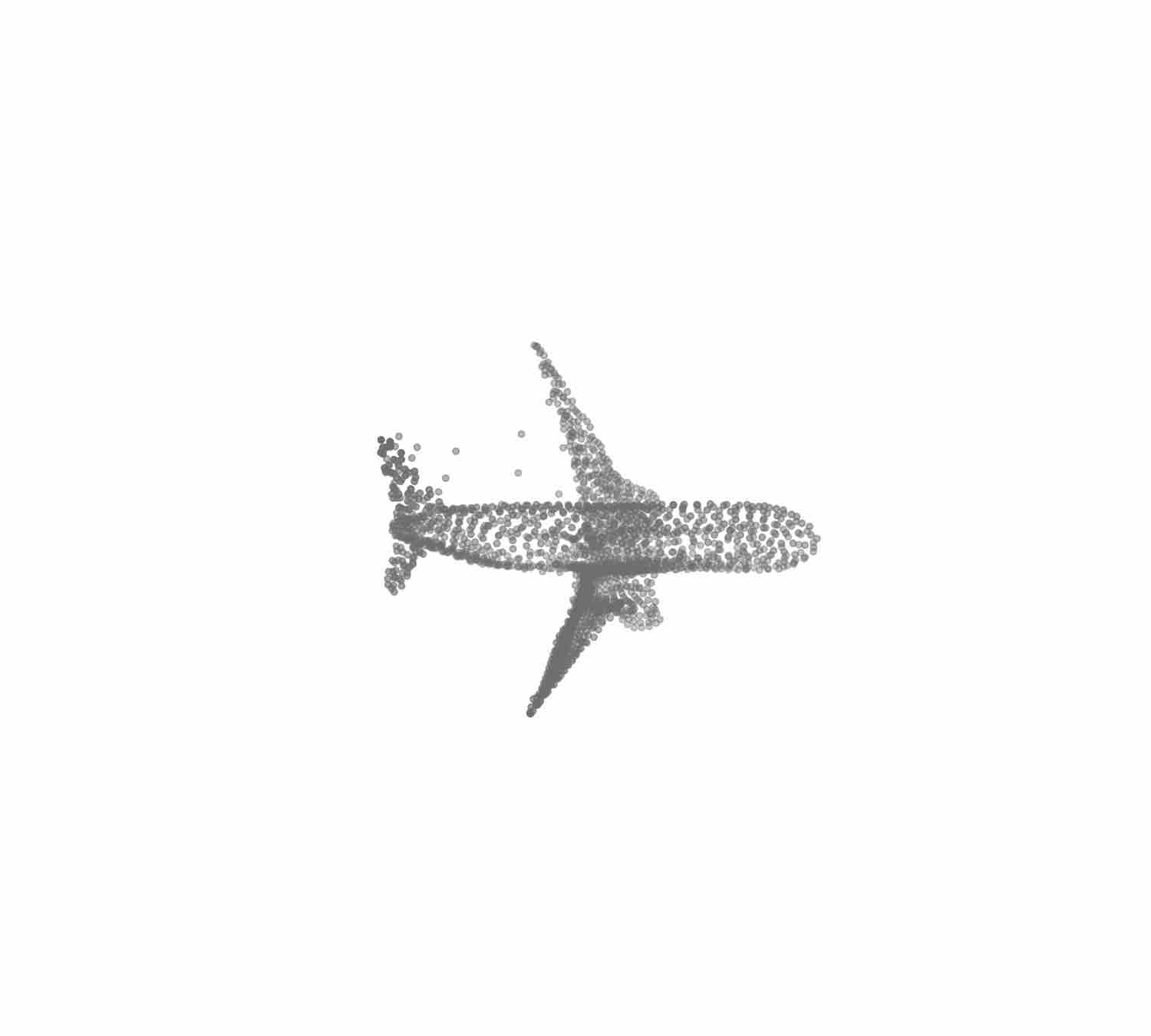} &
      \includegraphics[trim={5cm 4cm 3cm 4cm}, clip=true , scale=0.045] {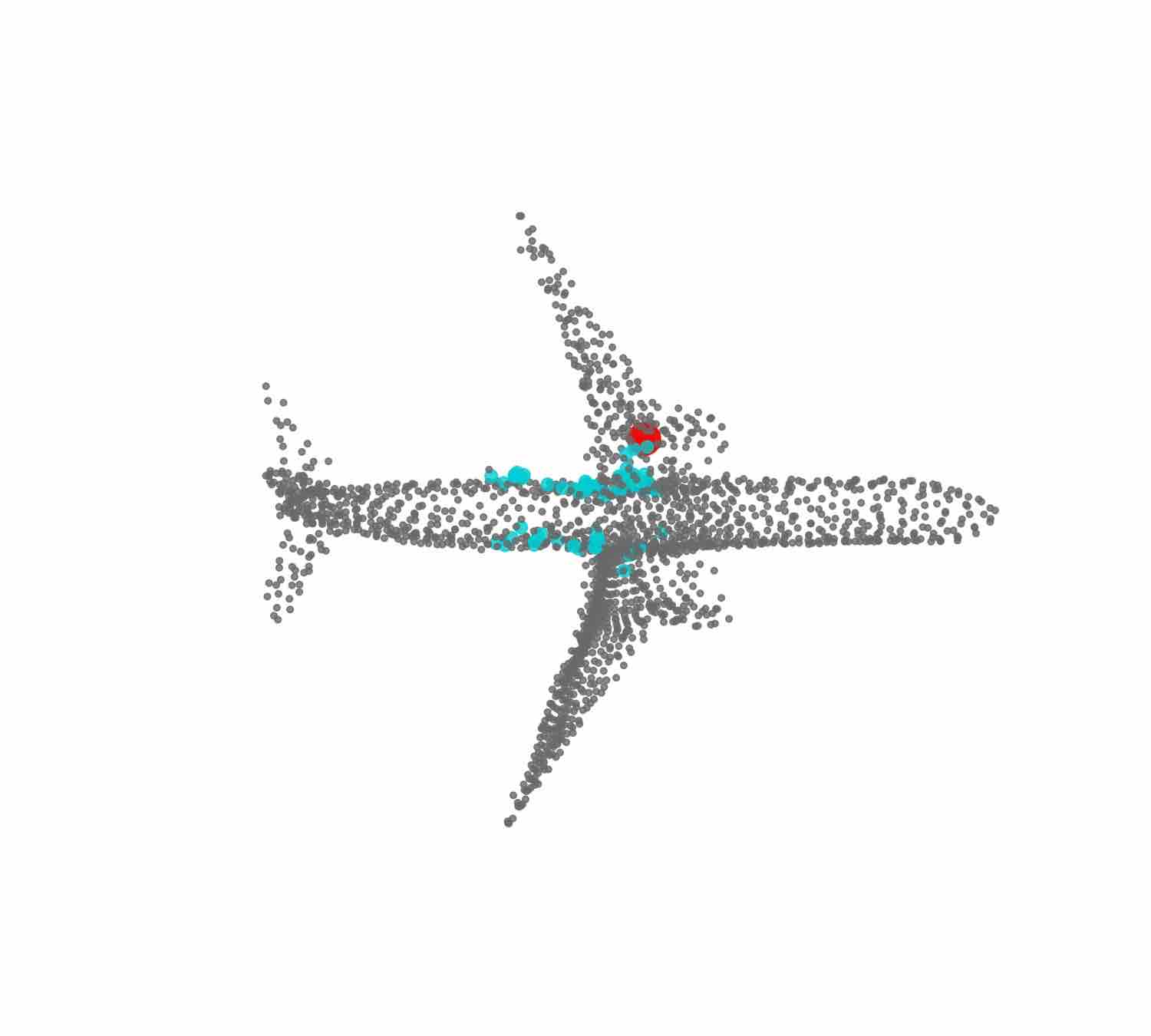} &
      \includegraphics[trim={6cm 6cm 4cm 6cm}, clip=true , scale=0.06] {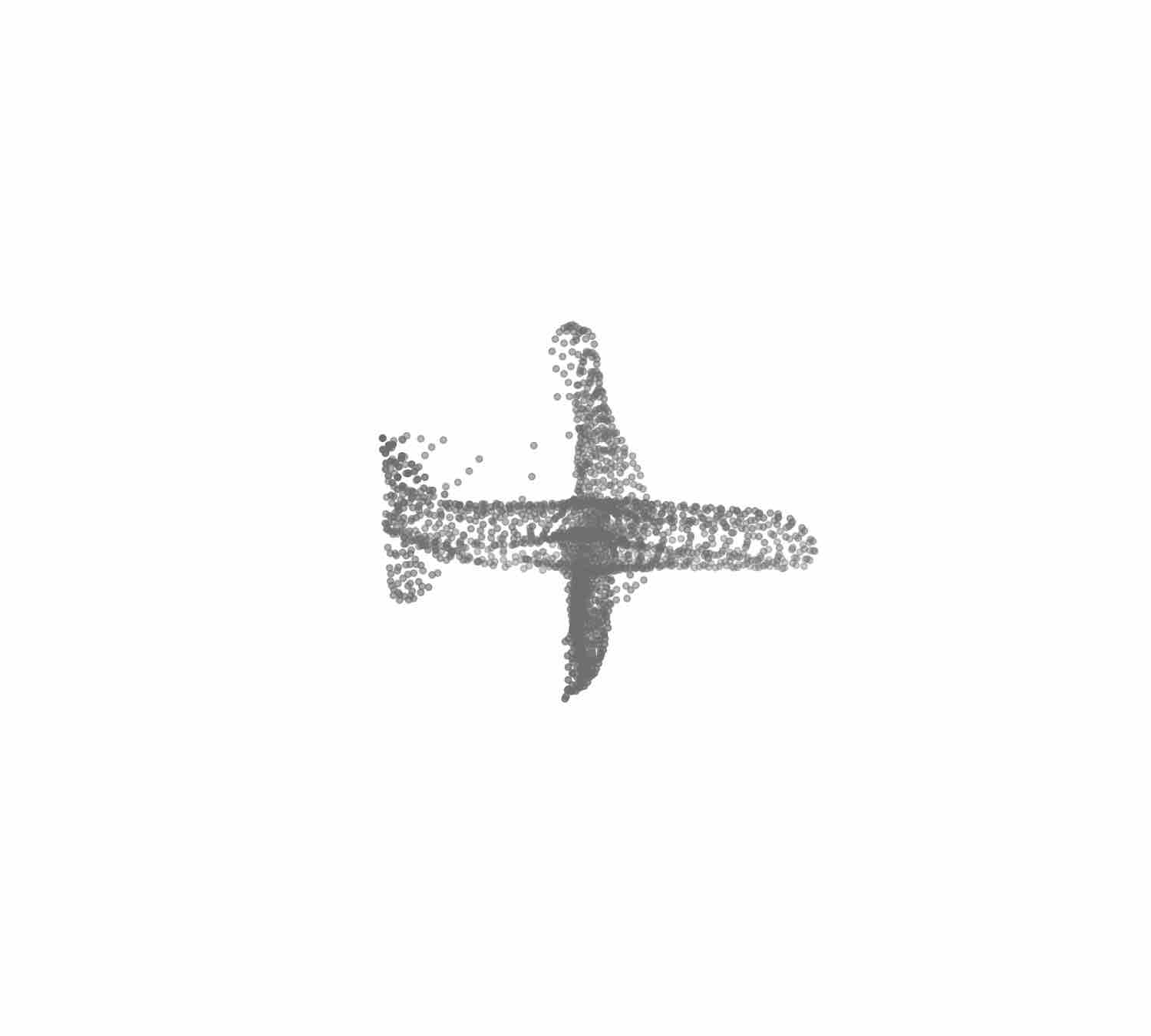} & 
      \includegraphics[trim={5cm 4cm 3cm 4cm}, clip=true , scale=0.045] {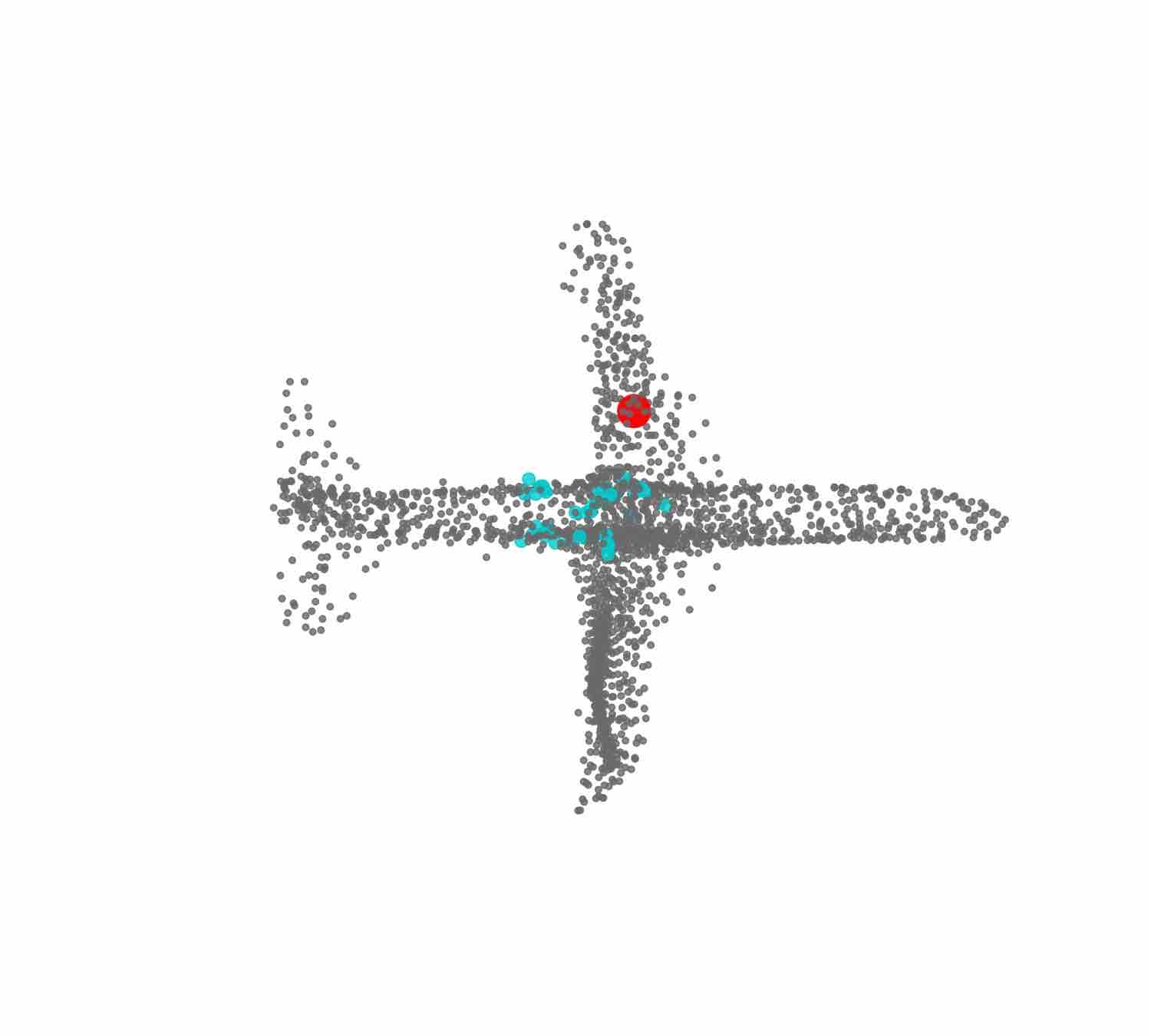} 
      \\      
      \hline
    \end{tabular}
  \end{center}  \caption{\label{tab:graph_topology_evolvement_laplacian} \textbf{Visualizing the evolvement of reconstructions and learnt graph topologies during the training process (based on the graph-Laplacian-matrix based filtering~\eqref{eq:lap_filter}).} In the \textit{Before graph filtering} columns, we show the coarse reconstructions produced by the folding module, which is the intermediate output; in the \textit{After graph filtering} columns, we show the reconstructions after the graph-filtering module, which is the final output. We put supervision only to the final output. We choose one point in the 2D lattice and record its neighboring points in the training process. The chosen point is colored in red and its neighboring points are colored in cyan. The chosen point and its neighboring points evolve to form an engine, which is a fine detail. Comparing to the reconstructed point clouds in the \textit{Before graph filtering} columns, the reconstructed point clouds in the \textit{After graph filtering} columns preserve more fine details, such as engines, tails, nodes, and wing tips; see the improvement on fine-detailed classification in Table~\ref{tab:sub_classification}.}
\end{table*}

\subsection{Visualization of Reconstruction and Graph Topology}
Tables~\ref{tab:graph_topology_evolvement_adj} and ~\ref{tab:graph_topology_evolvement_laplacian} show the evolvement of reconstructions and learnt graph topologies during the training process based on graph-adjacency-matrix and graph-Laplacian-matrix based filtering, respectively. The \textit{Before graph filtering} columns  show the coarse reconstructions produced by the folding module. Different from the reconstruction without graph filtering shown in Section~\ref{sec:reconstruction}, here we train the entire networks with three modules end-to-end and extract the output of the folding module as  the intermediate output. In other words, the supervision is not put on the output of the folding module directly. The \textit{After graph filtering} columns show the reconstructions after the graph-filtering module.  We see that (i) for three airplanes,  the reconstructed point clouds in the \textit{Before graph filtering} columns are similar, reflecting a shared and common shape of various airplanes. This indicates that  the folding module provides coarse reconstructions. (ii) The reconstructed point clouds in the \textit{After graph filtering} columns preserves more fine details, such as engines, tails, nodes, and wing tips. Note that the wing tips were missing without applying graph filtering; see Table~2~\cite{YangFST:18}. One can look at Figure~\ref{fig:architecture} for a more clear demonstration of the reconstructed wing tips from another visualization angle. This indicates that  the graph-filtering module refines the details and provides better reconstructions.   Comparing two types of graph filtering, the graph-adjacency-matrix-based filtering~\eqref{eq:adj_filtering} and the graph-Laplacian-matrix-based filtering~\eqref{eq:lap_filter} provide different, yet similar reconstruction performances.

To visualize the evolvement of the learnt graph topology, we select one point in the 2D lattice and record its neighbors in the learnt graph topology during the training process.  The chosen point is colored in red and its neighboring points are colored in cyan. We see that (i) during the training procedure, the learnt graph builds new connections and cuts the old connections adaptively for a better reconstruction performance; (ii) for three airplanes, the neighbors are similar, reflecting shared and common pairwise relationships in various airplanes; (iii) with the graph-topology-inference module and the graph-filtering module, both 3D coordinates and pairwise relationships are updating simultaneously to refine local shapes.

\subsection{3D Point Cloud Transfer Classification}
Features extracted from unsupervised-learning methods  can be used for supervised tasks. For example, principal component analysis is a common unsupervised-learning method that represents data via a learnt orthogonal transformation. The coordinates in the transformed coordinate system are often used as features in classification, regression~\cite{Bishop:06}. Similarly, we can take the latent code produced by the encoder of the proposed networks as features of a 3D point cloud. Based on those features, we can train a classifier to recognize the corresponding category of a 3D point cloud. Since the proposed networks and the classifier are trained by separate datasets, we call the task~\emph{transfer classification}.  In other words, the networks learn shape patterns from 3D point clouds in one dataset and are used to extract features of 3D point clouds in a different dataset.  The performance of the transfer classification shows the effectiveness of  the latent code and  the generalization ability of the networks.

 \begin{table*}[htb!]
\begin{center}
\begin{tabular}{ c|c|c|c|c } 
\hline
Method & Modality & \# code & MN40 & MN10 \\
\hline
SPH \cite{sph} & Voxels & 544 & 68.20 & 79.80 \\
LFD \cite{lfd} & Voxels & 4700 & 75.50 & 79.90 \\
T-L Network \cite{tlnetwork} & Voxels & - & 74.40 & - \\
VConv-DAE \cite{vconv} & Voxels & 6912 & 75.50 & 80.50 \\
3D-GAN \cite{3dgan} & Voxels & 7168 & 83.30 & 91.00 \\
VIP-GAN \cite{innergan} & Images & 4096 & \textbf{91.98} & 94.05\\
\hline
Latent-GAN \cite{AchlioptasDMG:17} & Points & 512 & 85.70 & 95.40\\
\hline
Our (Folding module only) & Points & 512 & 88.40 & 94.40 \\
Our (with fixed 2D lattice) &  Points & 512 & 84.27 \ & 92.11 \\

 Our (PointNet++ with adjacency~\eqref{eq:adj_filtering}) &  Points & 512 & 89.63 \ & 95.69\\

 Our (PointNet++ with Laplacian~\eqref{eq:lap_filter}) &  Points & 512 & 89.65 \ & 95.91 \\

Our (with graph-adjacency~\eqref{eq:adj_filtering}) &  Points & 512 & 89.67 \ & 95.63\\
Our (with graph-Laplacian~\eqref{eq:lap_filter}) &  Points & 512 & 89.55 \ & \textbf{95.93}\\
\hline
\end{tabular}
\end{center}
\caption{\label{tab:category_classification}\textbf{Comparison of classification accuracies on ModelNet10 and ModelNet40}. The proposed networks achieve the best classification accuracy in ModelNet10. VIP-GAN achieves the best performance in ModelNet40, but it uses a longer code and uses images as inputs. All methods use the same transfer-classification routine.}
\end{table*}

\subsubsection{Category classification}
We first consider the category classification. There are $10$ and $40$ categories in ModelNet10 (MN10) and ModelNet40 (MN40), respectively. Some categories include airplane, chair, and bed. We follow the same experimental setting in~\cite{AchlioptasDMG:17, YangFST:18, AtlasNet} to train the proposed networks with 3D point clouds sampled from the dataset of ShapeNet.  After the proposed networks are trained, we fix the parameters and run the networks to obtain the codes for 3D point clouds sampled from ModelNet10 and ModelNet40. We then split the codes into train/test datasets as the work in \cite{2016pointnet, YangFST:18, 3dgan, AtlasNet}. ModelNet40 (MN40) contains 9843/2468 3D mesh models in train/test datasets; ModelNet10 (MN10) contains 3991/909 3D mesh models in train/test datasets. Each point cloud contains 2048 3D points with their $[x, y, z]$ coordinates. We use a linear SVM as the classifier to recognize the corresponding category of a 3D point cloud. We compare the proposed networks with the other state-of-the-art unsupervised-learning models, including SPH, LFD, T-L Network, VConv-DAE, 3D GAN, VIP-GAN, LatentGAN~\cite{sph, lfd, tlnetwork, vconv, 3dgan, AchlioptasDMG:17}.  

Table~\ref{tab:category_classification} shows the comparison of classification accuracies. We see that (i) the proposed networks achieve the best classification accuracy in ModelNet10 and achieves the second best performance in ModelNet40; however, VIP-GAN uses a longer code (4096 vs. 512 in ours) and has a different input modality; it uses images, instead of 3D point clouds as inputs; (ii) in both datasets, two types of learnable graph filters~\eqref{eq:adj_filtering} and ~\eqref{eq:lap_filter} provide similar performances; (iii) in both datasets, when we replace the learnable graph topology by 
the fixed 2D lattice, the performance degenerates, indicating a naive graph topology
would provide misleading prior and make the training process harder; (iv) in both datasets, the proposed networks outperform the LatentGAN, which adopts 3D point clouds as inputs;  (v) in both datasets, the proposed networks perform better when the graph-filtering module is used to refine the reconstruction. The reason is that graph-topology-inference-module pushes the code to preserve both 3D coordinate information and pairwise relationship information, such that the decoder can reconstruct the 3D point clouds and the corresponding graph topology; and (vi) when we replace the encoder from PointNet to PointNet++, the classification performance does not significantly improved. The intuition could be that a powerful decoder is able to push the codewords to preserve informative features and release the pressure of using a sophisticated encoder.

\begin{table*}[htb!]
  \begin{center}
    \begin{tabular}{ c  c  c  c  c   }
      \hline 
      \multicolumn{5}{c}{\textbf{Airplane}  ( Train/Test :  551/142  , Number of classes: 9 )}\\
      \hline   
      Swept wings & Swept wings & Swept wings & Straight wings & Straight wings  \\
        back engines & two engines & four engines & two engines & four engines  \\
        \hline
      $\quad$&
      $\quad$&
      $\quad$&
      $\quad$&
      $\quad$\\
      \includegraphics[trim={12cm 12cm 12cm 12cm}, clip=true , scale=0.043] {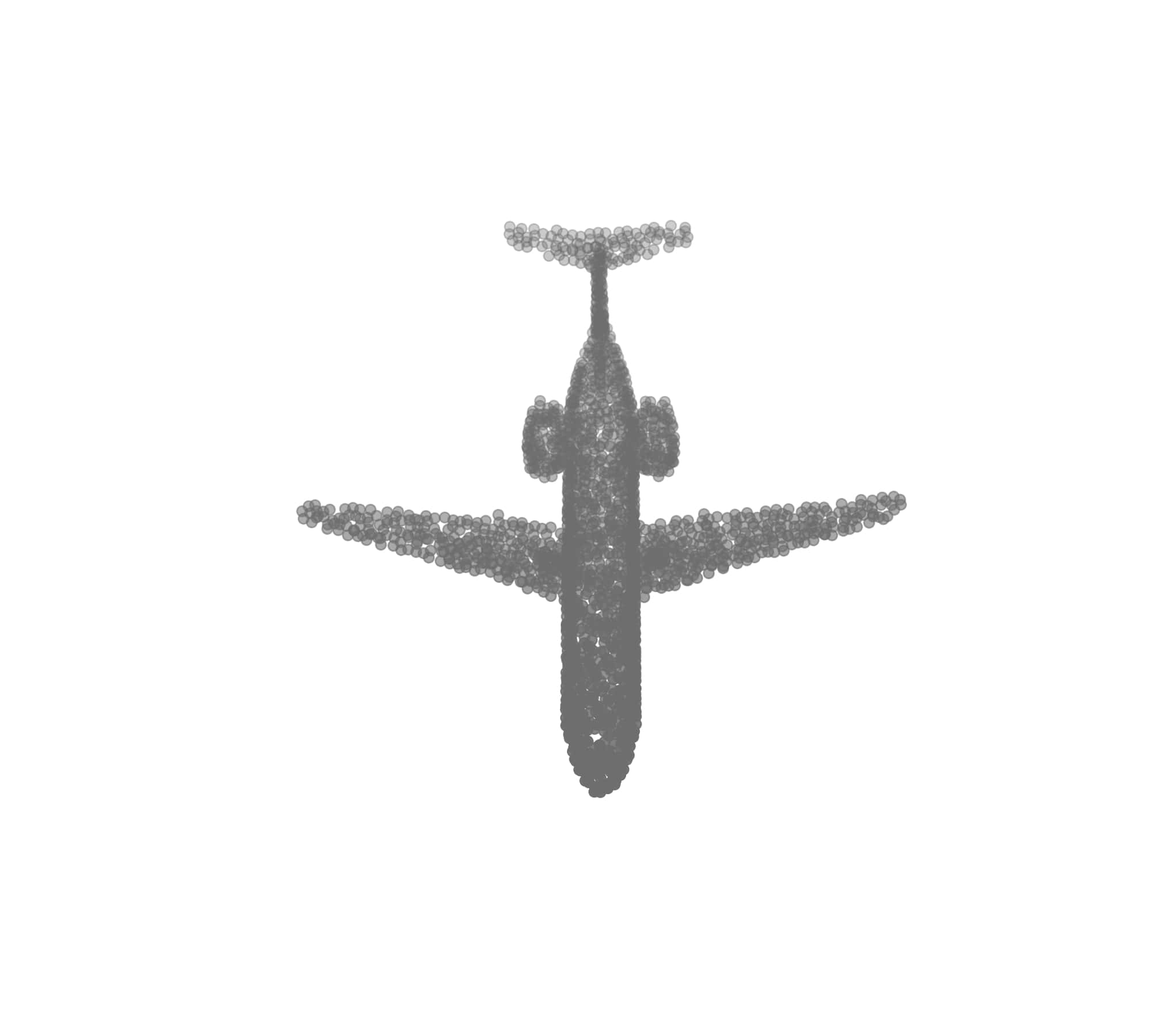} & 
      \includegraphics[trim={12cm 12cm 12cm 12cm}, clip=true , scale=0.043] {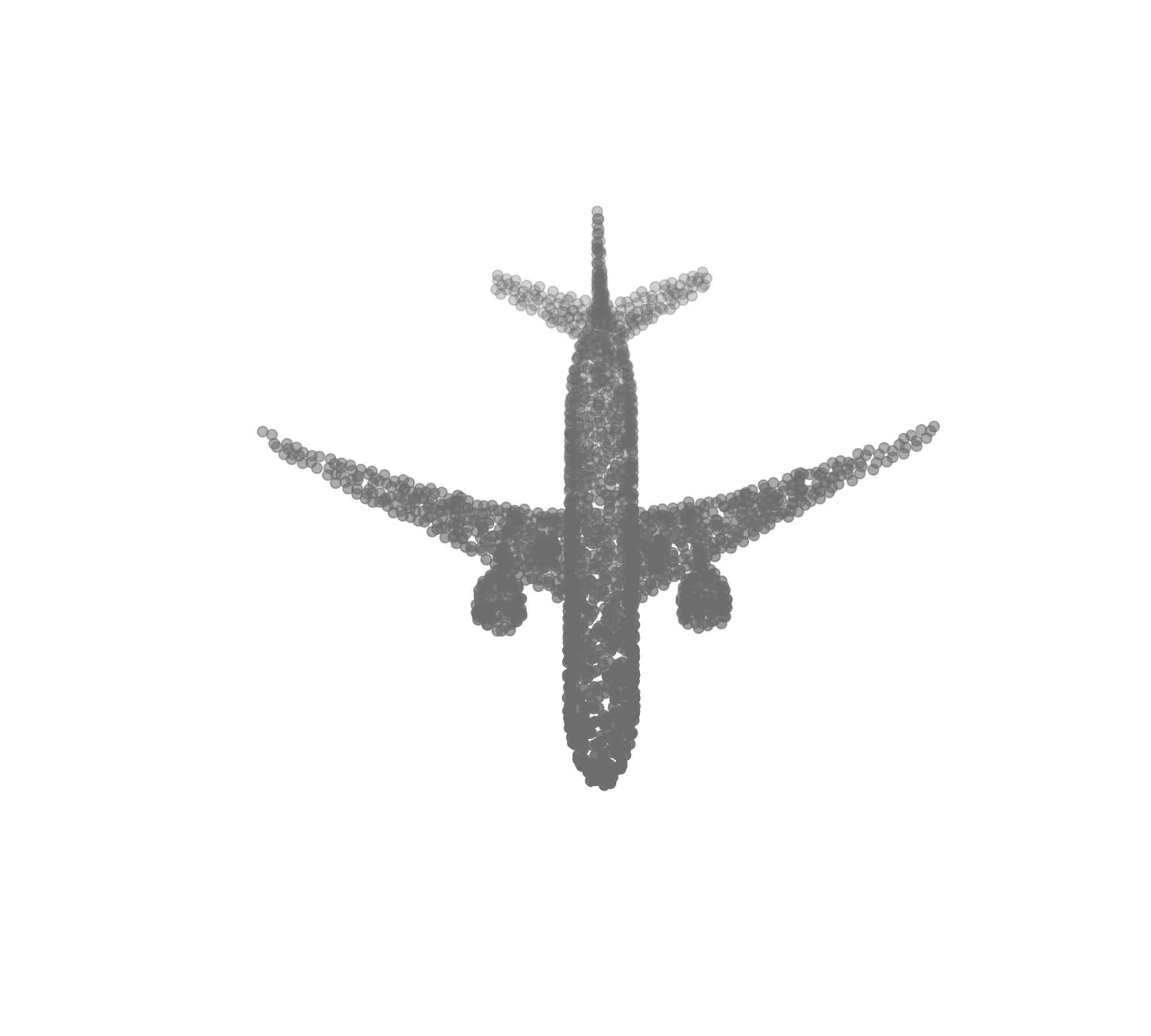} & 
      \includegraphics[trim={12cm 12cm 10cm 12cm}, clip=true , scale=0.043] {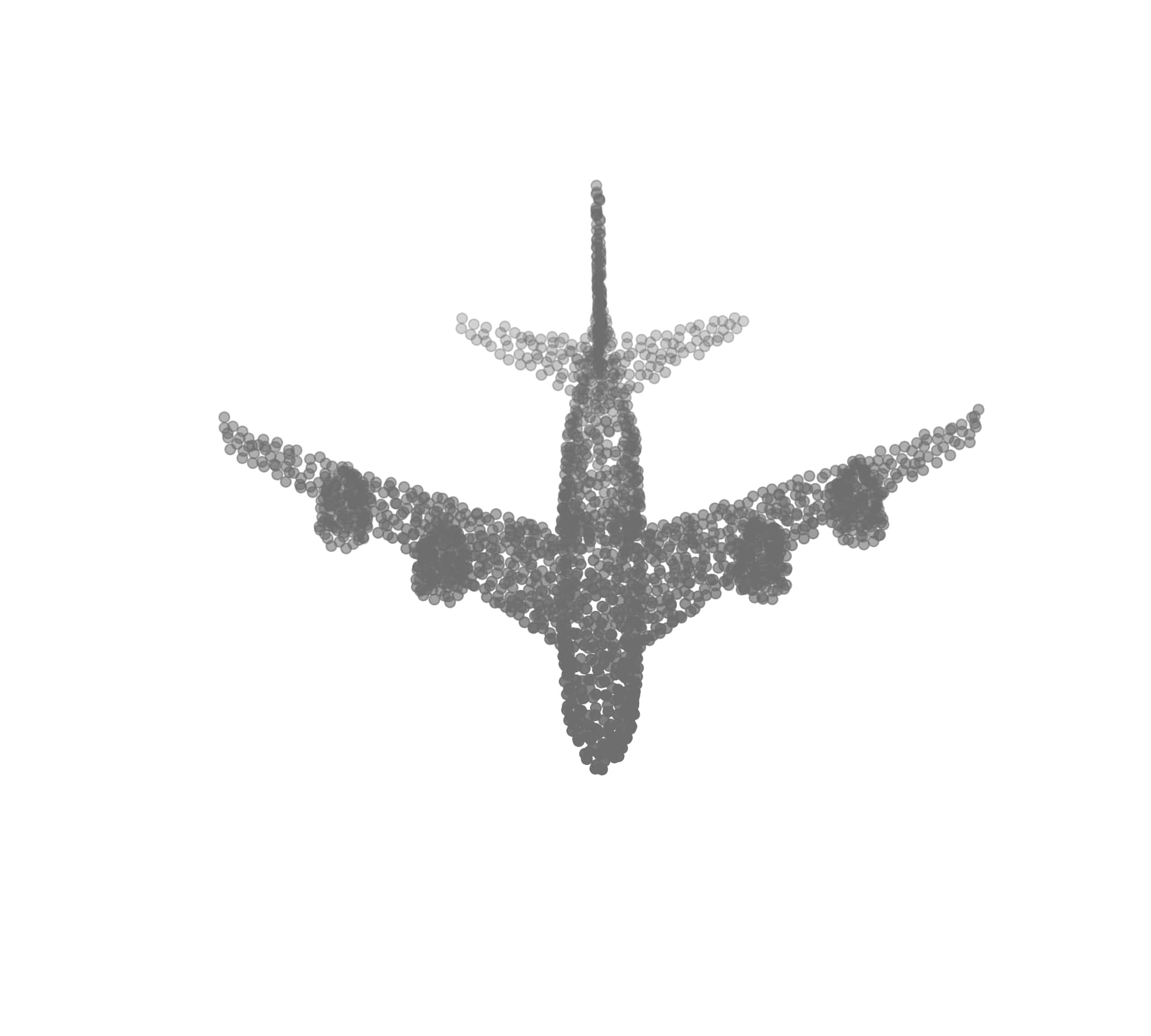} &
      \includegraphics[trim={11cm 12cm 10cm 12cm}, clip=true , scale=0.043] {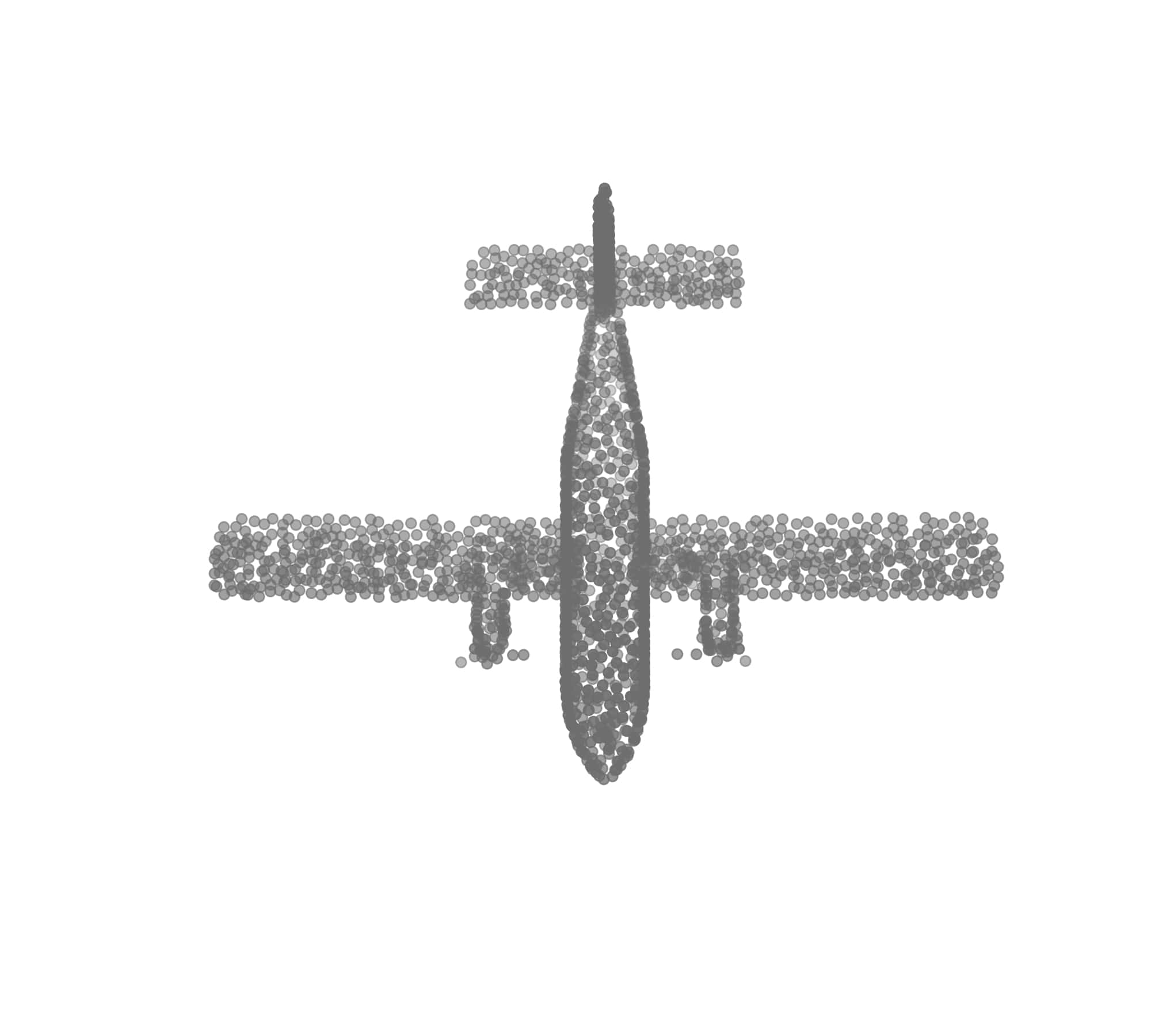} & 
      \includegraphics[trim={11cm 12cm 10.5cm 12cm}, clip=true , scale=0.043] {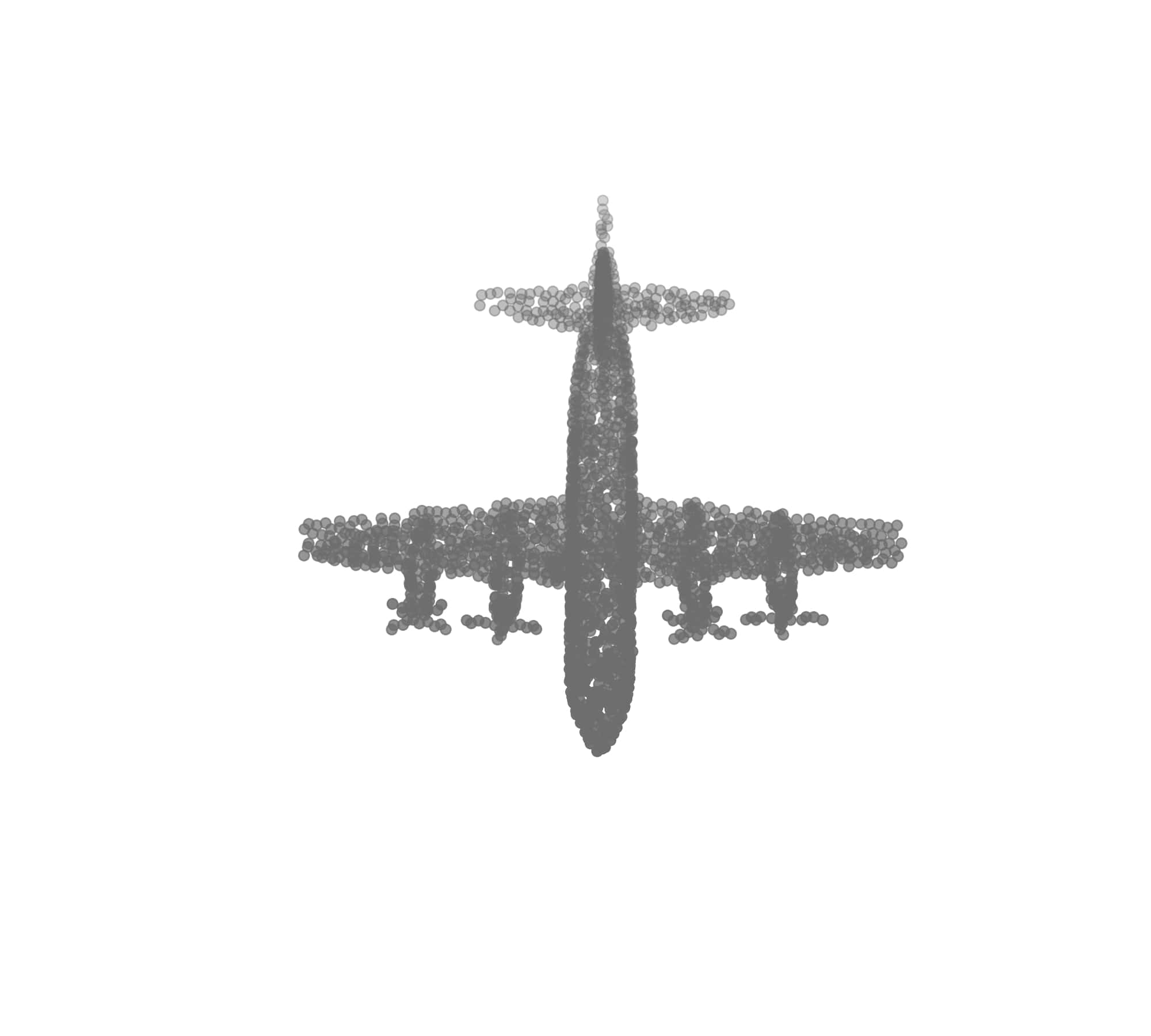} \\
      \hline
      
      \multicolumn{5}{c}{\textbf{Bed}  ( Train/Test :  433/112  , Number of classes: 8 )}\\
      \hline   
      Boardless & Front board & Front board & Double boards & Double boards  \\
      with pillows & no pillows & with pillows & no pillows & with pillows  \\
      \hline
      $\quad$&
      $\quad$&
      $\quad$&
      $\quad$&
      $\quad$\\
      \includegraphics[trim={12cm 12cm 12cm 12cm}, clip=true , scale=0.043] {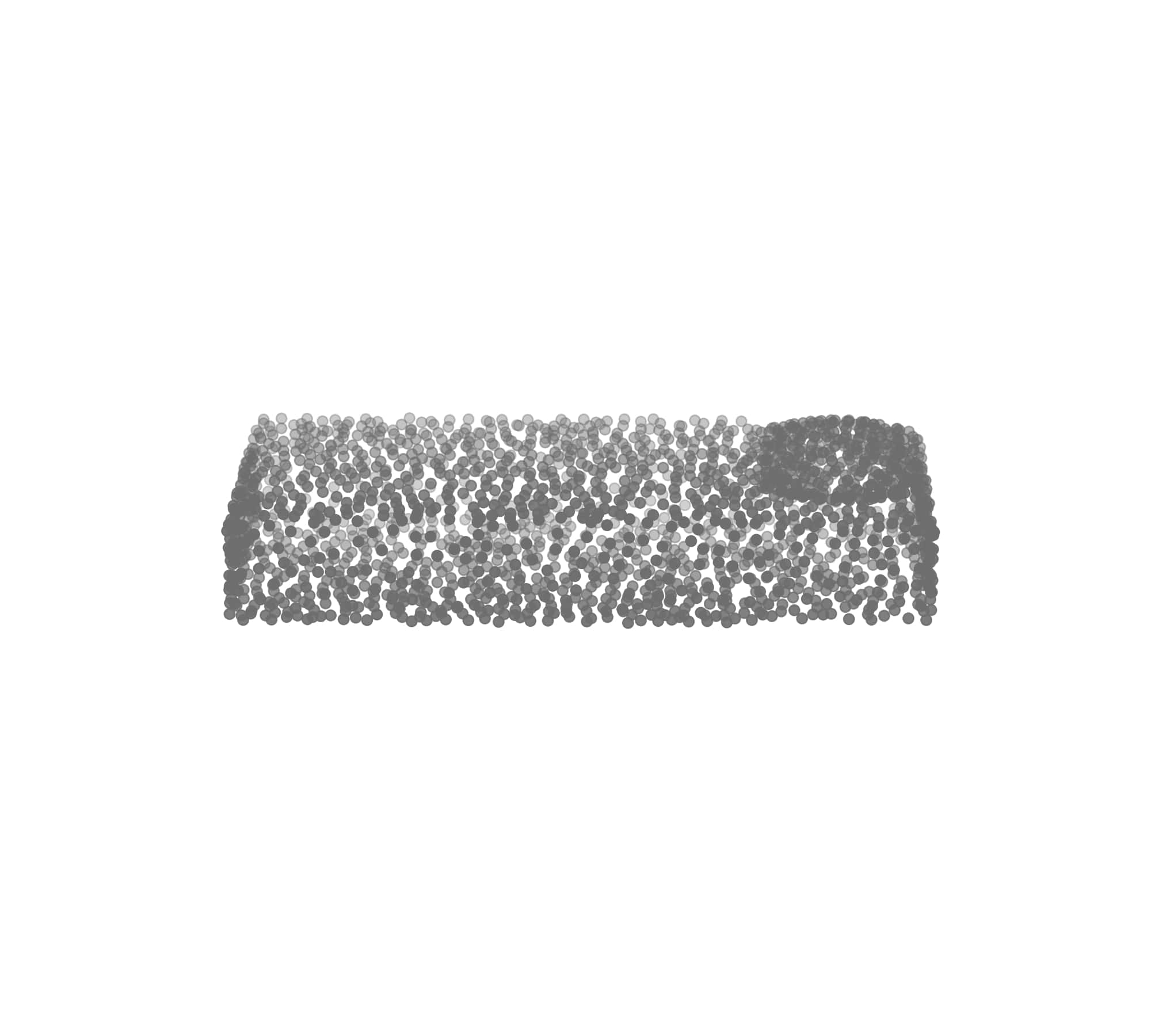} & 
      \includegraphics[trim={12cm 12cm 12cm 12cm}, clip=true , scale=0.043] {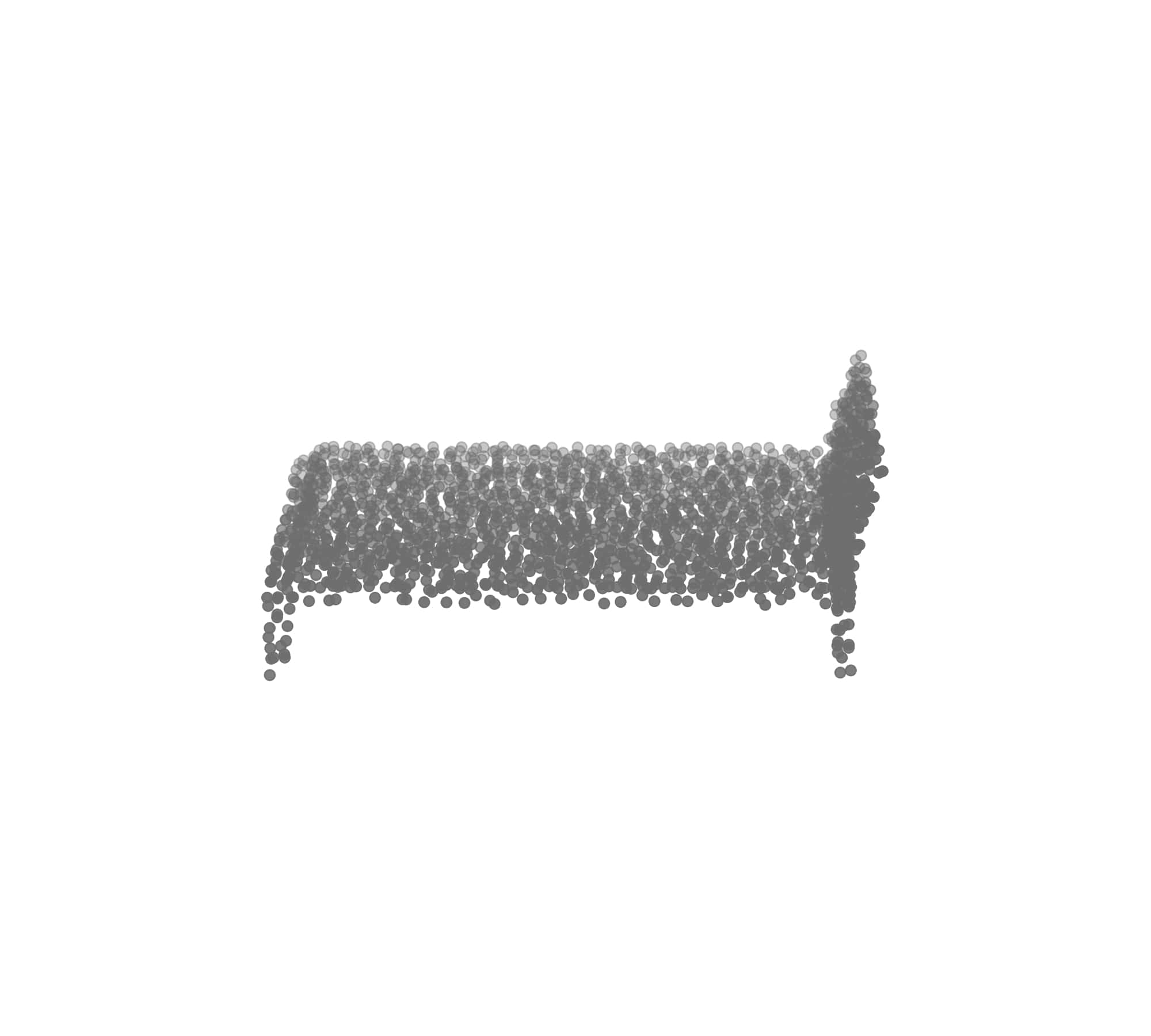} & 
      \includegraphics[trim={10.5cm 12cm 10.5cm 12cm}, clip=true , scale=0.043] {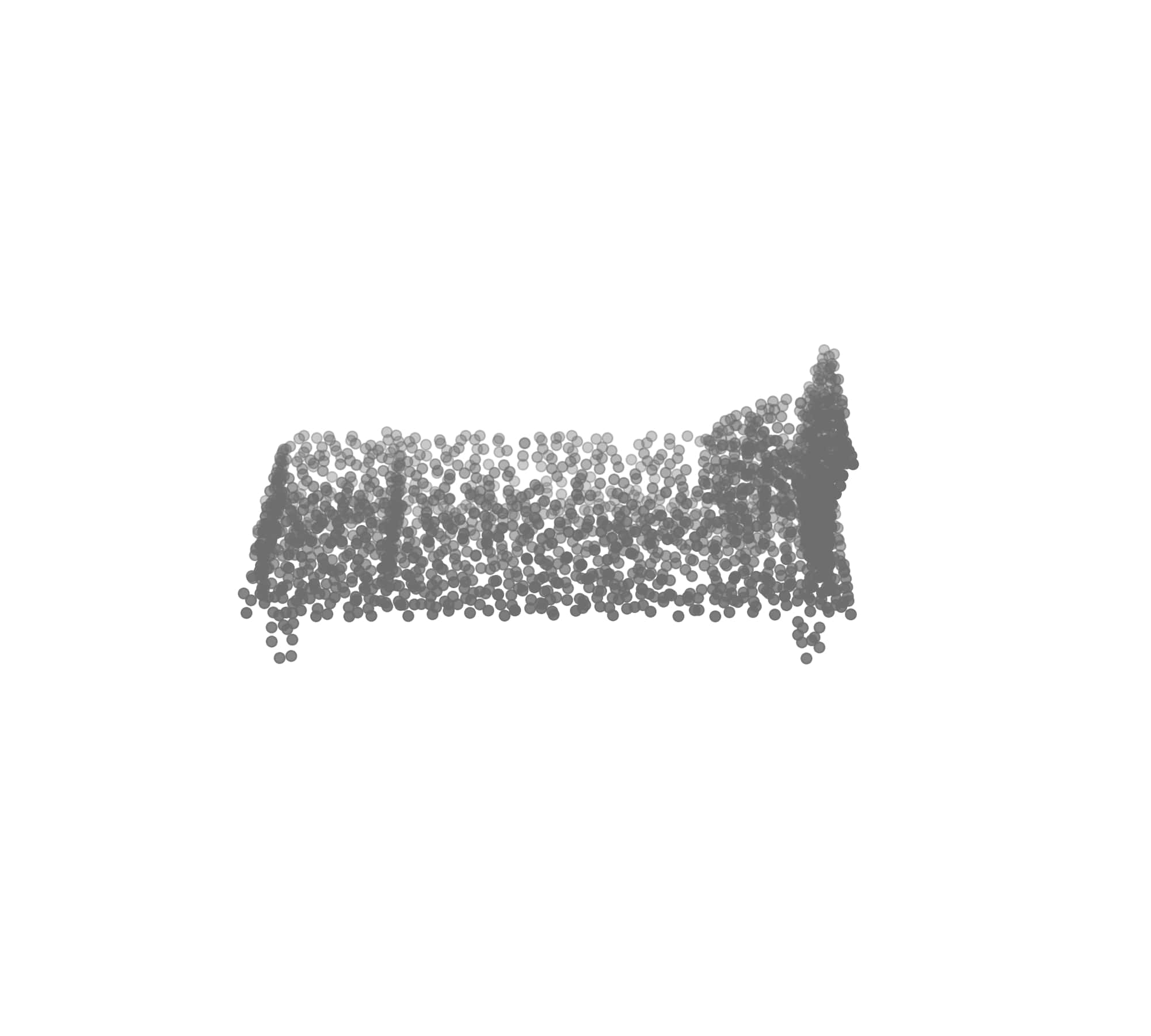} &
      \includegraphics[trim={10.5cm 12cm 10.5cm 12cm}, clip=true , scale=0.043] {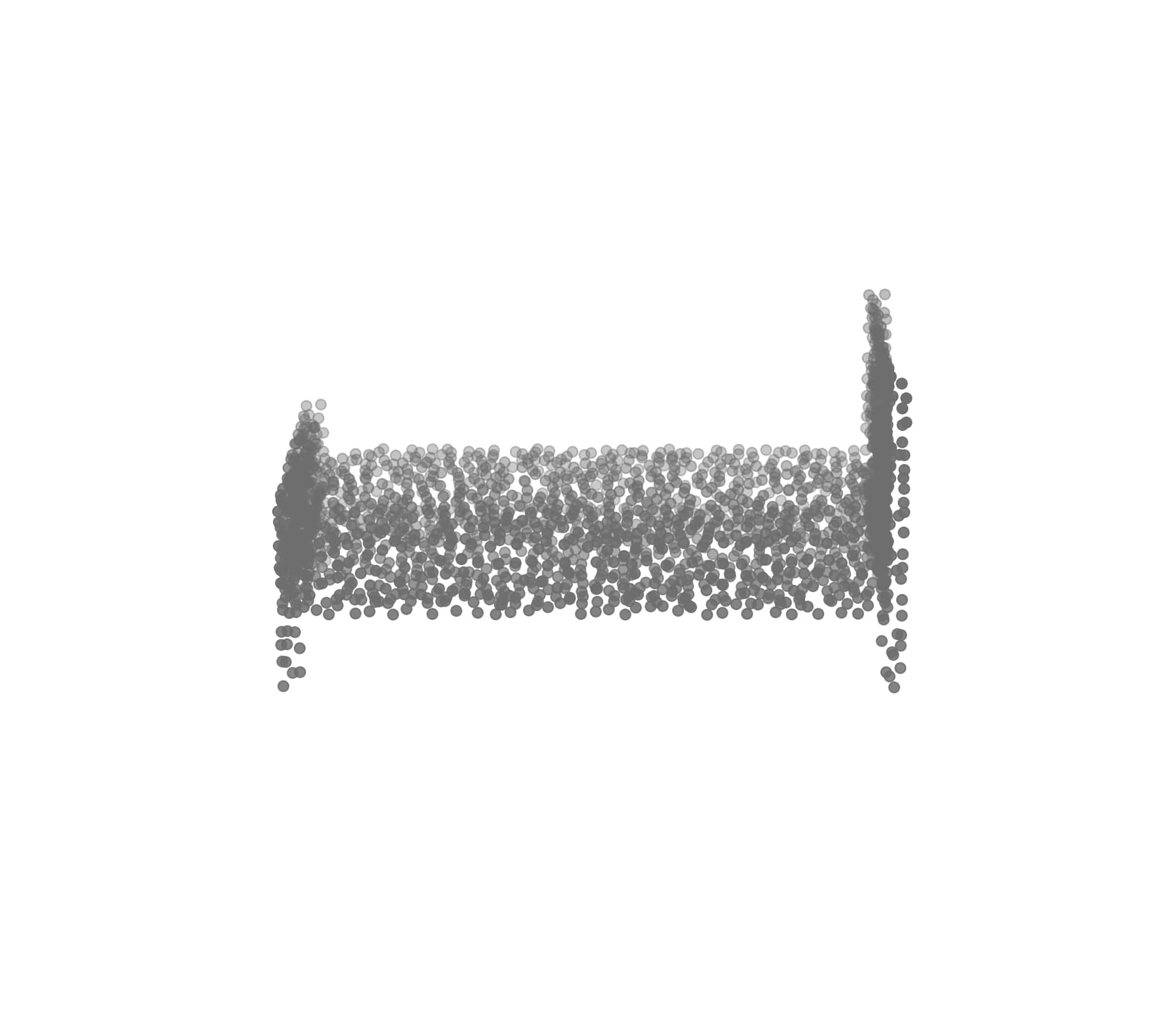} & 
      \includegraphics[trim={10.5cm 12cm 10.5cm 12cm}, clip=true , scale=0.043] {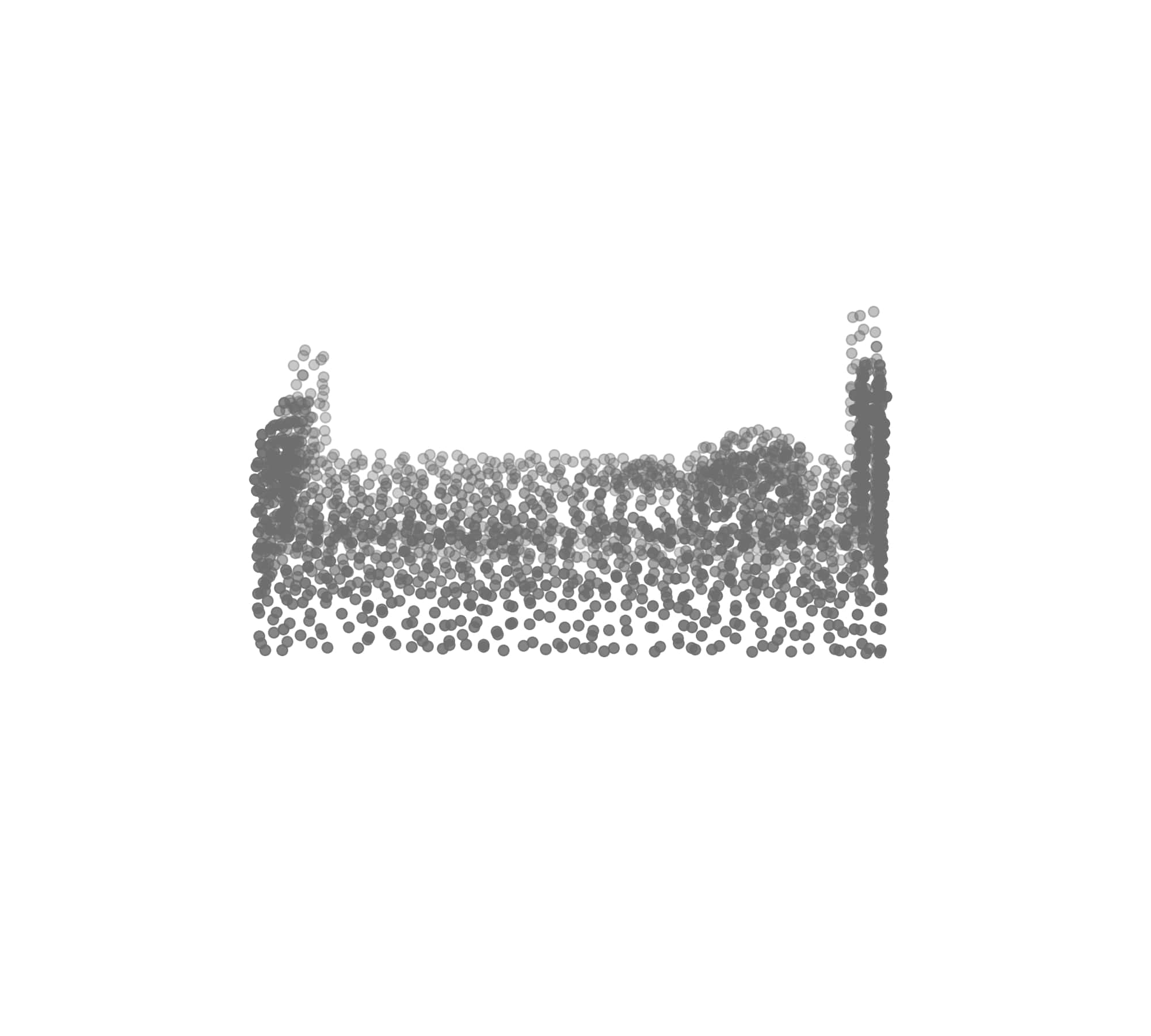} \\  
      \hline
      
      \multicolumn{5}{c}{\textbf{Chair} ( Train/Test :  707/187  , Number of classes: 16 )}\\
      \hline   
      Cantilever & Folding chair & Swivel chair & Armless chair & Wing chair  \\
      \hline  \\
      \includegraphics[trim={12cm 11cm 12cm 10cm}, clip=true , scale=0.04] {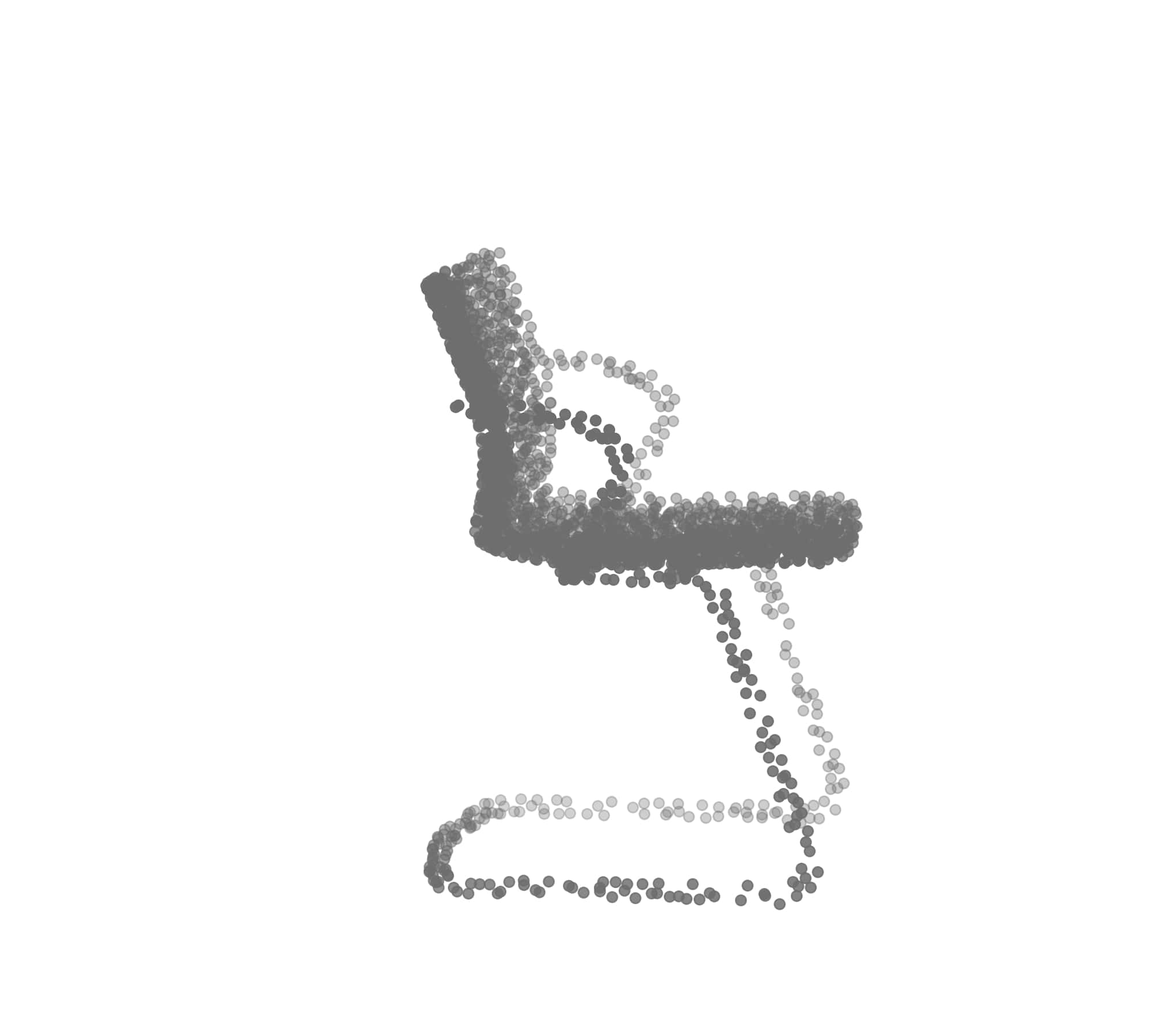} & 
      \includegraphics[trim={12cm 11cm 12cm 10cm}, clip=true , scale=0.04] {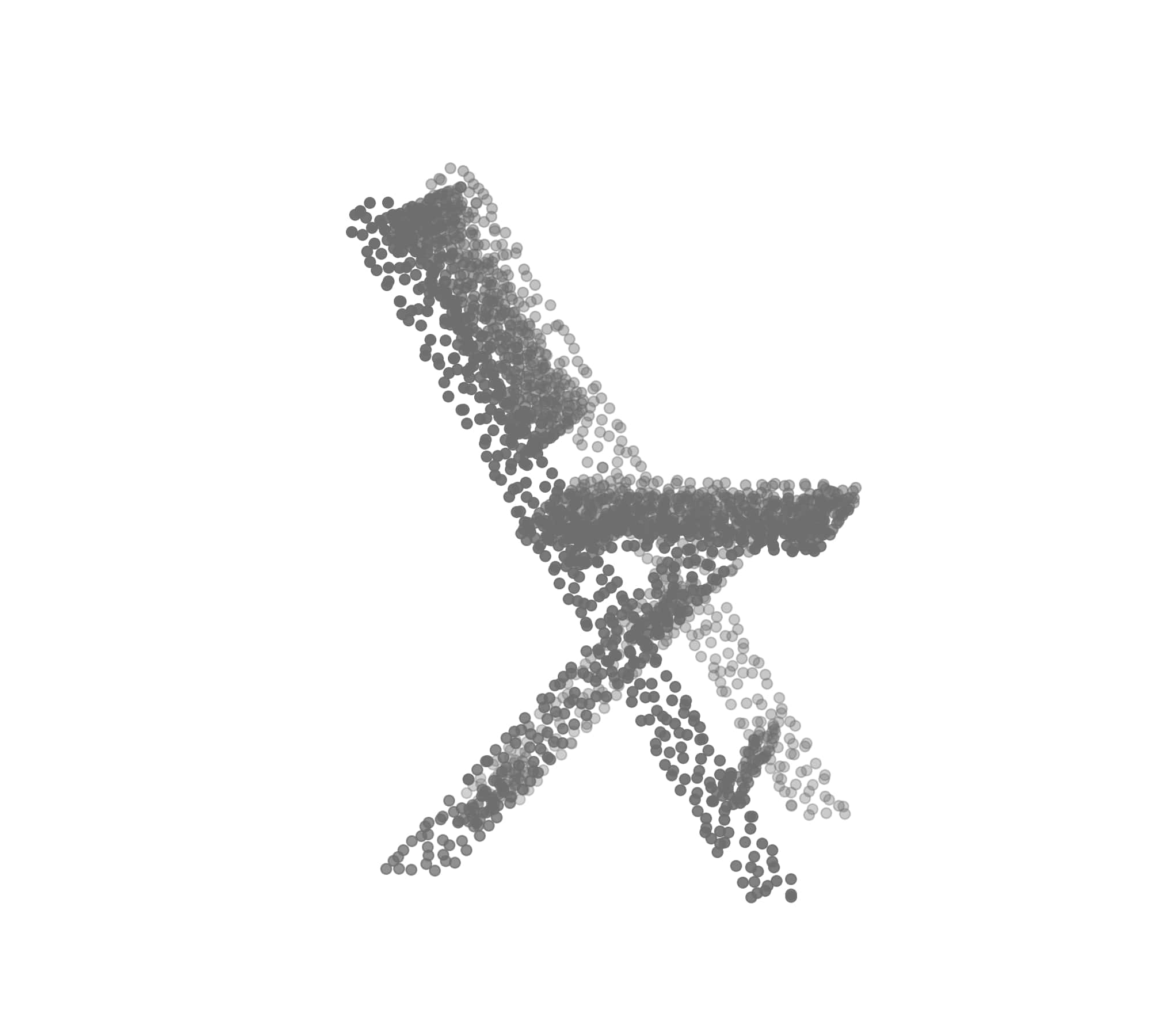} & 
      \includegraphics[trim={10.5cm 11cm 10.5cm 10cm}, clip=true , scale=0.04] {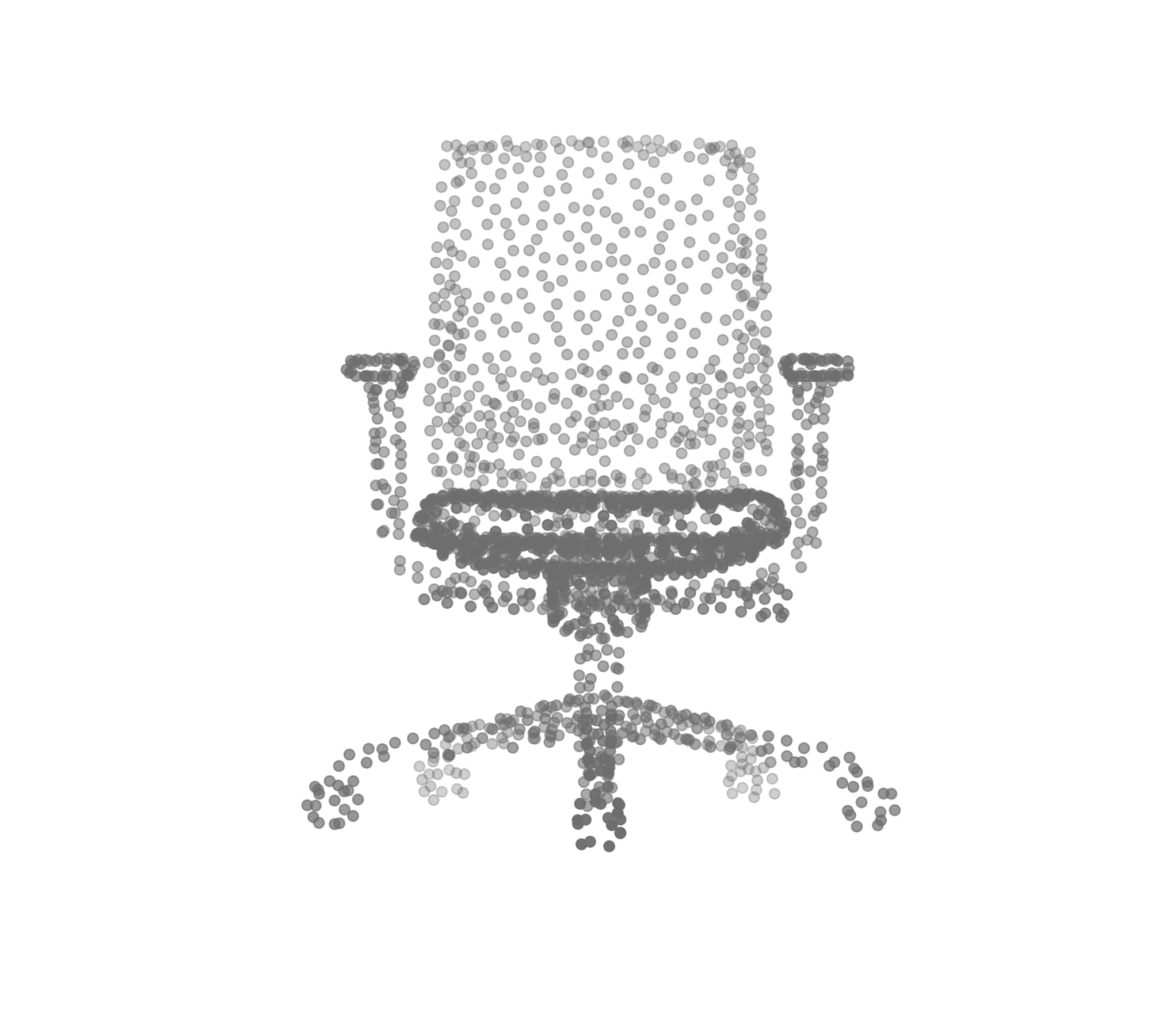} &
      \includegraphics[trim={10.5cm 11cm 10.5cm 10cm}, clip=true , scale=0.04] {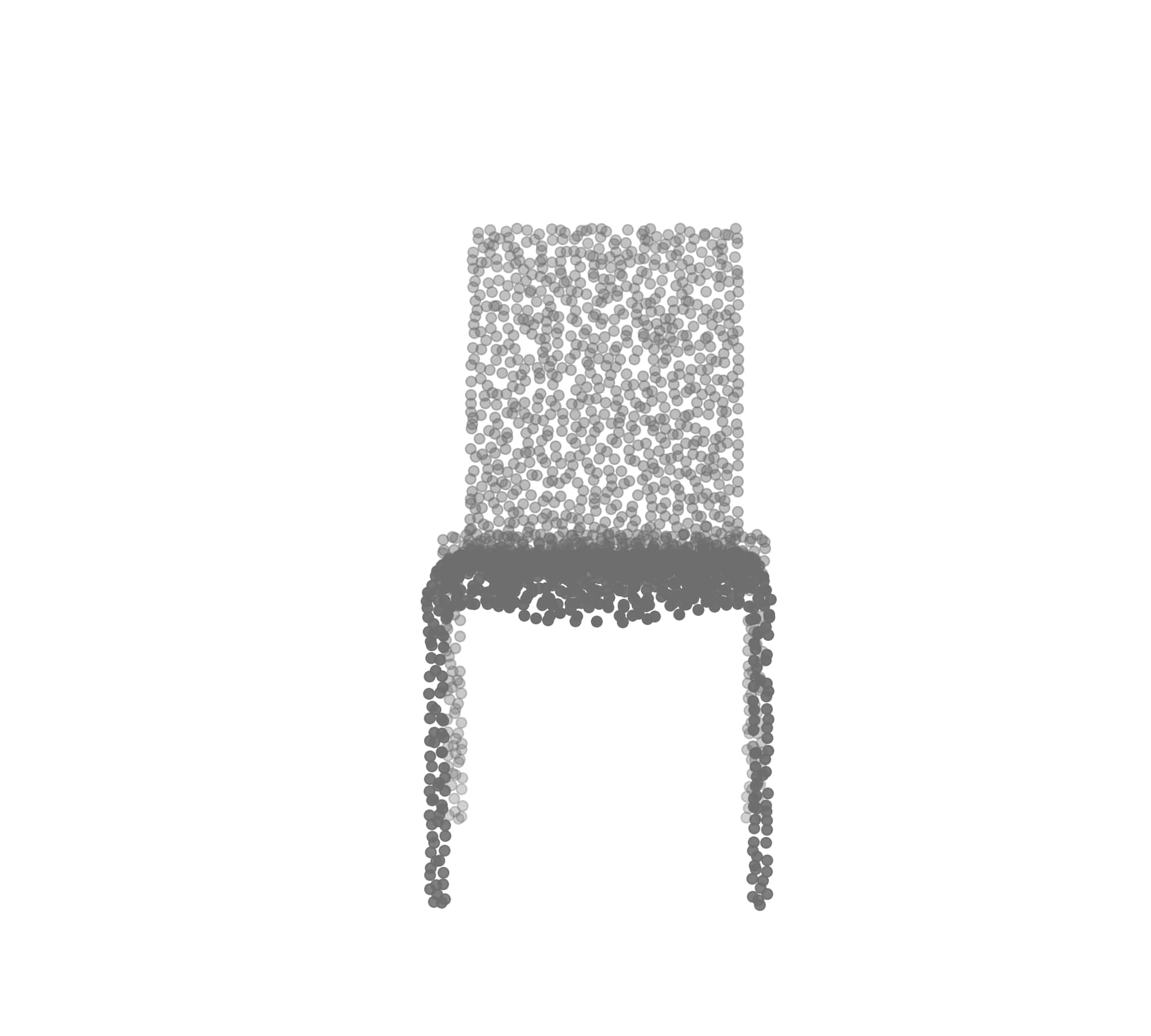} & 
      \includegraphics[trim={10.5cm 11cm 10.5cm 10cm}, clip=true , scale=0.04] {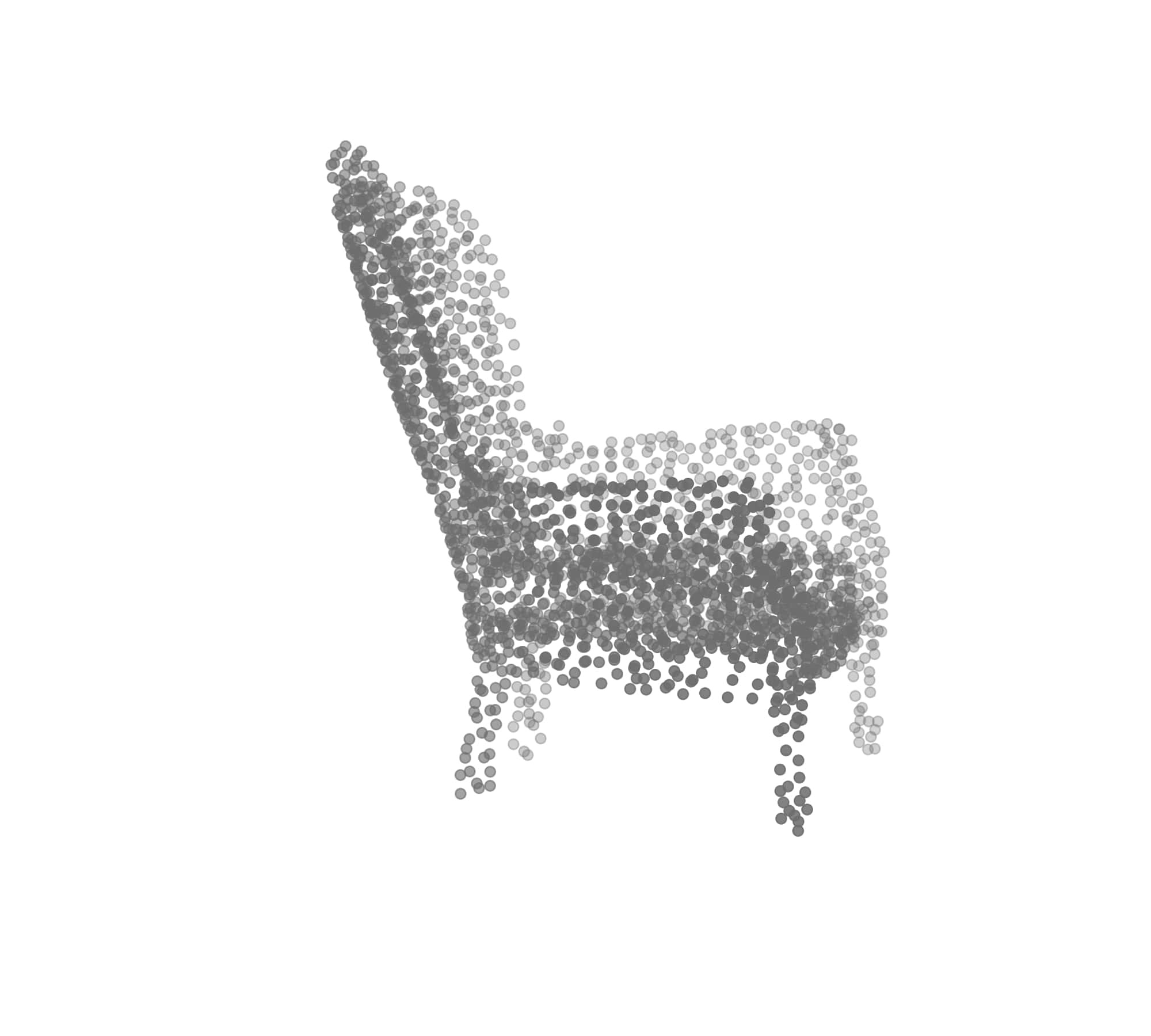} \\ 
      \hline
      
      \multicolumn{5}{c}{\textbf{Monitor}  ( Train/Test :  400/104  , Number of classes: 9) }\\
      \hline   
      V-shape base & Half-circle base & Half-ring base & Baseless & Rectangular base \\
      \hline  
      $\quad$&
      $\quad$&
      $\quad$&
      $\quad$&
      $\quad$\\
      \includegraphics[trim={12cm 12cm 12cm 12cm}, clip=true , scale=0.043] {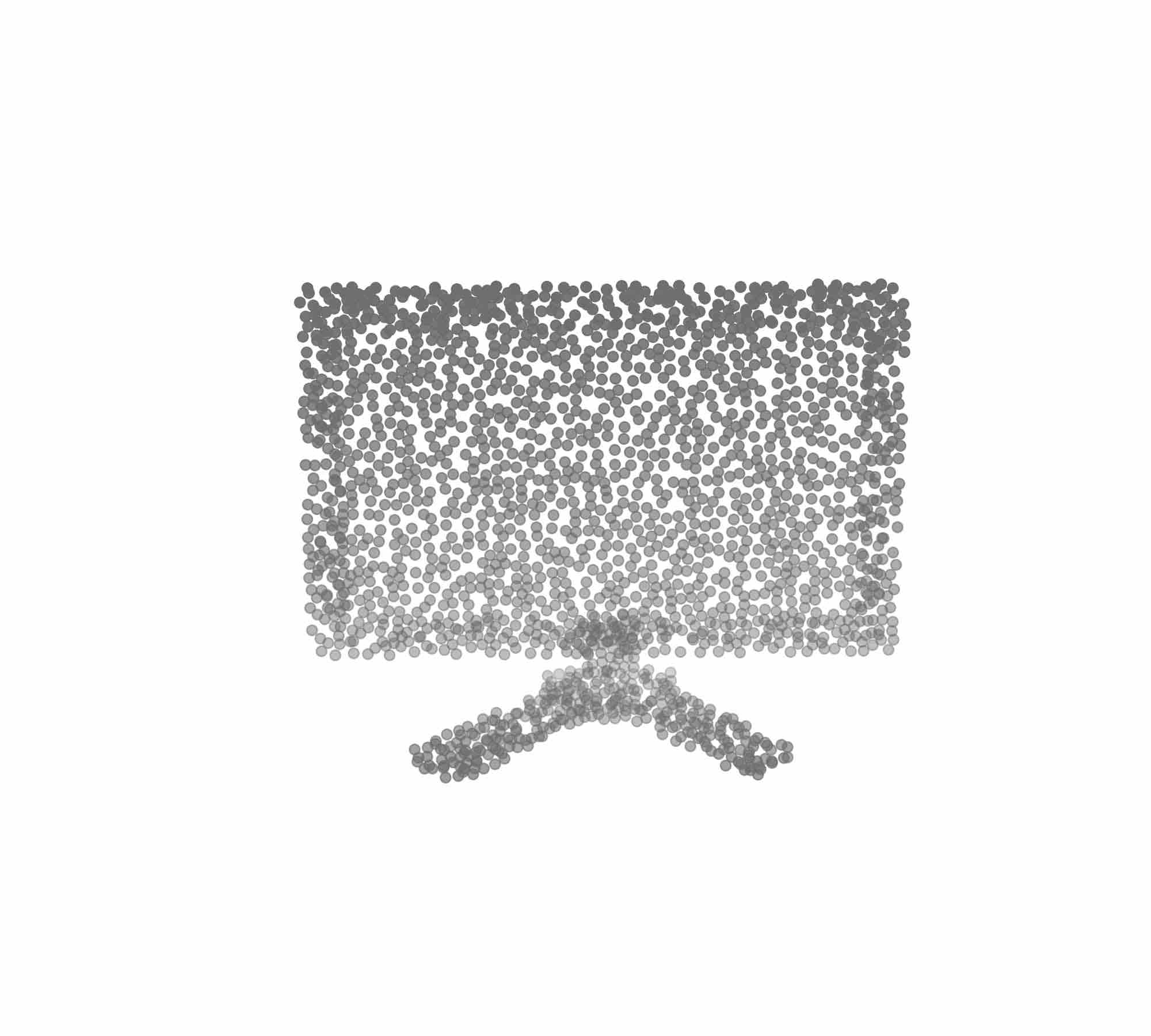} & 
      \includegraphics[trim={12cm 12cm 12cm 12cm}, clip=true , scale=0.043] {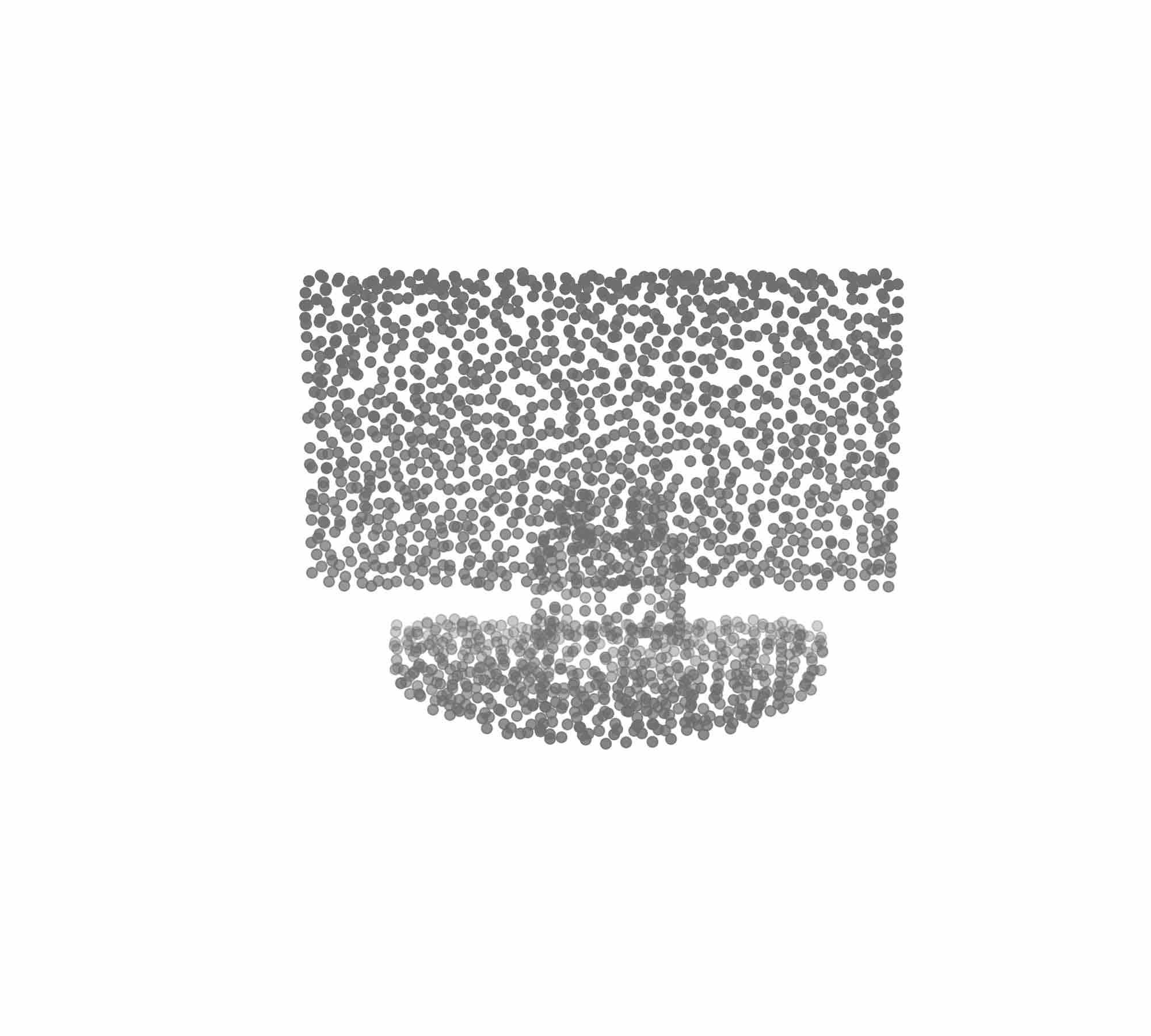} & 
      \includegraphics[trim={10.5cm 12cm 10.5cm 12cm}, clip=true , scale=0.043] {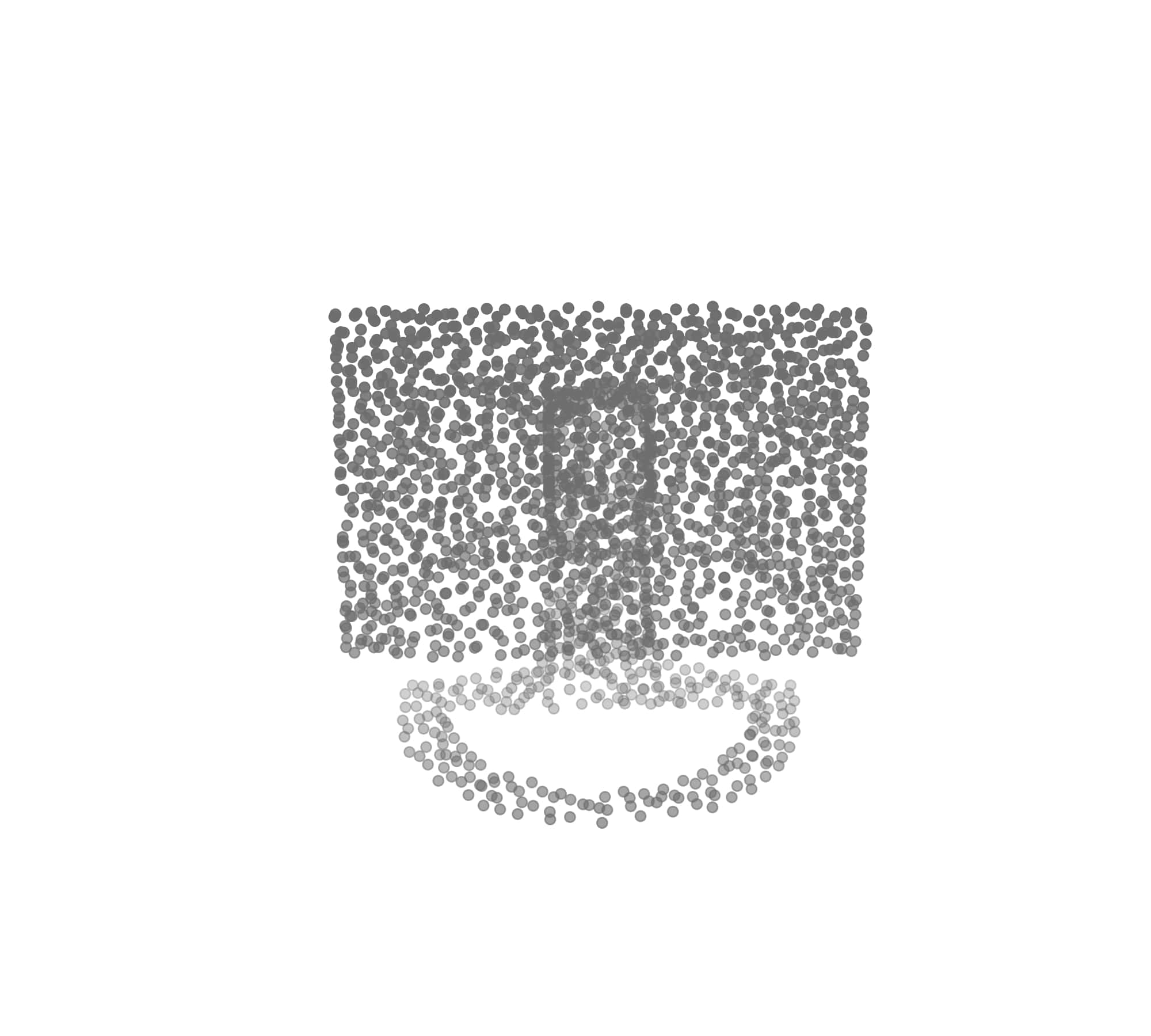} &
      \includegraphics[trim={10.5cm 12cm 10.5cm 12cm}, clip=true , scale=0.043] {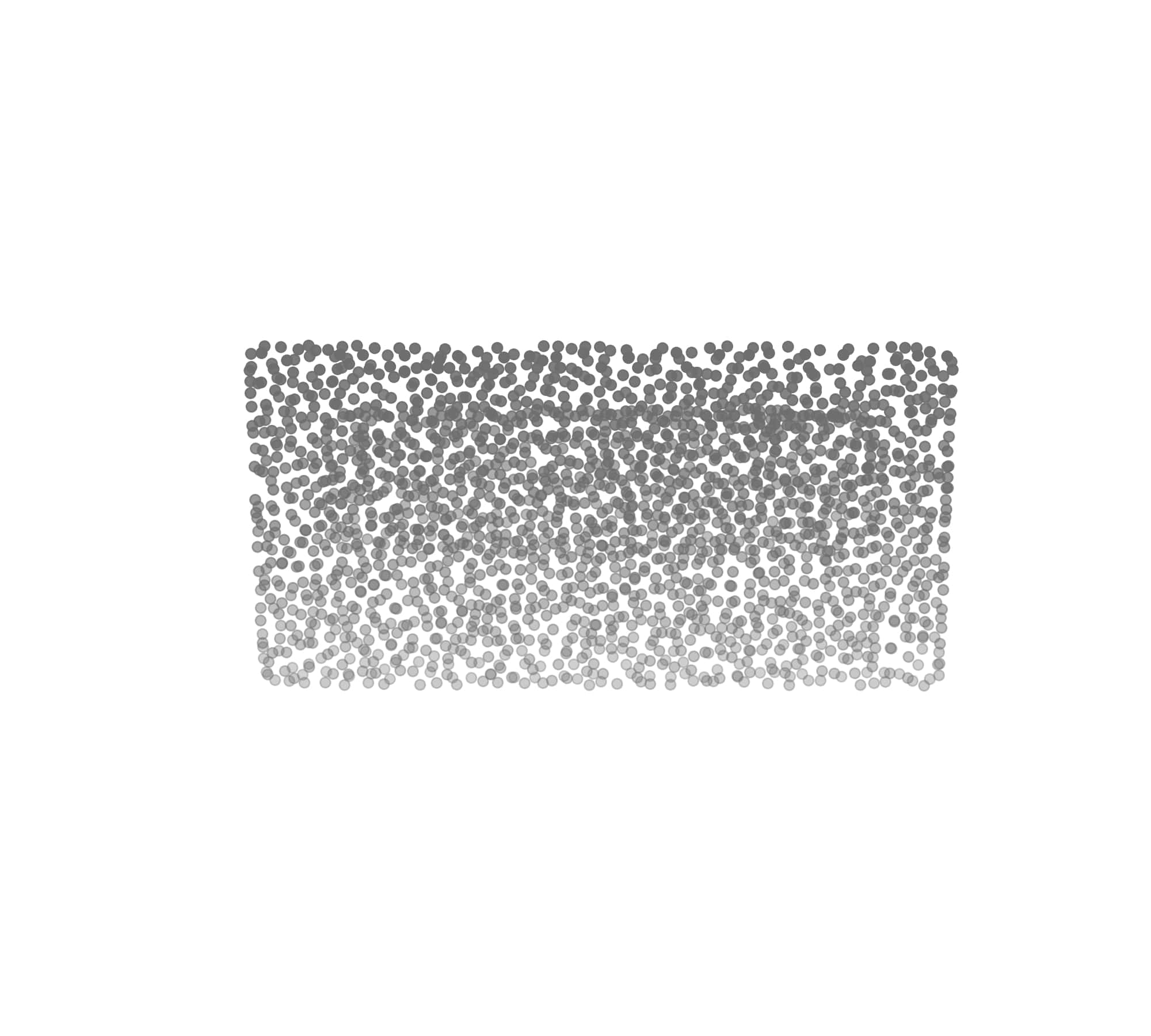} & 
      \includegraphics[trim={10.5cm 12cm 10.5cm 12cm}, clip=true , scale=0.043] {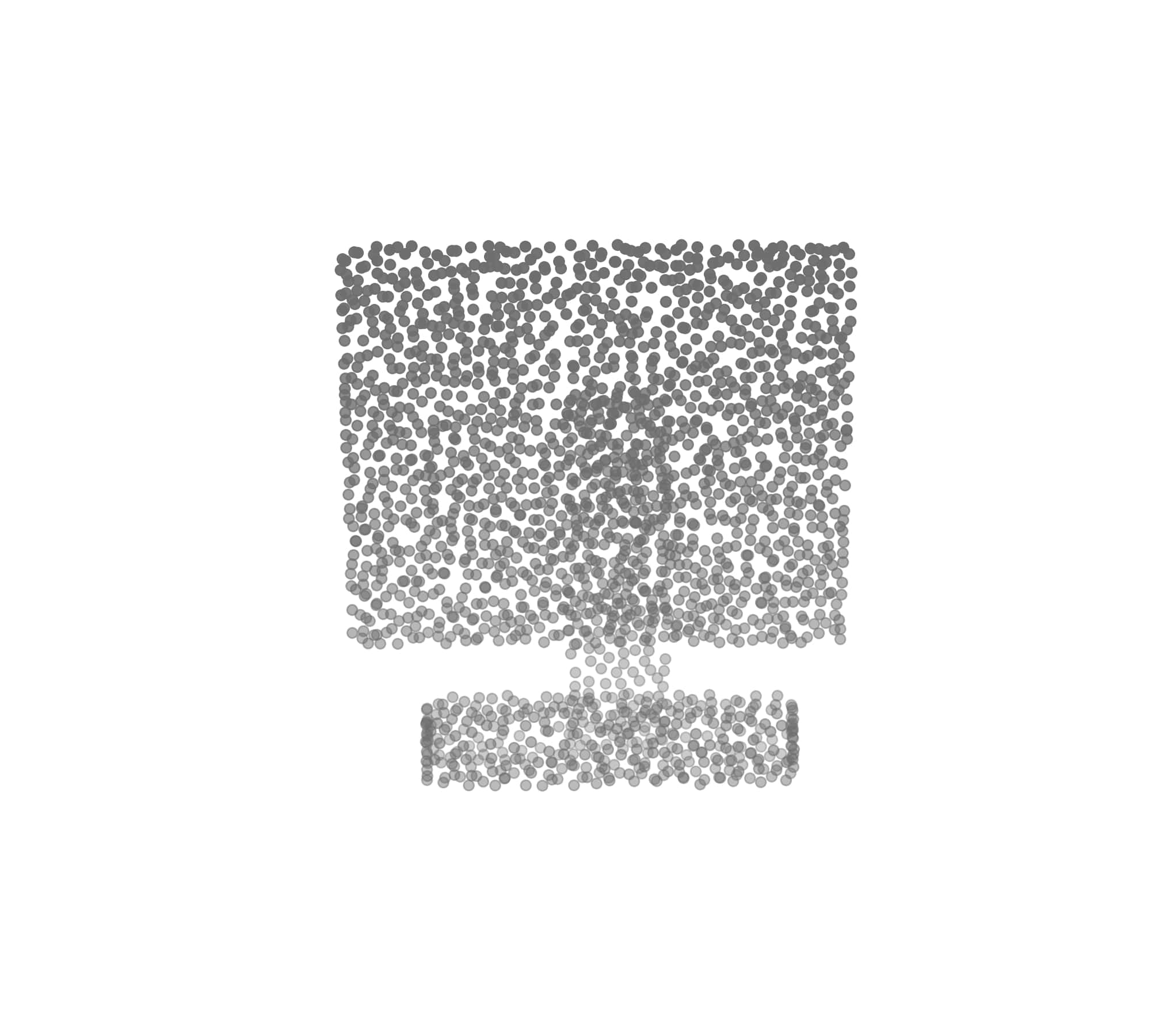} \\  
      \hline
    \end{tabular}
  \end{center}
  \caption{\label{tab:fine_grained_figures}\textbf{Dataset with subcategory labels}. We categorize airplanes, beds, chairs and monitors into $9$, $8$, $16$, $9$ classes, respectively. This  Dataset could be used for fine-grained classification, visualization, clustering, and many other tasks.}
\end{table*}

\subsubsection{Subcategory Classification}
Recognizing objects across various categories mainly requires global features, instead of local features. For example, even a model cannot capture engines of airplanes shown in Table~\ref{tab:fine_grained_figures}, the classifier can still correctly label it as an airplane. Here we consider validating the ability of the proposed networks to extract detailed features. For objects in the same category, we further group them into several subcategories based on some local shapes. We then train the models to classify subcategories; called the subcategory classification. To achieve a high classification accuracy in this task, a model has to capture local features.

We choose four categories in ModelNet40, including airplanes, beds, chairs, and monitors, and manually label subcategories for them. Table~\ref{tab:fine_grained_figures}  illustrates the subcategory dataset.  Here we label point clouds only based on 3D shapes, without considering their function or other information.  Airplanes are grouped into $9$ subcategories based on the shape of the wings and the number of engines. $551$ and $142$ airplanes are used for training and testing, respectively. Beds are grouped into $8$ subcategories based on the shape of the board and the presence of pillows. $433$ and $112$ beds are used for training and testing, respectively. Chairs are grouped into $16$ subcategories. $707$ and $187$ chairs are used for training and testing, respectively. Monitors are grouped into $9$ subcategories based on the shape of bases. $400$ and $104$ monitors are used for training and testing, respectively.

We compare the proposed networks with  LatentGAN, AtlasNet~\cite{AtlasNet}, 3DCapsNet~\cite{3dcapsule},
 and a supervised model, PointNet. Table~\ref{tab:sub_classification} shows the  comparison of classification accuracies in the subcategory dataset. We see that  (i) the proposed networks achieve significantly higher classification accuracies than LatentGAN in all four categories. LatentGAN fails to classify subcategories since fully-connected layers do not commit a dedicated design to exploit local geometric structures or point relationships; (ii) as an unsupervised model, the proposed networks even work better than PointNet in three categories, indicating a strong generalization ability of the proposed networks. Note that PointNet is trained with both cross-entropy loss and the hinge loss. The one with the hinge loss
is slightly worse; (iii) the proposed networks perform better when the graph-filtering module is used to refine the reconstruction. The results indicate that graph filtering has a significant advantage in preserving local geometric structures; and (iv) two types of graph filters~\eqref{eq:adj_filtering} and ~\eqref{eq:lap_filter} provide similar classification performances.

\begin{table}[htb!]
	\begin{center}
		\begin{tabular}{ c|cccc|c } 
			\hline
			Category & Airplane & Bed & Chair & Monitor & S?\\
			\hline
			 \# subcategories & 9 & 8 & 16 & 9 & - \\
			 \hline
			PointNet  (Softmax) & 81.69 & {\bf  79.46} & 62.57 & 72.12 & Y \\
			
			PointNet (Hinge loss) & 78.35 & 77.40 & 67.39 & 71.88 & Y \\
				
			LatentGAN & 13.38 & 51.78 & 48.13 & 42.31 & N \\

			AtlasNet & 83.57 & 72.23 & 78.54 & 71.69 & N\\
			
			3DCapsNet & 83.61 & 72.84 & 78.59 & 75.23 & N \\
			\hline
			Our (Folding only) & 83.29 & 73.21 & 77.05 & 72.12 & N \\
			
			Our (with GA) & 83.80 & 75.00 & {\bf 81.28} & {\bf 76.92} & N\\
			Our (with GL) & {\bf 83.97} & 76.65 & 81.23 & 75.43 & N\\
			\hline
		\end{tabular}
	\end{center}
	\caption{\label{tab:sub_classification}\textbf{Comparison of transfer classification accuracies on the subcategory dataset.} Proposed networks significantly outperform the other unsupervised model and outperforms the supervised model, PointNet, in three categories. S stands for supervised model; GA stands for the graph-adjacency-matrix-based filtering~\eqref{eq:adj_filtering}; GL stands for the graph-Laplacian-matrix-based filtering~\eqref{eq:lap_filter}.}
\end{table}

\subsection{Visualization of Clustering}
We use the latent code produced by the proposed networks to represent a 3D point cloud and use t-SNE to reduce the dimensionality to $2$ for visualization~\cite{MaatenH:08}. Here we use the graph-adjacency-matrix-based filtering~\eqref{eq:adj_filtering} to implement the graph-filtering module. Figure~\ref{fig:clustering} shows the clustering performance on both training and testing sets of ModelNet10 where each point represents a 3D point cloud and the associated color represents its ground-truth category. We see that 3D point clouds with the same categories are clustered together, indicating the proposed networks encode similar 3D point clouds to similar codes.

\begin{figure}[thb]
  \begin{center}
    \begin{tabular}{cc}
    \includegraphics[width=0.5\columnwidth]{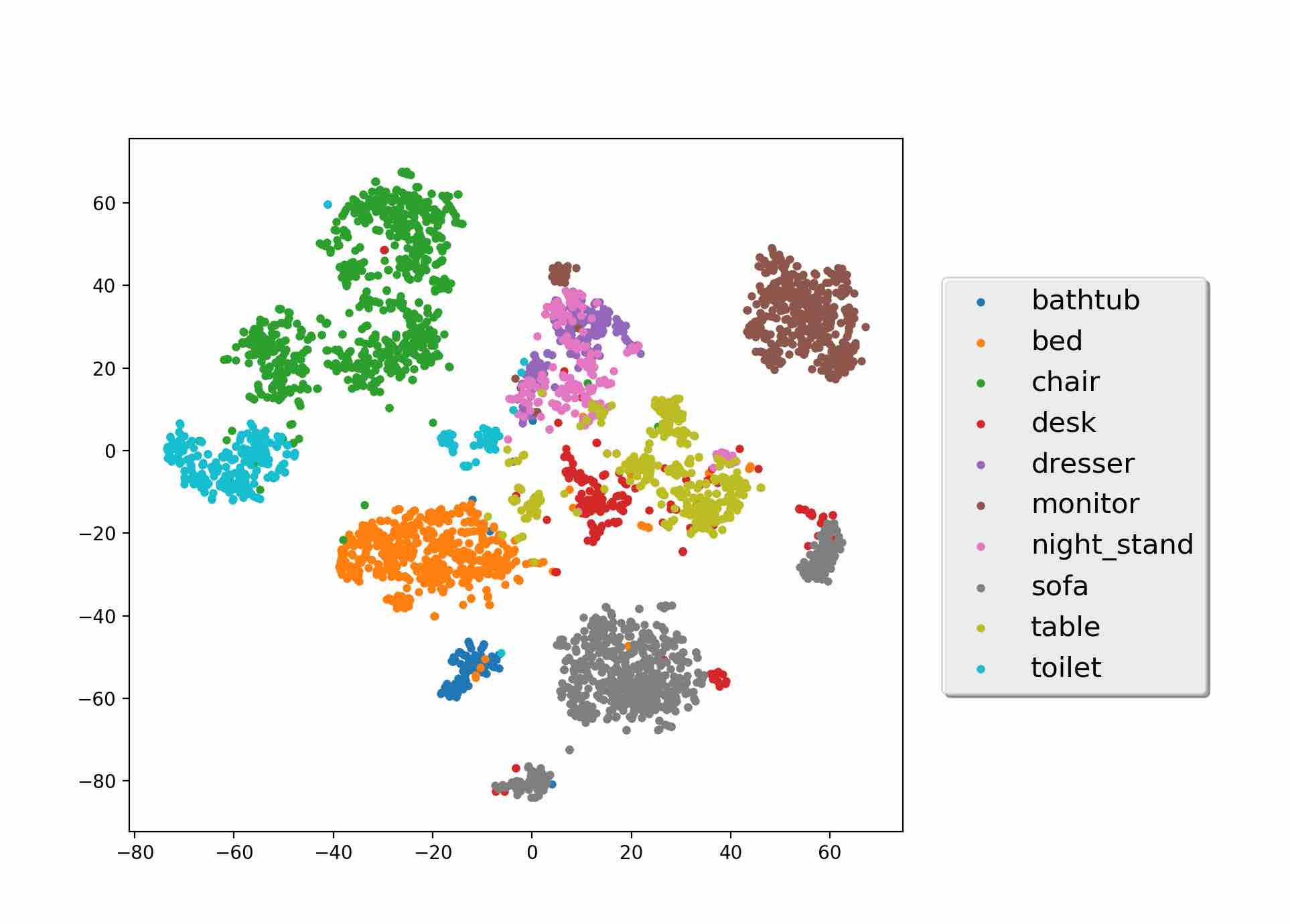}     & \includegraphics[width=0.5\columnwidth]{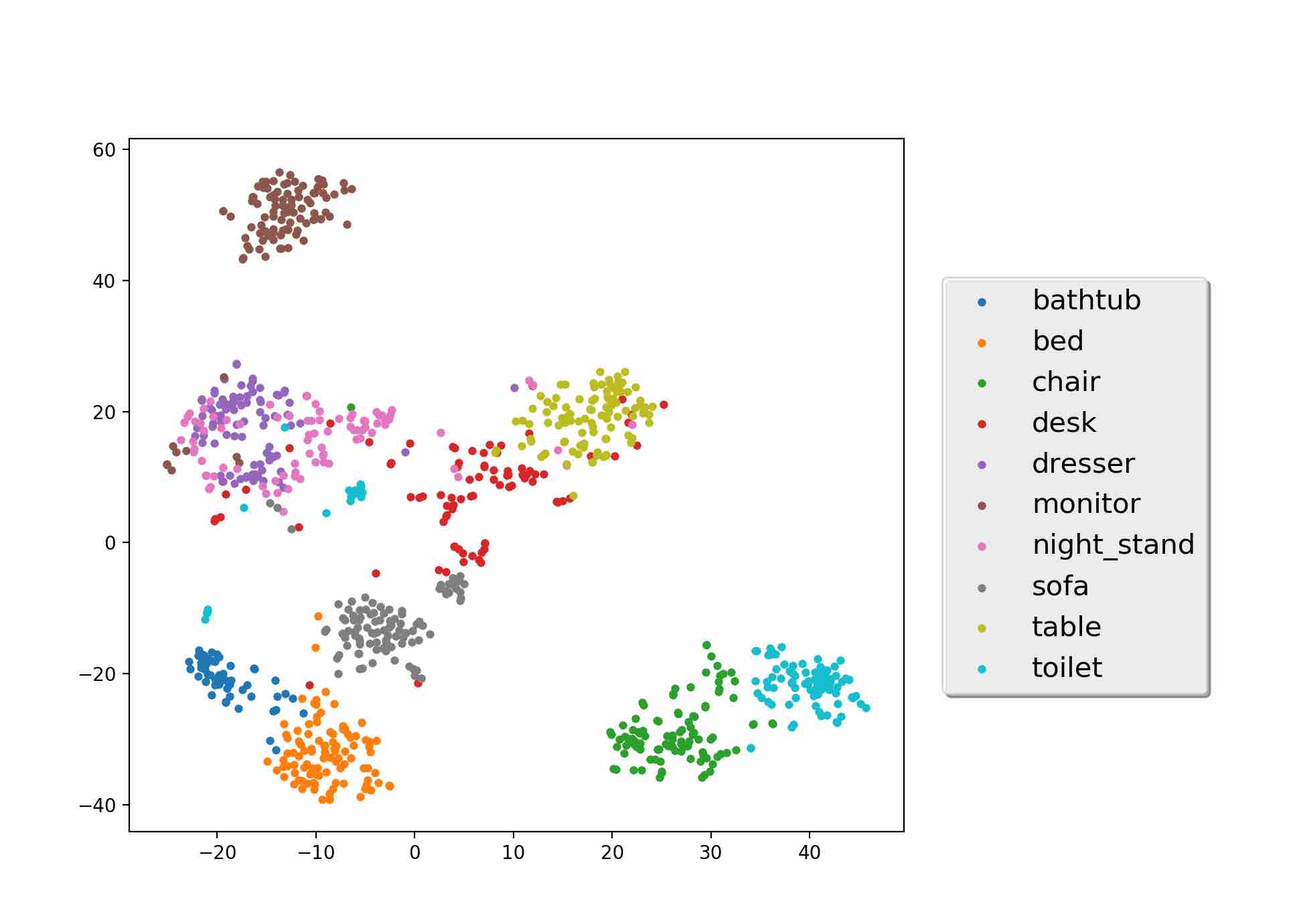}  
    \\
    {\small (a) Train.} &  {\small (b) Test.} 
  \end{tabular}
\end{center}
\caption{\label{fig:clustering} Clustering on ModelNet10 using t-SNE. Each point in the plots represents a 3D point cloud and the associated color represents its ground-truth category.}
\end{figure}

\subsection{Graph Spectral Analysis}
Table~\ref{fig:view_eigens} shows the spectral properties of the learned graph topologies based on either the graph adjacency matrix~\eqref{eq:adj_filtering} or the graph Laplacian matrix~\eqref{eq:lap_filter}. The first column shows the distribution of eigenvalues of the learned graph adjacency/Laplacian matrix, which indicates the graph frequencies. In each plot, the left side of the $x$-axis represents lower frequency and vice versa. We see that for the graph adjacency matrix, the eigenvalues associated with the torus point cloud and the airplane point cloud are similar; however for the graph Laplacian matrix, the eigenvalues of torus point cloud decrease much faster compared with the eigenvalues of airplane point cloud, which implies that it is easier to learn a torus than an airplane. We also calculate the first four eigenvectors of the graph adjacency/Laplacian matrix and color the corresponding reconstructed point clouds using the values of these eigenvectors, which are shown in Table~\ref{fig:view_eigens}. The colored point clouds listed in the second, third, fourth and fifth columns are segmented by the colors in a finer way. 

\begin{table*}[htb!]
  \begin{center}
    \begin{tabular}{ c | c  c  c  c}
      \hline 
      graph-filter &
      Eigenvalues & 
      1st eigenvector & 
      2nd eigenvector & 
      3rd eigenvector  \\
      \hline 
      $\quad$&
      $\quad$&
      $\quad$&
      $\quad$&
      $\quad$\\
      graph-adjacency & 
      \includegraphics[trim={1.5cm 2.5cm 4cm 4.5cm}, clip=true , scale=0.07] {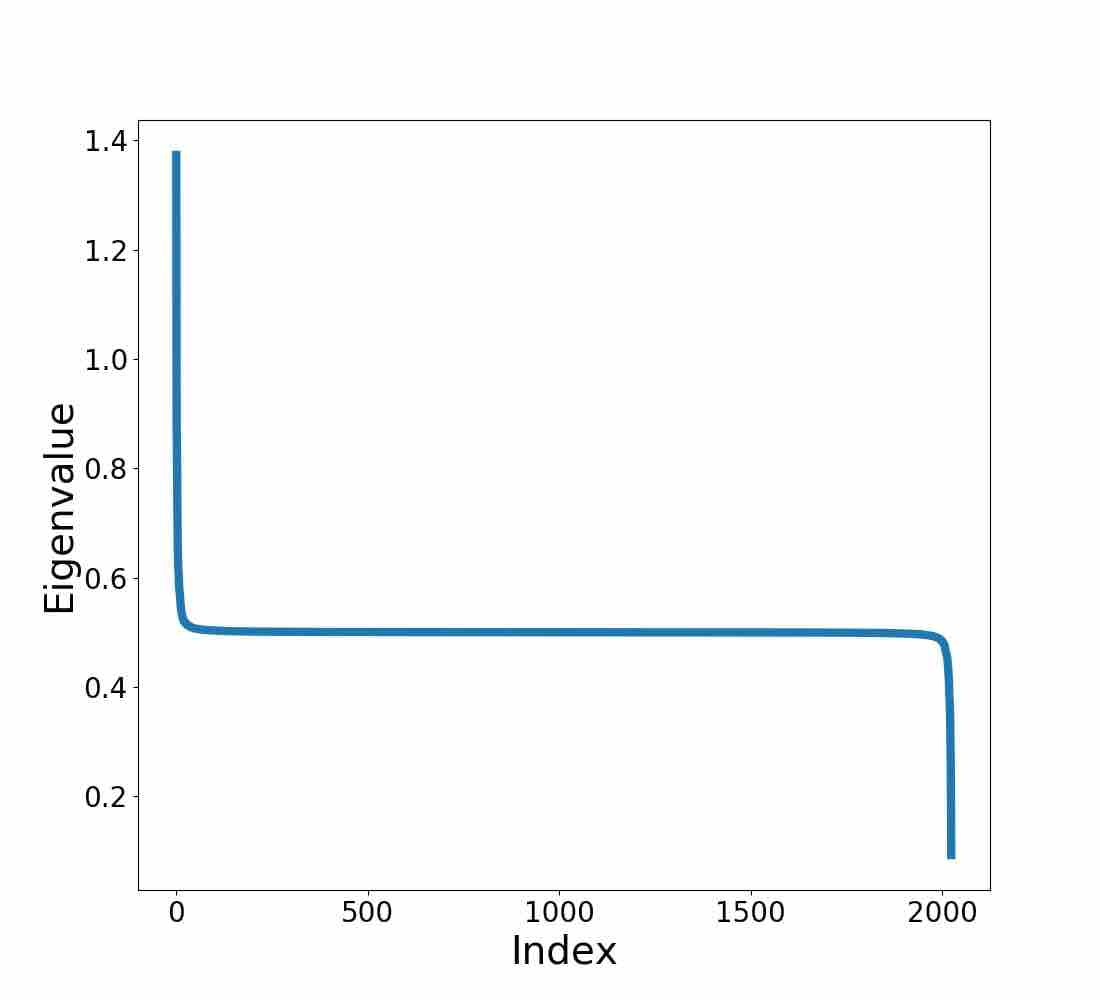} & 
      \includegraphics[trim={21cm 5cm 10cm 5cm}, clip=true , scale=0.045] {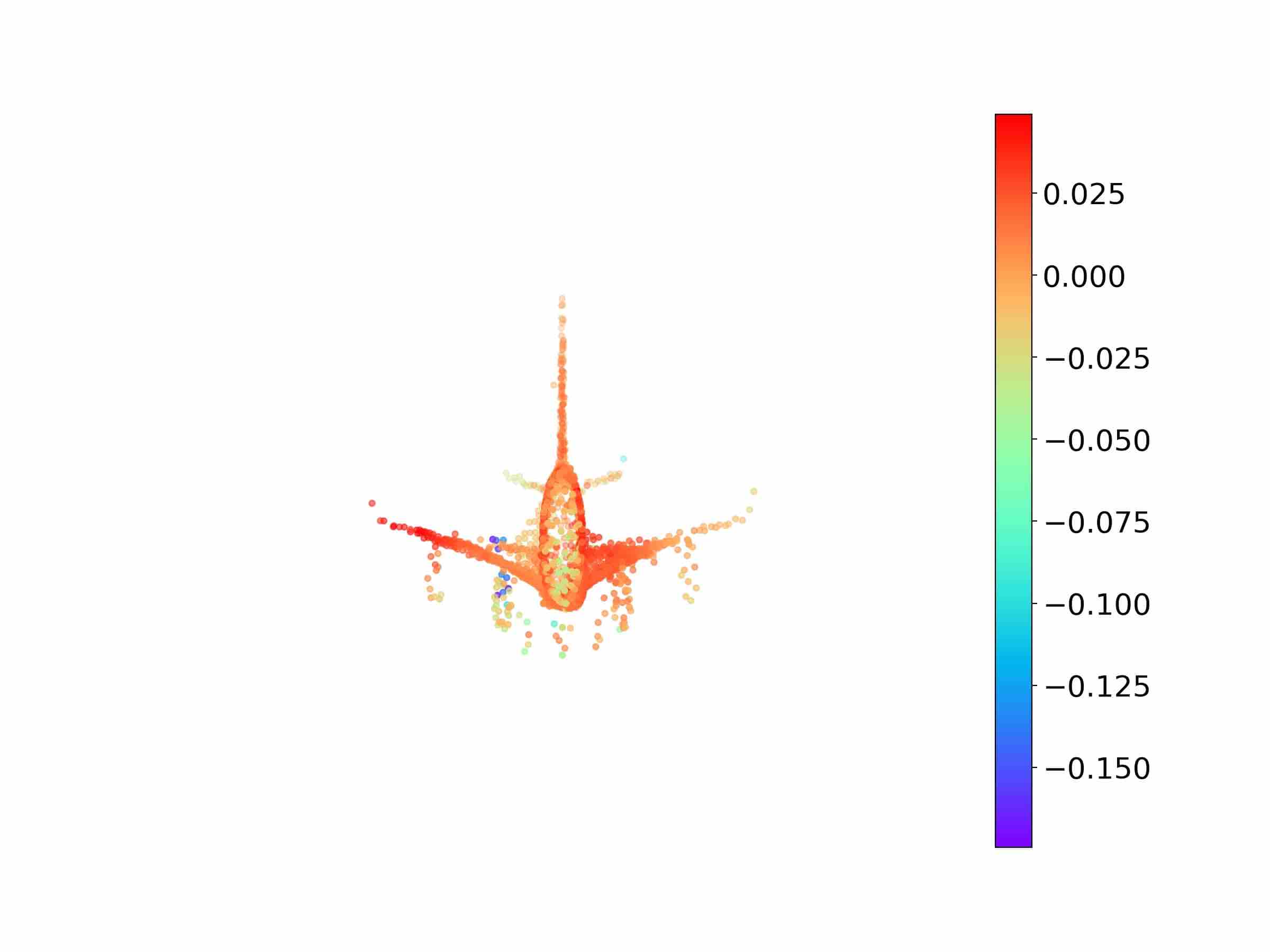} & 
      \includegraphics[trim={21cm 5cm 10cm 5cm}, clip=true , scale=0.045] {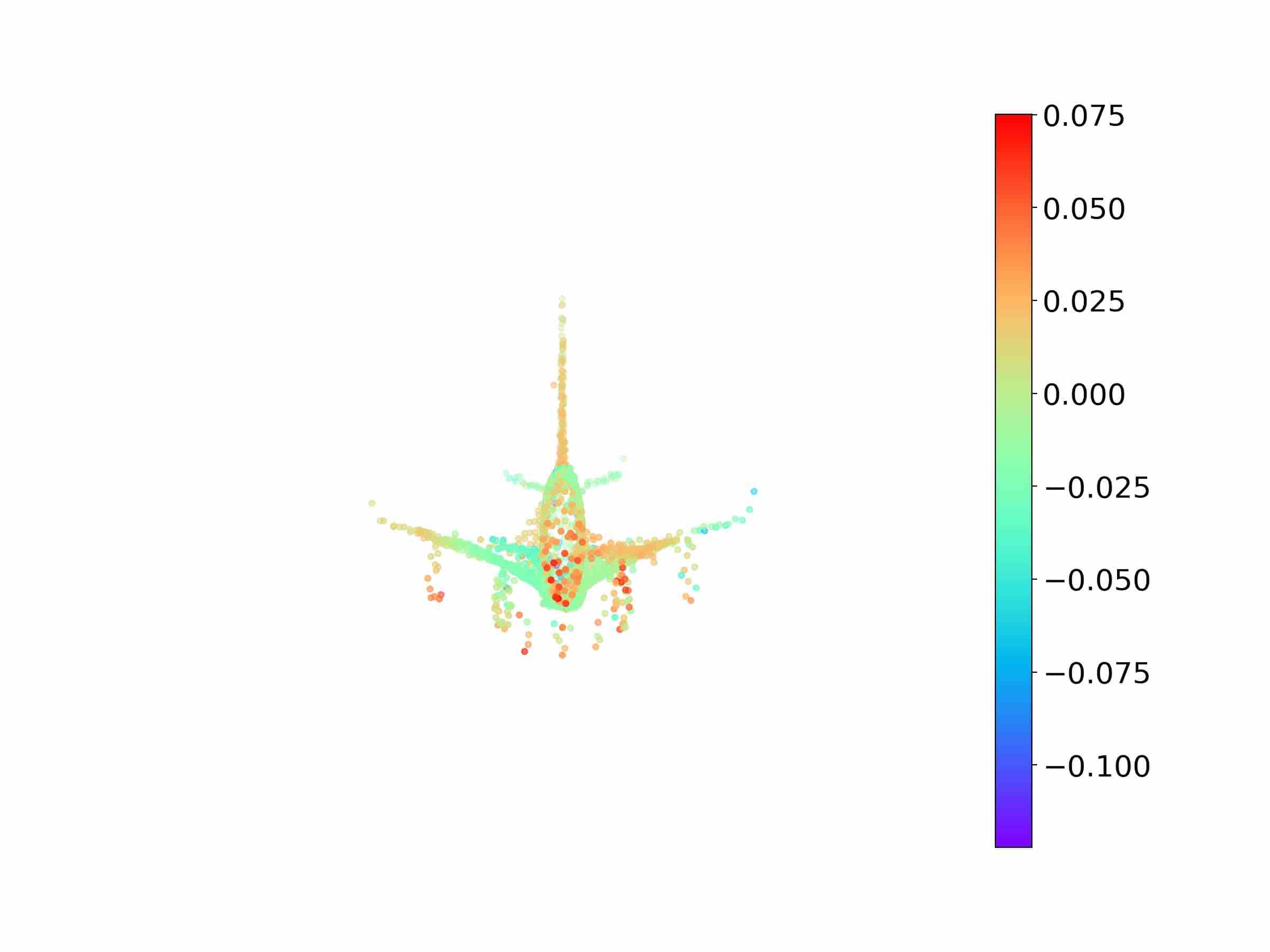} & 
      \includegraphics[trim={21cm 5cm 10cm 5cm}, clip=true , scale=0.045] {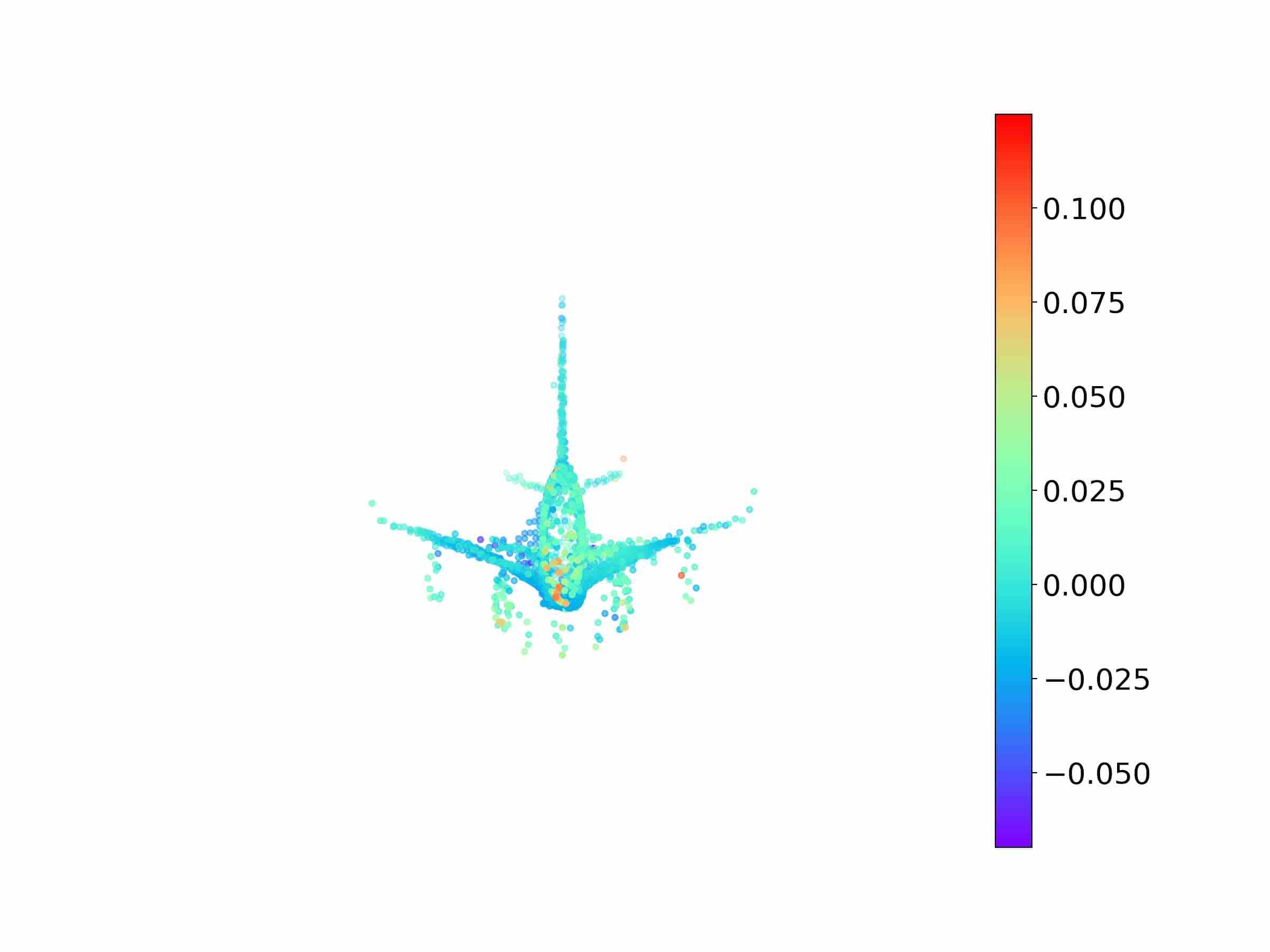}
      \\
      &
      \includegraphics[trim={1.5cm 2.5cm 4cm 4.5cm}, clip=true , scale=0.07] {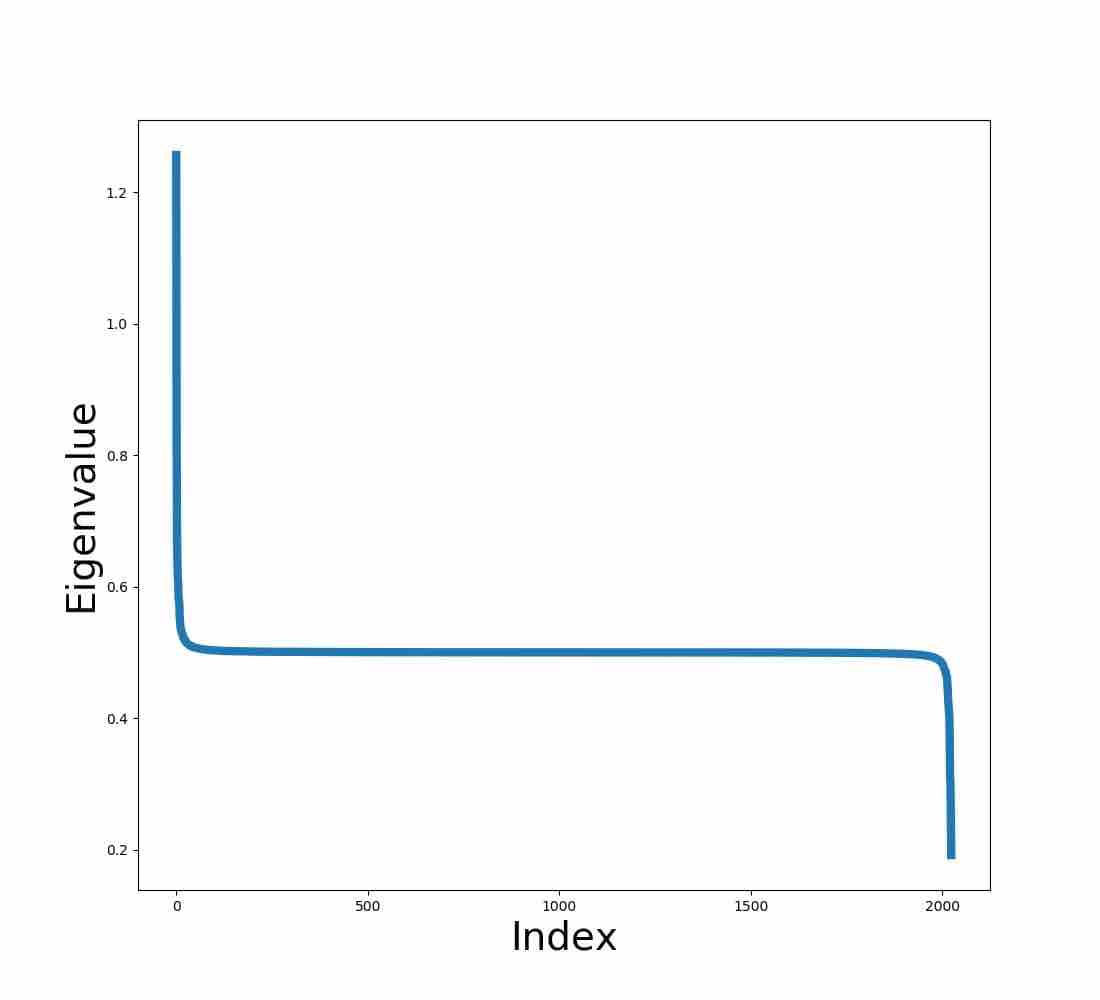} & 
      \includegraphics[trim={21cm 5cm 10cm 5cm}, clip=true , scale=0.045] {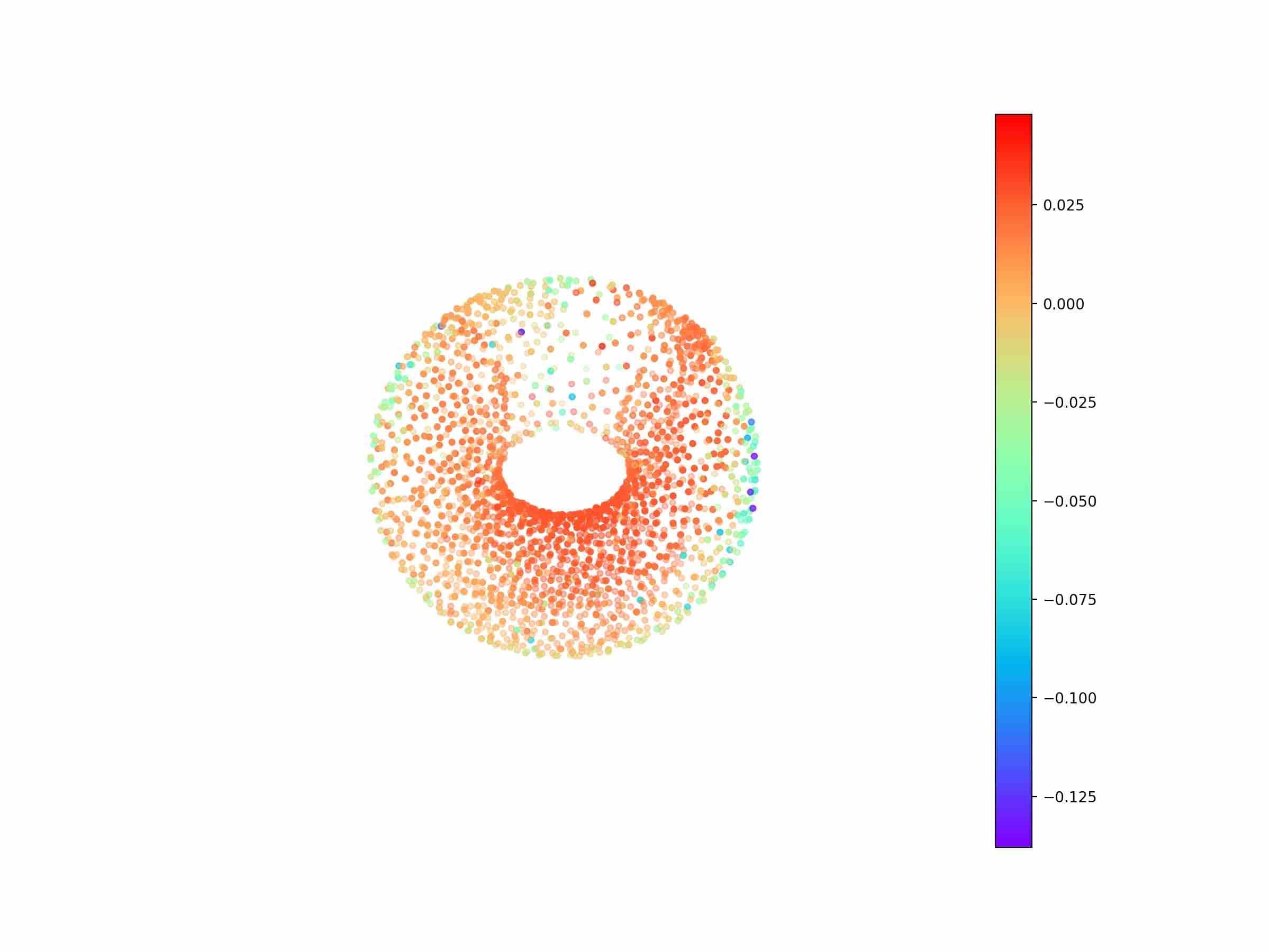} & 
      \includegraphics[trim={21cm 5cm 10cm 5cm}, clip=true , scale=0.045] {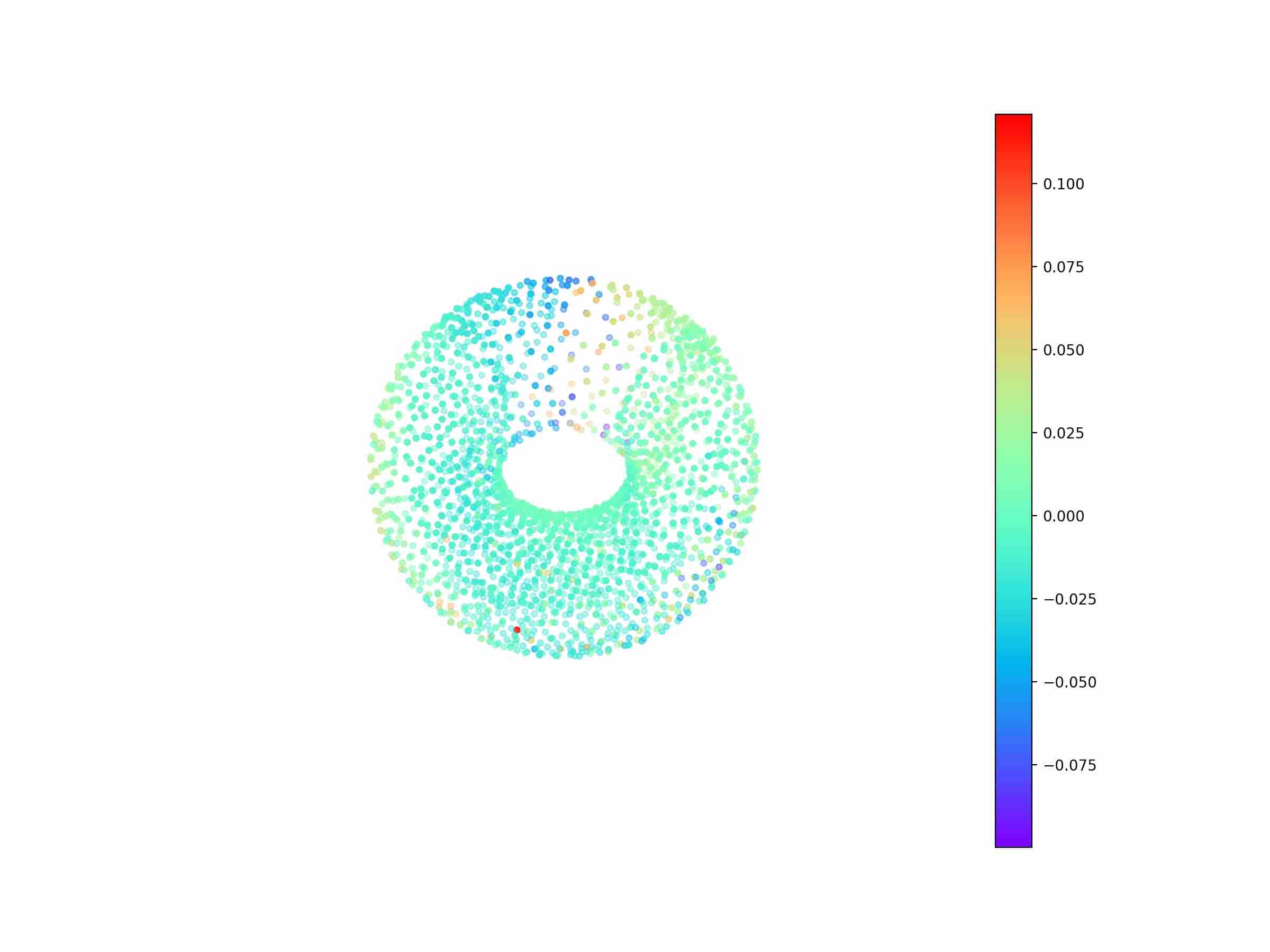} & 
      \includegraphics[trim={21cm 5cm 10cm 5cm}, clip=true , scale=0.045] {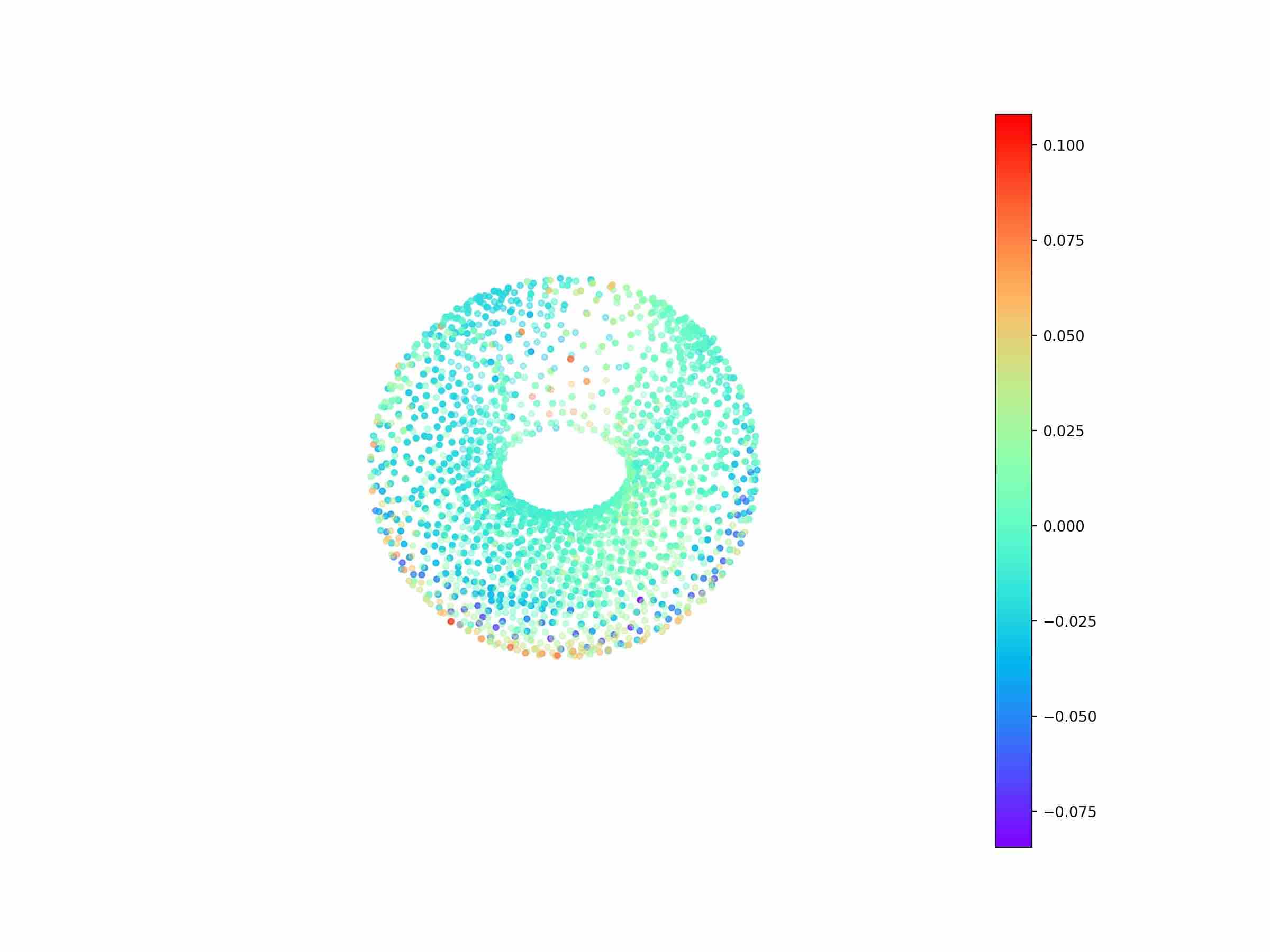} 
      \\
      \hline 
      $\quad$&
      $\quad$&
      $\quad$&
      $\quad$&
      $\quad$\\
      graph-Laplacian &
      \includegraphics[trim={1.5cm 2.5cm 4cm 4.5cm}, clip=true , scale=0.07] {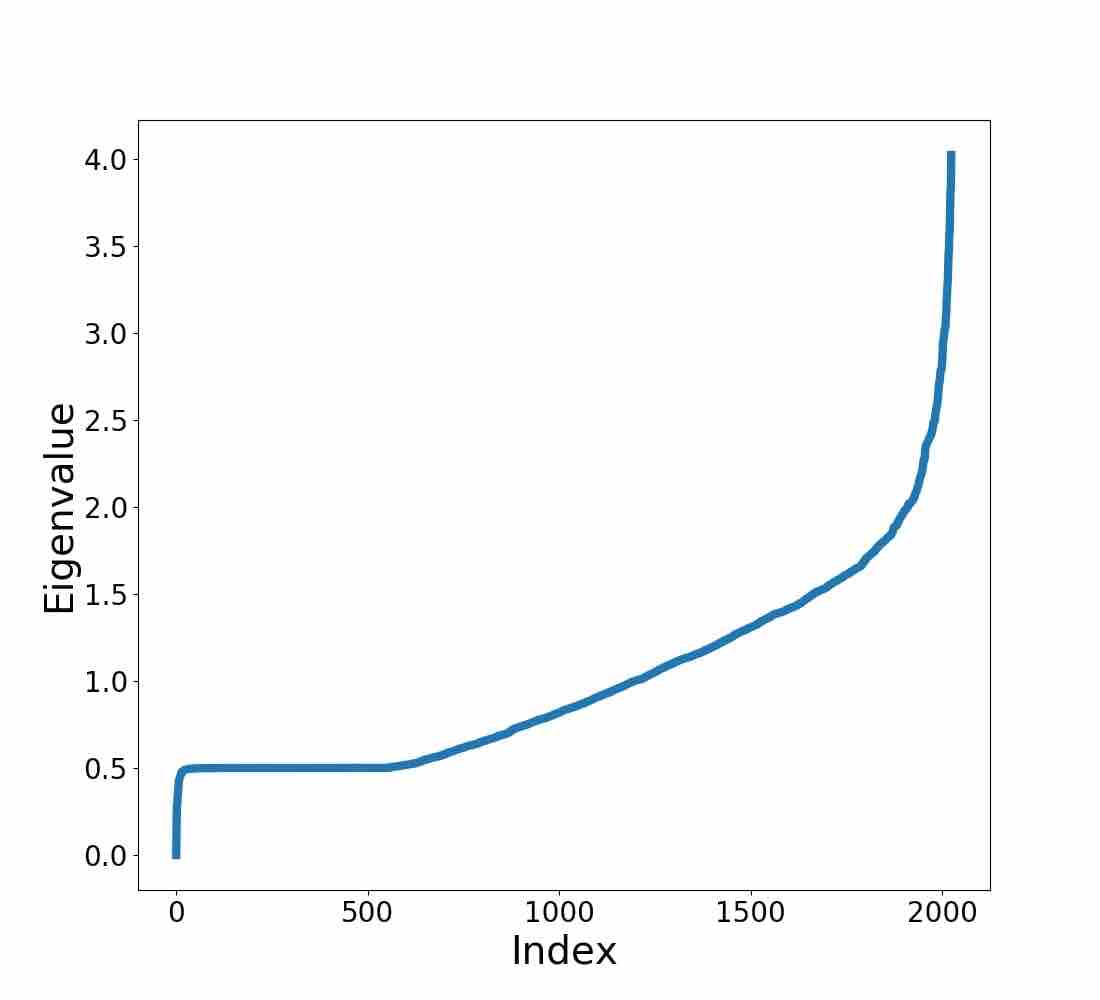} & 
      \includegraphics[trim={21cm 5cm 10cm 5cm}, clip=true , scale=0.045] {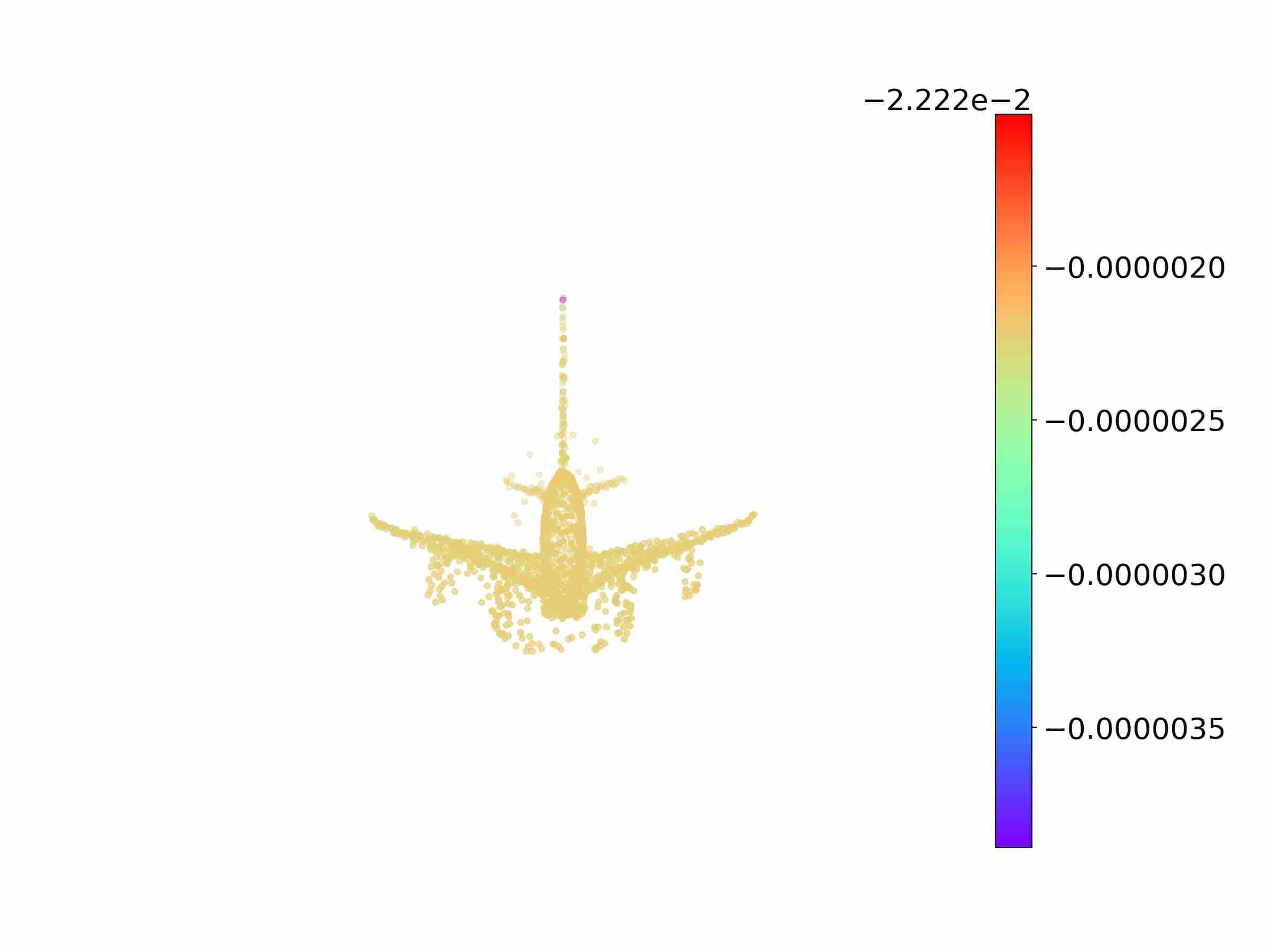} & 
      \includegraphics[trim={21cm 5cm 10cm 5cm}, clip=true , scale=0.045] {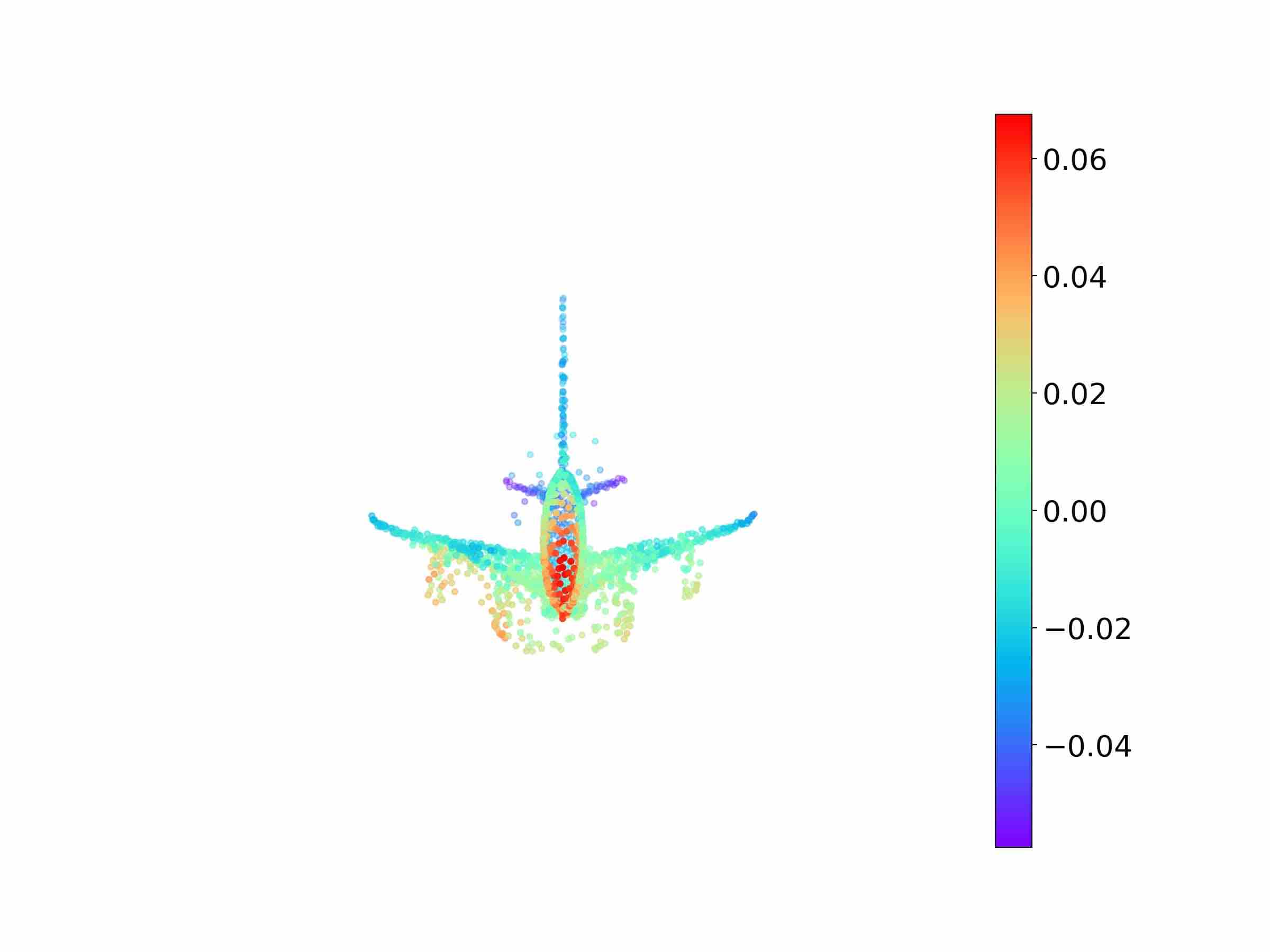} & 
      \includegraphics[trim={21cm 5cm 10cm 5cm}, clip=true , scale=0.045] {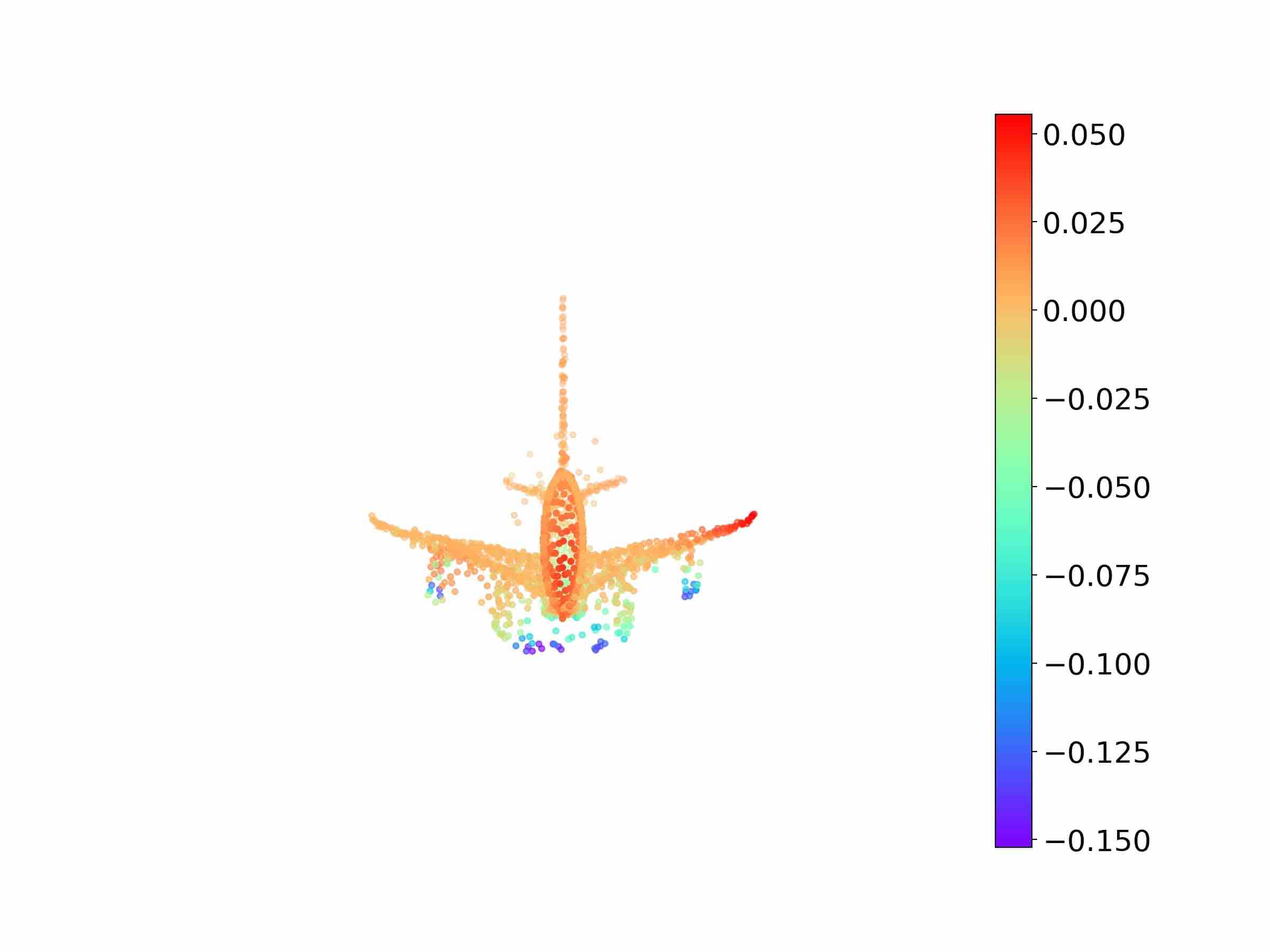}
      \\
      &
      \includegraphics[trim={1.5cm 2.5cm 4cm 4.5cm}, clip=true , scale=0.07] {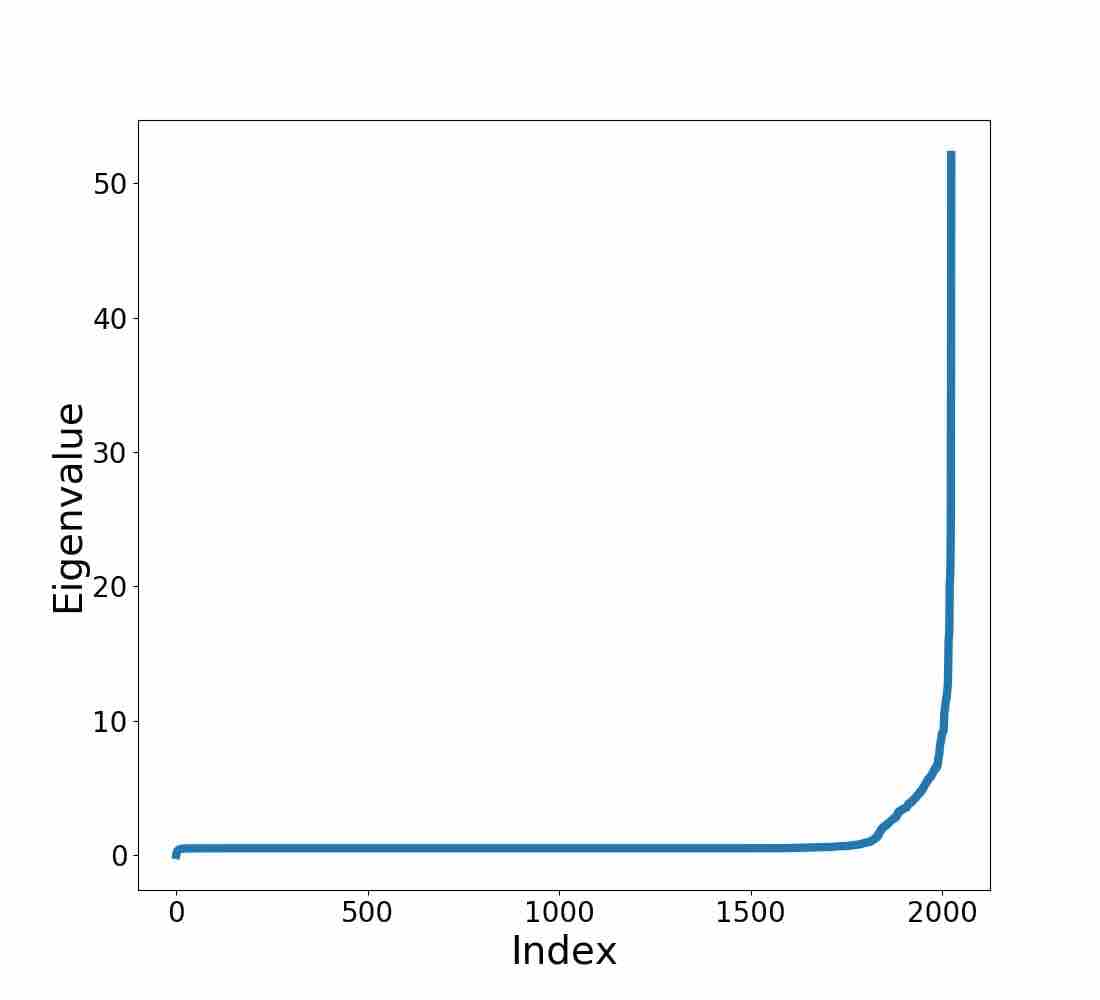} & 
      \includegraphics[trim={21cm 5cm 10cm 5cm}, clip=true , scale=0.045] {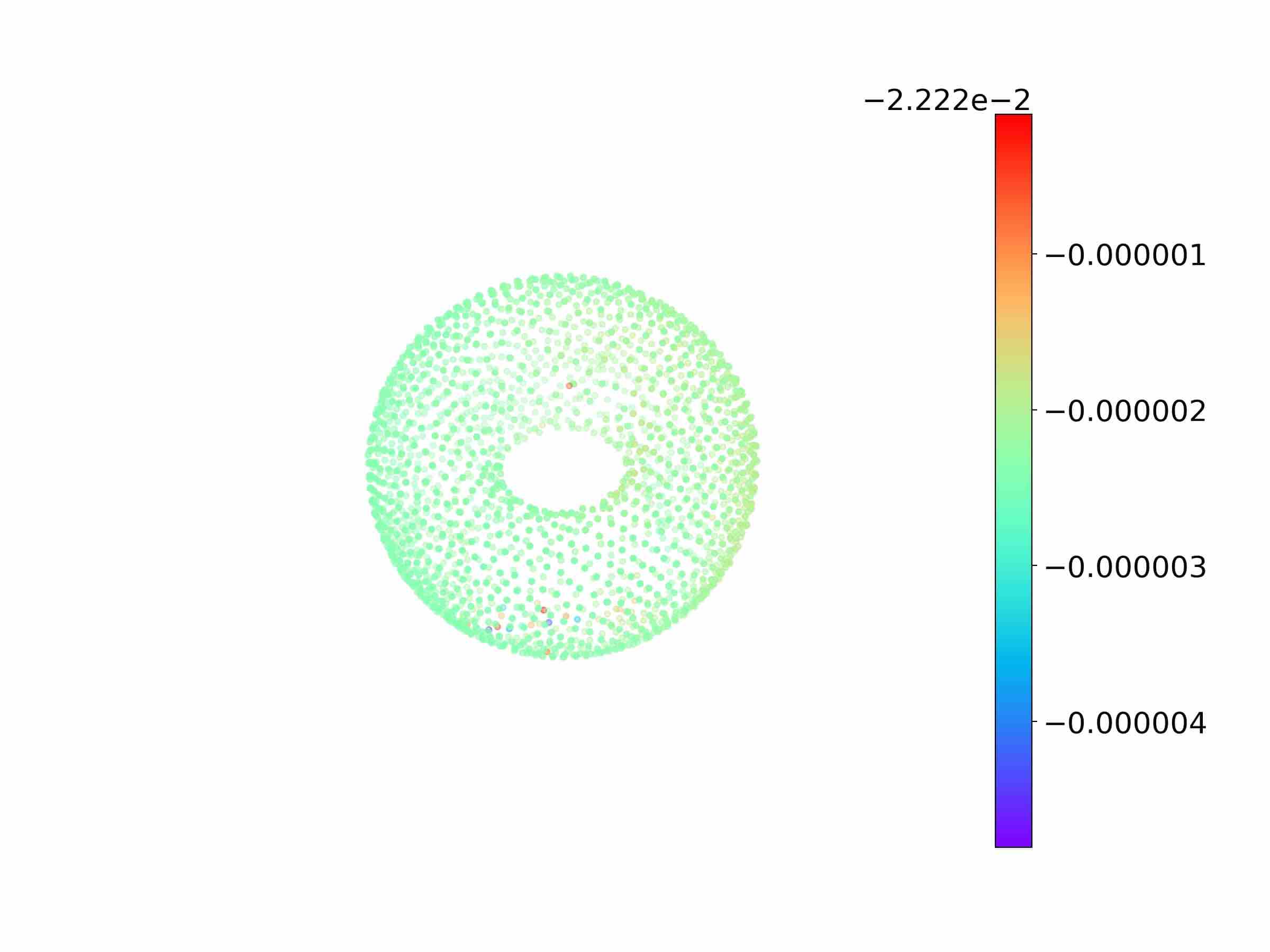} & 
      \includegraphics[trim={21cm 5cm 10cm 5cm}, clip=true , scale=0.045] {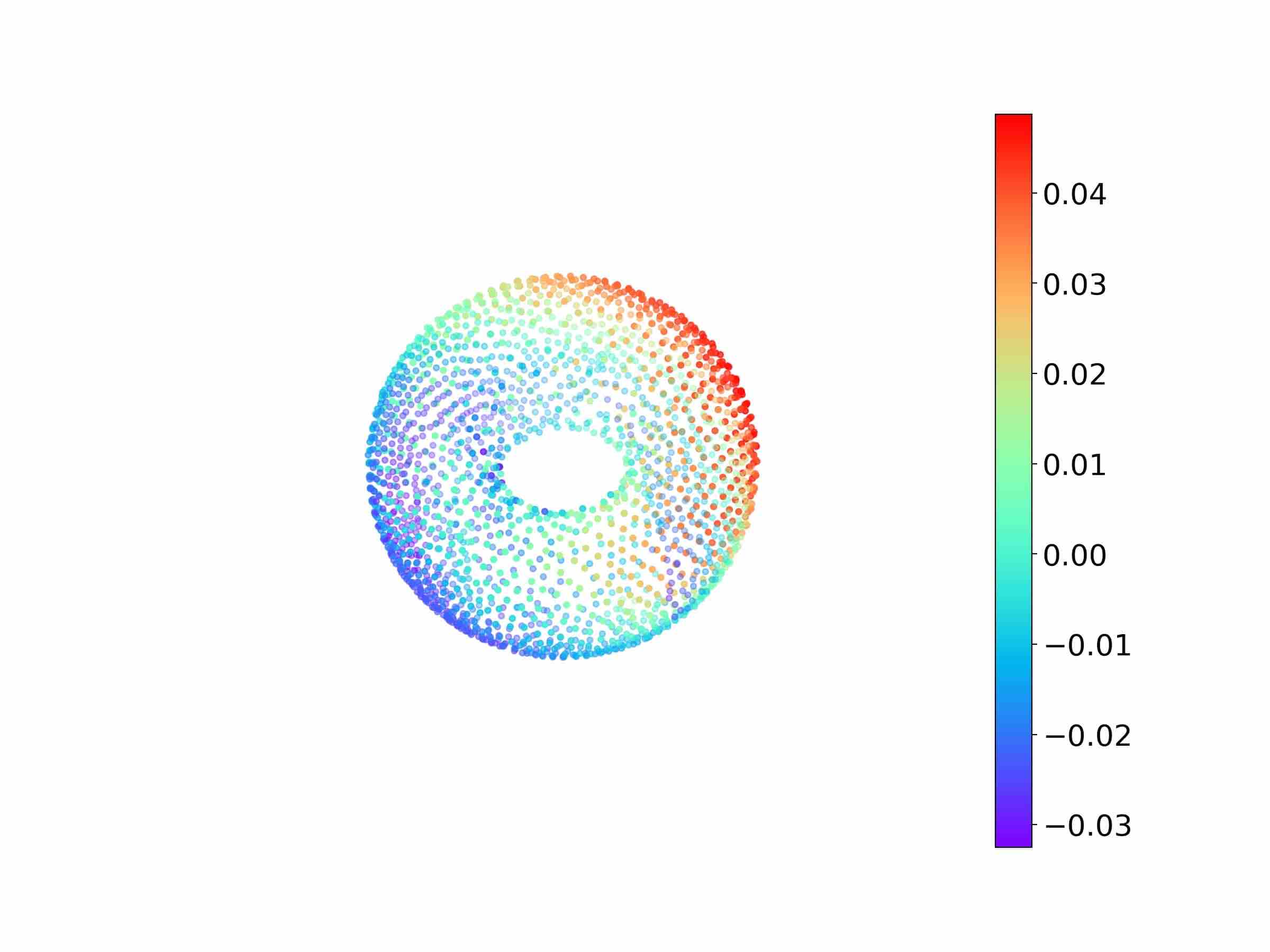} & 
      \includegraphics[trim={21cm 5cm 10cm 5cm}, clip=true , scale=0.045] {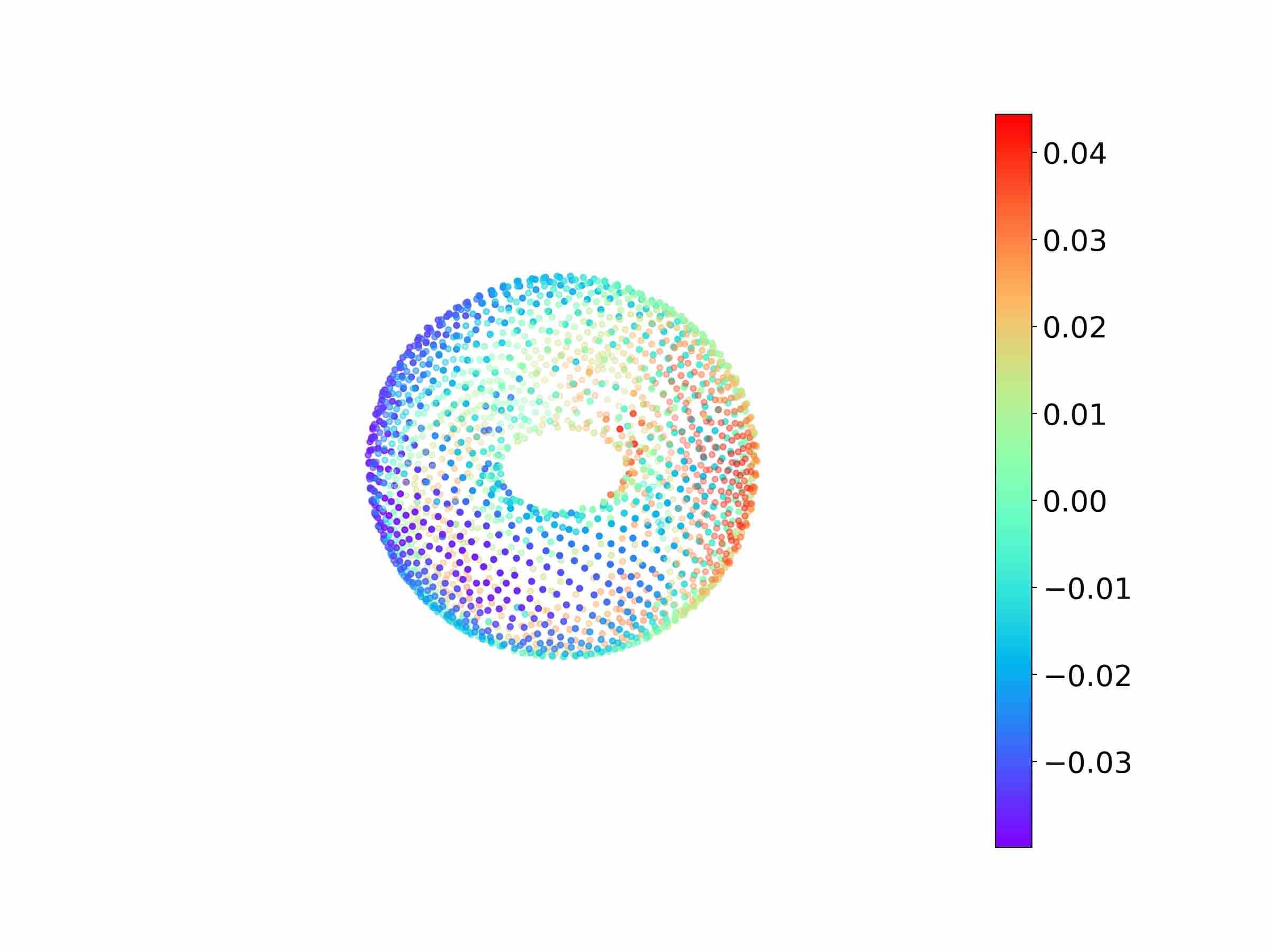}
      \\
      \hline
    \end{tabular}
  \end{center}
  \caption{\label{fig:view_eigens}\textbf{Visualization of graph spectral information for reconstructed 3D point cloud.} The first column shows the eigenvalues of the graph adjacency/Laplacian matrix obtained in the graph-topology-inference module~\eqref{eq:graph_lap}. The left side of the $x$-axis indicates lower frequency. From the second column to the fifth column, the plots show the reconstructed point clouds colored with the first four eigenvectors of the graph adjacency/Laplacian matrix. As eigenvalues increase, the corresponding eigenvectors represent a higher frequency and segment the point clouds in a finer way. }
\end{table*}

\begin{table*}[htb!]
	\begin{center}
		\begin{tabular}{ c|c c c c c c c c c } 
			\hline
			\hline 
		    \diagbox{Value of $M$}{Category} & Airplane & Car & Chair &  Guitar & Knife & Lamp & Laptop & Table\\
			\hline
			1024 (loss: $\times$ 1e-2) &  3.95 & 5.73 & 6.01 & 2.28 & 2.91 & 5.93 & 5.02 & 5.37 \\
			2025 (loss: $\times$ 1e-2) & 2.34 & 3.83 & 3.46 & 1.58 & 1.70 & 3.45 & 2.94 & 3.30 \\
			4096 (loss: $\times$ 1e-2) & 1.77 & 2.36 & 2.49 & 0.97 & 1.05 & 2.71 & 2.09 & 2.66 \\
			\hline
		\end{tabular}
	\end{center}
	\caption{\label{tab:num_point_loss}\textbf{Reconstruction loss as a function of the number of reconstructed points across eight categories in the ShapeNetCore dataset.} Increasing the number of reconstructed points leads to smaller reconstruction losses.}.
\end{table*}

\begin{table}[htb!]
\begin{center}
\begin{tabular}{ c|c|c|c|c } 
\hline
\hline
Value of $M$ & Modality & \# code & MN40 & MN10 \\
\hline
1024 & Points & 512 & 83.26 & 90.48\\
2048 & Points & 512 & \textbf{89.67} & \textbf{95.63} \\
4096 &  Points & 512 &  87.85 & 93.41 \\
\hline
\end{tabular}
\end{center}
\caption{\label{num_point_acc}
\textbf{Classification accuracies  as a function of the number of reconstructed points in the  datasets of ModelNet10 and ModelNet40}. Increasing the number of reconstructed points does not necessarily improve the classification performances.}
\end{table}

\begin{table}[htb!]
\begin{center}
\begin{tabular}{ c|c|c } 
\hline
\hline 
Method  &  MN40 & MN10 \\
\hline
Folding module only (CD) & 86.32 & 92.11 \\
Folding module only (AugCD)  & 88.40 & 94.40 \\
\hline 
Folding with graph-adjacency (CD)  & 88.24 \ & 94.10\\
Folding with graph-adjacency (AugCD) & \textbf{89.67} \ & 95.63\\
\hline 
Folding with graph-Laplacian (CD)  & 88.27 \ & 94.44\\
Folding with graph-Laplacian (AugCD) & 89.55 \ & \textbf{95.93}\\
\hline
\end{tabular}
\end{center}
\caption{\label{cd_classification}\textbf{Classification comparison between 
the Chamfer distance (CD) and the augmented Chamfer distance (AugCD) on ModelNet10 and ModelNet40}. The proposed networks trained with the AugCD achieve better classification performances in the ones trained with the CD.}
\end{table}

\begin{table}[htb!]
	\begin{center}
		\begin{tabular}{ c|cccc } 
			\hline
			\hline 
			Category & Airplane & Bed & Chair & Monitor\\
			\hline
			 \# subcategories & 9 & 8 & 16 & 9  \\
			 \hline
			 Folding module only (CD) & 80.15 & 71.20 & 71.93 & 68.82 \\
			Folding module only (AugCD) & 83.29 & 73.21 & 77.05 & 72.12 \\
			\hline 
			Folding with GA (CD) & 81.20 & 73.33 & 77.86 & 74.28\\
			Folding with GA (AugCD) & 83.80 & 75.00 & {\bf 81.28} & {\bf 76.92}\\
			\hline 
			Folding with GL (CD) & 82.13 & 73.65 & 77.28 & 73.75 \\
			Folding with GL (AugCD) & {\bf 83.97} & \textbf{76.65} & 81.23 & 75.43 \\
			\hline
		\end{tabular}
	\end{center}
	\caption{\label{cd_sub_classification}\textbf{Classification comparison between 
the Chamfer distance (CD) and the augmented Chamfer distance (AugCD) on the subcategory dataset.} The proposed networks trained with the AugCD significantly and consistently outperform the networks with the CD. GA stands for the graph-adjacency-matrix-based filtering; GL stands for the graph-Laplacian-matrix-based filtering. }
\end{table}

\subsection{Effect of the number of reconstructed points $M$}
Table~\ref{tab:num_point_loss} shows the reconstruction losses as a function of the number of reconstructed points $M$. We see that a larger value of $M$ leads to smaller reconstruction losses. The reason is that increasing the number of reconstructed points in a point cloud is equivalent to increasing the sampling density of reconstructed point clouds; thus, the points in input 3D point clouds are able to find closer points in reconstructed point clouds.  Note that (1) increasing the value of $M$ does not necessarily improve the transfer classification performance; see Table~\ref{num_point_acc}; and (2) increasing the value of $M$ would significantly increases the number of training parameters because the learned graph-adjacency/Laplacian matrix is an $M \times M$ matrix. Therefore, we usually set $M \approx N$. 

\subsection{Validation of augmented Chamfer distance}
As discussed in Section~\ref{sec:networks}, the augmented Chamfer distance is the Hausdorff distance between two 3D point clouds.  
Tables~\ref{cd_classification} and~\ref{cd_sub_classification} compare the classification accuracies between using the standard Chamfer distance (CD) and using the augmented Chamfer distance (AugCD) in the tasks of classification and fine-grained classification, respectively. We see that using the augmented Chamfer distance achieves a better performance than using the standard Chamfer distance.

\begin{table*}[htb!]
  \begin{center}
    \begin{tabular}{c | c  c  c  c  c }
      \hline 
    Original  & Spatial smoothness &  & Output  &  & Graph smoothness  \\
      & $\alpha = 0$ & $\alpha = 0.25$ & $\alpha = 0.5$ & $\alpha = 0.75$& $\alpha = 1$  \\
      \hline
      \includegraphics[trim={4cm 4cm 7cm 4cm}, clip=true , scale=0.12] {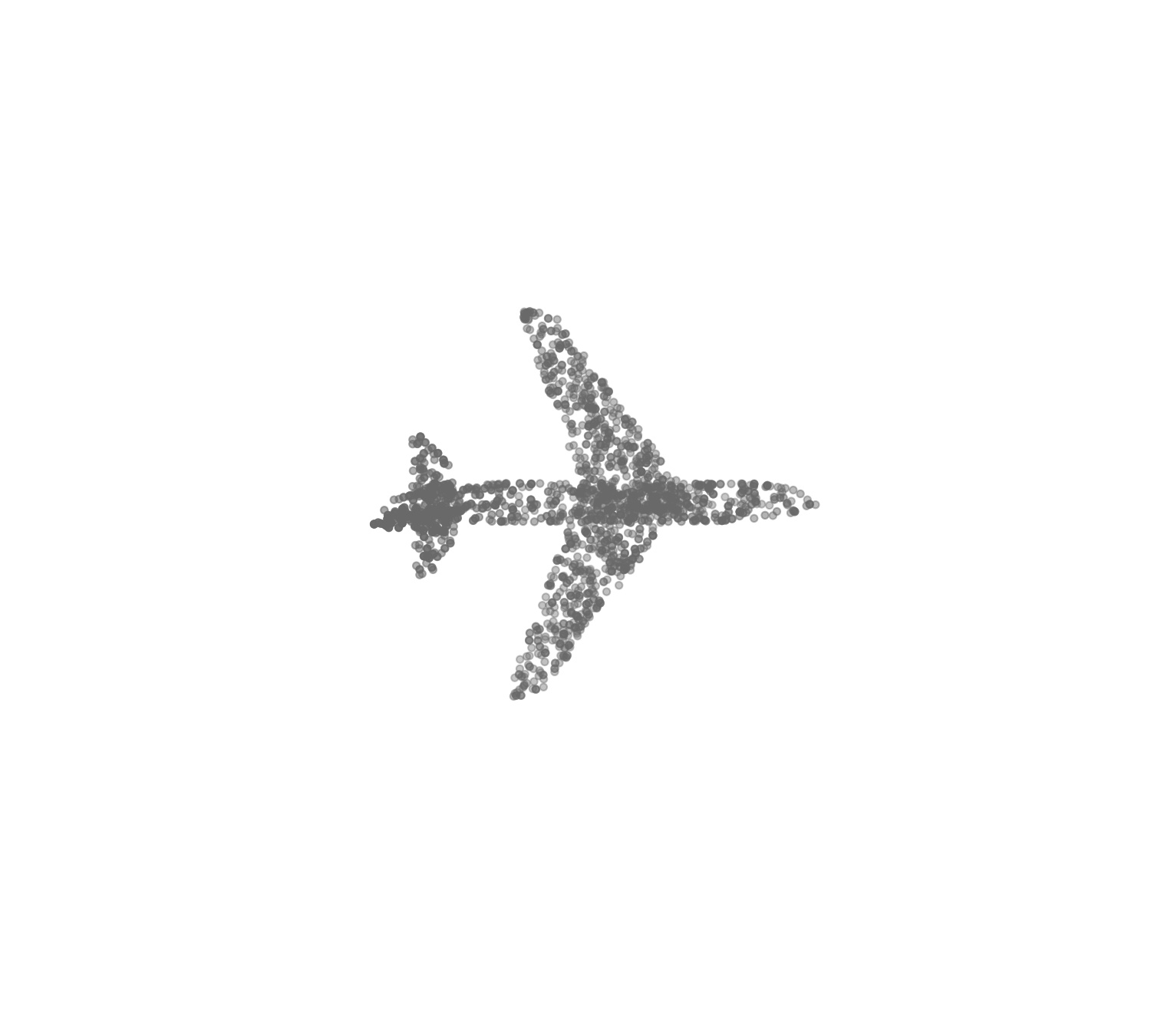} & 
      \includegraphics[trim={4cm 4cm 7cm 4cm}, clip=true , scale=0.12] {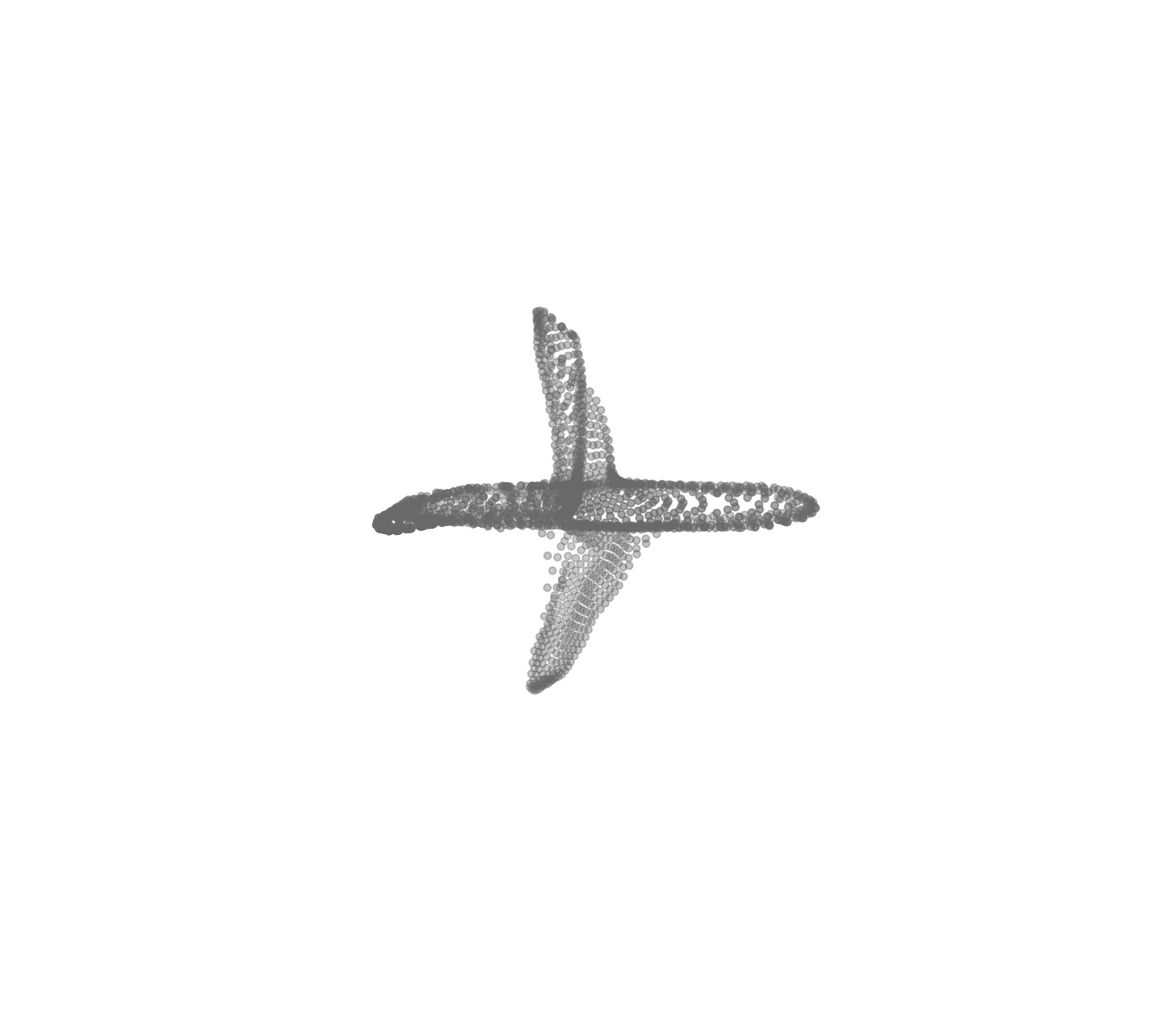} & 
      
      \includegraphics[trim={4cm 4cm 7cm 4cm}, clip=true , scale=0.12] {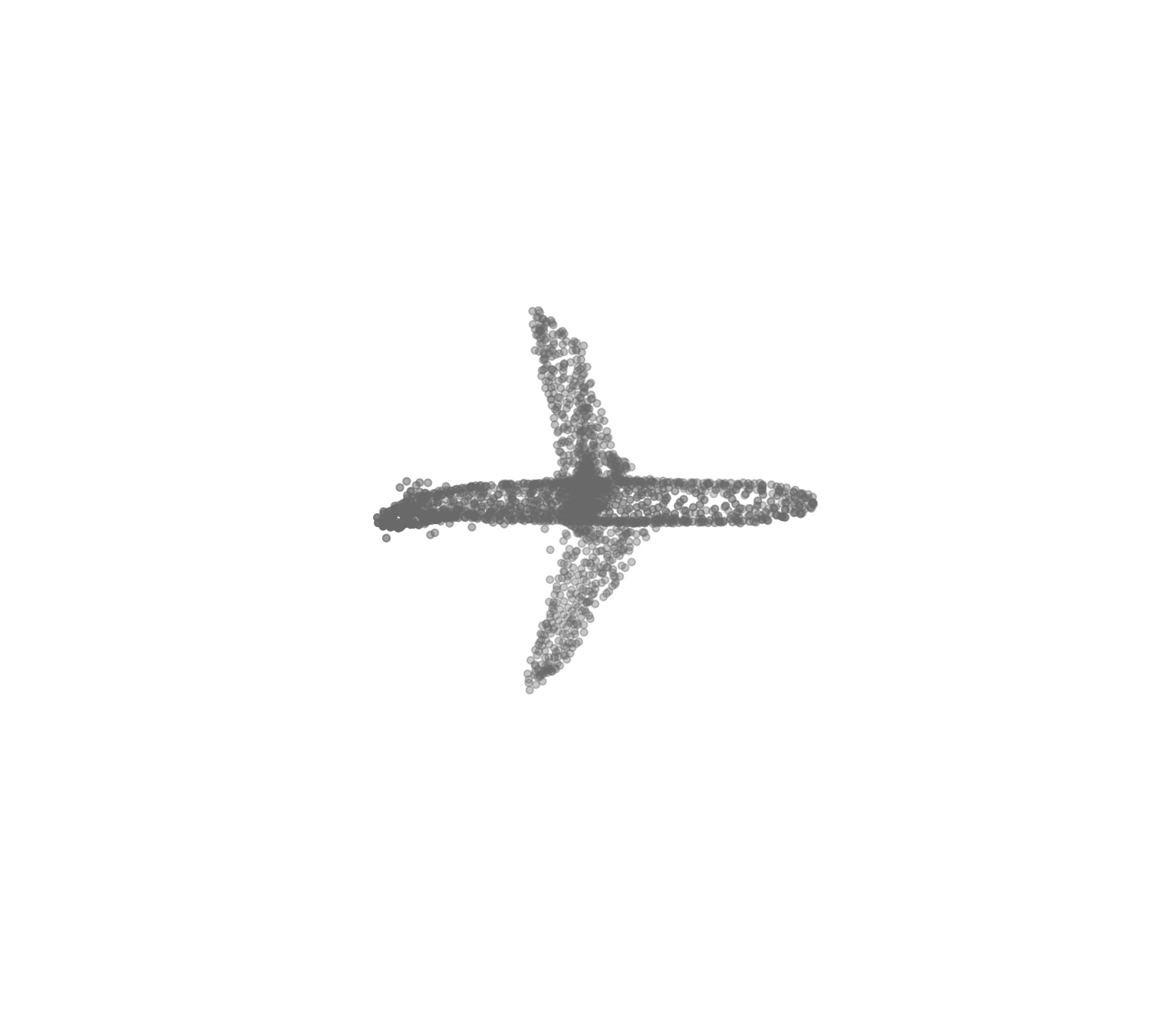} & 
      \includegraphics[trim={4cm 4cm 7cm 4cm}, clip=true , scale=0.12] {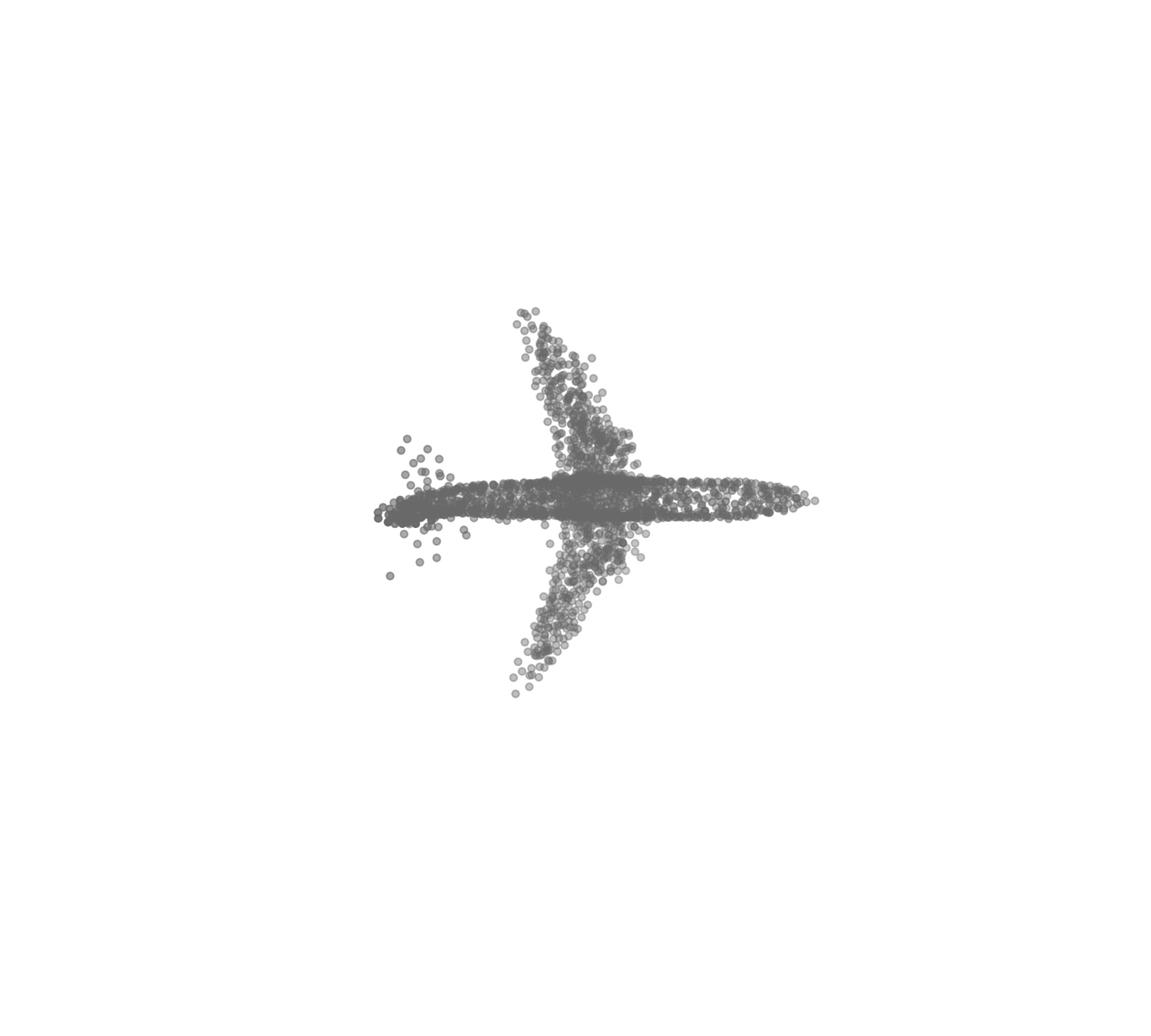} &
      
      \includegraphics[trim={4cm 4cm 7cm 4cm}, clip=true , scale=0.12] {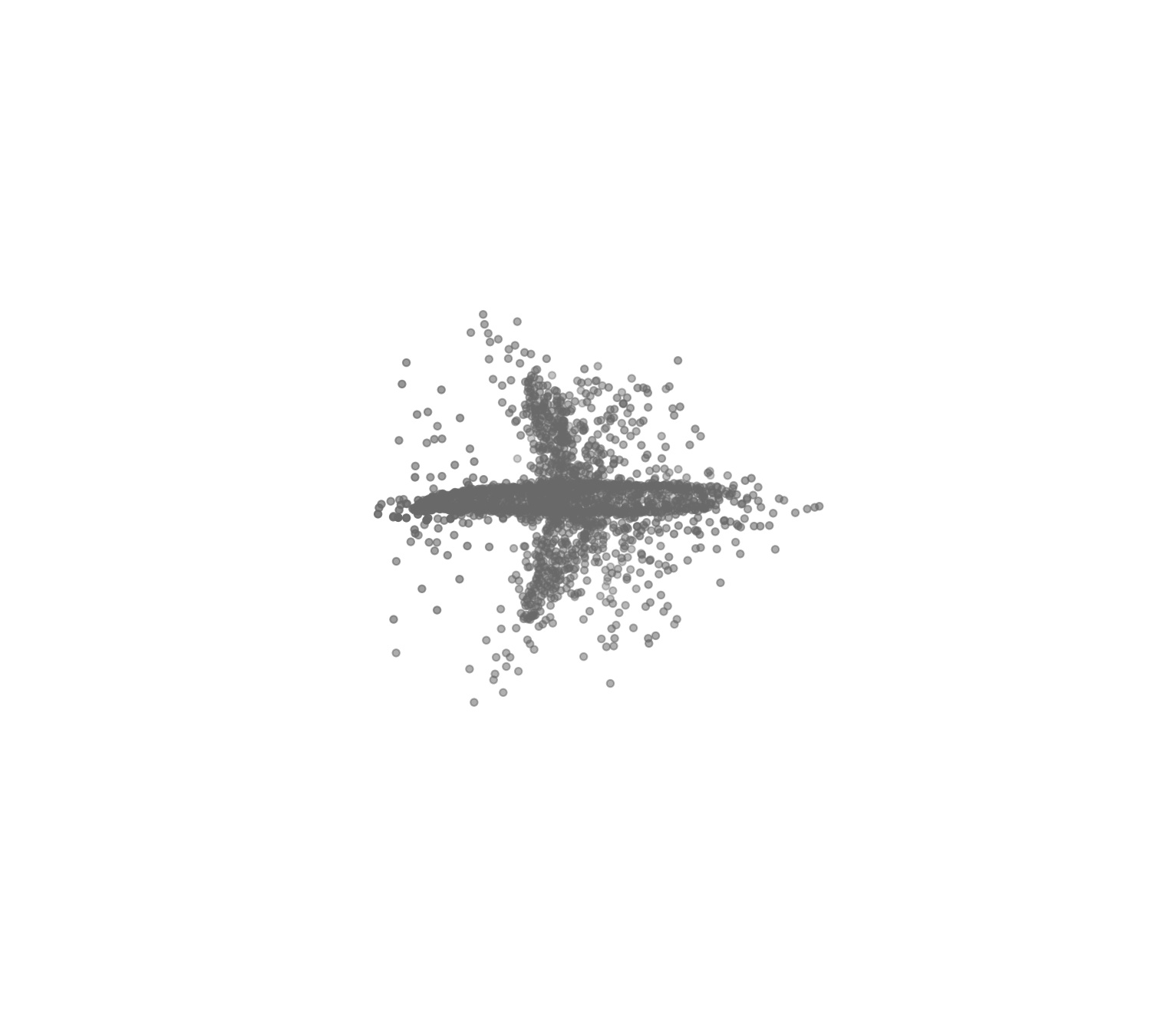} & 
      
      \includegraphics[trim={4cm 4cm 7cm 4cm}, clip=true , scale=0.12] {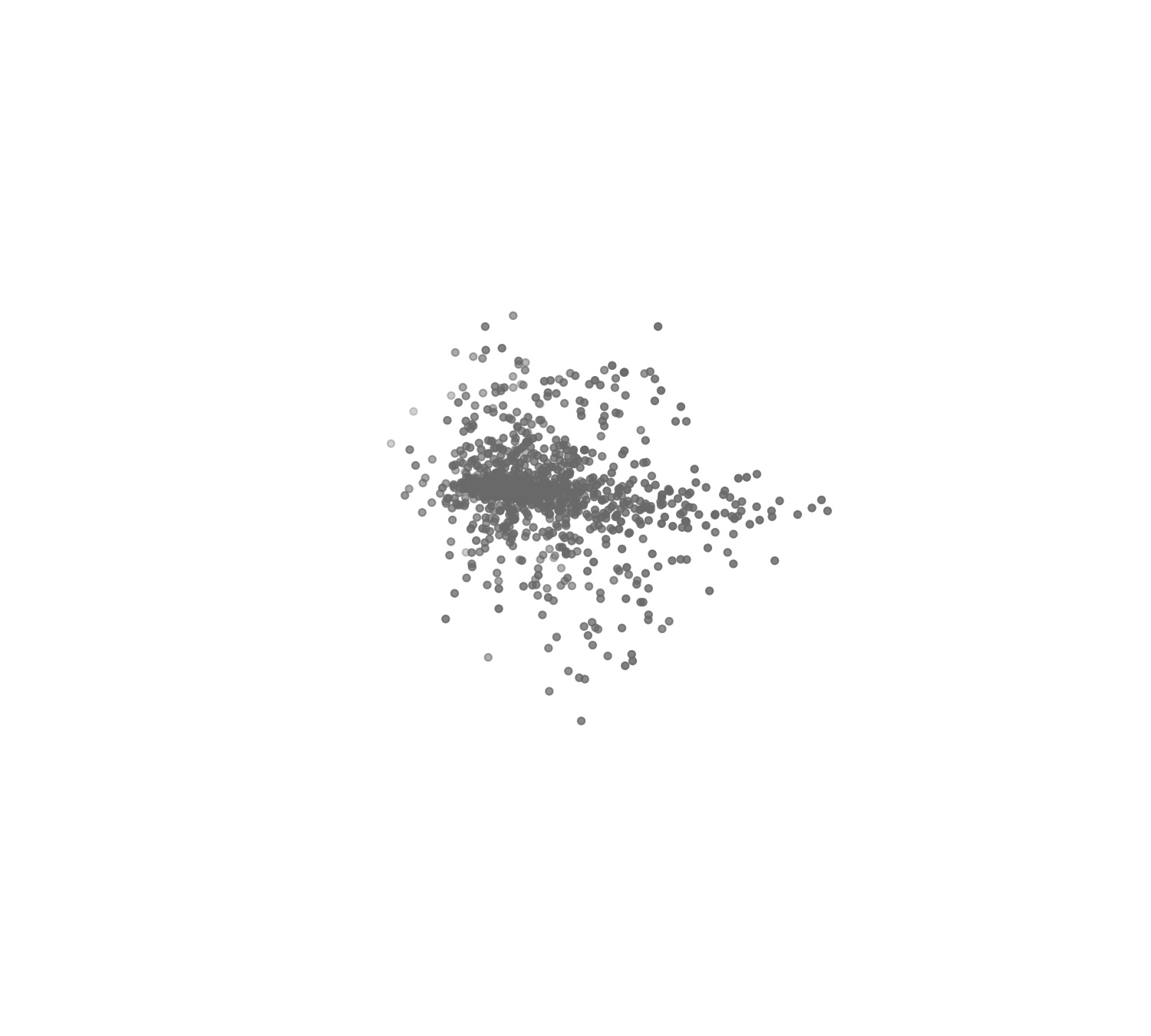}  \\
      
      \includegraphics[trim={4cm 4cm 7cm 4cm}, clip=true , scale=0.06] {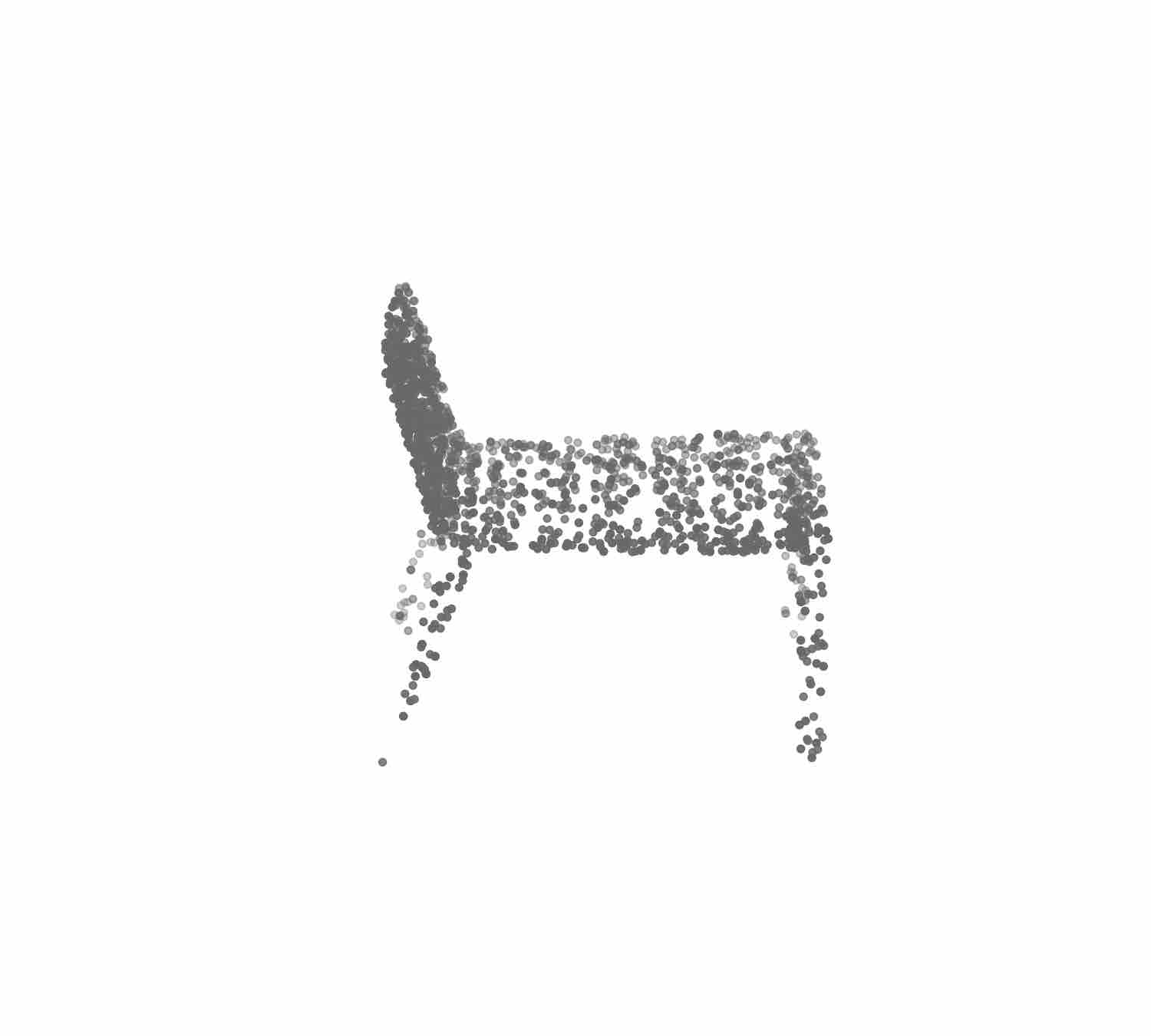} & 
      \includegraphics[trim={4cm 4cm 7cm 4cm}, clip=true , scale=0.06] {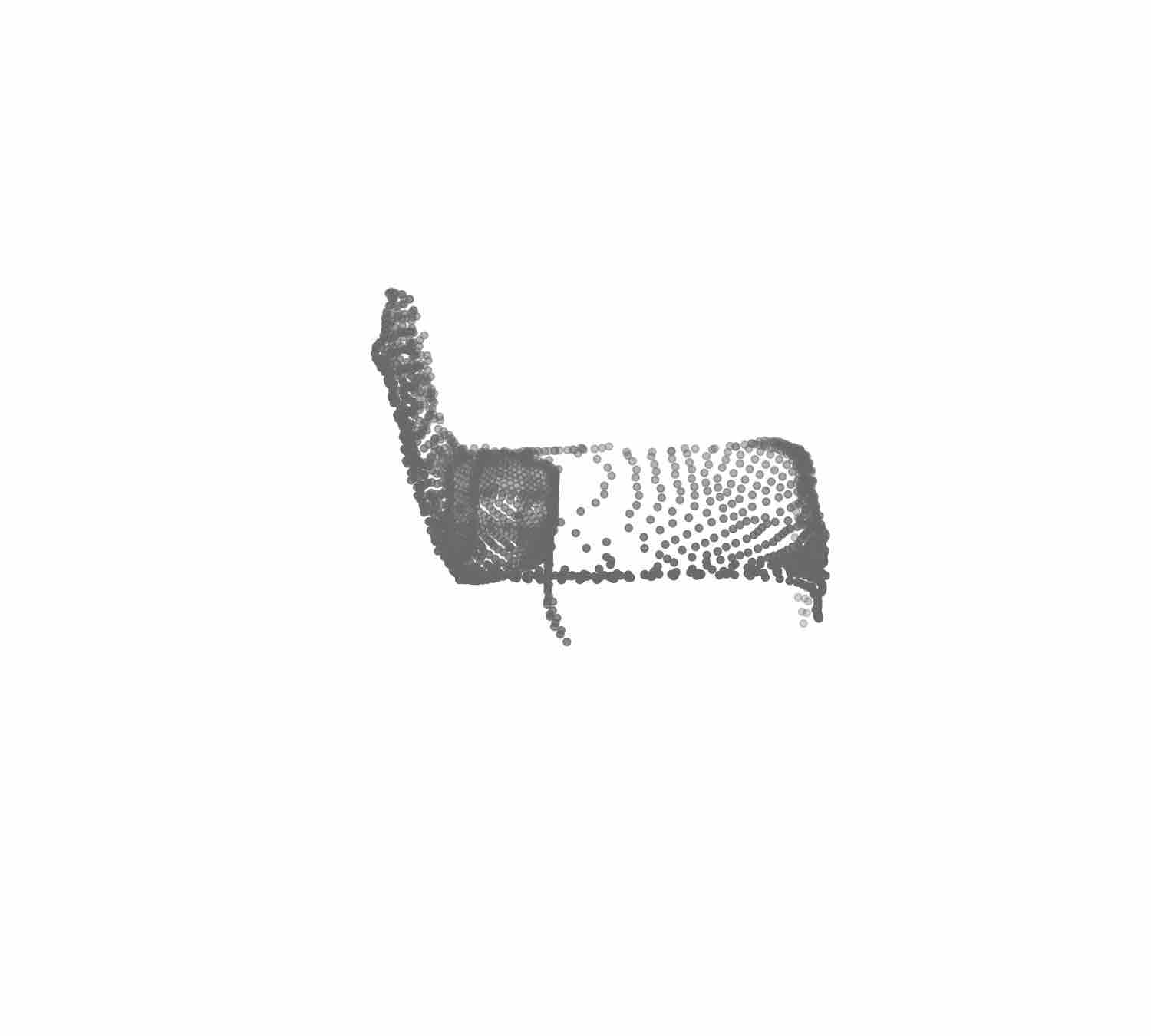} & 
      
      \includegraphics[trim={4cm 4cm 7cm 4cm}, clip=true , scale=0.06] {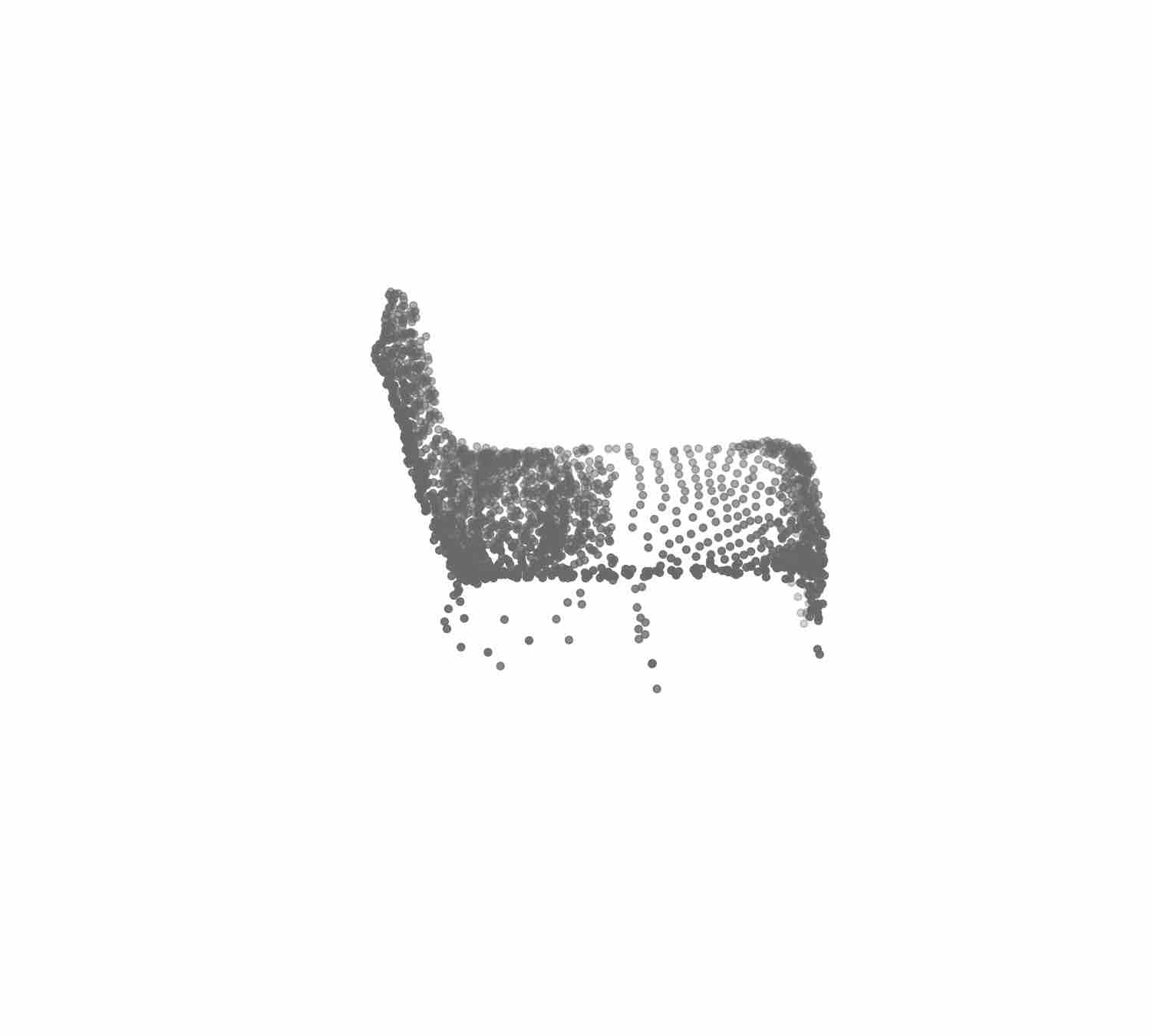} & 
      
      \includegraphics[trim={4cm 4cm 7cm 4cm}, clip=true , scale=0.06] {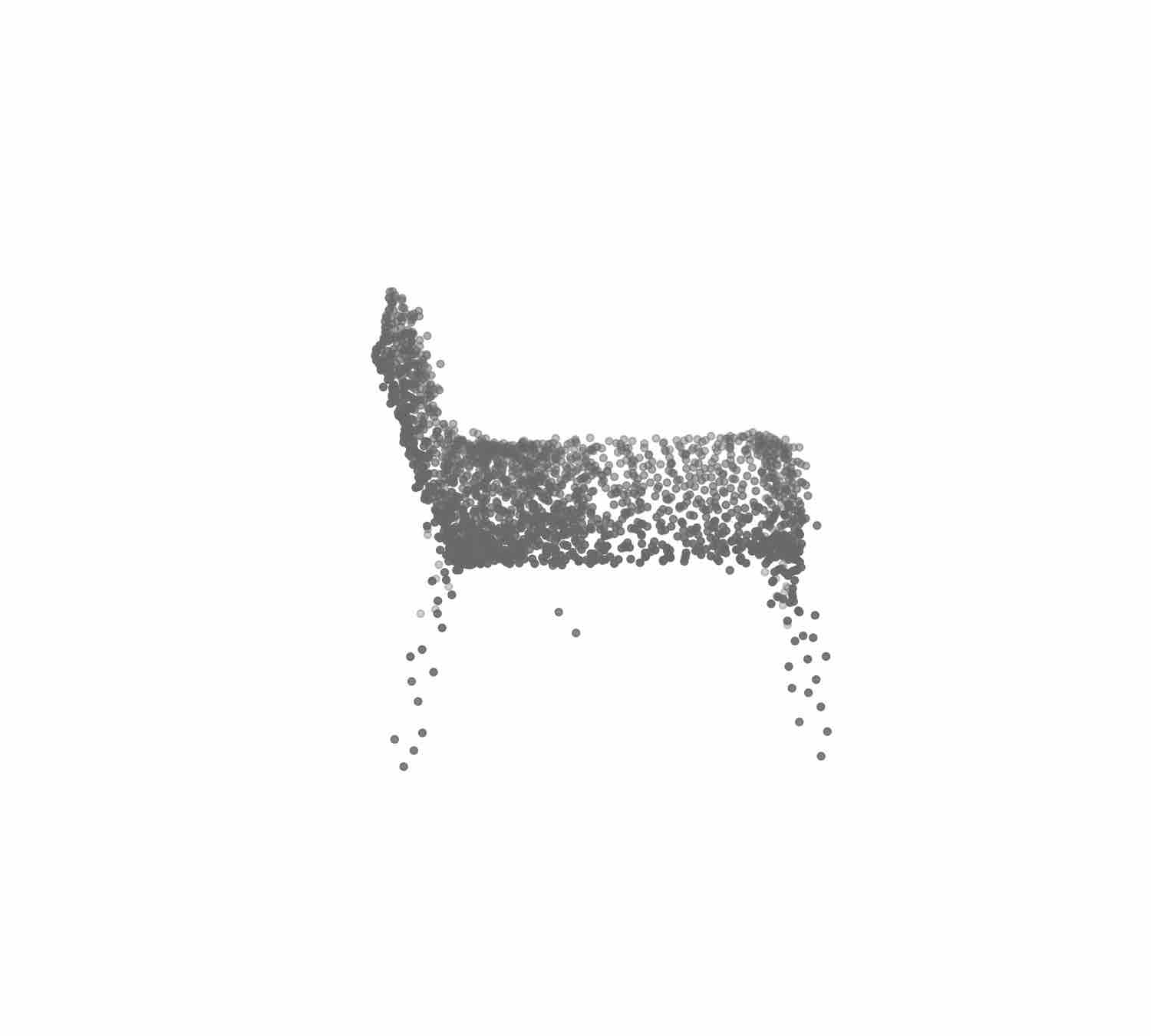} &
      
      \includegraphics[trim={4cm 4cm 7cm 4cm}, clip=true , scale=0.06] {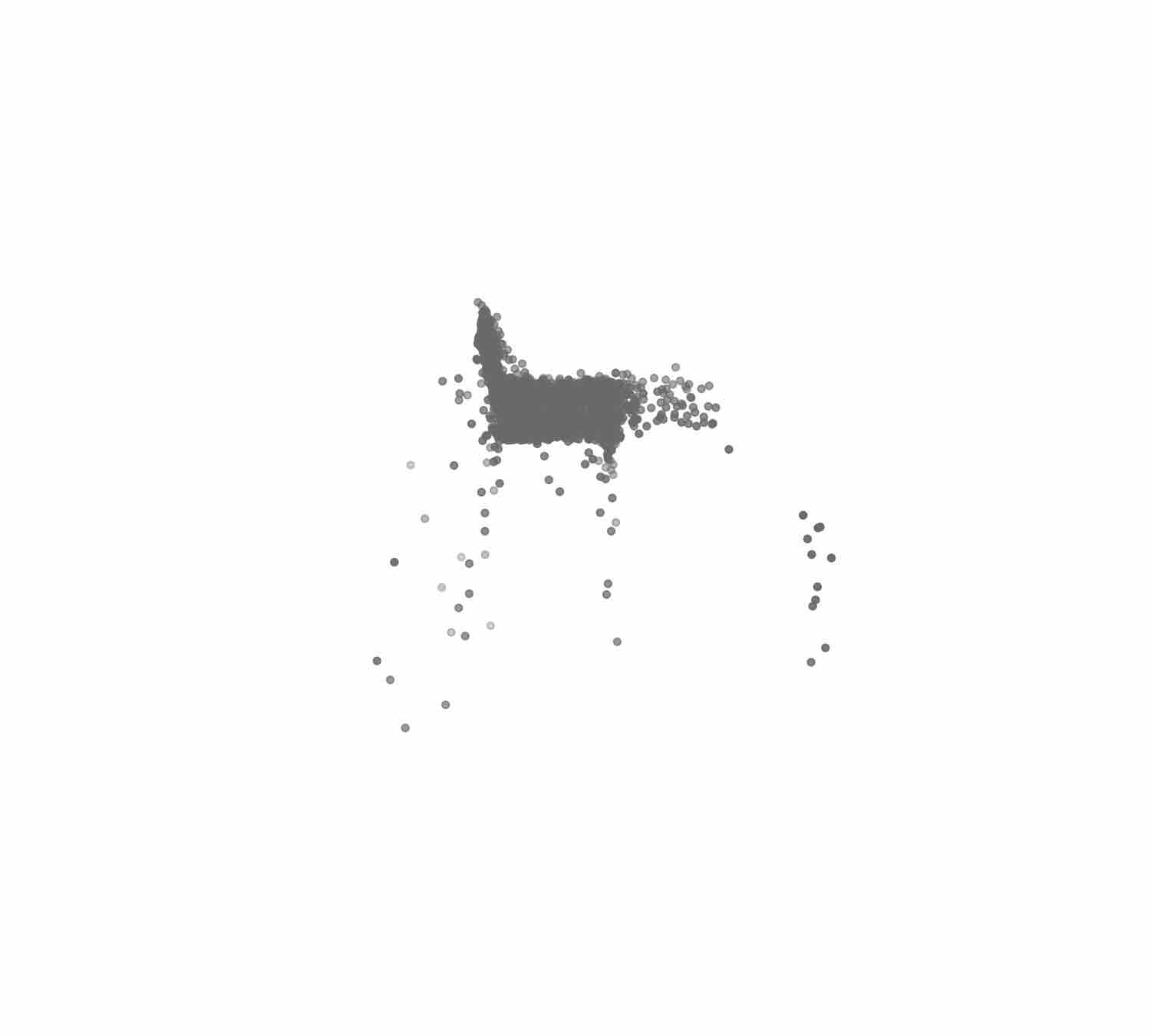} &
      
      \includegraphics[trim={4cm 4cm 7cm 4cm}, clip=true , scale=0.06] {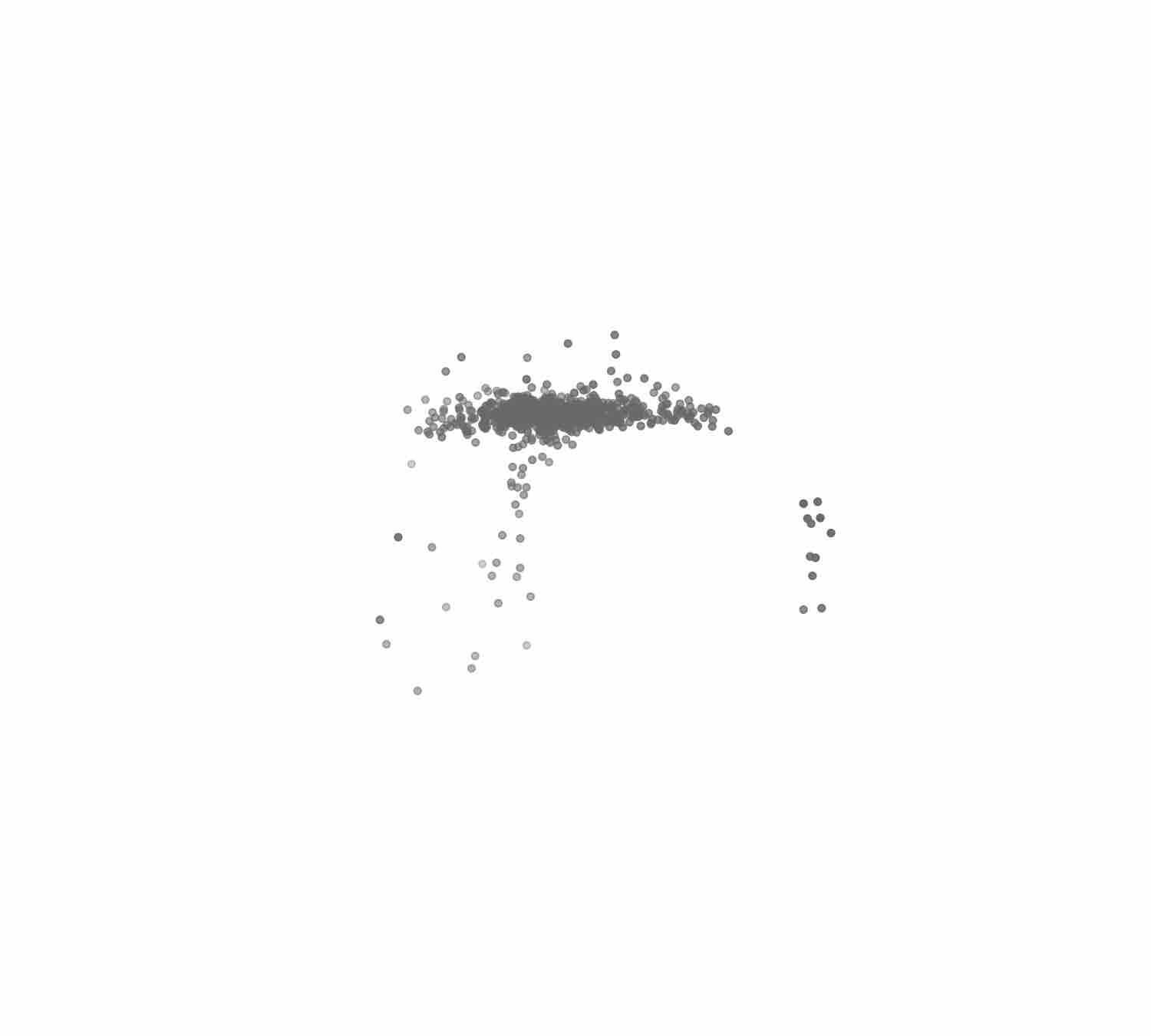}  \\
      \hline
    \end{tabular}
  \end{center}
  \caption{\label{tab:fine_grained_figures_adjacency}\textbf{Tradeoffs between spatial smoothness and graph smoothness (based on the graph-adjacency-matrix-based filtering).} We plot  the reconstructed point clouds as a function of $\alpha$; that is, $\X_\alpha = \left( (1-\alpha) \Id + \alpha \Adj \right) \X'$, where  $\X'$ is the coarse reconstruction produced by the folding module. When $\alpha = 0$, the reconstruction only depends on the folding module, which promotes the spatial smoothness and implies the dimension of the underlying surface is close to $2$; when $\alpha = 1$, the reconstruction promotes the graph smoothness and implies the dimension of the underlying surface is close to $3$. Here we use graph Haar filter, which is equivalent to $\alpha = 0.5$; it combines both spatial and graph smoothness. The intrinsic dimension of underlying surface is around $2.5$~\cite{manifold_dim}.}
\end{table*}

\begin{table*}[htb!]
  \begin{center}
    \begin{tabular}{c | c  c  c  c  c }
      \hline 
    Original  & Spatial smoothness &  & Output  &  & Graph smoothness  \\
      & $\alpha = 0$ & $\alpha = 0.25$ & $\alpha = 0.5$ & $\alpha = 0.75$& $\alpha = 1$  \\
      \hline
      \includegraphics[trim={4cm 4cm 7cm 4cm}, clip=true , scale=0.12] {figures/shapenet_core_class0/input_pc_5.jpg} & 
      \includegraphics[trim={4cm 4cm 7cm 4cm}, clip=true , scale=0.06] {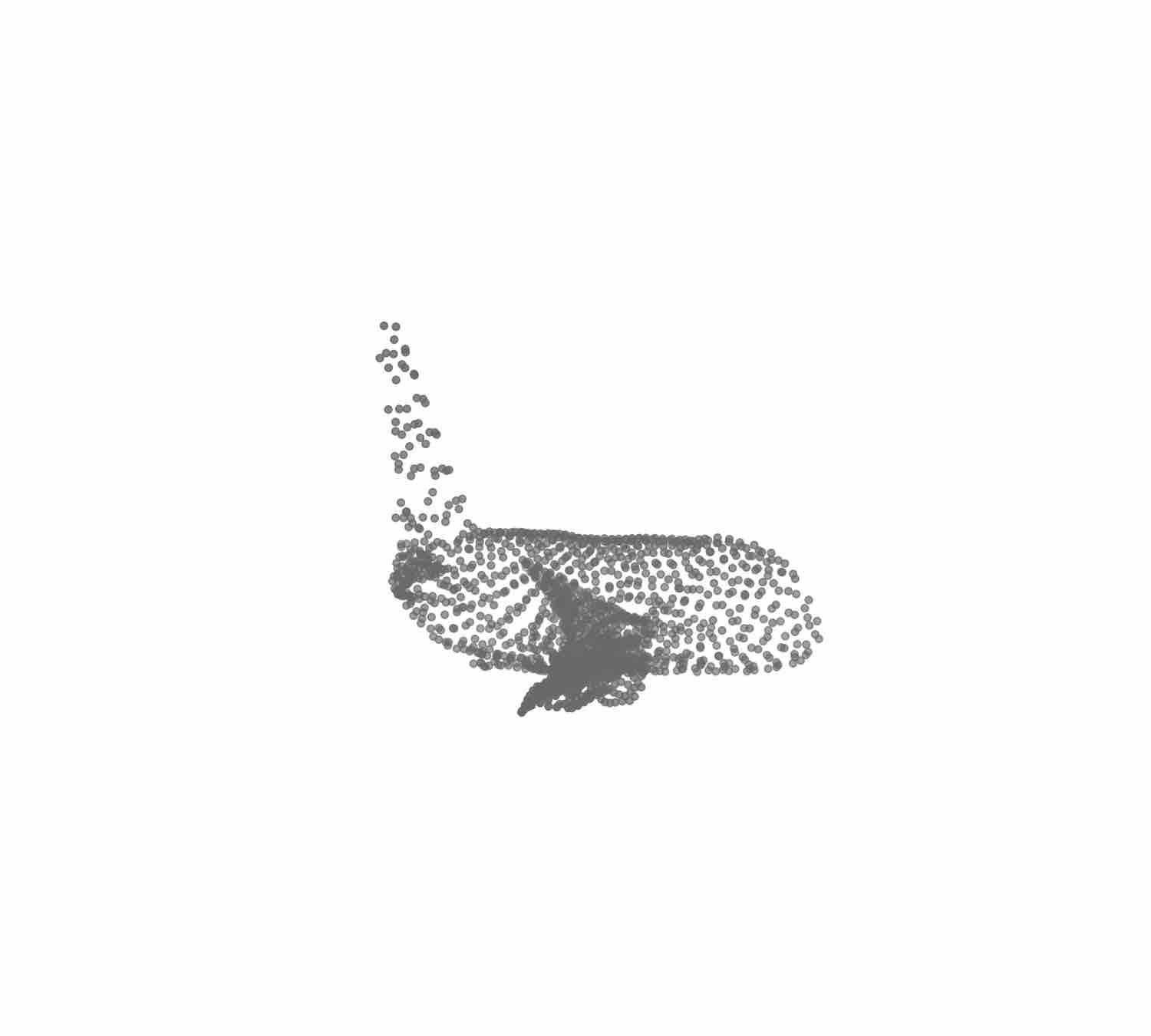} & 
      
      \includegraphics[trim={4cm 4cm 7cm 4cm}, clip=true , scale=0.06] {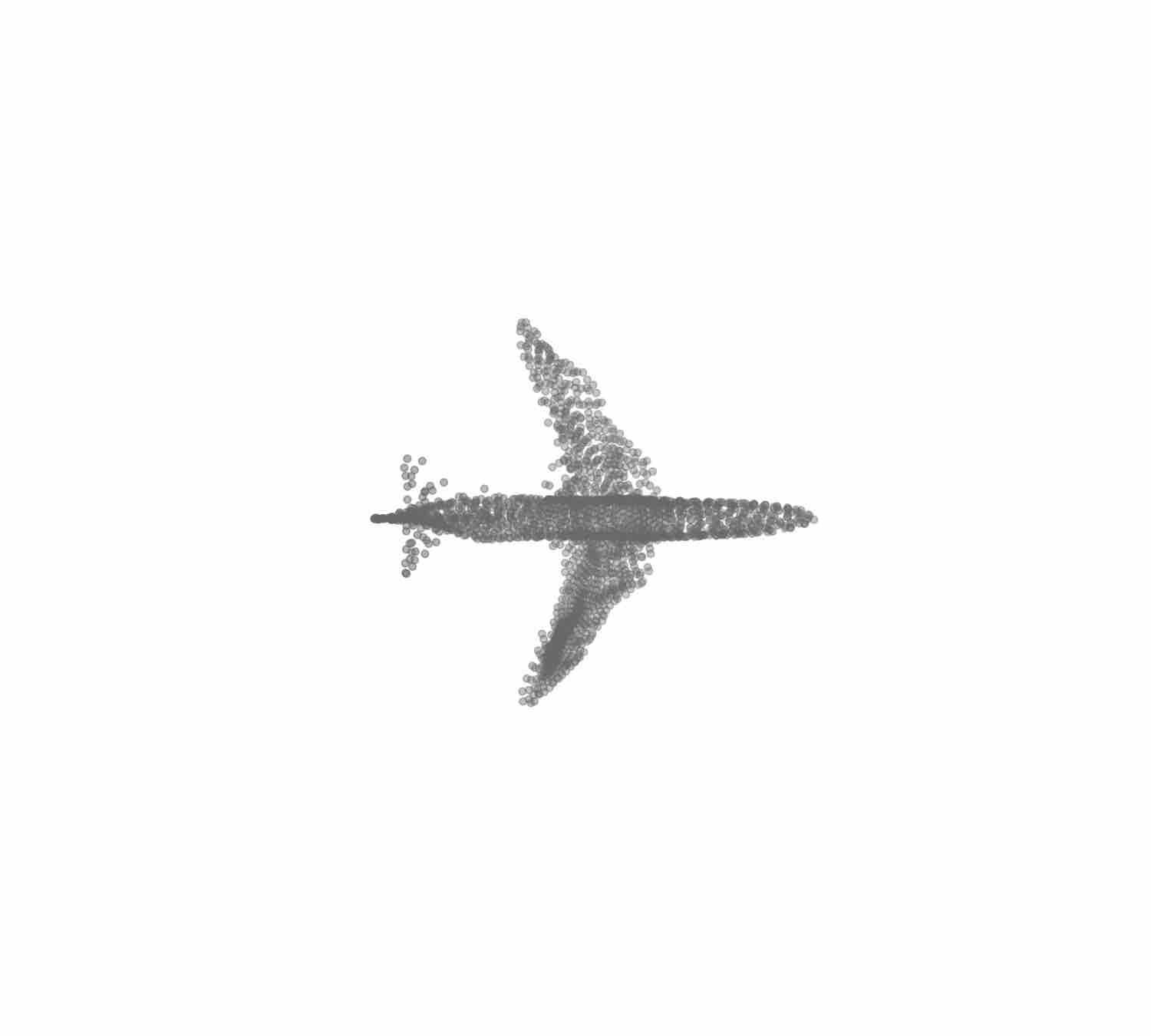} & 
      \includegraphics[trim={4cm 4cm 7cm 4cm}, clip=true , scale=0.06] {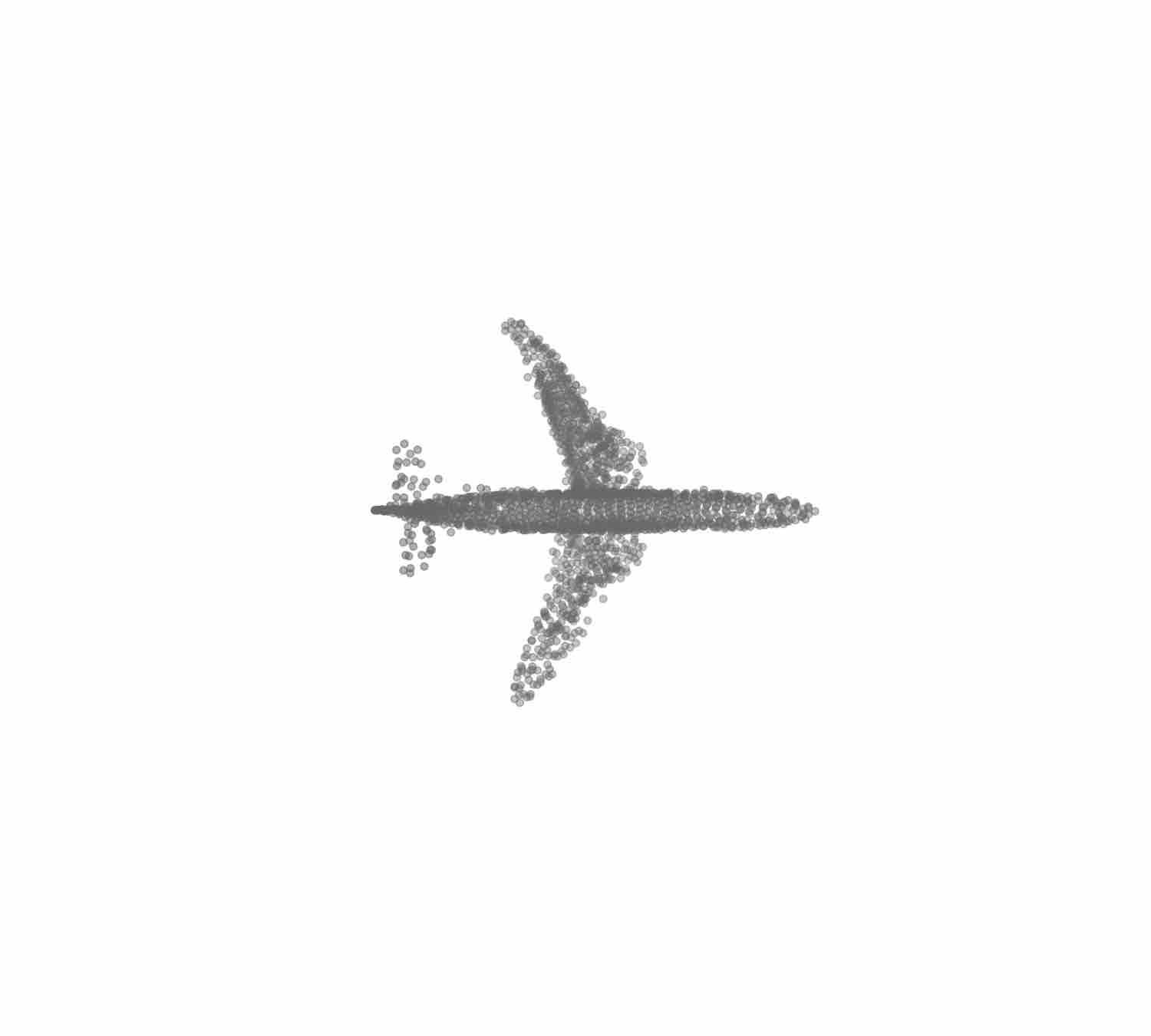} &
      
      \includegraphics[trim={4cm 4cm 7cm 4cm}, clip=true , scale=0.06] {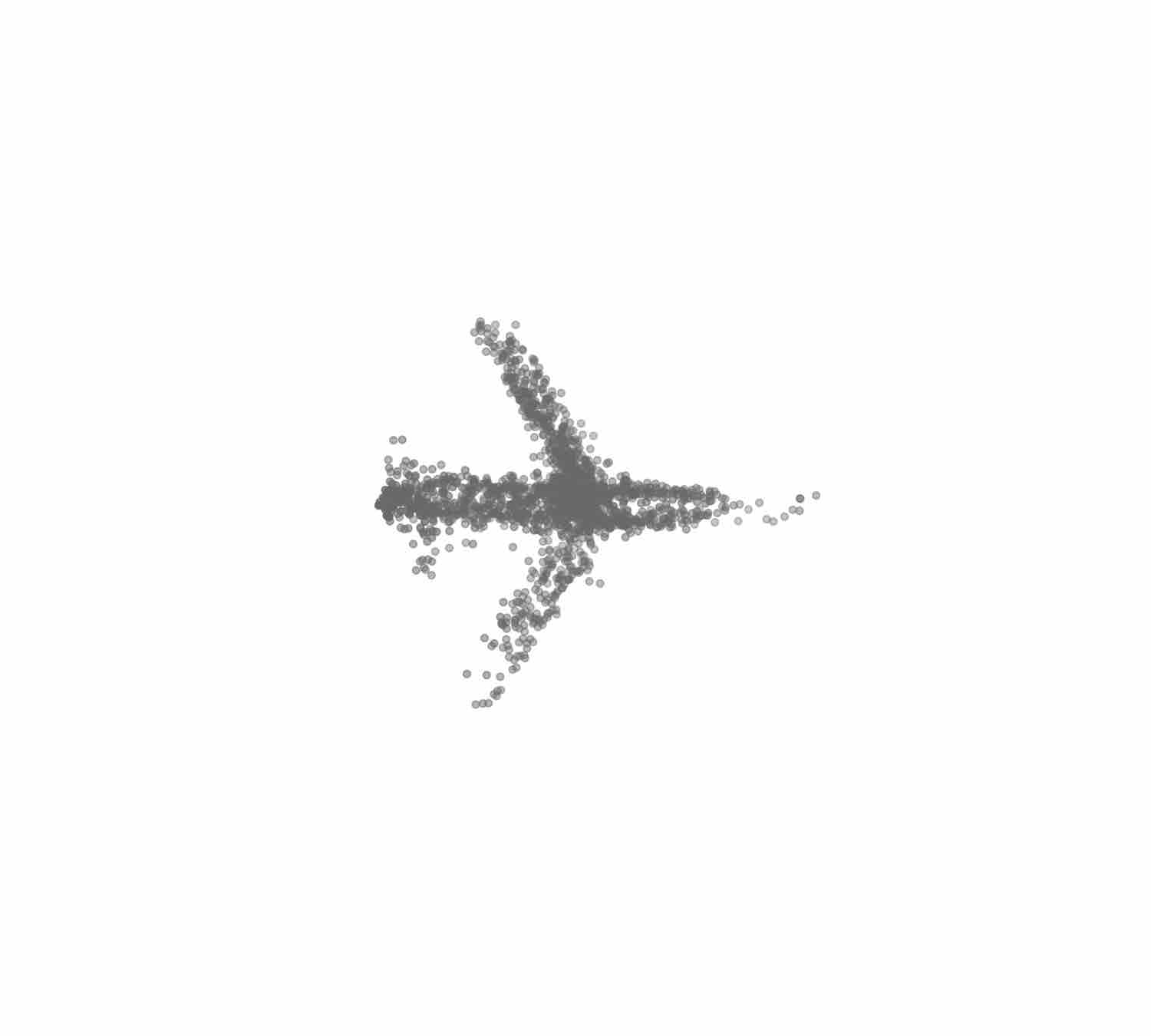} & 
      
      \includegraphics[trim={4cm 4cm 7cm 4cm}, clip=true , scale=0.06] {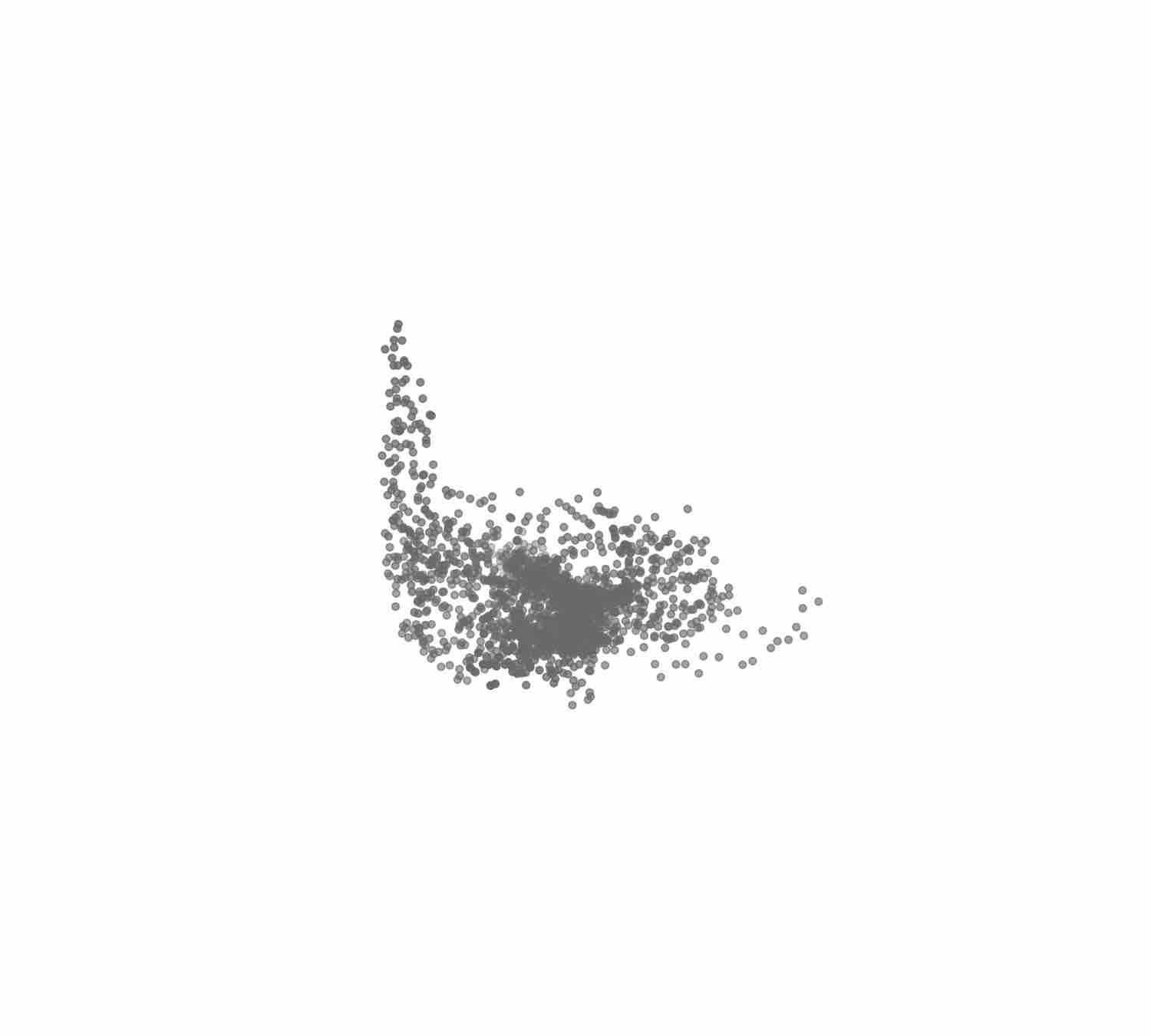}  \\
      
      \includegraphics[trim={4cm 4cm 7cm 4cm}, clip=true , scale=0.06] {figures/shapenet_core_class4/input_pc_2.jpg} & 
      \includegraphics[trim={4cm 4cm 7cm 4cm}, clip=true , scale=0.06] {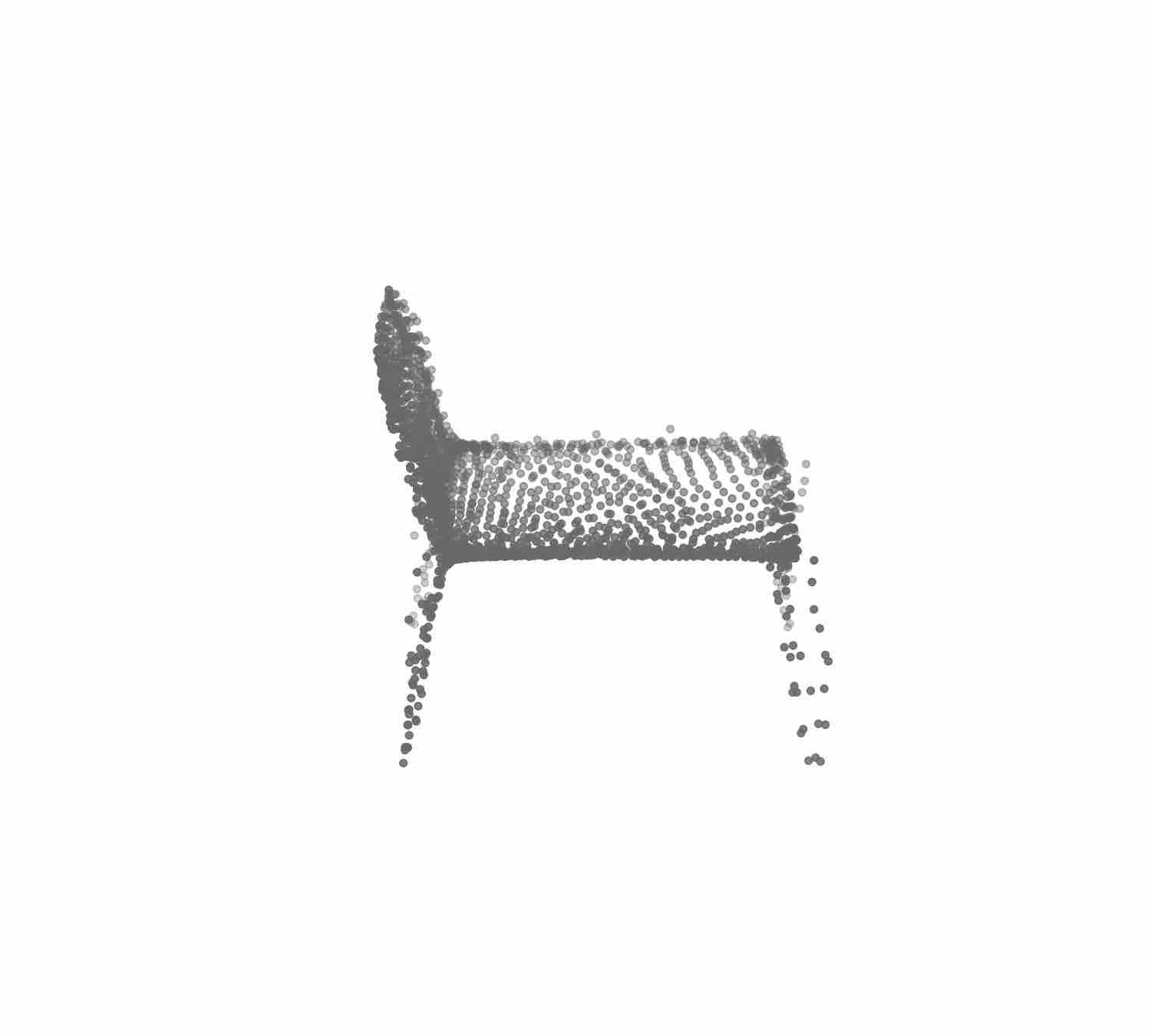} & 
      
      \includegraphics[trim={4cm 4cm 7cm 4cm}, clip=true , scale=0.06] {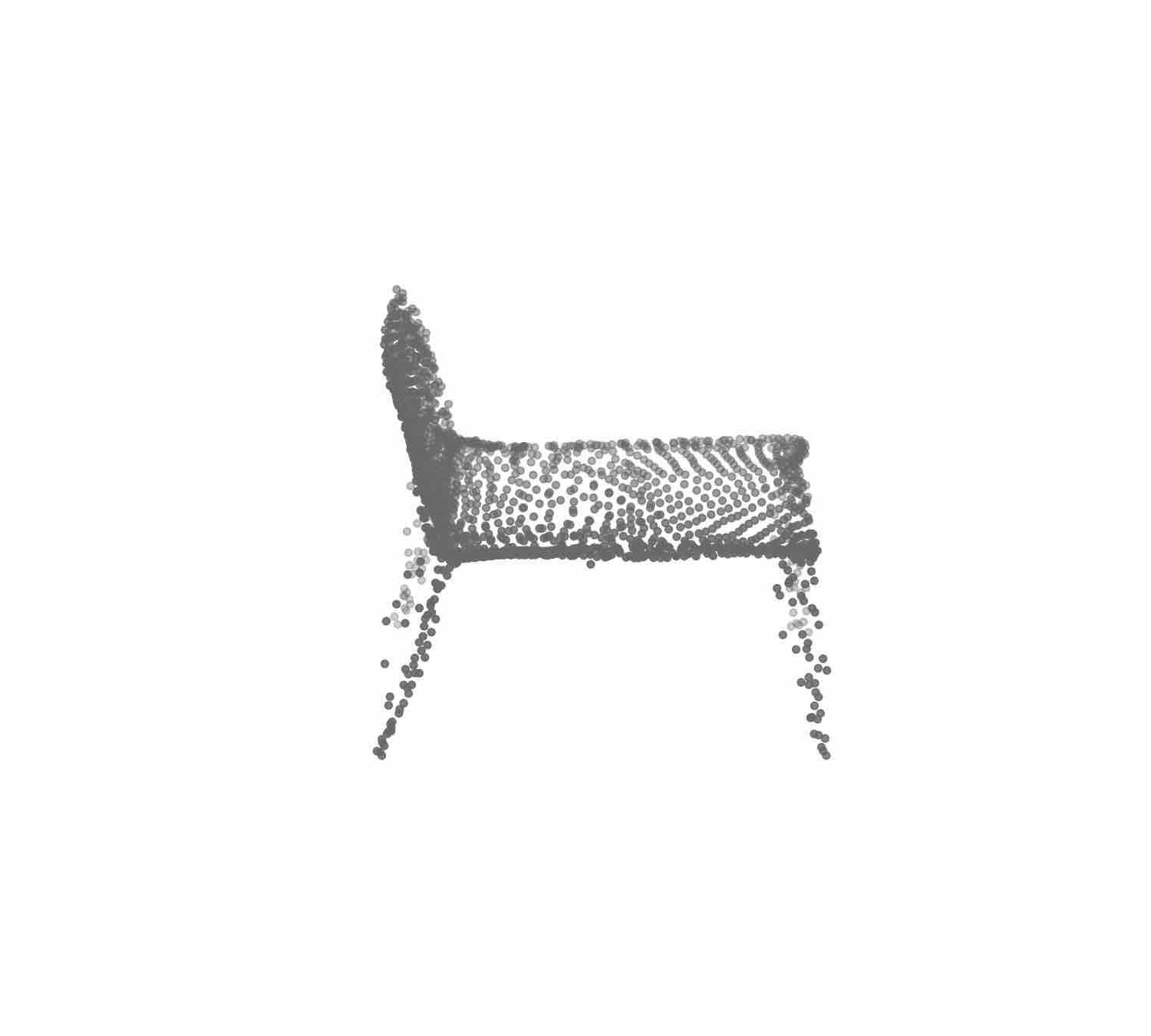} & 
      
      \includegraphics[trim={4cm 4cm 7cm 4cm}, clip=true , scale=0.06] {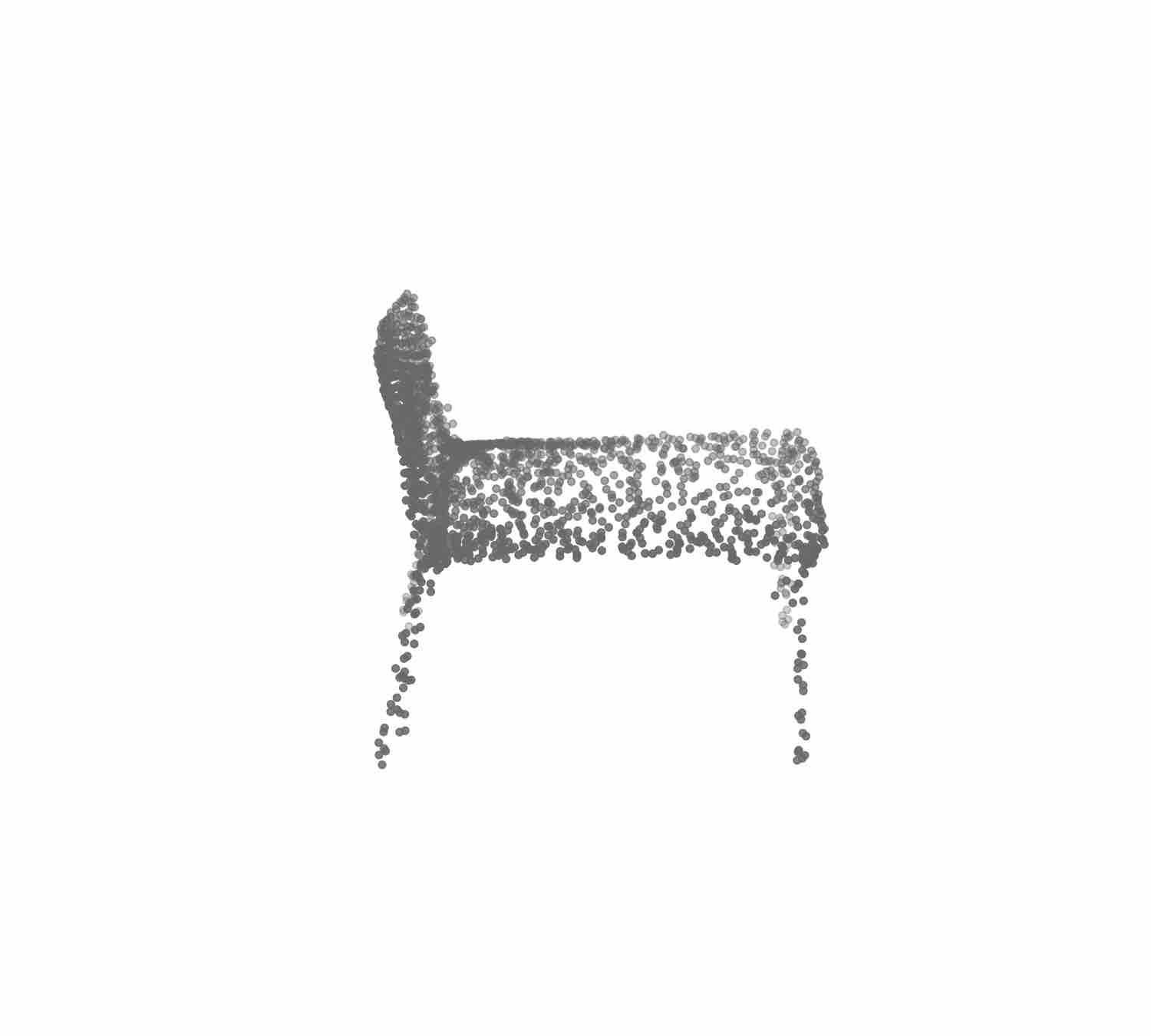} &
      
      \includegraphics[trim={4cm 4cm 7cm 4cm}, clip=true , scale=0.06] {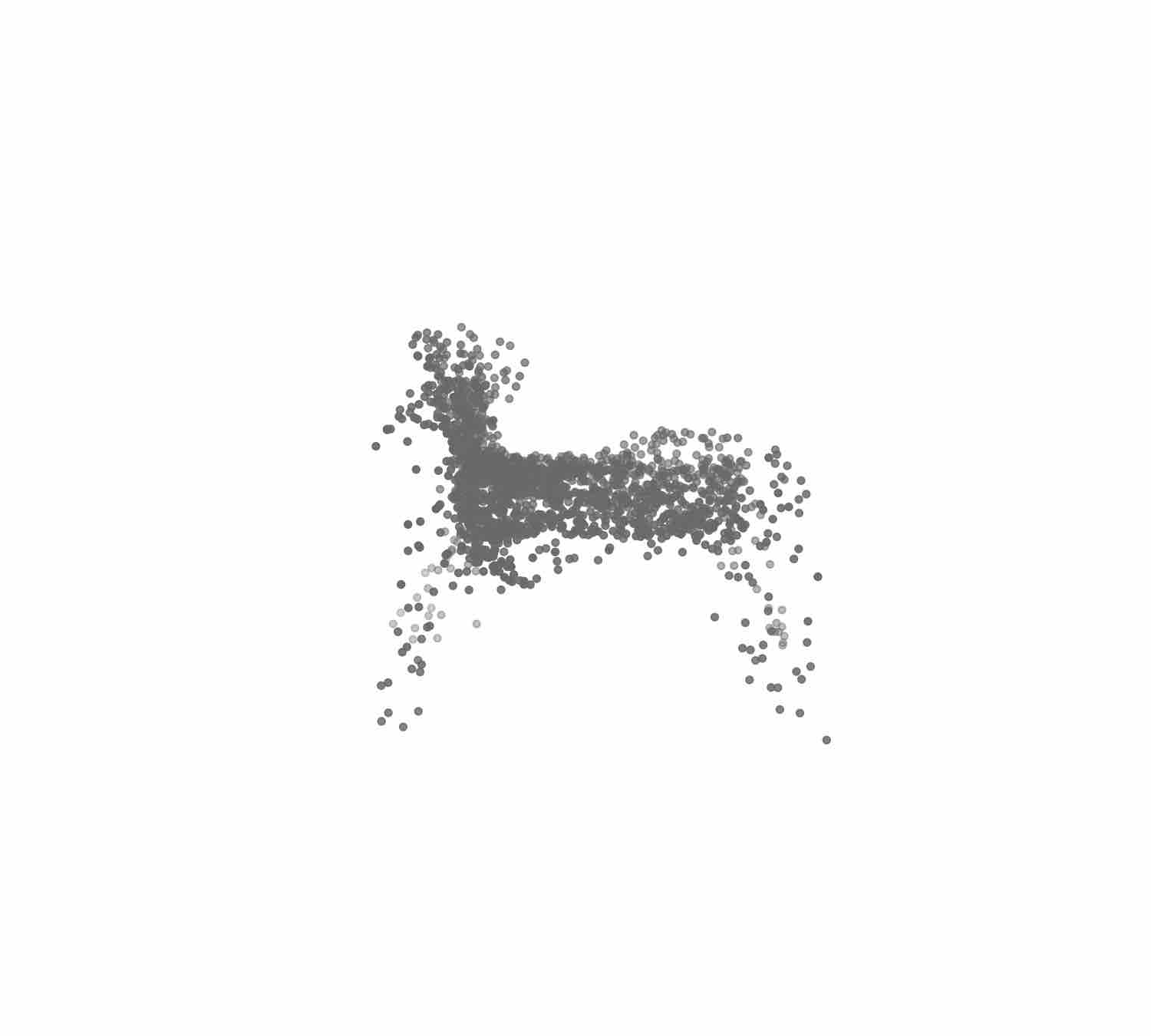} &
      
      \includegraphics[trim={4cm 4cm 7cm 4cm}, clip=true , scale=0.06] {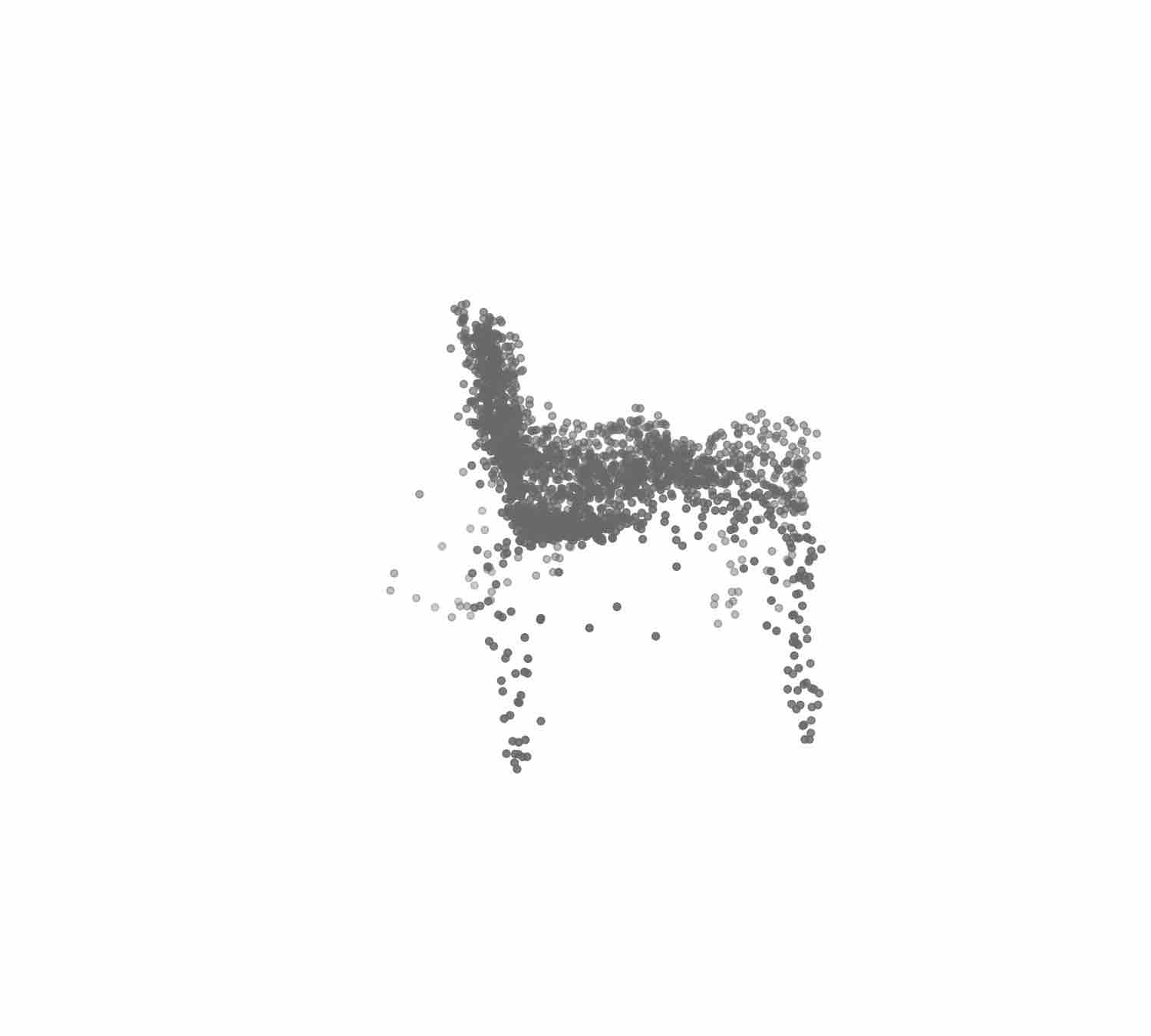}  \\
      \hline
    \end{tabular}
  \end{center}
  \caption{\label{tab:fine_grained_figures_lap}\textbf{Tradeoffs between spatial smoothness and graph smoothness (based on the graph-Laplacian-matrix-based filtering).} We plot the reconstructed point clouds as a function of $\alpha$; that is, $\X_\alpha = \left( \mu \Id +   \mathcal{L}  \right)^{-2\alpha}\X'$, where  $\X'$ is the coarse reconstruction produced by the folding module. When $\alpha = 0$, the reconstruction only depends on the folding module, which promotes the spatial smoothness, but loses some details; when $\alpha = 1$, the reconstruction promotes the graph smoothness, but less meaningful in the 3D space. Here the graph-Laplacian-matrix based filter~\eqref{eq:lap_filter} is equivalent to $\alpha = 0.5$; it combines both spatial and graph smoothness.}
\end{table*}

\subsection{Tradeoff between spatial and graph smoothness}
Tables~\ref{tab:fine_grained_figures_adjacency} and~\ref{tab:fine_grained_figures_lap} show the tradeoffs between spatial and graph smoothness based on graph adjacency matrix and graph Laplacian matrix, respectively. For the graph-adjacency-matrix based filter,  we plot  the reconstructed point clouds as a function of $\alpha$; that is, $\X_\alpha = \left( (1-\alpha) \Id + \alpha \Adj \right) \X'$, where  $\X'$ is the coarse reconstruction produced by the folding module. For the graph-Laplacian-matrix based filter,  we plot the reconstructed point clouds as a function of $\alpha$; that is, $\X_\alpha = \left( \mu \Id +   \mathcal{L}  \right)^{-2\alpha}\X'$. In both cases, when $\alpha = 0$, the reconstruction only depends on the folding module, which promotes the spatial smoothness and implies the intrinsic dimension of the underlying surface is close to $2$; when $\alpha = 1$, the reconstruction promotes the graph smoothness and implies the intrinsic dimension of the underlying surface is close to $3$. Here we use graph Haar filter, which is equivalent to $\alpha = 0.5$; it combines both spatial and graph smoothness. The intrinsic dimension of underlying surface is around $2.5$~\cite{manifold_dim}.

\section{Conclusions}
\label{sec:conclusions}
We propose an end-to-end deep autoencoder that achieves cutting-edge performance in unsupervised learning of 3D point clouds. The main novelties are that (i) we propose the folding module to fold a canonical 2D flat to the underlying surface of a 3D point cloud; (ii) we propose the graph-topology-inference module to model pairwise relationships between 3D points, pushing the latent code to preserve both coordinates and pairwise relationships of 3D points; (iii) we propose the graph-filtering module, which refines the coarse reconstruction to obtain the final reconstruction. We further provide an upper bound for the reconstruction loss and show that graph filtering lowers the upper bound. In the experiments, we validate the proposed networks in 3D point clouds reconstruction, visualization, and transfer classification. The experimental results show that (i) the proposed networks outperform state-of-the-art methods in the quantitative tasks; (ii) a graph topology can be inferred without specific supervision; and (iii) graph filtering improves the performances by reconstructing details.

\bibliographystyle{IEEEbib}
\bibliography{bibl_jelena}

\begin{thebibliography}{10}

\bibitem{YangFST:18}
Y.~Yang, C.~Feng, Y.~Shen, and D.~Tian,
\newblock ``Foldingnet: Point cloud auto-encoder via deep grid deformation,''
\newblock in {\em Proc. IEEE Conf. on Computer Vision and Pattern Recognition
  (CVPR)}, 2018, vol.~3.

\bibitem{meshlab}
P.~Cignoni, M.~Callieri, M.~Corsini, M.~Dellepiane, F.~Ganovelli, and
  G.~Ranzuglia,
\newblock ``Meshlab: an open-source mesh processing tool.,''
\newblock in {\em Eurographics Italian chapter conference}, 2008, vol. 2008,
  pp. 129--136.

\bibitem{PCL}
R.~B. Rusu and S.~Cousins,
\newblock ``3{D} is here: Point cloud library,''
\newblock in {\em {ICRA}}. 2011, pp. 1--4, IEEE.

\bibitem{3dmultiview}
Furukawa Y and J.~Ponce,
\newblock ``Accurate, dense, and robust multiview stereopsis,''
\newblock {\em IEEE transactions on pattern analysis and machine intelligence},
  vol. 32, no. 8, pp. 1362--1376, 2010.

\bibitem{3dsensor}
J.~Han, L.~Shao, D.~Xu, and J.~Shotton,
\newblock ``Enhanced computer vision with microsoft kinect sensor: A review,''
\newblock {\em IEEE transactions on cybernetics}, vol. 43, no. 5, pp.
  1318--1334, 2013.

\bibitem{duanMWP}
C.~Duan, S.~Chen, and J.~Kovacevic,
\newblock ``Weighted multi-projection: 3d point cloud denoising with estimated
  tangent planes,''
\newblock in {\em 2018 IEEE Global Conference on Signal and Information
  Processing (GlobalSIP)}. IEEE, 2018, pp. 725--729.

\bibitem{2016pointnet}
C.~Qi, H.~Su, K.~Mo, and L.~Guibas,
\newblock ``Pointnet: Deep learning on point sets for 3d classification and
  segmentation,''
\newblock {\em Proc. Computer Vision and Pattern Recognition (CVPR), IEEE},
  vol. 1, no. 2, pp. 4, 2017.

\bibitem{vox3}
Y.~Zhou and O.~Tuzel,
\newblock ``Voxelnet: End-to-end learning for point cloud based 3d object
  detection,''
\newblock {\em arXiv preprint arXiv:1711.06396}, 2017.

\bibitem{recognition}
R.~Klokov and V.~Lempitsky,
\newblock ``Escape from cells: Deep kd-networks for the recognition of 3d point
  cloud models,''
\newblock in {\em Proceedings of the IEEE International Conference on Computer
  Vision}, 2017, pp. 863--872.

\bibitem{YuLFCH:18}
L.~Yu, X.~Li, C-W. Fu, D.~Cohen-Or, and P-A. Heng,
\newblock ``Pu-net: Point cloud upsampling network,''
\newblock in {\em Proceedings of the IEEE Conference on Computer Vision and
  Pattern Recognition}, 2018, pp. 2790--2799.

\bibitem{3dgan}
J.~Wu, C.~Zhang, T.~Xue, B.~Freeman, and J.~Tenenbaum,
\newblock ``Learning a probabilistic latent space of object shapes via 3d
  generative-adversarial modeling,''
\newblock in {\em Advances in Neural Information Processing Systems}, 2016, pp.
  82--90.

\bibitem{AchlioptasDMG:17}
P.~Achlioptas, O.~Diamanti, I.~Mitliagkas, and L.~Guibas,
\newblock ``Representation learning and adversarial generation of 3d point
  clouds,''
\newblock {\em arXiv preprint arXiv:1707.02392}, 2017.

\bibitem{Bishop:06}
C.~M. Bishop,
\newblock {\em Pattern Recognition and Machine Learning},
\newblock Information Science and Statistics. Springer, 2006.

\bibitem{BerryBLPP:07}
Michael~W. Berry, Murray Browne, Amy~N. Langville, V.~Paul Pauca, and Robert~J.
  Plemmons,
\newblock ``Algorithms and applications for approximate nonnegative matrix
  factorization,''
\newblock {\em Comput. Stat. Data Anal.}, vol. 52, no. 1, pp. 155--173, Sept.
  2007.

\bibitem{Goodfellow-et-al-2016}
I.~Goodfellow, Y.~Bengio, and A.~Courville,
\newblock {\em Deep Learning},
\newblock MIT Press, 2016,
\newblock \url{http://www.deeplearningbook.org}.

\bibitem{sph}
M.~Kazhdan, T.~Funkhouser, and S.~Rusinkiewicz,
\newblock ``Rotation invariant spherical harmonic representation of 3d shape
  descriptors,''
\newblock in {\em Symposium on geometry processing}, 2003, vol.~6, pp.
  156--164.

\bibitem{lfd}
D-Y. Chen, X-P. Tian, Y-T. Shen, and M.~Ouhyoung,
\newblock ``On visual similarity based 3d model retrieval,''
\newblock in {\em Computer graphics forum}. Wiley Online Library, 2003,
  vol.~22, pp. 223--232.

\bibitem{tlnetwork}
R.~Girdhar, D.~F. Fouhey, M.~Rodriguez, and A.~Gupta,
\newblock ``Learning a predictable and generative vector representation for
  objects,''
\newblock in {\em European Conference on Computer Vision}. Springer, 2016, pp.
  484--499.

\bibitem{vconv}
A.~Sharma, O.~Grau, and M.~Fritz,
\newblock ``Vconv-dae: Deep volumetric shape learning without object labels,''
\newblock in {\em European Conference on Computer Vision}. Springer, 2016, pp.
  236--250.

\bibitem{FanSG:17}
H.~Fan, H.~Su, and L.~J. Guibas,
\newblock ``A point set generation network for 3d object reconstruction from a
  single image,''
\newblock in {\em {CVPR}}, 2017, pp. 2463--2471.

\bibitem{innergan}
Z.~Han, M.~Shang, Y-S. Liu, and M.~Zwicker,
\newblock ``View inter-prediction gan: Unsupervised representation learning for
  3d shapes by learning global shape memories to support local view
  predictions,''
\newblock {\em arXiv preprint arXiv:1811.02744}, 2018.

\bibitem{abs-1901-05103}
J.~J. Park, P.~Florence, J.~Straub, R.~A. Newcombe, and S.~Lovegrove,
\newblock ``Deep{SDF}: Learning continuous signed distance functions for shape
  representation,''
\newblock {\em CoRR}, vol. abs/1901.05103, 2019.

\bibitem{ShumanNFOV:13}
D.~I. Shuman, S.~K. Narang, P.~Frossard, A.~Ortega, and P.~Vandergheynst,
\newblock ``The emerging field of signal processing on graphs: {E}xtending
  high-dimensional data analysis to networks and other irregular domains,''
\newblock {\em IEEE Signal Process. Mag.}, vol. 30, pp. 83--98, May 2013.

\bibitem{SandryhailaM:14}
A.~Sandryhaila and J.~M.~F. Moura,
\newblock ``Big data processing with signal processing on graphs,''
\newblock {\em IEEE Signal Process. Mag.}, vol. 31, no. 5, pp. 80--90, Sept.
  2014.

\bibitem{OrtegaFKMV:18}
A.~Ortega, P.~Frossard, J.~Kova{\v c}evi{\'c}, J.~M.~F. Moura, and
  P.~Vandergheynst,
\newblock ``Graph signal processing: Overview, challenges, and applications,''
\newblock {\em Proceedings of the {IEEE}}, vol. 106, no. 5, pp. 808--828, 2018.

\bibitem{ShumanFV:16}
D.~I. Shuman, M.~J. Faraji, , and P.~Vandergheynst,
\newblock ``A multiscale pyramid transform for graph signals,''
\newblock {\em IEEE Trans. Signal Process.}, vol. 64, pp. 2119--2134, Apr.
  2016.

\bibitem{ChenSK:18}
S.~Chen, A.~Singh, and J.~Kova{\v c}evi{\'c},
\newblock ``Multiresolution representations for piecewise-smooth signals on
  graphs,''
\newblock {\em arXiv preprint arXiv:1803.02944}, pp. 6246--6253, 2018.

\bibitem{ChenVSK:15}
S.~Chen, R.~Varma, A.~Sandryhaila, and J.~Kova{\v c}evi{\'c},
\newblock ``Discrete signal processing on graphs: {S}ampling theory,''
\newblock {\em IEEE Trans. Signal Process.}, vol. 63, no. 24, pp. 6510--6523,
  Dec. 2015.

\bibitem{AnisGO:15}
A.~Anis, A.~Gadde, and A.~Ortega,
\newblock ``Efficient sampling set selection for bandlimited graph signals
  using graph spectral proxies,''
\newblock {\em IEEE Trans. Signal Process.}, 2015,
\newblock Submitted.

\bibitem{ChenVSK:16}
S.~Chen, R.~Varma, A.~Singh, and J.~Kova{\v c}evi{\'c},
\newblock ``Signal recovery on graphs: Fundamental limits of sampling
  strategies,''
\newblock {\em {IEEE} Trans. Signal and Information Processing over Networks},
  vol. 2, no. 4, pp. 539--554, 2016.

\bibitem{NarangGO:13}
S.~K. Narang, Akshay Gadde, and Antonio Ortega,
\newblock ``Signal processing techniques for interpolation in graph structured
  data,''
\newblock in {\em Proc. IEEE Int. Conf. Acoust., Speech, Signal Process.},
  Vancouver, May 2013, pp. 5445--5449.

\bibitem{ChenSMK:14}
S.~Chen, A.~Sandryhaila, J.~M.~F. Moura, and J.~Kova{\v c}evi{\'c},
\newblock ``Signal recovery on graphs: {Variation} minimization,''
\newblock {\em IEEE Trans. Signal Process.}, vol. 63, no. 17, pp. 4609--4624,
  Sept. 2015.

\bibitem{NarangO:12}
S.~K. Narang and A.~Ortega,
\newblock ``Perfect reconstruction two-channel wavelet filter banks for graph
  structured data,''
\newblock {\em IEEE Trans. Signal Process.}, vol. 60, pp. 2786--2799, June
  2012.

\bibitem{ChenSMK:14a}
S.~Chen, A.~Sandryhaila, J.~M.~F. Moura, and J.~Kova{\v c}evi{\'c},
\newblock ``Signal denoising on graphs via graph filtering,''
\newblock in {\em Proc. IEEE Glob. Conf. Signal Information Process.}, Atlanta,
  GA, Dec. 2014, pp. 872--876.

\bibitem{TremblayB:16}
N.~Tremblay and P.~Borgnat,
\newblock ``Subgraph-based filterbanks for graph signals,''
\newblock {\em IEEE Trans. Signal Process.}, vol. 64, pp. 3827--3840, Mar.
  2016.

\bibitem{NarangSO:10}
S.l~K. Narang, G.~Shen, and A.~Ortega,
\newblock ``Unidirectional graph-based wavelet transforms for efficient data
  gathering in sensor networks,''
\newblock in {\em Proc. IEEE Int. Conf. Acoust., Speech, Signal Process.},
  Dallas, TX, Mar. 2010, pp. 2902--2905.

\bibitem{DongTRF:18}
X.~Dong, D.~Thanou, M.~Rabbat, and P.~Frossard,
\newblock ``Learning graphs from data: {A} signal representation perspective,''
\newblock {\em CoRR}, vol. abs/1806.00848, 2018.

\bibitem{GamaMLR:19}
F.~Gama, A.~G. Marques, G.~Leus, and A.~Ribeiro,
\newblock ``Convolutional neural network architectures for signals supported on
  graphs,''
\newblock {\em {IEEE} Trans. Signal Processing}, vol. 67, no. 4, pp.
  1034--1049, 2019.

\bibitem{NiuCGTSK:18}
S.~Niu, S.~Chen, H.~Guo, C.~Targonski, M.~C. Smith, and J.~Kova{\v c}evi{\'c},
\newblock ``Generalized value iteration networks: Life beyond lattices,''
\newblock in {\em Proceedings of the Thirty-Second {AAAI} Conference on
  Artificial Intelligence, New Orleans, Louisiana, USA, February 2-7, 2018},
  2018, pp. 6246--6253.

\bibitem{ChenTFVK:18}
S.~Chen, D.~Tian, C.~Feng, A.~Vetro, and J.~Kova{\v c}evi{\'c},
\newblock ``Fast resampling of three-dimensional point clouds via graphs,''
\newblock {\em {IEEE} Trans. Signal Processing}, vol. 66, no. 3, pp. 666--681,
  2018.

\bibitem{ZhangR:18}
Y.~Zhang and M.~Rabbat,
\newblock ``A graph-cnn for 3d point cloud classification,''
\newblock {\em CoRR}, vol. abs/1812.01711, 2018.

\bibitem{maturana2015voxnet}
D.~Maturana and S.~Scherer,
\newblock ``Voxnet: A 3d convolutional neural network for real-time object
  recognition,''
\newblock in {\em Intelligent Robots and Systems (IROS), 2015 IEEE/RSJ
  International Conference on}. IEEE, 2015, pp. 922--928.

\bibitem{su2015multi}
H.~Su, S.~Maji, E.~Kalogerakis, and E.~Learned-Miller,
\newblock ``Multi-view convolutional neural networks for 3d shape
  recognition,''
\newblock in {\em Proceedings of the IEEE international conference on computer
  vision}, 2015, pp. 945--953.

\bibitem{graph_data1}
M.~M. Bronstein, J.~Bruna, Y.~LeCun, A.~Szlam, and P.~Vandergheynst,
\newblock ``Geometric deep learning: going beyond euclidean data,''
\newblock {\em IEEE Signal Processing Magazine}, vol. 34, no. 4, pp. 18--42,
  2017.

\bibitem{splatnet}
H.~Su, V.~Jampani, D.~Sun, S.~Maji, E.~Kalogerakis, M-H. Yang, and J.~Kautz,
\newblock ``Splatnet: Sparse lattice networks for point cloud processing,''
\newblock in {\em Proceedings of the IEEE Conference on Computer Vision and
  Pattern Recognition}, 2018, pp. 2530--2539.

\bibitem{graph_pc4}
Y.~Wang, Y.~Sun, Z.~Liu, S.~Sarma, M.~M. Bronstein, and J.~M. Solomon,
\newblock ``Dynamic graph cnn for learning on point clouds,''
\newblock {\em arXiv preprint arXiv:1801.07829}, 2018.

\bibitem{wang2018deep}
S.~Wang, S.~Suo, W-C. Ma, A.~Pokrovsky, and R.~Urtasun,
\newblock ``Deep parametric continuous convolutional neural networks,''
\newblock in {\em Proceedings of the IEEE Conference on Computer Vision and
  Pattern Recognition}, 2018, pp. 2589--2597.

\bibitem{tatarchenko2018tangent}
M.~Tatarchenko, J.~Park, V.~Koltun, and Q-Y. Zhou,
\newblock ``Tangent convolutions for dense prediction in 3d,''
\newblock in {\em Proceedings of the IEEE Conference on Computer Vision and
  Pattern Recognition}, 2018, pp. 3887--3896.

\bibitem{xu2018spidercnn}
Y.~Xu, T.~Fan, M.~Xu, L.~Zeng, and Y.~Qiao,
\newblock ``Spidercnn: Deep learning on point sets with parameterized
  convolutional filters,''
\newblock {\em arXiv preprint arXiv:1803.11527}, 2018.

\bibitem{li2018so}
J.~Li, M.~Ben Chen, and G.~H. Lee,
\newblock ``So-net: Self-organizing network for point cloud analysis,''
\newblock in {\em Proceedings of the IEEE Conference on Computer Vision and
  Pattern Recognition}, 2018, pp. 9397--9406.

\bibitem{vox1}
A.~Brock, T.~Lim, J.~M. Ritchie, and N.~Weston,
\newblock ``Generative and discriminative voxel modeling with convolutional
  neural networks,''
\newblock {\em arXiv preprint arXiv:1608.04236}, 2016.

\bibitem{choy20163d}
C.~Choy, D.~Xu, J.~Gwak, K.~Chen, and S.~Savarese,
\newblock ``3d-r2n2: A unified approach for single and multi-view 3d object
  reconstruction,''
\newblock in {\em Proceedings of the European Conference on Computer Vision
  ({ECCV})}, 2016.

\bibitem{deep_image1}
R.~Roveri, L.~Rahmann, A.~C. Oztireli, and M.~Gross,
\newblock ``A network architecture for point cloud classification via automatic
  depth images generation,''
\newblock in {\em Proceedings of the IEEE Conference on Computer Vision and
  Pattern Recognition}, 2018, pp. 4176--4184.

\bibitem{wang2018adaptive}
P-S. Wang, C-Y. Sun, Y.~Liu, and X.~Tong,
\newblock ``Adaptive o-cnn: A patch-based deep representation of 3d shapes,''
\newblock {\em arXiv preprint arXiv:1809.07917}, 2018.

\bibitem{AtlasNet}
T.~Groueix, M.~Fisher, V.~G. Kim, B.~Russell, and M.~Aubry,
\newblock ``{AtlasNet: A Papier-M\^ach\'e Approach to Learning 3D Surface
  Generation},''
\newblock in {\em Proceedings IEEE Conf. on Computer Vision and Pattern
  Recognition (CVPR)}, 2018.

\bibitem{sgpn}
W.~Wang, R.~Yu, Q.~Huang, and U.~Neumann,
\newblock ``Sgpn: Similarity group proposal network for 3d point cloud instance
  segmentation,''
\newblock in {\em Proceedings of the IEEE Conference on Computer Vision and
  Pattern Recognition}, 2018, pp. 2569--2578.

\bibitem{huang2018recurrent}
Q.~Huang, W.~Wang, and U.~Neumann,
\newblock ``Recurrent slice networks for 3d segmentation of point clouds,''
\newblock in {\em Proceedings of the IEEE Conference on Computer Vision and
  Pattern Recognition}, 2018, pp. 2626--2635.

\bibitem{shen2018}
Y.~Shen, C.~Feng, Y.~Yang, and D.~Tian,
\newblock ``Mining point cloud local structures by kernel correlation and graph
  pooling,''
\newblock in {\em Proceedings of the IEEE Conference on Computer Vision and
  Pattern Recognition}, 2018, vol.~4.

\bibitem{Newman:10}
M.~Newman,
\newblock {\em Networks: An Introduction},
\newblock Oxford University Press, 2010.

\bibitem{ChenNLL:19}
S.~Chen, S.~Niu, T.~Lan, and B.~Liu,
\newblock ``Large-scale 3d point cloud representations via graph inception
  networks with applications to autonomous driving,''
\newblock in {\em {ICIP}}, 2019.

\bibitem{ThanouCF:16}
D.~Thanou, P.~A. Chou, and P.~Frossard,
\newblock ``Graph-based compression of dynamic 3d point cloud sequences,''
\newblock {\em {IEEE} Trans. Image Processing}, vol. 25, no. 4, pp. 1765--1778,
  2016.

\bibitem{ZengCNPY:18}
J.~Zeng, G.~Cheung, M.~Ng, J.~Pang, and C.~Yang,
\newblock ``3d point cloud denoising using graph laplacian regularization of a
  low dimensional manifold model,''
\newblock {\em CoRR}, vol. abs/1803.07252, 2018.

\bibitem{LozesEL:15}
F.~Lozes, A.~Elmoataz, and O.~Lezoray,
\newblock ``Pde-based graph signal processing for 3-d color point clouds :
  Opportunities for cultural herihe arts and found promising,''
\newblock {\em {IEEE} Signal Process. Mag.}, vol. 32, no. 4, pp. 103--111,
  2015.

\bibitem{VetterliKG:12}
M.~Vetterli, J.~Kova{\v c}evi{\'c}, and V.~K. Goyal,
\newblock {\em Foundations of Signal Processing},
\newblock Cambridge University Press, Cambridge, 2014,
\newblock http://foundationsofsignalprocessing.org.

\bibitem{Chung:96}
F.~R.~K. Chung,
\newblock {\em Spectral Graph Theory (CBMS Regional Conference Series in
  Mathematics, No. 92)},
\newblock Am. Math. Soc., 1996.

\bibitem{RudinOF:92}
L.~I. Rudin, Osher, and E.~Fatemi,
\newblock ``Nonlinear total variation based noise removal algorithms,''
\newblock {\em Physica D}, , no. 1--4, pp. 259--268, Nov. 1992.

\bibitem{shapenet}
A.~Chang, T.~Funkhouser, L.~Guibas, P.~Hanrahan, Q.~Huang, Z.~Li, S.~Savarese,
  M.~Savva, S.~Song, H.~Su, J.~Xiao, L.~Yi, and F.~Yu,
\newblock ``Shapenet: An information-rich 3d model repository,''
\newblock {\em arXiv preprint arXiv:1512.03012}, 2015.

\bibitem{yi2016scalable}
L.~Yi, V.~G. Kim, D.~Ceylan, I.~Shen, M.~Yan, H.~Su, C.~Lu, Q.~Huang,
  A.~Sheffer, L.~Guibas, et~al.,
\newblock ``A scalable active framework for region annotation in 3d shape
  collections,''
\newblock {\em ACM Transactions on Graphics (TOG)}, vol. 35, no. 6, pp. 210,
  2016.

\bibitem{WuSKYZTX:15}
Z.~Wu, S.~Song, A.~Khosla, F.~Yu, L.~Zhang, X.~Tang, and J.~Xiao,
\newblock ``3d shapenets: A deep representation for volumetric shapes,''
\newblock in {\em The IEEE Conference on Computer Vision and Pattern
  Recognition (CVPR)}, June 2015.

\bibitem{MaatenH:08}
L.~van~der Maaten and G.~Hinton,
\newblock ``Visualizing data using {t-SNE},''
\newblock {\em Journal of Machine Learning Research}, vol. 9, pp. 2579--2605,
  2008.

\bibitem{3dcapsule}
Y.~Zhao, T.~Birdal, H.~Deng, and F.~Tombari,
\newblock ``3d point capsule networks,''
\newblock in {\em Conference on Computer Vision and Pattern Recognition}, 2019.

\bibitem{manifold_dim}
B.~K{\'e}gl,
\newblock ``Intrinsic dimension estimation using packing numbers,''
\newblock in {\em Advances in neural information processing systems}, 2003, pp.
  697--704.

\end{thebibliography}

\end{document}